%% file: full.tex
\documentclass{article} 

\def\TimesFont{} 
\input{ryantibs}

\input{preamble_alden}
\input{preamble_addison}
\def\per{\mathrm{per}}

\title{The Voronoigram: Minimax Estimation of Bounded Variation Functions From 
  Scattered Data}
\author{Addison J.\ Hu \and Alden Green \and Ryan J.\ Tibshirani}
\date{}

\usepackage{xr}

\begin{document}
\maketitle

\begin{abstract}
\input{abstract}
\end{abstract}

\input{main}

{\RaggedRight
\bibliographystyle{plainnat}
\bibliography{ryantibs}}

\clearpage
\appendix

\setcounter{equation}{0}
\setcounter{figure}{0}
\setcounter{table}{0}
\setcounter{algorithm}{0}
\setcounter{theorem}{0}
\setcounter{corollary}{0}
\setcounter{lemma}{0}
\setcounter{proposition}{0}

\input{app}

\end{document}

%% file: ryantibs.tex
\usepackage{amssymb,amsmath,amsthm,bbm}
\usepackage{verbatim,float,url,dsfont}
\usepackage{graphicx,subfigure,psfrag}
\usepackage{algorithm,algorithmic}
\usepackage{mathtools,enumitem}
\usepackage{multirow}
\usepackage{ragged2e}
\usepackage{xr-hyper}

\usepackage[colorlinks=true,citecolor=blue,urlcolor=blue,linkcolor=blue]{hyperref}
\usepackage[margin=1in]{geometry}
\usepackage[round]{natbib}

\usepackage[utf8]{inputenc} 
\usepackage[T1]{fontenc}    
\usepackage{booktabs}       
\usepackage{nicefrac}         
\usepackage{microtype}      

\ifdefined\TimesFont 
\usepackage{times} 
\fi

\ifdefined\ParSkip 
\usepackage{parskip} 
\fi

\newtheorem{theorem}{Theorem}
\newtheorem{lemma}{Lemma}
\newtheorem{corollary}{Corollary}
\newtheorem{proposition}{Proposition}
\theoremstyle{definition}
\newtheorem{remark}{Remark}
\newtheorem{definition}{Definition}

\newtheorem*{assumption*}{\assumptionnumber}
\providecommand{\assumptionnumber}{}
\makeatletter
\newenvironment{assumption}[2]{
  \renewcommand{\assumptionnumber}{Assumption #1#2}
  \begin{assumption*}
  \protected@edef\@currentlabel{#1#2}}
{\end{assumption*}}
\makeatother

\makeatletter
\newcommand*\rel@kern[1]{\kern#1\dimexpr\macc@kerna}
\newcommand*\widebar[1]{%
  \begingroup
  \def\mathaccent##1##2{%
    \rel@kern{0.8}%
    \overline{\rel@kern{-0.8}\macc@nucleus\rel@kern{0.2}}%
    \rel@kern{-0.2}%
  }%
  \macc@depth\@ne
  \let\math@bgroup\@empty \let\math@egroup\macc@set@skewchar
  \mathsurround\z@ \frozen@everymath{\mathgroup\macc@group\relax}%
  \macc@set@skewchar\relax
  \let\mathaccentV\macc@nested@a
  \macc@nested@a\relax111{#1}%
  \endgroup
}
\makeatother

\newcommand{\argmin}{\mathop{\mathrm{argmin}}}

\newcommand{\minimize}{\mathop{\mathrm{minimize}}}

\def\R{\mathbb{R}}

\def\Z{\mathbb{Z}}
\def\E{\mathbb{E}}
\def\P{\mathbb{P}}
\def\T{\mathsf{T}}
\def\Cov{\mathrm{Cov}}
\def\Var{\mathrm{Var}}
\def\half{\frac{1}{2}}
\def\th{^{\text{th}}}

\def\df{\mathrm{df}}
\def\dim{\mathrm{dim}}

\def\row{\mathrm{row}}
\def\nul{\mathrm{null}}

\def\nuli{\mathrm{nullity}}
\def\spa{\mathrm{span}}
\def\sign{\mathrm{sign}}

\def\hf{\hat{f}}

\def\htheta{\hat{\theta}}

\def\cE{\mathcal{E}}
\def\cF{\mathcal{F}}

\def\cH{\mathcal{H}}

\def\cN{\mathcal{N}}

\def\cZ{\mathcal{Z}}

%% file: preamble_alden.tex
\usepackage{mathrsfs}      

\newcommand{\Ebb}{\mathbb{E}}   
\newcommand{\Reals}{\mathbb{R}} 

\newcommand{\wt}[1]{\widetilde{#1}}
\newcommand{\wh}[1]{\widehat{#1}}
\newcommand{\mc}[1]{\mathcal{#1}}
\newcommand{\Rd}{\Reals^d}
\newcommand{\wb}[1]{\widebar{#1}}

\newcommand{\1}{1} 

\DeclareFontFamily{U}{mathx}{\hyphenchar\font45}
\DeclareFontShape{U}{mathx}{m}{n}{<-> mathx10}{}
\DeclareSymbolFont{mathx}{U}{mathx}{m}{n}
\DeclareMathAccent{\wc}{0}{mathx}{"71}

%% file: preamble_addison.tex
\DeclareMathOperator\DTV{\mathrm{DTV}}
\DeclareMathOperator\TV{\mathrm{TV}}

\DeclareMathOperator\BV{\mathrm{BV}}
\DeclareMathOperator\diver{\mathrm{div}}

\DeclareMathOperator{\proj}{\mathrm{proj}}
\def\Leb{\mu}  
\def\Part{V}

\def\Xint#1{\mathchoice
{\XXint\displaystyle\textstyle{#1}}%
{\XXint\textstyle\scriptstyle{#1}}%
{\XXint\scriptstyle\scriptscriptstyle{#1}}%
{\XXint\scriptscriptstyle\scriptscriptstyle{#1}}%
\!\int}
\def\XXint#1#2#3{{\setbox0=\hbox{$#1{#2#3}{\int}$ }
\vcenter{\hbox{$#2#3$ }}\kern-.6\wd0}}

\def\dashint{\Xint-}
\def\AvInt{\dashint}

\def\Vor{\mathrm{V}} 
\def\Eps{\varepsilon}
\def\kNN{k}

\def\WVor{\mathrm{V}}

\def\nVor{{n, \Vor}}
\def\nWVor{{\WVor}}
\def\nEps{{\varepsilon}}
\def\nKnn{{k}}

\def\PenMat{D}
\def\SurPenMat{T}
\def\SqAvMat{A}
\def\AvMat{\bar A}
\def\CellMax{M}
\def\ErrVec{{z_{1:n}}}
\def\Cell{\mathcal{C}}
\def\mn{\mathrm{min}}
\def\mx{\mathrm{max}}

\def\EpsPenMat{{\PenMat^\nEps}}
\def\KnnPenMat{{\PenMat^\nKnn}}
\def\CVorPenMat{{\tilde{\PenMat}^{\nWVor}}}
\def\UVorPenMat{{\check{\PenMat}^{\nWVor}}}

\def\XDom{\Omega}
\def\XSet{\mathscr{X}}
\def\one{{1}}
\def\cC{\mathcal{C}}

%% file: abstract.tex
We consider the problem of estimating a multivariate function $f_0$ of bounded
variation (BV), from noisy observations $y_i = f_0(x_i) + z_i$ made at random
design points $x_i \in \R^d$, $i=1,\ldots,n$. We study an estimator that forms
the Voronoi diagram of the design points, and then solves an optimization
problem that regularizes according to a certain discrete notion of total
variation (TV): the sum of weighted absolute differences of parameters
$\theta_i,\theta_j$ (which estimate the function values $f_0(x_i),f_0(x_j)$) at
all neighboring cells $i,j$ in the Voronoi diagram. This is seen to be
equivalent to a variational optimization problem that regularizes according to
the usual continuum (measure-theoretic) notion of TV, once we restrict the
domain to functions that are piecewise constant over the Voronoi diagram. 

The regression estimator under consideration hence performs (shrunken) local
averaging over adaptively formed unions of Voronoi cells, and we refer to it as
the \emph{Voronoigram}, following the ideas in \citet{koenker2005quantile},   
and drawing inspiration from Tukey's \emph{regressogram}
\citep{tukey1961curves}. Our contributions in this paper span both the 
conceptual and theoretical frontiers: we discuss some of the unique properties
of the Voronoigram in comparison to TV-regularized estimators that use other 
graph-based discretizations; we derive the asymptotic limit of the Voronoi TV
functional; and we prove that the Voronoigram is minimax rate optimal (up to log
factors) for estimating BV functions that are essentially bounded.   

%% file: main.tex
\section{Introduction}
\label{sec:introduction}
\input{introduction}

\section{Methods and basic properties}
\label{sec:methods_properties}
\input{methods_properties}

\section{Asymptotics for graph TV functionals}
\label{sec:asymptotic_limits}
\input{asymptotic_limits}

\section{Illustrative empirical examples}
\label{sec:experiments}
\input{experiments}

\section{Estimation theory for BV classes}
\label{sec:estimation_theory}
\input{estimation_theory}

\section{Discussion}
\label{sec:discussion}
\input{discussion}

%% file: introduction.tex
Consider a standard nonparametric regression setting, given observations
$(x_i,y_i)\in\XDom\times\R$, $i=1,\dots,n$, for an open and connected subset 
$\XDom$ of $\R^d$, and with
\begin{equation}
\label{eq:model}
y_i = f_0(x_i) + z_i, \quad i=1,\dots,n,
\end{equation}
for i.i.d.\ mean zero stochastic errors $z_i$, $i=1,\dots,n$. We are interested
in estimating the function $f_0$ under the working assumption that $f_0$
adheres to a certain notion of smoothness. A traditional smoothness assumption 
on $f_0$ involves its integrated squared derivatives, for example, the
assumption that 
\[
\int_\XDom \sum_{\|\alpha\|_1=2} (D^\alpha f)^2(x) \, dx 
\]
is small, where $\alpha = (\alpha_1,\dots,\alpha_d) \in \Z_+^d$ is a
multi-index and we write \smash{$D^\alpha = \frac{\partial^{\alpha_1}}{\partial 
    x_1} \dots \frac{\partial^{\alpha_d}}{\partial x_d}$} for the corresponding 
mixed partial derivative operator. This is the notion of smoothness
that underlies the celebrated \emph{smoothing spline} estimator in the
univariate case $d=1$ \citep{schoenberg1964spline} and the \emph{thin-plate
  spline} estimator when $d=2$ or 3 \citep{duchon1977splines}. We also note that
assuming $f_0$ is smooth in the sense of the above display is known as
second-order $L^2$ Sobolev smoothness (where we interpret $D^\alpha f$ as a weak  
derivative). 

Smoothing splines and thin-plate splines are quite popular and come with a
number of advantages. However, one shortcoming of using these methods, i.e., to
using the working model of Sobolev smoothness, is that it does not permit $f_0$
to have discontinuities, which limits its applicability. More broadly, smoothing
splines and thin-plate splines do not fare well when the estimand $f_0$
possesses heterogeneous smoothness, meaning that $f_0$ is more smooth at some
parts of its domain $\XDom$ and more wiggly at others. This motivates us to
consider regularity measured by the \emph{total variation} (TV) seminorm
\begin{equation}
\label{eq:total_variation}
\TV(f; \XDom) = \sup\left\{
    \int_\XDom f(x) \diver\phi(x) \, dx : \phi\in C_c^1(\XDom;\R^d), \,  
    \|\phi(x)\|_2 \leq 1 \; \text{for all $x\in\XDom$} \right\},
\end{equation}
where \smash{$C_c^1(\XDom;\R^d)$} denotes the space of continuously 
differentiable compactly supported functions from $\XDom$ to $\R^d$, and 
we use \smash{$\diver\phi = \sum_{i=1}^d \partial\phi_i/\partial_{x_i}$} for
the divergence of $\phi = (\phi_1, \dots, \phi_d)$. Accordingly, we define the
\emph{bounded variation} (BV) class on $\XDom$ by 
\[
\BV(\XDom) = \{ f \in L^1(\XDom): \TV(f; \XDom) < \infty \}, 
\] 
to contain all $L^1(\XDom)$ functions with finite TV. The definition in
\eqref{eq:total_variation} is often called the measure-theoretic definition of
multivariate TV. For simplicity we will often drop the notational dependence on 
$\XDom$ and simply write this as $\TV(f)$. This definition may appear
complicated at first, but it admits a few natural interpretations, which we
present next to help build intuition.    

\subsection{Perspectives on total variation}

Below are three perspectives on total variation. The first two reveal the way 
that TV acts on special types of functions; the third is a general equivalent
form of TV.  

\paragraph{Smooth functions.}

If $f$ is (weakly) differentiable with (weak) gradient \smash{$\nabla f = 
(\frac{\partial f}{\partial x_1}, \dots, \frac{\partial f}{\partial x_d})$},
then  
\begin{equation}
\label{eq:total_variation_smooth}
\TV(f) =  \int_\XDom \| \nabla f(x) \|_2 \, dx,
\end{equation}
provided that the right-hand here is well-defined and finite. Consider the
difference between this and the first-order $L^2$ Sobolev seminorm
\[
\int_{\XDom} \sum_{\|\alpha\|_1=1} (D^\alpha f)^2(x) \, dx 
= \int_{\XDom} \|\nabla f(x)\|_2^2 \, dx. 
\]
The latter uses the squared $\ell_2$ norm $\|\cdot\|_2^2$ in the integrand, 
whereas the former \eqref{eq:total_variation_smooth} uses the $\ell_2$ norm
$\|\cdot\|_2$. It turns out that this is a meaningful difference---one way to 
interpret this is as a difference between $L^2$ and $L^1$ regularity. Noting
that \smash{$\|x\|_1 \leq \sqrt{d} \|x\|_2$} for all $x \in \R^d$, the space
$\BV(\XDom)$ contains the first-order $L^1$ Sobolev space   
\[
W^{1,1}(\XDom) = \{ f \in L^1(\XDom): \int_{\XDom} \|\nabla f(x)\|_1 \, dx 
< \infty \}, 
\] 
which, loosely speaking, contains functions that can be more locally peaked and 
less evenly spread out (i.e., permits a greater degree of heterogeneity in
smoothness) compared to the first-order $L^2$ Sobolev space 
\[
W^{1,2}(\XDom) = \{ f \in L^2(\XDom): \int_{\XDom} \|\nabla f(x)\|_2^2 \, dx 
< \infty \}. 
\]      
It is important to emphasize, however, that $\BV(\XDom)$ is still much larger
than $W^{1,1}(\XDom)$, because it permits functions to have sharp
discontinuities. We discuss this next.

\paragraph{Indicator functions.}

If $S \subseteq \XDom$ is a set with locally finite perimeter, then the
indicator function $1_S$, which we define by $1_S(x) = 1$ for $x \in S$ 
and $0$ otherwise, satisfies 
\begin{equation}
\label{eq:total_variation_indicator}
\TV(1_S) =  \per(S),
\end{equation}
where $\per(S)$ is the perimeter of $S$ (or equivalently, the codimension 1 
Hausdorff measure of $\partial S$, the boundary of $S$). Thus, we see that that 
TV is tied to the geometry of the level sets of the function in question. 
Indeed, there is a precise sense in which this is true in full generality, as we
discuss next.   

\paragraph{Coarea formula.}

In general, for any $f \in \BV(\XDom)$, we have
\begin{equation}
\label{eq:total_variation_coarea}
\TV(f) = \int_{-\infty}^\infty \per\big(\{ x \in \XDom : f(x) > t \}\big) \, dt.
\end{equation}
This is known as the \emph{coarea formula} for BV functions (see, e.g., Theorem
5.9 in \citealp{evans2015measure}). It offers a highly intuitive picture of
what total variation is measuring: we take a slice through the graph of a
function $f$, calculate the perimeter of the set of points (projected down to
the $\XDom$-axis) that lie above this slice, and add up these perimeters over
all possible slices.

The coarea formula \eqref{eq:total_variation_coarea} also sheds light on why BV
functions are able to portray such a great deal of heterogeneous 
smoothness: all that matters is the total integrated amount of function growth,
according to the perimeter of the level sets, as we traverse the heights of
level sets. For example, if the perimeter has a component $\rho$ that persists
for a range of level set heights $[t, t+h]$, then this contributes the same
amount $\rho h$ to the TV as does a smaller perimeter component $\rho/100$
that persists for a larger range of level set heights $[t, t + 100h]$. To put it
differently, the former might represent a local behavior that is more spread 
out, and the latter a local behavior that is more peaked, but these two
behaviors can contribute the same amount to the TV in the end. Therefore, a ball
in the BV space---all $L^1$ functions $f$ such that $\TV(f) \leq r$---contains 
functions with a huge variety in local smoothness.  

\subsection{Why is estimating BV functions hard?}

Now that we have motivated the study of BV functions, let us turn towards the
problem of estimating a BV function from noisy samples. Given the centrality of
penalized empirical risk minimization in nonparametric regression, one might be
tempted to solve the TV-penalized variational problem   
\begin{equation}
\label{eq:tv_problem}
\minimize_{f \in \BV(\XDom)} \; \half \sum_{i=1}^{n} (y_i - f(x_i))^2 +  
\lambda \TV(f), 
\end{equation}
given data $(x_i, y_i)$, $i=1,\dots,n$ from the model \eqref{eq:model}, and
under the working assumption that $f$ has small TV. However, in short, solving
\eqref{eq:tv_problem} will ``not work'' in any dimension $d \geq 2$, in the
sense that it does not yield a well-defined estimator regardless of the choice
of tuning parameter $\lambda > 0$.  

When $d=1$, solving \eqref{eq:tv_problem} produces a celebrated estimator known
as the (univariate) \emph{TV denoising} estimator \citep{rudin1992nonlinear} or
the \emph{fused lasso} signal approximator \citep{tibshirani2005sparsity}. (More 
will be said about this shortly, under the related work subsection.) But for any
$d \geq 2$, problem \eqref{eq:tv_problem} is ill-posed, as the criterion does
not achieve its infimum. To see this, consider the function     
\[
  f_\epsilon = \sum_{i=1}^{n} y_i \cdot 1_{B(x_i, \epsilon)},
\]
where $B(x_i, \epsilon)$ denotes the closed $\ell_2$ ball of radius $\epsilon$
centered at $x_i$, and \smash{$1_{B(x_i, \epsilon)}$} denotes its indicator
function (which equals 1 on $B(x_i, \epsilon)$ and 0 outside of it). Now let us
examine the criterion in problem \eqref{eq:tv_problem}: for small enough
$\epsilon>0$, the function $f_\epsilon$ has a squared loss equal to 0, and has
TV penalty equal to $\lambda n c_d \epsilon^{d-1}$ (here $c_d>0$ is a constant
depending only on $d$). Hence, as $\epsilon \to 0$, the criterion value in
\eqref{eq:tv_problem} achieved by $f_\epsilon$ tends to 0. However, as $\epsilon
\to 0$, the function $f_\epsilon$ itself trivially approaches the zero function,
defined as $f(x) = 0$ for all $x$.\footnote{Just as with $L^p$ classes, elements
  in $\BV(\XDom)$ are actually only defined up to equivalence classes of
  functions. Hence, to make point evaluation well-defined in the random design 
  model \eqref{eq:model}, we must identify each equivalence class with a 
  representative; we use the \emph{precise representative}, which is defined at 
  almost every point $x$ by the limiting local average of a function around $x$;
  see Appendix \ref{app:precise_representative} for details. It is
  straightforward to see that the precise representative associated with
  $f_\epsilon$ converges to the zero function as $\epsilon \to 0$.}  
Note that this is true for any $\lambda > 0$, whereas the zero function
certainly cannot minimize the objective in \eqref{eq:tv_problem} for all
$\lambda > 0$. 

The problem here, informally speaking, is that the BV class is ``too big'' when 
$d \geq 2$; more formally, the evaluation operator is not continuous over the BV 
space---or in other words, convergence in BV norm\footnote{Traditionally defined
  by equipping the TV seminorm with the $L^1$ norm, as in \smash{$\|f\|_{\BV} =
    \|f\|_{L^1} + \TV(f)$}.}  
does not imply pointwise convergence---for $d \geq 2$. It is worth noting that
this problem is not specific to BV spaces and it also occurs with the $k\th$
order $L^p$ Sobolev space \smash{$W^{k,p}(\XDom) = \{ f \in L^p :
  \int_\XDom \sum_{\|\alpha\|_1=k} (D^\alpha f)^p(x) \, dx < \infty \}$}
when $pk \leq d$, 
which is called the subcritical regime. In the supercritical regime $pk > d$,
convergence in Sobolev norm implies pointwise convergence,\footnote{This is  
  effectively a statement about the everywhere continuity of the precise
  representative, which is a consequence of Morrey's inequality; see, e.g.,
  Theorem 4.10 in \citet{evans2015measure}.}   
but all bets are off when $pk \leq d$. Thus, just as the TV-penalized problem
\eqref{eq:tv_problem} is ill-posed for $d \geq 2$, the more familiar thin-plate
spline problem   
\[
\minimize_{f \in W^{2,2}(\XDom)} \; \half \sum_{i=1}^{n} (y_i - f(x_i))^2 + 
\lambda \int_\XDom \|\nabla^2 f(x)\|_F^2 \, dx
\]
is itself ill-posed when $d \geq 4$. (Here we use $\nabla^2 f(x)$ for the weak 
Hessian of $f$, and $\| \cdot \|_F$ for the Frobenius norm, so that the
second-order $L^2$ Sobolev seminorm can be written as
\smash{$\int_\XDom \sum_{\|\alpha\|_1=2} (D^\alpha f)^2(x) \, dx  
= \int_\XDom \|\nabla^2 f(x)\|_F^2 \, dx$}.)   

What can we do to circumvent this issue? Broadly speaking, previous approaches
from the literature can be stratified into two types. The first maintains the
smoothness assumption on $\TV(f_0)$ for the regression function $f_0$, but 
replaces the sampling model \eqref{eq:model} by a white noise model of the form 
\[
dY(x) = f_0(x) dx + \frac{\sigma}{\sqrt{n}} dW(x), \quad x \in \XDom, 
\]
where $dW$ is a Gaussian white noise process. Given this continuous-space
observation model, we can then replace the empirical loss term
\smash{$\sum_{i=1}^{n} (y_i - f(x_i))^2$} in \eqref{eq:tv_problem} by the
squared $L^2$ loss \smash{$\|Y - f\|_{L^2(\XDom)}^2 = \int_\XDom (Y(x) - f(x))^2
  \, dx$} (or some multiscale variant of this). The second type of approach
keeps the sampling model \eqref{eq:model}, but replaces the assumption on
$\TV(f_0)$ by an assumption on discrete total variation (which is based on the
evaluations of $f_0$ at the design points alone) of the form  
\[
\DTV(f_0) = \sum_{\{i,j\} \in E} w_{ij} |f_0(x_i) - f_0(x_j)|, 
\]
for an edge set $E$ and weights $w_{ij} \geq 0$. We then naturally replace the
penalty $\TV(f)$ in \eqref{eq:tv_problem} by $\DTV(f)$. More details on both
types of approaches will be given in the related work subsection.

The approach we take in the current paper sits in the middle, between the two
types. Like the first, we maintain a bona fide smoothness assumption on
$\TV(f_0)$, rather than a discrete version of TV. Like the second, we work in
the sampling model \eqref{eq:model}, and define an estimator by solving a
discrete version of \eqref{eq:tv_problem} which is always well posed, for any
dimension $d \geq 2$. In fact, the connections run deeper: the discrete problem
that we solve is not constructed arbitrarily, but comes from restricting the
domain in \eqref{eq:tv_problem} to a special finite-dimensional class of
functions, over which the penalty $\TV(f)$ in \eqref{eq:tv_problem} takes on an
equivalent discrete form.

\subsection{The Voronoigram}

This brings us to the central object in the current paper: an estimator defined 
by restricting the domain in the infinite-dimensional problem
\eqref{eq:tv_problem} to a finite-dimensional subspace, whose structure is 
governed by the Voronoi diagram of the design points $x_1,\dots,x_n \in 
\XDom$. In detail, let \smash{$\Part_i = \{ x \in \XDom: \|x_i - x\|_2 < \| x_j
  - x\| \}$} be the Voronoi cell\footnote{As we have defined it, each Voronoi
  cell is open, and thus a given function \smash{$f \in \cF^\Vor_n$} is not
  actually defined on the boundaries of the Voronoi cells. But this is a set of
  Lebesgue measure zero, and on this set it can be defined arbitrarily---any
  particular definition will not affect results henceforth.}   
associated with $x_i$, for $i=1,\dots,n$, and define
\[
\cF^\Vor_n = \spa\big\{ 1_{\Part_1}, \dots, 1_{\Part_n} \big\},
\]
where recall \smash{$1_{\Part_i}$} is the indicator function of $\Part_i$. In
words, \smash{$\cF^\Vor_n$} is a space of functions from $\XDom$ to $\R$ that 
are piecewise constant on the Voronoi diagram of $x_1,\dots,x_n$. (We remark
that this is most certainly a subspace of $\BV(\XDom)$, as each Voronoi cell has
locally finite perimeter; in fact, as we will see soon, the TV of each \smash{$f
  \in \cF^\Vor_n$} takes a simple and explicit closed form.) Now consider the  
finite-dimensional problem   
\begin{equation}
\label{eq:tv_voronoi}
\minimize_{f \in \cF^\Vor_n} \; \half \sum_{i=1}^{n} (y_i - f(x_i))^2 + 
\lambda \TV(f). 
\end{equation}
We call the solution to \eqref{eq:tv_voronoi} the \emph{Voronoigram} and 
denote it by \smash{$\hf^\Vor$}. This idea---to fit a piecewise constant
function to the Voronoi tessellation of the input domain $\XDom$---dates back to
(at least) \citet{koenker2005quantile}, where it was briefly proposed in Chapter
7 of this book (further discussion of related work will be given in Section 
\ref{sub:related_work}). It does not appear to have been implemented or studied
beyond this. Its name is inspired by Tukey's classic \emph{regressogram}
\citep{tukey1961curves}.     

Of course, there a many choices for a finite-dimensional subset of $\BV(\XDom)$
that we could have used for a domain restriction in \eqref{eq:tv_voronoi}. Why
\smash{$\cF^\Vor_n$}, defined by piecewise constant functions on the Voronoi 
diagram, as in \eqref{eq:tv_voronoi}? A remarkable feature of this choice is
that it yields an equivalent optimization problem
\begin{equation}
\label{eq:tv_voronoi_graph}
\minimize_{\theta \in \R^n} \; \half \sum_{i=1}^{n} (y_i - \theta_i)^2 +  
\lambda \hspace{-5pt} \sum_{\{i,j\} \in E^\Vor} w^\Vor_{ij} \cdot |\theta_i - 
\theta_j|, 
\end{equation}
for an edge set $E^\Vor$ defined by neighbors in the Voronoi graph, and weights
\smash{$w^\Vor_{ij}$} that measure the ``length'' of the shared boundary between
cells $\Part_i$ and $\Part_j$, to be defined precisely later (in Section
\ref{sub:voronoigram}). The equivalence between problems~\eqref{eq:tv_voronoi}
and~\eqref{eq:tv_voronoi_graph} sets $\theta_i = f(x_i)$, $i=1,\dots,n$, and is
driven by the following special fact: for such a pairing, whenever \smash{$f \in
\cF^\Vor_n$}, it holds (proved in Section \ref{sub:voronoigram}) that
\begin{equation}
\label{eq:tv_representation}
\TV(f) = \sum_{\{i,j\} \in E^\Vor} w^\Vor_{ij} \cdot |\theta_i - \theta_j|. 
\end{equation}
In this way we can view the Voronoigram as marriage between a
purely variational approach, which maintains the use of a continuum TV penalty
on a function $f$, and a purely discrete approach, which instead models
smoothness using a discrete TV penalty on a vector $\theta$ defined over a
graph. In short, the Voronoigram does both. 

\subsubsection{A first look at the Voronoigram}

From its equivalent discrete problem form \eqref{eq:tv_voronoi_graph}, we can
see that the penalty term drives the Voronoigram to have equal (or ``fused'')
evaluations at points $x_i$ and $x_j$ corresponding to neighboring cells in the
Voronoi tessellation. Generally, the larger the value of $\lambda \geq 0$, the
more neighboring evaluations will be fused together. Due to fact that each 
\smash{$f \in \cF^\Vor_n$} is constant over an entire Voronoi cell, this means
that the Voronoigram fitted function \smash{$\hf^\Vor$} is constant over 
adaptively chosen unions of Voronoi cells. Furthermore, based on what is known
about solutions of generalized lasso problems (details given in Section
\ref{sub:generalized_lasso}), we can express the fitted function here as  
\begin{equation}
\label{eq:voronoigram_sol}
\hf^\Vor = \sum_{k = 1}^{\hat{K}} (\bar{y}_k - \hat{s}_k) \cdot 1_{\hat{R}_k},    
\end{equation}
where \smash{$\hat{K}$} is the number of connected components that appear in the
solution \smash{$\htheta^\Vor$} over the Voronoi graph, \smash{$\hat{R}_k$} 
denotes a union of Voronoi cells associated with the $k\th$ connected component, 
\smash{$\bar{y}_k$} denotes the average of response points $y_i$ such that
\smash{$x_i \in \hat{R}_k$}; and \smash{$\hat{s}_k$} is a data-driven shrinkage
factor. To be clear, each of \smash{$\hat{K}$}, \smash{$\hat{R}_k$},
\smash{$\bar{y}_k$}, and \smash{$\hat{s}_k$} here are data-dependent
quantities---they fall out of the structure of the solution in problem
\eqref{eq:tv_voronoi_graph}.   

Thus, like the regressogram, the Voronoigram estimates the regression function
by fitting (shrunken) averages over local regions; but unlike the regressogram,
where the regions are fixed ahead of time, the Voronoigram is able to choose its
regions \emph{adaptively}, based on the geometry of the design points $x_i$   
(owing to the use of the Voronoi diagram) and on how much local variation is
present in the response points $y_i$ (a consideration inherent to the
minimization in \eqref{eq:tv_voronoi_graph}, which trades off between the loss 
and penalty summands).      

Figure \ref{fig:voronoigram_intro} gives a simple example of the Voronoigram and
its adaptive structure in action. 

\begin{figure}[htb]
\centering
\includegraphics[width=0.9\linewidth]{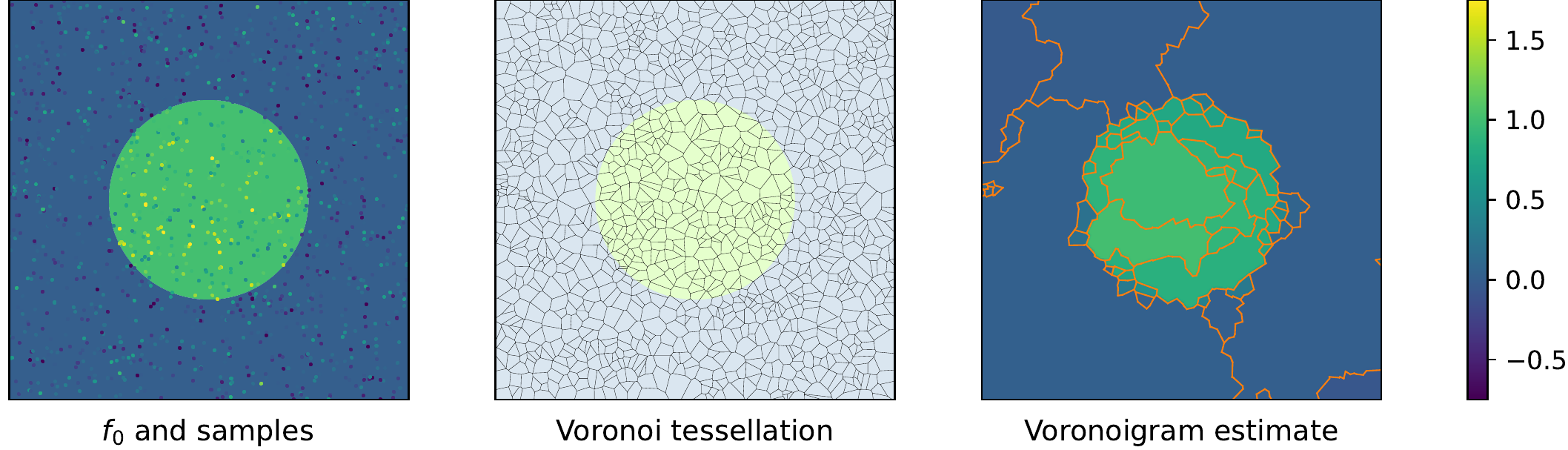}
\caption{\it
  A simple example of using the Voronoigram to estimate a function $f_0$, from  
  noisy observations. Left: $f_0$ and noisy observations made at $n=1274$ random
  points in $d=2$ dimensions. Center: the Voronoi tessellation, whose cells
  constitute the piecewise constant basis for the Voronoigram. Right: the
  Voronoigram estimate (at a certain choice of $\lambda$), with the resulting
  adaptively~chosen constant pieces---over which it performs
  averaging---outlined in orange.     
}
\label{fig:voronoigram_intro}
\end{figure}

\subsection{Summary of contributions}

Our primary practical and methodological contribution is to motivate and study
the Voronoigram as a nonparametric regression estimator for BV functions in a
multivariate scattered data (random design) setting, including comparing and
contrasting it to two related approaches: discrete TV regularization using
$\varepsilon$-neighborhood or $k$-nearest neighbor graphs. A summary is as
follows (a more detailed summary is given in Section
\ref{sub:graph_comparisons}). 

\begin{itemize}
\item The graph used by Voronoigram---namely, the Voronoi adjacency graph---is
  \emph{tuning-free}. This stands in contrast to $\varepsilon$-neighborhood or
  $k$-nearest neighbor graphs, which require a choice of a local radius 
  $\varepsilon$ or number of neighbors $k$, respectively.    

\item The Voronoigram penalty becomes \emph{density-free} in large samples,
  which is term we use to describe the fact that it converges to ``pure'' total  
  variation, independent of the density $p$ of the (random) design points   
  $x_1,\dots,x_n$. This follows from one of our main theoretical results
  (reiterated below), and it stands in contrast to the TV penalties based on
  $\varepsilon$-neighborhood and $k$-nearest neighbor graphs, which are known to
  asymptotically approach particular $p$-weighted versions of total variation.       

\item The Voronoigram estimator yields a natural passage from a discrete set of 
  fitted values \smash{$\hf^\Vor(x_i)$}, $i=1,\dots,n$ to a fitted function
  \smash{$\hf^\Vor$} defined over the entire input domain $\XDom$: this is
  simply given by local constant extrapolation of each fitted value
  \smash{$\hf^\Vor(x_i)$} to its containing Voronoi cell
  $\Part_i$. (Equivalently, \smash{$\hf^\Vor(x)$} is given by the
  1-nearest neighbor prediction rule based on \smash{$(x_i, \hf^\Vor(x_i))$},
  $i=1,\dots,n$.) Further, thanks to \eqref{eq:tv_representation}, we know that
  such an extrapolation method is \emph{complexity-preserving}: the discrete TV
  of \smash{$\htheta^\Vor_i$}, $i=1,\ldots,n$ is precisely the same as the
  continuum TV of the extrapolant \smash{$\hf^\Vor$}. Other graph-based TV
  regularization methods do not come with this property.
\end{itemize}

\noindent
On the theoretical side, our primary theoretical contributions are twofold,
summarized below.

\begin{itemize}
\item We prove that the Voronoi penalty functional, applied to evaluations of
  $f$ at i.i.d.\ design points $x_1,\ldots,x_n$ from a density $p$, converges to
  $\TV(f)$, as $n \to \infty$ (see Section \ref{sec:asymptotic_limits} for
  details). The fact that its asymptotic limit here is independent of $p$ is
  both important and somewhat remarkable.     

\item We carry out a comprehensive minimax analysis for $L^2$ estimation over 
  $\BV(\XDom)$. The highlights (Section \ref{sec:estimation_theory} gives
  details): for any fixed $d \geq 2$ and regression function $f_0$ with
  $\TV(f_0) \leq L$ and $\|f_0\|_{L^\infty} \leq M$ (where $L,M>0$ are
  constants), a modification of the Voronoigram estimator \smash{$\hf^\Vor$} in
  \eqref{eq:tv_voronoi}---defined by simply clipping small weights
  \smash{$w^\Vor_{ij}$} in the penalty term---converges to $f_0$ at the 
  squared $L^2$ rate $n^{-1/d}$ (ignoring log terms). We prove that this matches
  the minimax rate (up to log terms) for estimating a regression function $f_0$  
  that is bounded in TV and $L^\infty$. Lastly, we prove that an even simpler
  \emph{unweighted} Voronoigram estimator---defined by setting all edge weights
  in \eqref{eq:tv_voronoi} to unity---also obtains the optimal rate (up to log 
  terms), as do more standard estimators based on discrete TV regularization
  over $\varepsilon$-neighborhood and $k$-nearest neighbor graphs.      
\end{itemize}

\subsection{Related work}
\label{sub:related_work}

The work of \citet{mammen1997locally} marks an important early contribution
promoting and studying the use of TV as a regularization functional, in
univariate nonparametric regression. These authors considered a variational
problem similar to \eqref{eq:tv_problem} in dimension $d=1$, with a generalized
penalty $\TV(D^k f)$, the TV of the $k\th$ weak derivative $D^k f$ of $f$. They
proved that the solution is always a spline of degree $k$ (whose knots may lie
outside the design points if $k \geq 2$) and named the solution the
\emph{locally adaptive regression spline} estimator. A related, more recent idea 
is \emph{trend filtering}, proposed by \citet{steidl2006splines, kim2009trend}
and extended by \citet{tibshirani2014adaptive} to the case of arbitrary design
points. Trend filtering solves a discrete analog of the locally adaptive
regression spline problem, in which the penalty $\TV(D^k f)$ is replaced by the
discrete TV of the $k\th$ discrete derivative of $f$---based entirely on 
evaluations of $f$ at the design points. 

\citet{tibshirani2014adaptive} showed that trend filtering, like the
Voronoigram, admits a special duality between discrete and continuum
representations: the trend filtering optimization problem is in fact the
restriction of the variational problem for locally adaptive regression splines
to a particular finite-dimensional space of $k\th$ degree piecewise
polynomials. The key fact underlying this equivalence is that for any function
$f$ in this special piecewise polynomial space, its continuum penalty $\TV(D^k
f)$ equals its discrete penalty (discrete TV applied to its $k\th$ discrete
derivatives), a result analogous to the property \eqref{eq:tv_representation} of
functions \smash{$f \in \cF^\Vor_n$}. Thus we can view the Voronoigram a
generalization of this core idea, at the heart of trend filtering, to multiple
dimensions---albeit restricted the case $k=0$.

We note that similar ideas to locally adaptive regression splines and trend
filtering were around much earlier; see, e.g., \citet{schuette1978linear,
  koenker1994quantile}. \citet{tibshirani2022divided} provides an account of 
the history of these and related ideas in nonparametric smoothing, and also
makes connections to numerical analysis---the study of discrete splines in 
particular. It is worth highlighting that when $k=0$, the locally adaptive
regression spline and trend filtering estimators coincide, and reduce to a
method known as \emph{TV denoising}, which has even earlier roots in applied
mathematics (to be covered shortly).  

Beyond the univariate setting, there is still a lot of related work to cover
across different areas of the literature, and we break up our exposition into
parts accordingly.    

\paragraph{Continuous-space TV methods.}

The seminal work of \citet{rudin1992nonlinear} introduced TV regularization in
the context of signal and image denoising, and has spawned to a large body of 
follow-up work, mostly in the applied mathematics community, where this is
called the \emph{Rudin-Osher-Fatemi} (ROF) model for TV denoising. See, e.g.,
\citet{rudin1994total, vogel1996iterative, chambolle1997image,
  chan2000highorder}, among many others. In this literature, the observation   
model is traditionally continuous-time (univariate), or continuous-space
(multivariate)---this means that, rather than having observations at a finite
set of design points, we have an entire observation process (deterministic or
random), itself a function over a bounded and connected subset of $\R^d$. TV
regularization is then used in a variational optimization context, and
discretization usually occurs (if at all) as part of numerical optimization
schemes for solving such variational problems.

Statistical analysis in continuous-space observation models traditionally
assumes a white noise regression model, which has a history of study for
adaptive kernel methods (via Lepski's method) or wavelet methods in particular,
see, e.g., \citet{lepski1997optimal1, lepski1997optimal2, neumann2000multi, 
  kerkyacharian2001nonlinear, kerkyacharian2008nonlinear}. In this general area
of the literature, the recent paper of \citet{delalamo2021frameconstrained} is 
most related to our paper: these authors consider a multiresolution
TV-regularized estimator in a multivariate white noise model, and derive minimax
rates for $L^p$ estimation of TV and $L^\infty$ bounded functions. When $p=2$,  
they establish a minimax rate (ignoring log factors) of $n^{-1/d}$ on the
squared $L^2$ error scale, for arbitrary dimension $d \geq 2$, which agrees with
our results in Section \ref{sec:estimation_theory}.  

\paragraph{Discrete, lattice-based TV methods.}

Next we discuss purely discrete TV-regularization approaches, in which both the 
observation model and the penalty are discrete, and are based on function values
at a discrete sequence of points. Such approaches can be further delineated into
two subsets: models and methods based on discrete TV over lattices
(multi-dimensional grid graphs), and those based on discrete TV over geometric
graphs (such as $\varepsilon$-neighborhood or $k$-nearest neighbor graphs 
constructed from the design points). We cover the former first, and the latter
second.    

For lattice-based TV approaches, \citet{tibshirani2005sparsity} marks an
early influential paper studying discrete TV regularization over univariate and 
bivariate lattices, under the name \emph{fused lasso}.\footnote{The original
  work here proposed discrete TV regularization on the coefficients of regressor 
  variables that obey an inherent lattice structure. If we denote the matrix of 
  regressors by $X$, then a special case of this is simply $X=I$ (the identity
  matrix), which reduces to the TV denoising problem. In some papers, the
  resulting estimator is sometimes referred to as the fused lasso \emph{signal 
    approximator}.} This generated much follow-up work in statistics, e.g.,
\citet{friedman2007pathwise, hoefling2010path, tibshirani2011solution,   
  arnold2016efficient}, among many others. In terms of theory, we highlight
\citet{hutter2016optimal}, who established sharp upper bounds for the
estimation error of TV denoising over lattices, as well as
\citet{sadhanala2016total}, who certified optimality (up to log factors) by
giving minimax lower bounds. The rate here (ignoring log factors) for
estimating signals with bounded discrete TV, in mean squared error across the
lattice points, is $n^{-1/d}$. This holds for an arbitrary dimension $d \geq 2$,
and agrees with our results in Section
\ref{sec:estimation_theory}. Interestingly, \citet{sadhanala2016total} also
prove that the minimax linear rate over the discrete TV is class is
constant---which means that the best estimator that is linear in the response
vector $y \in \R^n$, of the form \smash{$\hf(x) = w(x)^\T y$}, is
\emph{inconsistent} in terms of its max risk (over signals with bounded discrete
TV). We do not pursue minimax linear analysis in the current paper but expect a
similar phenomenon to hold in our setting.

Lastly, we highlight \citet{sadhanala2017higher, sadhanala2021multivariate},
who proposed and studied an extension of trend filtering on lattices. Just like
univariate trend filtering, the multivariate version allows for an arbitrary
smoothness order $k \geq 0$, and reduces to TV denoising (or the fused lasso) on
a lattice for $k=0$. In the lattice setting, the theoretical picture is fairly
complete: for general $k,d$, denoting by $s = (k+1)/d$ the effective degree of
smoothness, the minimax rate for estimating signals with bounded $k\th$ order
discrete TV is $n^{-s}$ for $s \leq 1/2$, and \smash{$n^{-2s/(2s+1)}$} for $s >
1/2$. The minimax linear rates display a phase transition as well: constant for
$s \leq 1/2$, and \smash{$n^{-(2s-1)/(2s)}$} for $s > 1/2$. In our setting, we
do not currently have error analysis, let alone an estimator, for higher-order
notions of TV smoothness (general $k \geq 1$). With continuum TV and scattered 
data (random design), this is more challenging to formulate. However, the
lattice-based world continues to provides goalposts for what we would hope to
find in future work. 

\paragraph{Discrete, graph-based TV methods.}

Turning to graph-based TV regularization methods, as explained above, much of
the work in statistics stemmed from \citet{tibshirani2005sparsity}, and the
algorithmic and methodological contributions cited above already considers
general graph structures (beyond lattices). While one can view our proposal as a
special instance of TV regularization over a graph---the Voronoi adjacency
graph---it is important to recognize the independent and visionary work of
\citet{koenker2004penalized}, which served as motivation for us and intimately
relates to our work. These authors began with a triangulation of scattered
points in $d=2$ dimensions (say, the Delaunay triangulation) and defined a
nonparametric regression estimator called the \emph{triogram} by minimizing, 
over functions $f$ that are continuous and piecewise linear over the
triangulation, the squared loss of $f$ plus a penalty on the TV of the
\emph{gradient} of $f$. This is in some ways completely analogous to the
problem we study, except with one higher degree of smoothness. As we mentioned
in the introduction above, in \citet{koenker2005quantile} the author actually
proposes the Voronoigram as a lower-order analog of the triogram, but the method
has not, to our knowledge, been studied beyond this brief proposal.  

Outside of \citet{koenker2004penalized}, existing work involving TV
regularization on graphs relies on geometric graphs like
$\varepsilon$-neighborhood or $k$-nearest neighbor graphs. In terms of
theoretical analysis, the most relevant paper to discuss is the recent work of
\citet{padilla2020adaptive}: they study TV denoising on precisely these two
types of geometric graphs ($\varepsilon$-neighborhood and $k$-nearest neighbor
graphs), and prove that it achieves an estimation rate in squared $L^2$ error of
$n^{-1/d}$, but require that $f_0$ is more than TV bounded---they require it to
satisfy a certain piecewise Lipschitz assumption. Although we primarily study TV
regularization over the Voronoi adjacency graph, we build on some core analysis
ideas in \citet{padilla2020adaptive}. In doing so, we are able to prove that the
Voronoigram achieves the squared $L^2$ error rate $n^{-1/d}$, and we only
require that $\TV(f_0)$ and \smash{$\|f_0\|_{L^\infty}$} are bounded (with the
latter condition actually necessary for nontrivial estimation rates over BV
spaces when $d \geq 2$, as we explain in Section
\ref{sub:impossibility}). Furthermore, we are able to generalize the results of
\citet{padilla2020adaptive}, and we prove that the TV-regularized estimator over
$\varepsilon$-neighborhood and $k$-nearest neighbor graphs achieves the same
rate under the same assumptions, removing the need for the piecewise Lipschitz
condition. See Remark \ref{rem:padilla} for a more detailed discussion. We also
mention that earlier ideas from \citet{wang2016trend, padilla2018dfs} are
critical analysis tools in \citet{padilla2020adaptive} and critical for our
analysis as well. 

Finally, we would like to mention the recent and parallel work of
\citet{green2021minimax1, green2021minimax2}, which studies  
regularized estimators by discretizing Sobolev (rather than TV) functionals over 
neighborhood graphs, and establishes results on estimation error and minimaxity
entirely complementary to ours, but with respect to Sobolev smoothness classes.

%% file: methods_properties.tex
In this section, we discuss some basic properties of our primary object of
study, the Voronoigram, and compare these properties to those of related  
methods. 

\subsection{The Voronoigram and TV representation}
\label{sub:voronoigram}

We start with a discussion of the property behind
\eqref{eq:tv_representation}---we call this a \emph{TV representation} property
of functions in \smash{$\cF^\Vor_n$}, as their total variation over $\XDom$ can
be represented exactly in terms of their evaluations over $x_1,\dots,x_n$.  

\begin{proposition}
\label{prop:tv_representation}
For any $x_1,\dots,x_n$, with Voronoi tessellation $\Part_1,\dots,\Part_n$, and
any \smash{$f \in \cF^\Vor_n =  \spa\{ 1_{\Part_1}, \dots, 1_{\Part_n} \}$} of
the form  
\[
  f = \sum_{i=1}^n \theta_i \cdot 1_{\Part_i},
\]
it holds that
\begin{equation}
\label{eq:tv_representation_explicit}
\TV(f) = \sum_{i,j=1}^n \cH^{d-1}( \partial \Part_i \cap \partial \Part_j )
\cdot |\theta_i - \theta_j|,  
\end{equation}
where $\cH^{d-1}$ denotes Hausdorff measure of dimension $d-1$, and
$\partial \Part_i$ denotes the boundary of $\Part_i$. 
\end{proposition}

The proof of this proposition follows from the measure-theoretic definition
\eqref{eq:total_variation} of total variation, and we defer it to Appendix
\ref{app:tv_representation}. In a sense, the above result is a natural extension
of the property that the TV of an indicator function is the perimeter of the
underlying set, recall \eqref{eq:total_variation_indicator}.     

Note that \eqref{eq:tv_representation_explicit} in Proposition
\ref{prop:tv_representation} reduces to the property
\eqref{eq:tv_representation} claimed in the introduction, once we define weights    
\begin{equation}
\label{eq:voronoi_weights}
w^\Vor_{ij} = \cH^{d-1}( \partial \Part_i \cap \partial \Part_j ), 
\quad i,j = 1,\dots,n,
\end{equation}
and define the edge set \smash{$E^\Vor$} to contain all $\{i,j\}$ such that
\smash{$w^\Vor_{ij} \not= 0$}. In words, each \smash{$w^\Vor_{ij}$} is the
surface measure (length, in dimension $d=2$) of the shared boundary between
$\Part_i$ and $\Part_j$. We say that $i,j$ are adjacent with respect to the
Voronoi diagram provided that \smash{$w^\Vor_{ij} \not= 0$}. Using this
nomenclature, we can think of \smash{$E^\Vor$} as the set of all adjacent pairs
$i,j$. This defines a weighted undirected graph on $\{1,\dots,n\}$, which we
call the \emph{Voronoi adjacency graph} (the Voronoi graph for short). We denote
this by \smash{$G^\Vor = ([n], E^\Vor, w^\Vor)$}, where here and throughout we
write $[n] = \{1,\dots,n\}$.     

Backing up a little further, we remark that \eqref{eq:voronoi_weights} also
provides the remaining details needed to completely describe the Voronoigram
estimator in \eqref{eq:tv_voronoi}. By the TV representation property 
\eqref{eq:tv_representation}, we see that we can equivalently express the
penalty in \eqref{eq:tv_voronoi} as that in \eqref{eq:tv_voronoi_graph}, which
certifies the equivalence between the two problems.  Of course, since
\eqref{eq:tv_representation} is true of all functions in \smash{$\cF^\Vor_n$},
it is also true of the Voronoigram solution \smash{$\hf^\Vor$}. Hence, to 
summarize the relationship between the discrete \eqref{eq:tv_voronoi_graph} and
continuum \eqref{eq:tv_voronoi} problems, once we solve for the Voronoigram
fitted values \smash{$\htheta^\Vor_i = \hf^\Vor(x_i)$}, $i=1,\dots,n$ at the
design points, we extrapolate via
\begin{equation}
\label{eq:voronoi_extrapolate}
\hf^\Vor = \sum_{i=1}^n \htheta^\Vor_i \cdot 1_{\Part_i}, 
\quad \text{which satisfies} \quad 
\TV(\hf^\Vor) = \sum_{\{i,j\} \in E^\Vor} w^\Vor_{ij} \cdot |\htheta^\Vor_i -
\htheta^\Vor_j|. 
\end{equation}
In other words, the continuum TV of the extrapolant \smash{$\hf^\Vor$} is
exactly the same as the discrete TV of the vector of fitted values
\smash{$\htheta^\Vor$}. This is perhaps best appreciated when discussed relative  
to alternative approaches based on discrete TV regularization on graphs, which
do not generally share the same property. We revisit this in Section
\ref{sub:graph_comparisons}. 

\subsection{Insights from generalized lasso theory}
\label{sub:generalized_lasso}

Consider a generalized lasso problem of the form:
\begin{equation}
\label{eq:generalized_lasso}
\minimize_{\theta \in \R^n} \; \half \|y - \theta\|_2^2 + \lambda \|D
\theta\|_1,
\end{equation}
where $y = (y_1,\dots,y_n) \in \R^n$ is a response vector and $D \in \R^{m 
  \times n}$ is a penalty operator (as problem \eqref{eq:generalized_lasso} has 
identity design matrix, it is hence sometimes also called a generalized lasso
signal approximator problem). The Voronoigram is a special case of a generalized
lasso problem: that is, problem \eqref{eq:tv_voronoi_graph} can be equivalently
expressed in the more compact form \eqref{eq:generalized_lasso}, once we take
\smash{$D = D^\Vor$}, the edge incidence operator of the Voronoi adjacency 
graph. In general, given an weighted undirected graph $G = ([n], E, w)$, we
denote its edge incidence operator $D(G) \in \R^{m  \times n}$; recall that this 
is a matrix whose number of rows equals the number of edges, $m = |E|$, and if
edge $\ell$ connects nodes $i$ and $j$, then 
\begin{equation}
\label{eq:edge_indicidence_op}
\big[ D(G) \big]_{\ell k} =  
\begin{cases}
+w_{ij} & k = i \\
-w_{ij} & k = j \\
0 & \text{otherwise}.
\end{cases}
\end{equation}
Thus, to reiterate the equivalence using the notation just introduced, the
penalty operator in the generalized lasso form \eqref{eq:generalized_lasso} of
the Voronoigram \eqref{eq:tv_voronoi_graph} is \smash{$D^\Vor = D(G^\Vor)$}, the
edge incidence operator of the Voronoi graph \smash{$G^\Vor$}. And, as is clear
from the discussion, the Voronoigram is not just an instance of an arbitrary
generalized lasso problem, it is an instance of TV denoising on a graph. 
Alternative choices of graphs for TV denoising will be discussed in Section
\ref{sub:other_graphs}.              

What does casting the Voronoigram in generalized lasso form do for us? It
enables us to use existing theory on the generalized lasso to read off results
about the structure and complexity of Voronoigram estimates.
\citet{tibshirani2011solution, tibshirani2012degrees} show the following about
the solution \smash{$\htheta$} in problem \eqref{eq:generalized_lasso}: if we
denote by \smash{$A = \{ i \in [n] : (D \htheta)_i \not= 0\}$} the active set 
corresponding to \smash{$D \htheta$}, and \smash{$s = \sign( (D\htheta)_A )$}
the active signs, then we can write
\begin{equation}
\label{eq:generalized_lasso_sol}
\htheta = P_{\nul(D_{-A})} (y - \lambda D_A^\T s).
\end{equation}
where $D_A$ is the submatrix of $D$ with rows that correspond to $A$,
\smash{$D_{-A}$} is the submatrix with the complementary set of rows, and 
\smash{$P_{\nul(D_{-A})}$} is the projection matrix onto \smash{$\nul(D_{-A})$},
the null space of \smash{$D_{-A}$}. When we take $D = D(G)$, the edge incidence
operator on a graph $G$, the null space \smash{$D_{-A}$} has a simple analytic
form that is spanned by indicator vectors on the connected components of the
subgraph of $G$ that is induced by removing the edges in $A$. This allows us to
rewrite \eqref{eq:generalized_lasso_sol}, for a generic TV denoising estimator
\smash{$\htheta = \htheta(G)$}, as 
\begin{equation}
\label{eq:tv_denoising_sol}
\big[ \htheta(G) \big]_i = \sum_{k = 1}^{\hat{K}} (\bar{y}_k - \hat{s}_k) \cdot
1 \big\{ i \in \hat{C}_k \big\},  \quad i = 1,\dots,n
\end{equation}
where \smash{$\hat{K}$} is the number of connected components of the subgraph of
$G$ induced by removing edges in $A$, \smash{$\hat{C}_k$} denotes the $k\th$ such
connected component, \smash{$\bar{y}_k$} denotes the average of points $y_i$
such that \smash{$i \in \hat{C}_k$}, and \smash{$\hat{s}_k$} denotes the average
of the values \smash{$\lambda (D_A^\T s)_i$} over \smash{$i \in \hat{C}_k$}.  

What is special about the Voronoigram is that \eqref{eq:tv_denoising_sol},
combined with the structure of \smash{$\cF^\Vor_n$} (piecewise constant
functions on the Voronoi diagram), leads to an analogous piecewise constant
representation on the \emph{original input domain} $\XDom$, as written and 
discussed in \eqref{eq:voronoigram_sol} in the introduction. Here each 
\smash{$\hat{R}_k = \{ \Part_i : i \in \hat{C}_k \}$}, the union of Voronoi
cells of points in connected component \smash{$\hat{C}_k$}.

Beyond local structure, we can learn about the complexity of the Voronoigram 
estimator---vis-a-vis its \emph{degrees of freedom}---from generalized lasso 
theory. In general, the (effective) degrees of freedom of an estimator
\smash{$\htheta$} is defined as \citep{efron1986biased, hastie1990generalized}:
\[
\df(\htheta) = \frac{1}{\sigma^2} \sum_{i=1}^n \Cov( \htheta_i, y_i ),
\]
where $\sigma^2 = \Var(z_i)$ denotes the noise variance in the data model
\eqref{eq:model}. \citet{tibshirani2011solution, tibshirani2012degrees} prove
using Stein's lemma \citep{stein1981estimation} that when each $z_i \sim N(0,
\sigma^2)$ (i.i.d.\ for $i=1,\dots,n$), it holds that  
\begin{equation}
\label{eq:generalized_lasso_df}
\df(\htheta) = \E[ \nuli(D_{-A}) ],
\end{equation}
where \smash{$\nuli(D_{-A})$} is the nullity (dimension of the null space) of
\smash{$D_{-A}$}, and recall $A$ is the active set corresponding to \smash{$D
  \htheta$}. For $D = D(G)$ and \smash{$\htheta = \htheta(G)$}, the TV denoising
estimator over a graph $G$, the result in \eqref{eq:generalized_lasso_df}
reduces to  
\begin{equation}
\label{eq:tv_denoising_df}
\df\big( \htheta(G) \big) = \E\big[ \text{\# of connected components in 
  \smash{$\htheta(G)$}} \big].  
\end{equation}
As a short interlude, we note that this somewhat remarkable because the
connected components are adaptively chosen in the graph TV denoising estimator,
and yet it does not appear that we ``pay extra'' for this data-driven selection
in \eqref{eq:tv_denoising_df}. This is due to the $\ell_1$ penalty that appears
in the TV denoising criterion, which induces a ``counterbalancing'' shrinkage  
effect---recall we fit shrunkage averages, rather than averages, in
\eqref{eq:tv_denoising_sol}. For more discussion, see
\citet{tibshirani2015degrees}.       

The result \eqref{eq:tv_denoising_df} is true of any TV denoising estimator,
including the Voronoigram. However, what is special about the Voronoigram is
that we are able to write this purely in terms of the fitted function
\smash{$\hf^\Vor$}:   
\begin{equation}
\label{eq:voronoigram_df}
\df(\htheta^\Vor) = \E\big[ \text{\# of locally constant regions in 
  \smash{$\hf^\Vor$}} \big],   
\end{equation}
because by construction the number of locally constant regions in
\smash{$\hf^\Vor$} is equal to the number of connected components in 
\smash{$\htheta^\Vor$}.\footnote{For this to be true, strictly speaking, we
  require that for each $i$ and $j$ in different connected components with
  respect to the subgraph defined by the active set $A$ of
  \smash{$\htheta^\Vor$}, we have \smash{$\htheta_i \not= \htheta_j$}. However,
  for any fixed $\lambda$, this occurs with probability  one if the response
  vector $y$ is drawn from a continuous probability distribution; see
  \citet{tibshirani2012degrees, tibshirani2013lasso}.}   


\subsection{Alternatives: \texorpdfstring{$\varepsilon$}{eps}-neighborhood and
  kNN graphs}      
\label{sub:other_graphs}

We now review two more standard graph-based alternatives to the Voronoigram: TV
denoising over $\varepsilon$-neighborhood and $k$-nearest neighbor (kNN)
graphs. Discrete TV over such graphs has been studied by many, including
\citet{wang2016trend} (experimentally), and \citet{garciatrillos2016continuum,
garciatrillos2019variational, padilla2020adaptive} (formally).  The general
recipe is to run TV denoising over a graph $G = ([n], E, w)$ formed using the
design points $x_1,\dots,x_n$. We note that it suffices to specify the weight
function here, since the edge set is simply defined by all pairs of nodes that
are assigned nonzero weights. For the \emph{$\varepsilon$-neighborhood graph},
we take
\begin{equation}
\label{eq:eps_weights}
w^\Eps_{ij} = 
\begin{cases}
1 & \|x_i - x_j\|_2 \leq \varepsilon \\
0 & \text{otherwise},
\end{cases}
\quad i,j = 1,\dots,n,
\end{equation}
where $\varepsilon > 0$ is a user-defined tuning parameter. For the
(symmetrized) \emph{$k$-nearest neighbor graph}, we take 
\begin{equation}
\label{eq:knn_weights}
w^\kNN_{ij} = 
\begin{cases}
1 & \|x_i - x_j\|_2 \leq \max\big\{ 
\|x_i-x_{(k)}(x_i)\|_2, \,
\|x_j-x_{(k)}(x_j)\|_2 \big\} \\ 
0 & \text{otherwise},
\end{cases}
\quad i,j = 1,\dots,n,
\end{equation}
where \smash{$x_{(k)}(x_i)$} denotes the element of
$\{x_1,\dots,x_{i-1},x_{i+1},\dots,x_n\}$ that is
$k\th$ closest in $\ell_2$ distance to $x_i$ (and we break ties arbitrarily, if
needed), and $k \in [n]$ is a user-defined tuning parameter. 

We denote the resulting graphs by \smash{$G^\Eps$} and \smash{$G^\kNN$},
respectively, and the resulting graph-based TV denoising estimators by
\smash{$\htheta^\Eps = \htheta(G^\Eps)$} and \smash{$\htheta^\kNN =
  \htheta(G^\kNN)$}, respectively. To be explicit, these solve
\eqref{eq:generalized_lasso} when the penalty operators are taken to be the 
relevant edge incidence operators \smash{$D = D(G^\Eps)$} and \smash{$D =
  D(G^\kNN)$}, respectively.  

It is perhaps worth noting that the $\varepsilon$-neighborhood graph is a
special case of a \emph{kernel graph} whose weight function is of the form
$w_{ij} = K(\|x_i-x_j\|_2)$ for a kernel function $K$. Though we choose to
analyze the $\varepsilon$-neighborhood graph for simplicity, much of our
theoretical development for TV denoising on this graph carries over to more
general kernel graphs, with suitable conditions on $K$. We remark that the kNN
and Voronoi graphs do not fit neatly in kernel form, as the weight they assign
to $i,j$ depends not only $x_i,x_j$ but also on $x_1,\dots,x_n$. That said, in 
either case the graph weights are well-approximated by kernels asymptotically;
see Appendix \ref{app:asymptotic_limits} for the effective kernel for the
Voronoi graph.   

\subsection{Discussion and comparison of properties}
\label{sub:graph_comparisons}

We begin with some similarities, starting by recapitulating the properties
discussed in the second-to-last subsection: all three of \smash{$\htheta^\Eps$}, 
\smash{$\htheta^\kNN$}, and \smash{$\htheta^\Vor$}---the TV denoising estimators
on the $\varepsilon$-neighborhood, kNN, and Voronoi graphs, respectively---have
adaptively~chosen piecewise constant structure, as per
\eqref{eq:tv_denoising_sol} (though to be clear, they will have generically
different connected components for the same response vector $y$ and tuning
parameter $\lambda$). All three estimators also have a simple unbiased estimate
for their degrees of freedom, as per \eqref{eq:tv_denoising_df}. And lastly, all
three are given by solving a highly structured convex optimization problems for
which a number of efficient algorithms exist; see, e.g.,
\citet{osher2005iterative, chambolle2009total, goldstein2010geometric,
  hoefling2010path, chambolle2011first, tibshirani2011solution, landrieu2016cut,
  wang2016trend}.      

A further notable property that all three estimators share, which has not yet been
discussed, is \emph{rotational invariance}. This means that, for any orthogonal
$U \in \R^{d \times d}$, if we were to replace each design point $x_i$ by
\smash{$\tilde{x}_i = U x_i$} and recompute the TV denoising estimate using the
$\varepsilon$-neighborhood, kNN, or Voronoi graphs (and with the same response
vector $y$ and tuning parameter $\lambda$) then it will remain unchanged. This
is true because the weights underlying these three graphs---as we can see from
\eqref{eq:voronoi_weights}, \eqref{eq:eps_weights}, and
\eqref{eq:knn_weights}---depend on the design points only through the pairwise
$\ell_2$ distances $\|x_i-x_j\|_2$, which an orthogonal transformation
preserves.     

We now turn the a discussion of the differences between these graphs and their
use in denoising. 

\paragraph{Auxiliary tuning parameters.}

TV denoising over the $\varepsilon$-neighborhood and $k$-nearest neighbor graphs
each have an ``extra'' tuning parameter when compared the Voronoigram: a tuning
parameter associated with learning the graph itself ($\varepsilon$ and $k$,
respectively). This auxiliary tuning parameter must be chosen carefully in order 
for the discrete TV penalty to be properly behaved; as usual, we can turn to
theory (e.g., \citealp{garciatrillos2016continuum,
  garciatrillos2019variational}) to prescribe the proper asymptotic scaling for
such choices, but in practice these are really just guidelines. Indeed, as we
vary $\varepsilon$ and $k$ we can typically find an observable practical impact
on the performance of TV denoising estimators using their corresponding graphs,
especially for the  $\varepsilon$-neighborhood graph (for which $\varepsilon$
impacts connectedness). One may see this by comparing the results of 
Section \ref{sec:experiments} to those of Appendix
\ref{app:sensitivity_analysis}.  All in all, the need to appropriately choose
auxiliary tuning parameters when using these graphs for TV denoising is a
complicating factor for the practitioner.  

\paragraph{Connectedness.}

A related practical consideration: only the Voronoi adjacency graph is
guaranteed to be connected (cf.\ Lemma
\ref{lem:topological-connectedness-equiv-graph-connectedness} in the appendix),
while the kNN and $\varepsilon$-neighborhood graphs have varying degrees of
connectedness depending on their auxiliary parameter.  In particular, the
$\varepsilon$-neighborhood graph is susceptible to isolated points.  This can be
problematic in practice: having many connected components and in particular
having isolated points prevents the estimator from properly denoising, leading
to degraded performance. This phenomenon is studied in Section
\ref{sub:function_estimation}, where the $\varepsilon$-neighborhood graph, grown
to have roughly the same average degree as the Voronoi adjacency and kNN graphs,
sees worse performance when used in TV denoising.  A workaround is to grow the
$\varepsilon$-neighborhood graph to be denser; but of course this increases the
computational burden in learning the estimator and storing the graph.

\paragraph{Computation.}

On computation of the graphs themselves, the Voronoi diagram of $n$ points in
$d$ dimensions has worst-case complexity of \smash{$O(n\log n + n^{\lceil
    d/2\rceil})$} \citep{aurenhammer2000voronoi}.\footnote{Note that the Voronoi
  adjacency graph as considered in Section \ref{sub:voronoigram} intersects the
  Voronoi diagram with the domain $\XDom$ on which the $n$ points are sampled,
  which incurs the additional step of checking whether each vertex of the
  Voronoi diagram belongs in $\XDom$.  For simple domains (say, the unit cube), 
  this can be done in constant time for each edge as they are enumerated during
  graph construction.}   
In applications, this worst-case complexity may be pessimistic; for example,  
\citet{dwyer1991higher} finds that the Voronoi diagram of $n$ points
sampled uniformly at random from the $d$-dimensional unit ball may be
computed in linear expected time. 

On the other hand, the \smash{$O(n\log n + n^{\lceil d/2\rceil})$} runtime does
not include calculation of the weights \eqref{eq:voronoi_weights} on the edges
of the Voronoi adjacency graph, which significantly increases the computational
burden, especially in higher dimensions (it is essentially intractable for $d
\geq 4$). One alternative is to simply use the unweighted Voronoi adjcacency
graph for denoising---dropping the weights \smash{$w^\Vor_{ij}$} in the summands
in  \eqref{eq:tv_voronoi_graph} but keeping the same edge structure---which we
will see, in what follows, has generally favorable practical and theoretical
(minimax) performance.          

Construction of the $\varepsilon$-neighborhood and kNN graphs, in a brute-force 
manner, has complexity $O(dn^2)$ in each case. The complexity of building the
$k$-nearest neighbor graph can be improved to $O(dn\log n)$ by using $k$-d trees
\citep{friedman1977algorithm}. This is dominated by initial cost of building the 
$k$-d tree itself, so a practitioner seeking to tune over the number of nearest
neighbors is able to build kNN graphs at different levels of density
relatively efficiently. As far as we know, there is no analogous general-purpose 
algorithmic speedup for the $\varepsilon$-neighborhood graph, but practical 
speedups may be possible by resorting to approximation techniques (for example,
using random projections or hashing). 



\paragraph{Extrapolation.}

A central distinction between the Voronoigram and TV denoising methods based on
$\varepsilon$-neighborhood and kNN graphs is that the latter methods are
purely discrete, which means that---as defined---they really only produce fitted
values (estimates of the underlying regression function values) at the design
points, and not an entire fitted function (an estimate of the underlying
function). Meanwhile, the Voronoigram produces a fitted function \emph{via} the
fitted values at the design points. Recall the equivalence between problems
\eqref{eq:tv_voronoi} and \eqref{eq:tv_voronoi_graph}, and the 
central property between the discrete and continuum estimates highlighted in
\eqref{eq:voronoi_extrapolate}---to rephrase once again, this says that
\smash{$\hf^\Vor$} is just as complex in continuous-space (as measured by
continuum TV) as \smash{$\htheta^\Vor$} is in discrete-space (as measured by
discrete TV).      

We note that it would also be entirely natural to extend the fitted values
\smash{$\htheta_i = \hf(x_i)$}, $i=1,\dots,n$ from TV denoising using the
$\varepsilon$-neighborhood or kNN graph as a piecewise constant function over 
the Voronoi cells $\Part_1,\dots,\Part_n$,  
\[
\hf = \sum_{i=1}^n \htheta_i \cdot 1_{\Part_i}.
\]
To see this, observe that this is nothing more than the ubiquitous 1-nearest 
neighbor (1NN) prediction rule performed on the fitted values,  
\[
\hf(x) = \hf(x_i), \quad \text{where $\|x - x_i\|_2 = \min_{j=1,\dots,n} \|x -
  x_j\|_2$}.   
\]
However, this extension \smash{$\hf$} does not generally satisfy the property
that its continuum TV is equal to the graph-based TV of \smash{$\htheta$} (with
respect to the original geometric graph, be it $\varepsilon$-neighborhood or
kNN). The complexity-preserving property in \eqref{eq:voronoi_extrapolate} of
the Voronoigram is truly special.\footnote{In fact, this occurs for not one but
  two natural notions of complexity: TV, as in \eqref{eq:voronoi_extrapolate}, 
  and degrees of freedom, as in \eqref{eq:voronoigram_df}. The latter says that
  \smash{$\hf^\Vor$} has just as many locally constant regions (connected
  subsets of $\XDom$) as \smash{$\htheta^\Vor$} has connected components (with
  respect to the Voronoi adjacency graph). This is not true in general for the
  1NN extensions fit to TV denoising estimates on $\varepsilon$-neighborhood or
  kNN graphs; see Section \ref{sub:extrapolation} and Figure
  \ref{fig:03_FunctionEstimationPredictions} in particular.}        




\medskip
\noindent
We finish by summarizing two more points of comparison for discrete TV on the
Voronoi graph versus $\varepsilon$-neighborhood and kNN graphs. These will come
to light in the theory developed later, but are worth highlighting now. First,
discrete TV on the Voronoi adjacency graph, the $\varepsilon$-neighborhood
graph, and the kNN graph can be said to each track different population-level 
quantities---the most salient difference being that discrete TV on a Voronoi
graph in the large-sample limit does not depend on the distribution of the
design points, unlike the other two graphs (compare
\eqref{eq:asymptotic_limit_voronoi} to \eqref{eq:asymptotic_limit_eps} and
\eqref{eq:asymptotic_limit_knn}). Second, while TV denoising on all three
graphs obtains the minimax error rate for functions that are bounded in TV and
$L^\infty$, on the $\varepsilon$-neighborhood and the kNN graphs TV denoising is
furthermore manifold adaptive, and it is not clear the same is true of the
Voronoigram (see Remark \ref{rem:padilla} following Theorem
\ref{thm:voronoi_ub}).

%% file: asymptotic_limits.tex
Having introduced, discussed, and compared graph-based formulations of total
variation---with respect to the Voronoi, $k$-nearest neighbor, and
$\varepsilon$-neighborhood graphs---a natural question remains: as we grow the
number of design points $n$ used to construct the graphs, do these discrete
notions of TV approach particular continuum notions of TV?  Answers to these
questions, aside from being of fundamental interest, will help us better
understand the effects of using these different graph-based TV regularizers in 
the context of nonparametric regression. 

The asymptotic limits for the TV functional defined over the
$\varepsilon$-neighborhood and $k$-nearest neighbor graphs have in fact already
been derived by \citet{garciatrillos2016continuum} and
\citet{garciatrillos2019variational}, respectively. These results are reviewed
in Remark \ref{rem:trillos}, following the presentation of our main result in this
section, Theorem \ref{thm:asymptotic_limit_voronoi}, on the asymptotic limit for
TV over the Voronoi graph. First, we introduce some helpful notation. Given $G = 
([n], E, w)$, a weighted undirected graph, we denote its corresponding discrete
TV functional by  
\begin{equation}
\label{eq:discrete_total_variation}
\DTV(\theta ; w) = \sum_{\{i,j\} \in E} w_{ij} |\theta_i - \theta_j|.
\end{equation}
Given $x_1,\dots,x_n \in \XDom$, and $f : \XDom \to \R$, we also use the
shorthand $f(x_{1:n}) = (f(x_1),\dots,f(x_n)) \in \R^n$.  

Next we introduce an assumption that we require on the sampling distribution of
the random design points.  

\begin{assumption}{A}{1}
\label{assump:density_bounded}
The design distribution has density $p$ (with respect to Lebesgue measure),
which is bounded away from 0 and $\infty$ uniformly on $\XDom = (0,1)^d$; that 
is, there exist constants \smash{$p_{\mn}, p_{\mx}$} such that  
\[
0 < p_{\mn} \leq p(x) \leq p_{\mx} < \infty, \quad \text{for all $x \in
  \XDom$}. 
\]
\end{assumption}

We are now ready to present our main result in this section.

\begin{theorem}
\label{thm:asymptotic_limit_voronoi}
Assume that $x_1,\dots,x_n$ are i.i.d.\ from a distribution satisfying
Assusmption \ref{assump:density_bounded}, and additionally assume its density
$p$ is Lipschitz: $|p(y) - p(x)| \leq L\|y - x\|_2$ for all $x,y \in \XDom$ and
some constant $L>0$. Consider the Voronoi graph whose edge weights are 
defined in \eqref{eq:voronoi_weights}. For any fixed $d \geq 2$ and $f \in 
C^2(\XDom)$, as $n \to \infty$, it holds that     
\begin{equation}
\label{eq:asymptotic_limit_voronoi}
\DTV\Big(f(x_{1:n}); \, w^\Vor\Big) \to c_d \int_{\XDom} \|\nabla f(x)\|_2 \,dx,
\end{equation}
in probability, where $c_d$ is the constant 
\[
c_d = \frac{\eta_{d - 2}^2}{d - 1}\int_{0}^{\infty} 
\int_{0}^{\infty} t^d s^{d - 2} \exp\biggl( -\Leb_d\Bigl\{\frac{t^2}{4} +
s^2\Bigr\}^{d/2} \biggr) \,ds \,dt,
\]
and $\eta_{d-2}$ denotes the Hausdorff measure of the $(d-2)$-dimensional unit
sphere, and $\Leb_d$ the Lebesgue measure of the $d$-dimensional unit ball.  
\end{theorem}

The proof of Theorem \ref{thm:asymptotic_limit_voronoi} is long and involved and
deferred to Appendix \ref{app:asymptotic_limits}. A key idea in the proof is
show that the weights \eqref{eq:voronoi_weights} have an asymptotically
equivalent kernel form, for a particular (closed-form) kernel that we refer to
as the \emph{Voronoi kernel}. We believe this result is itself significant and
may be of independent interest.   

We now make some remarks.

\begin{remark}
The assumption that $f$ is twice continuously differentiable, $f \in
C^2(\XDom)$, in Theorem \ref{thm:asymptotic_limit_voronoi} is used to simplify
the proof; we believe this can be relaxed, but we do not attempt to do so. It is
worth recalling that under this condition, the right-hand side in 
\eqref{eq:asymptotic_limit_voronoi} is a scaled version of the TV of $f$, since
in this case \smash{$\TV(f) = \int_\XDom \|\nabla f(x)\|_2 \, dx$}.  
\end{remark}

\begin{remark}
\label{rem:trillos}
The fact that the asymptotic limit of the Voronoi TV functional is
\emph{density-free}, meaning the right-hand side in
\eqref{eq:asymptotic_limit_voronoi} is (a scaled version of) ``pure'' total  
variation and does not depend on $p$, is somewhat remarkable. This stands in
contrast to the asymptotic limits of TV functionals defined over
$\varepsilon$-neighborhood and kNN graphs, which turn out to be density-weighted
versions of continuum total variation. We transcribe the results of 
\citet{garciatrillos2016continuum} and \citet{garciatrillos2019variational} to
our setting, to ease the comparison. From \citet{garciatrillos2016continuum},
for the $\varepsilon$-neighborhood weights \eqref{eq:eps_weights} and any
sequence $\varepsilon=\varepsilon_n$ satisfying certain scaling conditions, it
holds as $n \to \infty$ that  
\begin{equation}
\label{eq:asymptotic_limit_eps}
\frac{1}{n^2 \varepsilon_n^{d+1}} \DTV \Big(f(x_{1:n}); \, w^\Eps \Big) \to c'_d  
\int_{\XDom} \|\nabla f(x)\|_2 \, p^2(x) \,dx,
\end{equation}
in a particular notion of convergence, for a constant $c'_d>0$. From
\citet{garciatrillos2019variational}, for the kNN weights \eqref{eq:knn_weights}
and any sequence $k=k_n$ satisfying certain scaling conditions, defining 
\smash{$\bar\varepsilon_n = (k_n/n)^{1/d}$}, it holds as $n \to \infty$ that  
\begin{equation}
\label{eq:asymptotic_limit_knn}
\frac{1}{n^2 \bar\varepsilon_n^{d+1}} \DTV \Big(f(x_{1:n}); \, w^\kNN \Big) \to
c''_d \int_{\XDom} \|\nabla f(x)\|_2 \, p^{1-1/d}(x) \,dx,   
\end{equation}
again in a particular notion of convergence, and for a constant $c_d''>0$.  

These differences have interesting methodological interpretations. First, recall 
that traditional regularizers used in nonparametric regression---which includes 
those in smoothing splines, thin-plate splines, and locally adaptive regression
splines, trend filtering, RKHS estimators, and so on---are not 
design-density dependent. In this way, the Voronoigram adheres closer to the 
statistical mainstream than TV denoising on $\varepsilon$-neighborhood or kNN
graphs, since the regularizer in the Voronoigram tracks ``pure'' TV in large
samples. Furthermore, by comparing \eqref{eq:asymptotic_limit_eps} to
\eqref{eq:asymptotic_limit_voronoi} we see that, relative to the Voronoigram, TV
denoising on the $\varepsilon$-neighborhood graph does not assign as strong a
penalty to functions that are wiggly in low-density regions and smoother in
high-density regions. TV denoising on the $k$-nearest neighbor graph lies in
between the two: the density $p$ appears in \eqref{eq:asymptotic_limit_knn}, but
raised to a smaller power than in \eqref{eq:asymptotic_limit_eps}.  

We may infer from this scenarios in which density-weighted TV denoising would be
favorable to density-free TV denoising and vice versa. In a sampling model where
the underlying regression function exhibits more irregularity in a low-density
region of the input space, we would expect a density-weighted method to perform
better since the density weighting provides a larger effective ``budget'' for
the penalty, leading to greater regularization and variance reduction
overall. Conversely, in a sampling model where the regression function exhibits
greater irregularity in a high-density region, we would expect a density-free
method to have a comparative advantage because the density weighting gives 
rise to a smaller ``budget'', hampering the ability to properly regularize. In
Section \ref{sec:experiments}, we consider sampling models that reflect these
qualities and assess the performance of each method empirically.  
\end{remark}

\begin{remark}
It is worth noting that it should be possible to remove the density dependence
in the asymptotic limits for the TV functionals over the
$\varepsilon$-neighborhood and kNN graphs. Following seminal ideas in
\citet{coifman2006diffusion}, we would first form an estimate \smash{$\hat{p}$}
of the design density $p$, and then we would reweight the
$\varepsilon$-neighborhood and kNN graphs to precisely cancel the dependence on
$p$ in their limiting expressions. Under some conditions (which includes
consistency of \smash{$\hat{p}$}) this should guarantee that the asymptotic
limits are density-free, that is, in our case, the reweighted
$\varepsilon$-neighborhood and kNN discrete TV functionals converge to ``pure''
TV.
\end{remark}

%% file: experiments.tex
In this section, we empirically examine the properties elucidated in the last
section. We first investigate whether the large sample behavior of the three
graph-based TV functionals of interest matches the prediction from
asymptotics. We then examine the use of each as a regularizer in nonparametric
regression. Our experiments are not intended to be comprehensive,  
but are meant to tease out differences that arise from the interplay between the 
density of the design points and regions of wiggliness in the regression
function.

\subsection{Basic experimental setup}

Throughout this section, our experiments center around a single function, in 
dimension $d=2$: the indicator function of a ball of radius \smash{$r_0 =
  \frac{1}{4}$} centered at \smash{$x_0 = (\frac{1}{2}, \frac{1}{2}) \in \R^2$},   
\begin{equation}
\label{eq:indicator_ball}
f_0 = 1\{x \in B(x_0, r_0)\},
\end{equation}
supported on $\XDom=(0, 1)^2$. This is depicted in the upper display of
Figure \ref{fig:00_FunctionWithSamples} using a wireframe plot. 

\begin{figure}[htb]
\centering
\includegraphics[width=0.8\linewidth]{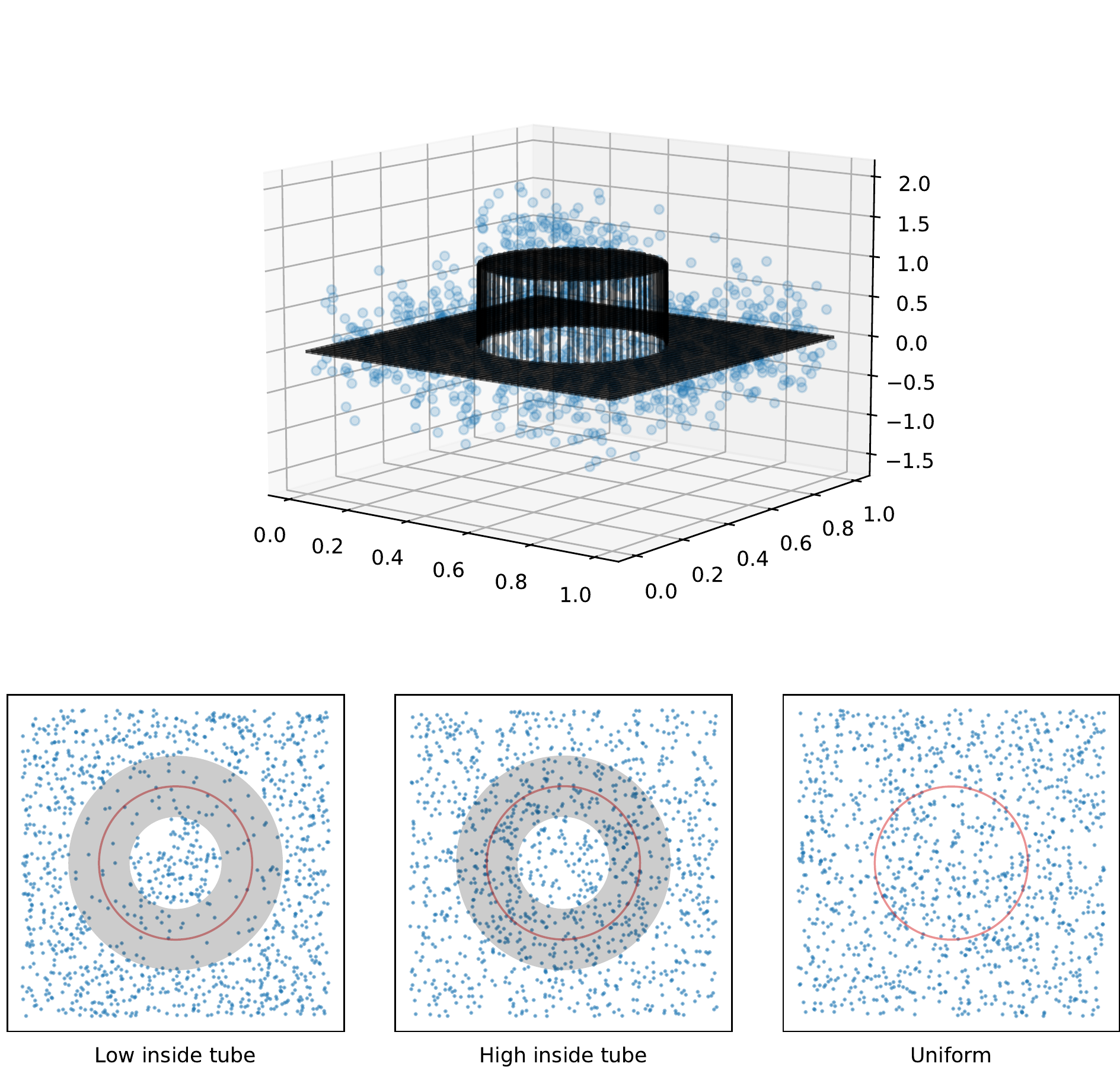}
\caption{\it
  Illustration of the basic experimental setup used in this section. Top: the
  function $f_0$ in \eqref{eq:indicator_ball} depicted using a wireframe plot,
  along with $n=1274$ noisy evaluations of $f_0$ in blue (the noise level is set
  such that the signal-to-noise ratio is 1.) Bottom row: $n=1274$ samples from 
  each of the three design distributions considered. The boundary of the set
  $B(x_0, r_0)$ is denoted in red, and the annulus $A$ is shaded in translucent
  gray.  
}
\label{fig:00_FunctionWithSamples}
\end{figure}

We also consider three choices for the distribution $P$ of the design points
$x_1,\dots,x_n$, supported on $\XDom$.    

\begin{enumerate}
\item ``Low inside tube'': the sampling density $p$ is $0.295$ on an annulus $A$
  centered at $x_0$ that has inner radius $r_0-0.1$ and outer radius
  $r_0+0.1$. (The density on $\XDom\setminus A$ is set to a constant value such
  that $p$ integrates to 1.) 

  \item ``High inside tube'': the sampling density $p$ is $1.2$ on $A$ (with
    again a constant value chosen on $\XDom\setminus A$ such that $p$ integrates
    to 1.)   

  \item ``Uniform'': the sampling distribution is uniform on $\XDom$.
\end{enumerate}

\noindent
We illustrate these sampling distributions empirically by drawing $n=1274$
observations from each and plotting them on the lower set of plots in
Figure \ref{fig:00_FunctionWithSamples}.  We note that the ``high'' density
value of $1.2$ for the ``high inside tube'' sampling distribution yields an
empirical distribution that---by eye---is indistinguishable from the empirical 
distribution formed from uniformly drawn samples. However, as we will soon see,
this departure from uniform is nonetheless large enough that the large sample
behavior of the TV functionals on Voronoi adjacency, $\varepsilon$-neighborhood,
and $k$-nearest neighbor graphs admit discernable differences.

\subsection{Total variation estimation}
\def\graphscale{25}

We examine the of use of the Voronoi adjacency, $k$-nearest neighbor, and 
$\varepsilon$-neighborhood graphs, built from a random sample of design 
points, to estimate the total variation of the function $f_0$ in
\eqref{eq:indicator_ball}. To be clear, here we compute (using the notation
\eqref{eq:discrete_total_variation} introduced in the asymptotic limits
section): 
\[
\DTV \Big( f_0(x_{1:n}); w \Big) = \sum_{\{i,j\} \in E} w_{ij} |f_0(x_i) -
f_0(x_j)|, 
\]
for three choices of edge weights $w$: Voronoi \eqref{eq:voronoi_weights},
$\varepsilon$-neighborhood \eqref{eq:eps_weights}, and kNN
\eqref{eq:knn_weights}.  

We let the number of design points $n$ range from $10^2$ to $10^5$,
logarithmically spaced, with 20 repetitions independently drawn from each design
distribution for each $n$. The $k$-nearest neighbor graph is built using
\smash{$k = \lfloor C_1 \log^{1.1} n\rfloor$}, and the
$\varepsilon$-neighborhood graph is built using \smash{$\varepsilon =
  C_2(\log^{1.1} n / n)^{1/2}$}, where $C_1,C_2$ are constants chosen such that 
the average degree of these graphs is roughly comparable to the average degree
of the Voronoi adjacency graph (which has no tuning parameter). We note that it
is possible to obtain marginally more stable results for the $k$-nearest
neighbor and $\varepsilon$-neighborhood graphs by taking $C_1, C_2$ to be
larger, and thus making the graphs denser. These results are deferred to
Appendix \ref{app:sensitivity_analysis}, though we remark that the need to
separately tune over such auxiliary parameters to obtain more stable results is
a disadvantage of the kNN and $\varepsilon$-neighborhood methods (recall also
the discussion in Section \ref{sub:graph_comparisons}).

Figure \ref{fig:04_TVEstimationResults} shows the results under the three design
distributions outlined previously. For each sample size $n$ and for each graph,
we plot the average discrete TV, and its standard error, with respect to the 20
repetitions. We additionally plot the limiting asymptotic values predicted by
the theory---recall \eqref{eq:asymptotic_limit_voronoi},
\eqref{eq:asymptotic_limit_eps}, \eqref{eq:asymptotic_limit_knn}---as horizontal
lines. Generally, we can see that the discrete TV, as measured by each of the
three graphs, approaches its corresponding asymptotic limit. The standard error
bars for the Voronoi graph tend to be the narrowest, whereas those for the kNN
and $\varepsilon$-neighborhood graphs are generally wider.  In the rightmost
plot, showing the results under uniform sampling, the asymptotic limits of the
discrete TV for the three methods match, since the density weighting is
nullified by the uniform distribution.

\begin{figure}[htb]
\centering
\includegraphics[width=0.98\linewidth]{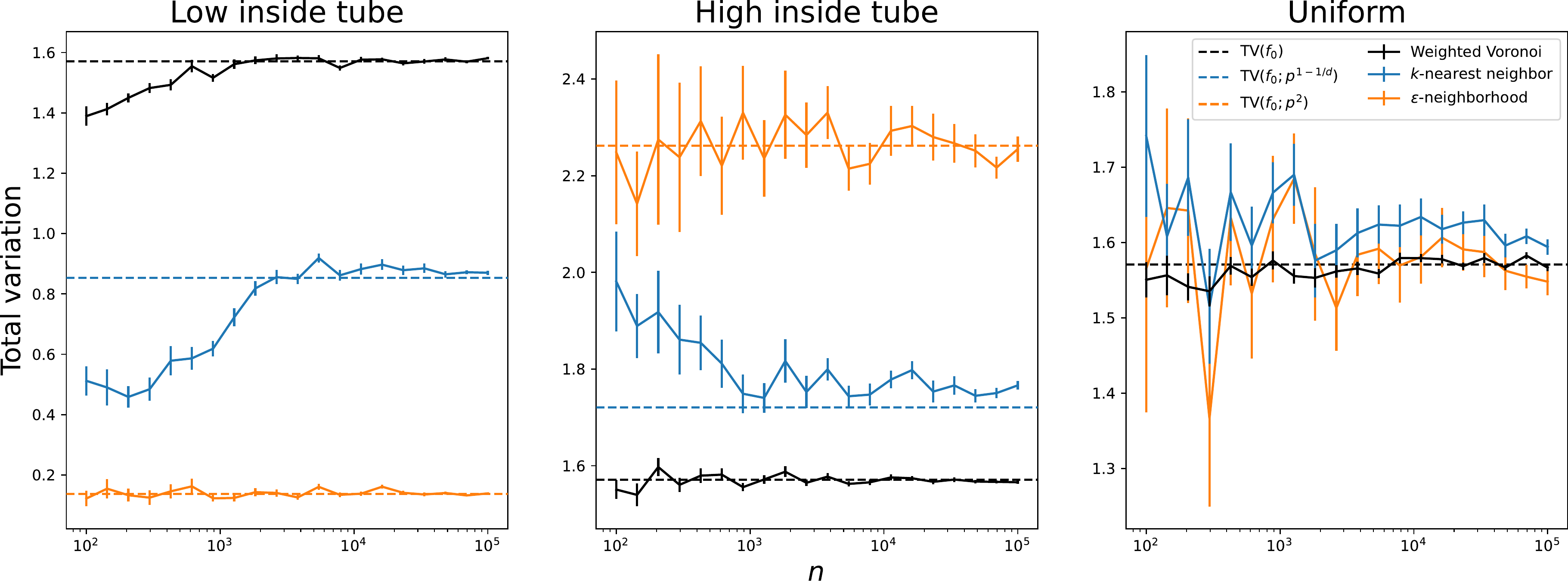}
\caption{\it
  Results from the TV estimation experiment (``weighted Voronoi'' refers to the 
  usual Voronoi adjacency graph, with weights in \eqref{eq:voronoi_weights}, and  
  is used to distinguish it from the Voronoi adjacency graph with unit edge
  weights, which will appear in later experiments). We see that the discrete TV  
  as measured by each graph converges to its asymptotic limit, drawn as a dashed
  horizontal line, as $n$ grows (note that the $x$-axis is on a log scale). 
}
\label{fig:04_TVEstimationResults}
\end{figure}

To give a qualitative sense of their differences, Figure
\ref{fig:05_TVEstimationGraphs} displays the graphs from each of the methods for
a draw of $n=1274$ samples under each sampling distribution.  Note that the
Voronoi adjacency and kNN graphs are connected (this is always the case for the
former), whereas this is not true of the $\varepsilon$-neighborhood graph
(recall Section \ref{sub:graph_comparisons}), with the most noticable contrast
being in the ``low inside tube'' sampling model. This relates to the notion that
the Voronoi and kNN graphs effectively use an adaptive local bandwidth, versus
the fixed bandwidth used by the $\varepsilon$-neighborhood graph. Comparing the
former two (Voronoi and kNN graphs), we also see that there are fewer ``holes''
in the Voronoi graph as it has the quality that it seeks neighbors ``in each
direction'' for each design point.

\begin{figure}[htbp]
\centering
\includegraphics[width=0.95\linewidth]{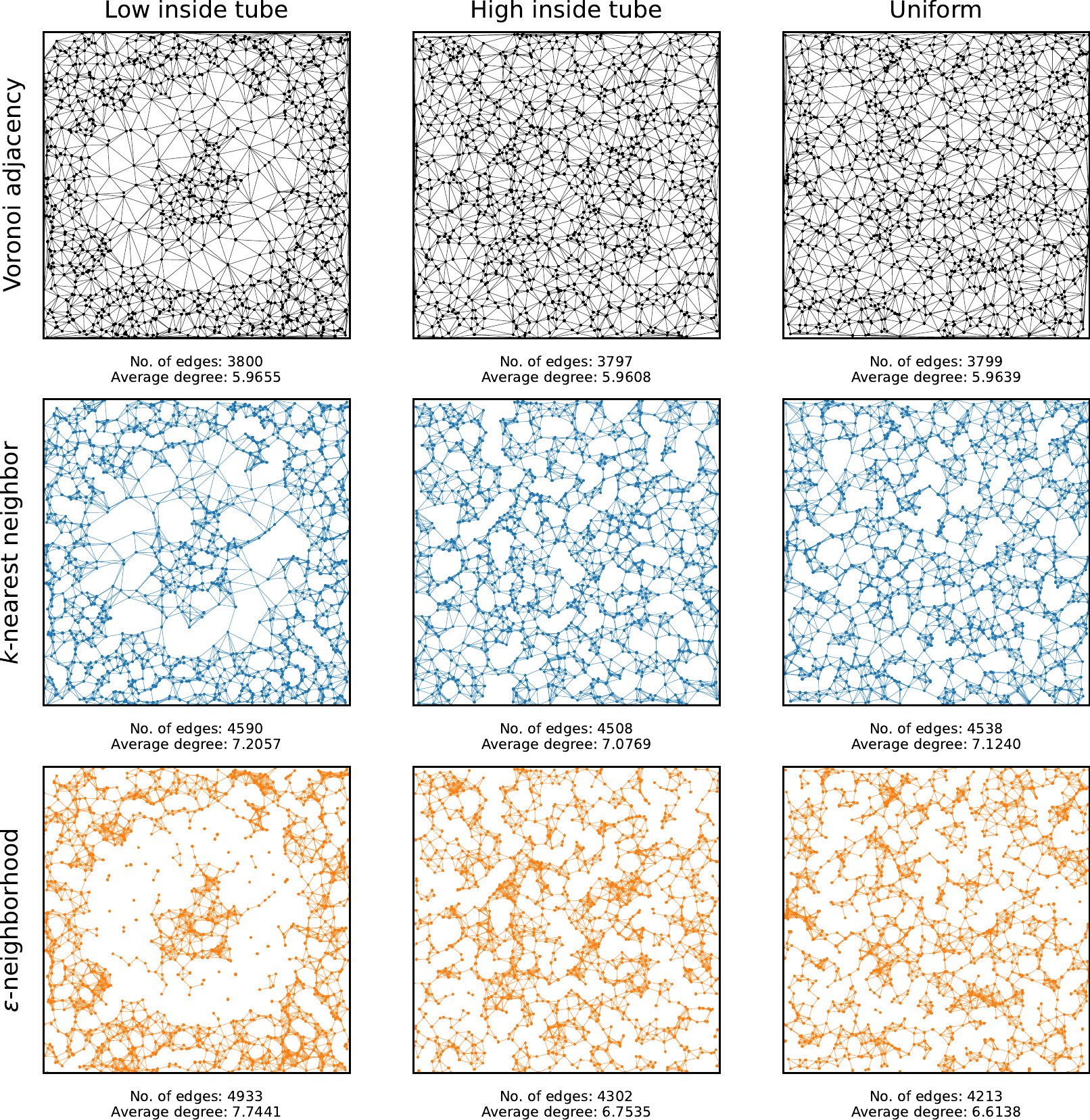}
\caption{\it
  Visualization of the Voronoi, kNN, and $\varepsilon$-neighborhood graphs for a  
  sample of $n=1274$ design points from each of the three sampling distributions
  considered. We see qualitatively very different behaviors in these three graph
  models, and we can also intuit the different asymptotic limits of their
  discrete TV functionals; for example, the strong dependence of the
  $\varepsilon$-neighborhood graph on the sampling density is quite noticeable
  in the ``low inside tube'' setting (bottom left plot).    
}
\label{fig:05_TVEstimationGraphs}
\end{figure}

\subsection{Regression function estimation}
\label{sub:function_estimation}

Next we study the use of discrete TV from the Voronoi, $k$-nearest neighbor, and
$\varepsilon$-neighborhood graphs as a penalty in a nonparametric regression
estimator. In other words, given noisy observations as in \eqref{eq:model} of
the function $f_0$ in \eqref{eq:indicator_ball}, we solve the graph TV denoising
problem \eqref{eq:generalized_lasso} with penalty operator $D$ equal to the edge 
incidence matrix corresponding to the Voronoi \eqref{eq:voronoi_weights},
$\varepsilon$-neighborhood \eqref{eq:eps_weights}, and kNN
\eqref{eq:knn_weights} graphs. 

We fix $n=1274$, and draw each $z_i \sim N(0,\sigma^2)$, where the noise level
$\sigma^2>0$ is chosen so that the signal-to-noise ratio, defined as
\[
\mathrm{SNR} = \frac{\Var(f_0(x_i))}{\sigma^2},
\]
is equal to 1. (Here $\Var(f_0(x_i))$ denotes the variance of $f_0(x_i)$ with
respect to the randomness from drawing $x_i \sim P$.) Each graph TV denoising
estimator is fit over a range of values for the tuning parameter $\lambda$, and
at value of $\lambda$ we record the $L^2(P_n)$ mean squared error
\[
\frac{1}{n} \sum_{i=1}^n \big( \hf(x_i) - f_0(x_i) \big)^2.
\]
Figure \ref{fig:02_FunctionEstimationResults} shows the average of this
$L^2(P_n)$ error, along with its standard error, across the 20 repetitions. The
$x$-axis is parametrized by an estimated degrees of freedom for each $\lambda$
value, to place the methods on common footing---that is, recalling the general 
formula in \eqref{eq:tv_denoising_df} for any TV denoising estimator, we convert
each value of $\lambda$ to the average number of resulting connected components
over the 20 repetitions.    

The results of Figure \ref{fig:02_FunctionEstimationResults} broadly align with
the expectations set forth at the end of Section \ref{sec:asymptotic_limits}:
the density-weighted methods (using kNN and $\varepsilon$-neighborhood graphs)
perform better when the irregularity is concentrated in a low density area
(``low inside tube''), and the density-free method (the Voronoigram) does better
when the irregularity is concentrated in a high density area (``high inside
tube''). We also observe that across all settings, the best performing estimator
tends to be the most parsimonious---the one that consumes the fewest degrees of
freedom when optimally tuned.

\begin{figure}[htb]
\centering
\includegraphics[width=\linewidth]{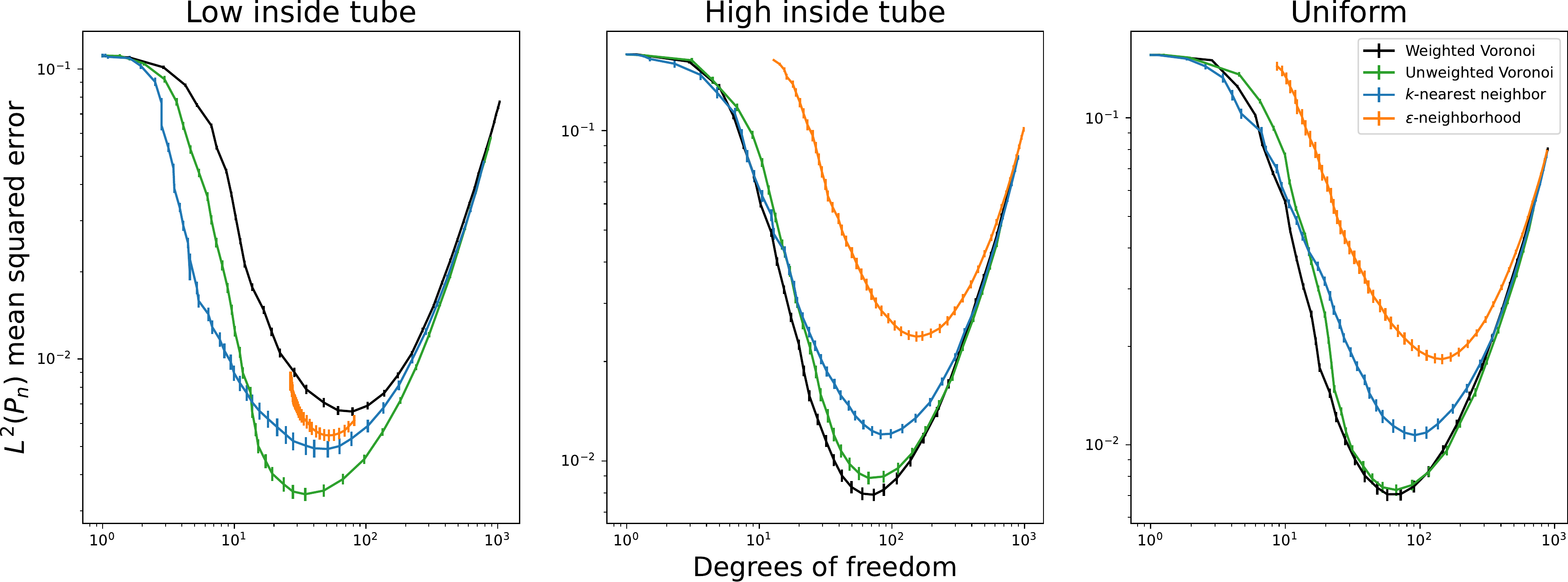}
\caption{\it
  Results from the function estimation experiment (``weighted Voronoi'' refers
  to the usual Voronoi graph and ``unweighted Voronoi'' the graph with the same
  edge structure but unit edge weights). We see that the density-weighted
  methods---TV denoising over the kNN and $\varepsilon$-neighborhood
  graphs---generally do better in the ``low inside tube'' setting, where the
  irregularity in $f_0$ is concentrated in a low density region of the design
  distribution. Conversely, density-free method---TV denoising on the Voronoi
  graph, also known as the Voronoigram---does better in the ``high inside tube''
  scenario, where irregularity is concentrated in a high density region. Lastly,
  TV denoising on the unweighted Voronoi graph does very well in each scenario.     
}
\label{fig:02_FunctionEstimationResults}
\end{figure}

In the ``low inside tube'' setting (leftmost panel of Figure
\ref{fig:02_FunctionEstimationResults}), we see that $\varepsilon$-neighborhood
graph total variation does worse than its kNN counterpart, even though we would
have expected the former to outperform the latter (because it weights the
density more heavily; cf.\ \eqref{eq:asymptotic_limit_eps} and
\eqref{eq:asymptotic_limit_knn}). The poor performance of TV denoising over the 
$\varepsilon$-neighborhood graph may be ascribed to the large number of
disconnected points (see Figures \ref{fig:05_TVEstimationGraphs} and
\ref{fig:03_FunctionEstimationPredictions}), whose fitted values it cannot
regularize.  These disconnected points are also why the minimal degrees of
freedom obtained this estimator (as $\lambda\rightarrow\infty$) is relatively
large compared to kNN TV denoising and the Voronoigram, across all three
settings.  In Appendix \ref{app:sensitivity_analysis}, we conduct sensitivity
analysis where we grow the $\varepsilon$-neighborhood and kNN graphs more
densely while retaining a comparable average degree (to each other). In that
case, the performance of the corresponding estimators becomes comparable (the
$\varepsilon$-neighborhood graph still has some disconnected points), which just 
emphasizes the peril of graph denoising methods which permit isolated points. 

Interestingly, under the uniform sampling distribution (rightmost panel of
Figure \ref{fig:02_FunctionEstimationResults}), where the asymptotic
limits of the discrete TV functionals over the Voronoi, kNN, and
$\varepsilon$-neighborhood graph are the same, we see that the Voronoigram
performs best in mean squared error, which is encouraging empirical evidence in
its favor.

Finally, Figure \ref{fig:02_FunctionEstimationResults} also displays the error
of the \emph{unweighted Voronoigram}---which we use to refer to TV denoising on
the unweighted Voronoi graph, obtained by setting each \smash{$w^\Vor_{ij} = 1$}
in \eqref{eq:tv_voronoi_graph}. This is somewhat of a ``surprise winner'': it
performs close to the best in each of the sampling scenarios, and is
computationally cheaper than the Voronoigram (it avoids the expensive step of
computing the Voronoi edge weights, which requires surface area calculations). We
lack an asymptotic characterization for discrete TV on the unweighted Voronoi
graph, so we cannot provide a strong a priori explanation for the favorable
performance of the unweighted Voronoigram over our experimental
suite. Nonetheless, in view of the example adjacency graphs in Figure
\ref{fig:05_TVEstimationGraphs}, we hypothesize that its favorable performance
is at least in part due to the adaptive local bandwidth that is inherent to the
Voronoi graph, which seeks neighbors ``in each direction'' while avoiding edge
crossings.  Moreover, in Section \ref{sec:estimation_theory} we show that the
unweighted Voronoigram shares the property of minimax rate optimality (for
estimating functions bounded in TV and $L^\infty$), further strengthening its
case.

\subsection{Extrapolation: from fitted values to functions}
\label{sub:extrapolation}

As the last part of our experimental investigations, we consider extrapolating
the graph TV denoising estimators, which represent a sequence of fitted values 
at the design points: \smash{$\hf(x_i)$}, $i=1,\dots,n$, to a entire fitted 
function: \smash{$\hf(x)$}, $x \in \XDom$. As discussed and motivated in Section
\ref{sub:graph_comparisons}, we use the 1NN extrapolation rule for each
estimator. This is equivalently viewed as piecewise constant extrapolation over 
the Voronoi tessellation. 

\begin{figure}[htb]
\centering
\includegraphics[width=0.9\linewidth]{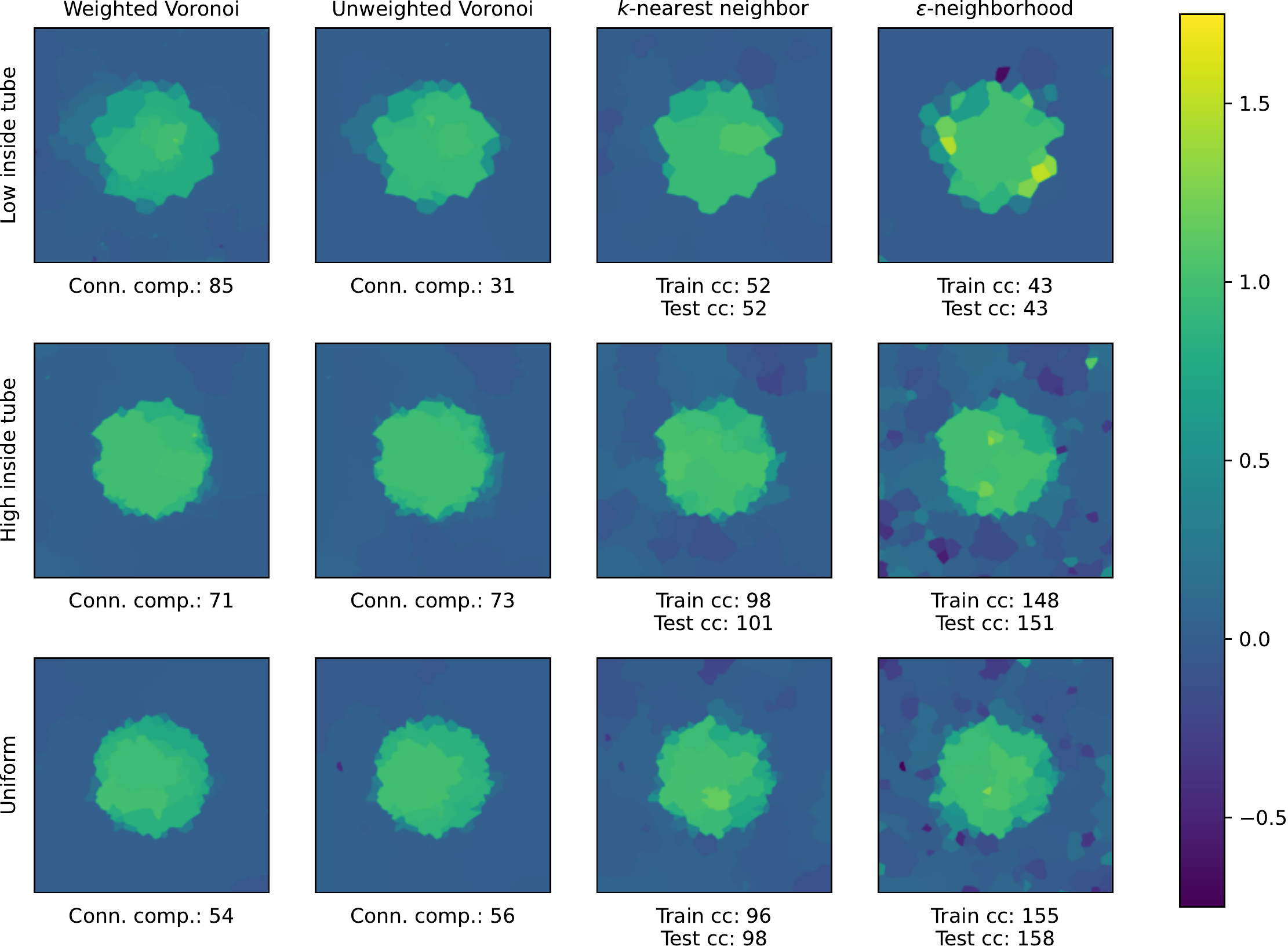}
\caption{\it
  Extrapolants from graph TV denoising estimates, using 1NN extrapolation. We can
  see several qualitative differences, for example, the issues posed by isolated
  points in the $\varepsilon$-neighborhood graph. We also note that the number
  of connected components in the graph used to learn the estimator (which gives
  an unbiased estimate of its degrees of freedom) is guaranteed to match the
  number of connected components in the extrapolant \emph{only} for the Voronoi
  methods.  
}
\label{fig:03_FunctionEstimationPredictions}
\end{figure}

Figure \ref{fig:03_FunctionEstimationPredictions} plots the extrapolants for
each TV denoising estimator, fitted over a particular sample of $n=1274$ points
from each design distribution. In each case, the estimator was tuned to have
optimal mean squared error (cf.\ Figure \ref{fig:02_FunctionEstimationResults}).
From these visualizations, we are able to clearly understand where certain
estimators struggle; for example, we can see the effect of isolated components
in the $\varepsilon$-neighborhood graph in the ``low inside tube'' setting, and
to a lesser extent in the ``high inside tube'' and uniform sampling settings
too. As for the Voronoigram, we previously observed (cf.\ Figure
\ref{fig:02_FunctionEstimationResults} again) that it struggles in the ``low
inside tube'' setting due to the large weights placed on edges crossing the
annulus, and in the upper left plot of Figure
\ref{fig:03_FunctionEstimationPredictions} we see ``patchiness'' around the
annulus, where large jumps are heavily penalized, rather than sharper jumps made
by other estimators (including its unweighted sibling). This is underscored by
the large number of connected components in the Voronoigram versus others in the
``low inside tube'' setting.

Lastly, because the partition induced by the 1NN extrapolation rule is exactly
the Voronoi diagram, we note that the number of connected components on the
training set $\{x_1,\dots,x_n\}$---as measured by connectedness of the fitted
values \smash{$\hf(x_1),\dots,\hf(x_n)$} over the Voronoi graph---always matches
the number of connected components on the test set $\XDom$---as measured by
connectedness of the extrapolant \smash{$\hf$} over the domain $\XDom$. This is
not true of TV denoising over the kNN and $\varepsilon$-neighborhood graphs,
where we can see a mismatch between connectedness pre- and post-extrapolation.

%% file: estimation_theory.tex
In this section, we analyze error rates for estimating $f_0$ given data as in
\eqref{eq:model}, under the assumption that $f_0$ has bounded total
variation. Thus, of central interest will be a (seminorm) ball in the BV space, 
which we denote by      
\[
\BV(L) = \{ f \in L^1(\XDom): \TV(f) \leq L \}.
\]
For simplicity, here and often throughout this section, we suppress the
dependence on the domain $\XDom$ when referring to various function classes of
interest. We use $P$ for the design distribution, and we will primarily be
interested in error in the $L^2(P)$ norm, defined as 
\[
\| \hf - f_0 \|_{L^2(P)}^2 = \int \big( \hf(x) - f_0(x) \big)^2 \, dP(x).
\]
We also use $P_n$ for the empirical distribution of sample $x_1,\dots,x_n$ of
design points, and we will also be concerned with error in the $L^2(P_n)$ norm, 
defined as  
\[
\| \hf - f_0 \|_{L^2(P_n)}^2 = \frac{1}{n} \sum_{i=1}^n \big( \hf(x) - f_0(x)
\big)^2. 
\]
We will generally use the terms ``error'' and ``risk'' interchangeably. Finally,
we will consider the following assumptions, which we refer to as \emph{the
  standard assumptions}.   
\begin{itemize}
\item The design points $x_i$, $i=1,\dots,n$, are i.i.d.\ from a distribution $P$
  satisfying Assumption \ref{assump:density_bounded}.  
\item The response points $y_i$, $i=1,\dots,n$, follow \eqref{eq:model}, with
  i.i.d.\ errors $z_i \sim N(0,\sigma^2)$, $i=1,\dots,n$.  
\item The dimension satisfies $d \geq 2$ and remains fixed as $n \to \infty$.
\end{itemize}
Note that under Assumption \ref{assump:density_bounded}, asymptotic statements
about $L^2(P)$ and $L^2(\Leb)$ errors are equivalent, with $\Leb$ denoting
Lebesgue measure (the uniform distribution) on $\XDom$, since it holds that
\smash{$p_{\mn} \| g \|_{L^2(\Leb)}^2 \leq \| g \|_{L^2(P)}^2 \leq p_{\mx} \| g
  \|_{L^2(\Leb)}^2$} for any function $g$. 

\subsection{Impossibility result without \texorpdfstring{$L^\infty$}{L infty}
  boundedness} 
\label{sub:impossibility}

A basic issue to explain at the outset is that, when $d \geq 2$, consistent
estimation over the BV class $\BV(L)$ is \emph{impossible} in $L^2(P)$
risk. This is in stark contrast to the univariate setting, $d=1$, in which
TV-penalized least squares \citep{mammen1997locally, sadhanala2019additive},
and various other estimators, offer consistency. 

One way to see this is through the fact that $\BV(\XDom)$ does not compactly
embed into $L^2(\XDom)$ for $d \geq 2$, which implies that $L^2$ estimation over
$\BV(L)$ is impossible (see Section 5.5 of \citet{johnstone2015gaussian} for a
discussion of this phenomenon in the Gaussian sequence model). We now state 
this impossibility result and provide a more constructive proof, which sheds
more light on the nature of the problem.        

\begin{proposition}
\label{prop:impossibility}
Under the standard assumptions, there exists a constant $c>0$ (not depending on
$n$) such that     
\[
\inf_{\hf} \sup_{f_0 \in \BV(1) \cap L^2(\XDom)} \E\| \hf - f_0 \|_{L^2(P)}^2
\geq c > 0,  
\]
where the infimum is taken over all estimators \smash{$\hf$} that are measurable
functions of the data $(x_i,y_i)$, $i=1,\dots,n$. 
\end{proposition}

\begin{proof}
As explained above, under Assumption \ref{assump:density_bounded} we may
equivalently study $L^2(\Leb)$ risk, which we do henceforth in this proof. We
simply denote \smash{$\|\cdot\|_{L^2} = \|\cdot\|_{L^2(\Leb)}$}. Consider the
two-point hypothesis testing problem of distinguishing
\[
H_0: f^\star_0 = 0 \quad \text{versus} \quad
H_1: f^\star_1 = \frac{\epsilon^{-d/2}}{2d} \cdot \1_{(0,\epsilon)^d}, 
\]
where $0 < \epsilon < 1$. By construction, $f \in L^2(\XDom)$ and $\TV(f) \leq
1$ for each of \smash{$f = f^\star_0$} and \smash{$f =
  f^\star_1$}. Additionally, we have \smash{$\|f^\star_0 -   f^\star_1\|_{L^2} = 
  \frac{1}{2d}$}. It follows from a standard reduction that  
\begin{align}
\nonumber
\inf_{\hf} \sup_{f_0 \in \BV(1) \cap L^2(\XDom)} \E\|\hf - f_0\|_{L^2} 
&\geq \inf_{\hf}\sup_{f_0 \in \{f^\star_0, f^\star_1\}} \E\|\hf - f_0\|_{L^2} \\
\label{eq:testing_bound}
&\geq \inf_{\psi}\Big( \P_{H_0}(\psi = 1) + \P_{H_1}(\psi=0) \Big),   
\end{align}
where the infimum in the rightmost expression is over all measurable tests 
$\psi$. Now, conditional on the event
\[
\cE = \{ x_i \not\in (0,\epsilon)^d, \, i=1,\dots,n \}, 
\]
the distributions are the same under null and alternative hypotheses,
\smash{$\P_{H_0}(\cdot | \cE) = \P_{H_1}(\cdot | \cE)$}. Additionally, note that
we have \smash{$\P(\cE) \geq (1 - p_\mx\epsilon^d)^n$} under Assumption  
\ref{assump:density_bounded}. Consequently, for any test $\psi$, 
\begin{align*}
\P_{H_1}(\psi = 1) 
&= \P_{H_1}(\psi = 1 | \cE) \P(\cE) + \P_{H_1}(\psi = 1 | \cE^c) \P(\cE^c) \\
&\leq \P_{H_1}(\psi = 1) + 1 - (1 - p_\mx\epsilon^d)^n \\
&= \P_{H_0}(\psi = 1) + 1 - (1 - p_\mx\epsilon^d).
\end{align*}
In other words, just rearranging the above, we have shown that 
\[
\P_{H_0}(\psi = 1) + \P_{H_1}(\psi = 0) \geq (1 - p_\mx\epsilon^d).
\]
Taking $\epsilon \to 0$, and plugging this back into \eqref{eq:testing_bound},
establishes the desired result.  
\end{proof}

The proof of Proposition \ref{prop:impossibility} reveals one reason why
consistent estimation over $\BV(L)$ is not possible: when $d\geq 2$, functions
of bounded variation can have ``spikes'' of arbitrarily small width but large 
height, which cannot be witnessed by any finite number of samples. (We note that
this has nothing to do with noise in the response, and the proposition still
applies in the noiseless case with $\sigma=0$.) This motivates a solution: in 
the remainder of this section, we will rule out such functions by additionally
assuming that $f_0$ is bounded in $L^\infty$.  

\subsection{Minimax error: upper and lower bounds}

Henceforth we assume that $f_0$ has bounded TV \emph{and} has bounded $L^\infty$
norm, that is, we consider the class 
\[
\BV_\infty(L, M) = \{ f \in L^1(\XDom): \TV(f) \leq L, \, \|f\|_{L^\infty} \leq
M \}.
\]
Here \smash{$\|\cdot\|_{L^\infty} = \|\cdot\|_{L^\infty(\XDom)}$} is the
essential supremum norm on $\XDom$. Perhaps surprisingly, additionally assuming
that $f_0$ is bounded in $L^\infty$ dramatically improves prospects for
estimation. The following theorem shows that two different and simple
modifications of the  Voronoigram, appropriately tuned, each achieve a
$n^{-1/d}$ rate of convergence in its sup risk over $\BV_\infty(L,M)$, modulo
log factors. 

\begin{theorem}
\label{thm:voronoi_ub}
Under the standard assumptions, consider either of the following modified
Voronoigram estimators \smash{$\htheta$}: 
\begin{itemize}
\item the minimizer in the Voronoigram problem \eqref{eq:tv_voronoi_graph}, once
  we replace each weight \smash{$w^\Vor_{ij}$} by a clipped version defined as 
  \smash{$\tilde{w}^\Vor_{ij} = \max\{c_0 n^{-(d-1)/d}, w^\Vor_{ij}\}$}, for any
  constant $c_0>0$.         
\item the minimizer in the Voronoigram problem \eqref{eq:tv_voronoi_graph}, once
  we replace each weight \smash{$w^\Vor_{ij}$} by 1.   
\end{itemize}
Let \smash{$\lambda = c\sigma\tau_n(\log n)^{1/2+\alpha}$} for any $\alpha>1$
and a constant $c>0$, where \smash{$\tau_n=n^{(d-1)/d}$} for the clipped weights
estimator and $\tau_n=1$ for the unit weights estimator. There exists another
constant $C>0$ such that for all sufficiently large $n$ and $f_0 \in
\BV_\infty(L,M)$, the estimated function \smash{$\hf = \sum_{i=1}^n \htheta_i
  \cdot 1_{\Part_i}$} (which is piecewise constant over the Voronoi diagram) 
satisfies 
\begin{equation}
\label{eq:voronoi_ub}   
\E\| \hf - f_0 \|_{L^2(P)}^2 \leq C \bigg(
\frac{\sigma L(\log n)^{5/2 + \alpha + 1/d}}{n^{1/d}} + 
\frac{(\log n)^{1+\alpha}}{n} +
\frac{L M(\log n)^{1 + 1/d}}{n^{1/d}} \bigg).
\end{equation}
\end{theorem}

We now certify that this upper bound is tight, up to log factors, by providing a
complementary lower bound. 

\begin{theorem}
\label{thm:minimax_lb}
Under the standard assumptions, provided that $n,L,M$ satisfy 
\smash{$c_0 (M^2 n)^{-\frac{(d-1)}{d}} \leq L \leq C_0 (M^2 n)^{1/d}$} 
for constants $C_0>c_0>0$, the minimax risk satisfies 
\begin{equation}
\label{eq:minimax_lb}
\inf_{\hf} \sup_{f_0 \in\BV_\infty(L, M)}  
\E\| \hf - f_0 \|_{L^2(P)}^2 \geq CLM (M^2 n)^{-1/d},
  \end{equation}
for another constant $C>0$, where the infimum is taken over all estimators
\smash{$\hf$} that are measurable functions of the data $(x_i,y_i)$,
$i=1,\dots,n$.   
\end{theorem}

Taken together, Theorems \ref{thm:voronoi_ub} and \ref{thm:minimax_lb} establish
that the minimax rate of convergence over $\BV_\infty(1,1)$ is $n^{-1/d}$,
modulo log factors. Further, after subjecting it to minor modifications---either
clipping small edge weights, or setting all edge weights to unity (the latter
being particularly desirable from a computational point of view)---the
Voronoigram is minimax rate optimal, again up to log factors.

The proof of the lower bound \eqref{eq:minimax_lb} is Theorem
\ref{thm:minimax_lb} is fairly standard and can be found in Appendix
\ref{app:estimation_theory}. The proof of the upper bound \eqref{eq:voronoi_ub}
in Theorem \ref{thm:voronoi_ub} is much more involved, and the key steps are
described over Sections \ref{sub:analysis_empirical} and
\ref{sub:analysis_population} (with the details deferred to Appendix
\ref{app:estimation_theory}). Before moving on to key parts of the analysis, we
make several remarks.

\begin{remark}
It is not clear to us whether clipping small weights in the Voronoigram penalty
as we do in Theorem \ref{thm:voronoi_ub} (via \smash{$\tilde{w}^\Vor_{ij} = 
  \max\{c_0 n^{-(d-1)/d}, w^\Vor_{ij}\}$}) is actually needed, or whether the
unmodified estimator \eqref{eq:tv_voronoi_graph} itself attains the same or a    
similar upper bound, as in \eqref{eq:voronoi_ub}. In particular, it may be that
under Assumption \ref{assump:density_bounded}, the surface area of the
boundaries of Voronoi cells (defining the weights) are already lower bounded in 
rate by $n^{-(d-1)/d}$, with high probability; however this is presently unclear
to us.     
\end{remark}

\begin{remark}
\label{rem:random_design}
The design points \emph{must be random} in order to have nontrivial rates of 
convergence in our problem setting. If $x_1,\dots,x_n$ were instead fixed, then
for $d \geq 2$ and any $n$ it is possible to construct \smash{$f_0 \in
  \BV_{\infty}(1,1)$} with $f_0(x_i)=0$, $i=1,\dots,n$ and (say)
\smash{$\|f\|_{L^2} = 1/2$}. Standard arguments based on reducing to a
two-point hypothesis testing problem (as in the proof of Proposition 
\ref{prop:impossibility}) reveal that the minimax rate in $L^2$ is trivially
lower bounded by a constant, rendering consistent estimation impossible once
again.       

This is completely different from the situation for $d = 1$, where the minimax
risks under fixed and random design models for TV bounded functions are
basically equivalent. Fundamentally, this is because for $d \geq 2$ the space 
$\BV(\XDom)$ does not compactly embed into $C^0(\XDom)$, the space of continuous
functions (whereas for $d = 1$, all functions in $\BV(\XDom)$ possess at 
least an approximate form of continuity). Note carefully that this is a
different issue than the failure of $\BV(\XDom)$ to compactly embed into 
$L^2(\XDom)$, and that it is not fixed by intersecting a TV ball with an
$L^\infty$ ball.   
\end{remark}

\begin{remark}
We can generalize the definition of total variation in
\eqref{eq:total_variation}, by generalizing the norm we use to constrain the
``test'' function $\phi$ to an arbitrary norm $\|\cdot\|$ on $\R^d$. (See
\eqref{eq:total_variation_general} in the appendix.) The original definition in
\eqref{eq:total_variation} uses the $\ell_2$ norm, $\|\cdot\| =
\|\cdot\|_2$. What would minimax rates look if we used a different choice of
norm to define TV? Suppose that we use an $\ell_p$ norm, for any $p \geq 1$;
that is, suppose we take $\|\cdot\| = \|\cdot\|_p$ as the norm to constrain the
``test'' functions in the supremum. Then under this change, the minimax rate
will still remain $n^{-1/d}$, just as in Theorems \ref{thm:voronoi_ub} and
\ref{thm:minimax_lb}. This is simply due to the fact that $\ell_p$ norms are
equivalent on $\R^d$ (thus a unit ball in the TV-$\ell_p$ seminorm will be
sandwiched in between two balls in TV-$\ell_2$ seminorm of constant radii). 
\end{remark}

\begin{remark}
The minimax rate for estimating a Lipschitz function, that is, the minimax rate
over the class 
\[
\mathrm{Lip}(L) = \{ f: \XDom \to \Reals \,:\, |f(x) - f(z)| \leq L\|x - z\|_2
\; \text{for all $x,z \in \XDom$} \},
\]
is $n^{-2/(2 + d)}$ in squared $L^2$ risk, for constant $L>0$ (not growing with
$n$); see, e.g., \citet{stone1982optimal}. When $d = 2$, this is equal to
$n^{-1/2}$, implying that the minimax rates for estimation over
$\mathrm{Lip}(1)$ and \smash{$\BV_{\infty}(1,1)$} match (up to log
factors). This is despite the fact that $\mathrm{Lip}(1)$ is a strict subset of
\smash{$\BV_{\infty}(1,1)$}, with the latter containing far more diverse
functions, such as those with sharp discontinuities (indicator functions being a
prime example). When $d \geq 3$, we can see that the minimax rates drift apart,
with that for \smash{$\BV_{\infty}(1,1)$} being slower than $\mathrm{Lip}(1)$,
increasingly so for larger $d$.       
\end{remark}

\begin{remark}
A related point worthy of discussion is about what types of estimators can
attain optimal rates over $\mathrm{Lip}(1)$ and \smash{$\BV_{\infty}(1,1)$}.
For $\mathrm{Lip}(1)$, various \emph{linear smoothers} are known to be optimal,
which describes an estimator \smash{$\hf$} of the form \smash{$\hf(x) = w(x)^\T
y$} for a weight function $w : \XDom \to \R^n$ (the weight function can depend
on the design points but not on the response vector $y$). This includes kNN
regression and kernel smoothing, among many other traditional methods. 
For \smash{$\BV_{\infty}(1,1)$}, meanwhile, we have shown that the (modified) 
Voronoigram estimator is optimal (modulo log factors), which is highly
\emph{nonlinear} as a function of $y$. All other examples of minimax rate
optimal estimators that we provide in Section \ref{sub:other_estimators} are
nonlinear in $y$ as well. In fact, we conjecture that no linear smoother can
achieve the minimax rate over \smash{$\BV_{\infty}(1,1)$}. There is very strong 
precedence for this, both from the univariate case \citep{donoho1998minimax} 
and from the multivariate lattice case \citep{sadhanala2016total}. We leave a
minimax linear analysis over \smash{$\BV_{\infty}(1,1)$} to future work.  
\end{remark}

\begin{remark}
\label{rem:padilla}
Lastly, we comment on the relationship to the results obtained in
\citet{padilla2020adaptive}. These authors study TV denoising over the
$\varepsilon$-neighborhood and kNN graphs; our analysis also extends to
cover these estimators, as shown in Section \ref{sub:other_estimators}. They
obtain a comparable squared $L^2$ error rate of $n^{-1/d}$, under a related but
different set of assumptions. In one way, their assumptions are more restrictive
than ours, because they require conditions on $f_0$ that are stronger than TV 
and $L^\infty$ boundedness: they require it to satisfy an additional assumption      
that generalizes piecewise Lipschitz continuity, but is difficult to assess, in
terms of understanding precisely which functions have this property. (They also
directly consider functions that are piecewise Lipschitz, but this assumption is
so strong that they are able to remove the BV assumption entirely and attain the 
same error rates.)    

In another way, the results in \citet{padilla2020adaptive} go beyond ours, since
they accomodate the case when the design points lie on a manifold, in which case
their estimation rates are driven by the intrinsic (not ambient) dimension. Such 
manifold adaptivity is possible due to strong existing results on the properties 
of the $\varepsilon$-neighborhood and kNN graphs in the manifold setting. Is is
unclear to us whether the Voronoi graph has similar properties. This would be an
interesting topic for future work.        
\end{remark}

\subsection{Analysis of the Voronoigram: $L^2(P_n)$ risk}
\label{sub:analysis_empirical}

We outline the analysis of the Voronoigram. The analysis proceeds in three
parts. First, we bound the $L^2(P_n)$ risk of the Voronoigram in terms of the
discrete TV of the underlying signal over the Voronoi graph. Second, we bound  
this discrete TV in terms of the continuum TV of the underlying function. This
is presented in Lemmas \ref{lem:voronoi_empirical_dtv_ub} and
\ref{lem:voronoi_dtv_ub}, respectively. The third step is to bound the $L^2(P)$
risk after extrapolation (to a piecewise constant function on the Voronoi
diagram), which is presented in Lemma \ref{lem:voronoi_extrapolate_ub} in the
next subsection. All proofs are deferred until Appendix
\ref{app:estimation_theory}.    

For the first part, we effectively reduce the discrete analysis of the
Voronoigram---in which we seek to upper bound its $L^2(P_n)$ risk in terms of
its discrete TV---to the analysis of TV denoising on a grid. Analyzing this
estimator over a grid is desirable because a grid graph has nice spectral
properties (cf.\ the analyses in \citet{wang2016trend, hutter2016optimal,
sadhanala2016total, sadhanala2017higher, sadhanala2021multivariate} which all
leverage such properties). In the language of functional analysis, the core idea
here is an \emph{embedding} between the spaces defined by the discrete TV
operators with respect to one graph $G$ and another $G'$, of the form
\[
\| D(G') \, \theta\|_1 \leq C_n \| D(G) \, \theta\|_1, \quad \text{for all
  $\theta \in \R^n$}, 
\]
where $D(G), D(G')$ denote their respective edge incidence operators. This
approach was pioneered in \citet{padilla2018dfs}, who used it to study error
rates for TV denoising in quite a general context. It is also the key behind
the analysis of TV denoising on the $\varepsilon$-neighborhood and kNN graph in
\citet{padilla2020adaptive}, who also perform a reduction to a grid graph. The
next lemma, inspired by this work, shows that the analogous reduction is
available for the Voronoi graph.

\begin{lemma}
\label{lem:voronoi_empirical_dtv_ub}
Under the standard assumptions, consider either of the two modified Voronoi
weighting schemes defined in Theorem \ref{thm:voronoi_ub}:
\begin{itemize}
\item $\tilde{w}^\Vor_{ij} = \max\{c_0 n^{-(d-1)/d}, w^\Vor_{ij}\}$ for each
  $i,j$ such that $w^\Vor_{i,j}>0$; 
\item $\check{w}^\Vor_{ij} = 1$ for each $i,j$ such that $w^\Vor_{i,j}>0$.
\end{itemize}
Let $D$ denote the edge incidence operator corresponding to the modified graph, 
and \smash{$\htheta$} the solution in \eqref{eq:generalized_lasso}
(equivalently, it is the solution in \eqref{eq:tv_voronoi_graph} after
substituting in the modified weights). Then there exists a matrix $D'$, that can 
be viewed as a suitably~modified edge incidence operator corresponding to a
$d$-dimensional grid graph, such that
\begin{equation}
\label{eq:voronoi_embedding}
\| D' \theta\|_1 \leq C_n \tau_n \| D \theta\|_1, \quad \text{for all
  $\theta \in \R^n$}, 
\end{equation}
with probability at least $1-3/n^4$ (with respect to the distribution of design
points), where $C_n>0$ grows polylogarithmically in $n$ and $\tau_n$ is
the scaling factor defined in Theorem~\ref{thm:voronoi_ub}. Further, letting  
\smash{$\lambda = c\sigma\tau_n(\log n)^{1/2+\alpha}$} for any $\alpha>1$ and a
constant $c>0$, there exists another constant $C>0$ such that for all
sufficiently large $n$ and $f_0 \in \BV(\XDom)$, 
\begin{equation}
\label{eq:voronoi_empirical_dtv_ub}
\E\bigg[\frac{1}{n} \|\htheta - \theta_0\|_2^2 \bigg]
\leq C \bigg(  
\frac{\sigma\tau_n(\log n)^{1/2+\alpha}}{n}
\E\|D \theta_0\|_1 + \frac{(\log n)^\alpha}{n}
\bigg), 
\end{equation}
where we denote $\theta_0 = (f_0(x_1), \dots, f_0(x_n)) \in \R^n$. 
\end{lemma}

Notice that, in equivalent notation, we can write the left-hand side in
\eqref{eq:voronoi_empirical_dtv_ub} as \smash{$n^{-1}\|\htheta - \theta_0\|_2^2
  = \|\hf - f_0\|_{L^2(P_n)}^2$}, for the estimated function satisfying
\smash{$\hf(x_i) =  \htheta_i$}, $i=1,\dots,n$; and for the $\ell_1$ term on the
right-hand side in \eqref{eq:voronoi_empirical_dtv_ub} we can write \smash{$\|D
  \theta_0\|_1 = \DTV(f_0(x_{1:n}); \, w)$} for suitable edge weights
$w$---either of the two choices defined in bullet points at the start of the
theorem---over the Voronoi graph.    

As we can see, the $L^2(P_n)$ risk of the Voronoigram depends on the discrete TV
of the true signal over the Voronoi graph. A natural question to ask, then, is
whether a function bounded in continuum TV is also bounded in discrete TV, when
the latter is measured using the Voronoi graph. Our next result answers this in
the affirmative. It is inspired by analogous results developed in
\citet{green2021minimax1, green2021minimax2} for Sobolev functionals.   

\begin{lemma}
\label{lem:voronoi_dtv_ub}
Under Assumption \ref{assump:density_bounded}, there exists a constant $C>0$ 
such that for all sufficiently large $n$ and $f_0 \in \BV(\XDom)$, with $w$
denoting either of the two choices of edge weights given at the start of Lemma
\ref{lem:voronoi_empirical_dtv_ub}, 
\begin{equation}
\label{eq:voronoi_dtv_ub}
\E\Big[ \DTV \Big( f_0(x_{1:n}); \, w \Big) \Big] \leq C \bar\tau_n (\log
n)^{1+1/d} \TV(f_0),
\end{equation}
where \smash{$\bar\tau_n = n^{(d-1)/d} / \tau_n$} (which is 1 for the clipped
weights estimator and \smash{$n^{(d-1)/d}$} for the unit weights estimator).  
\end{lemma}

Lemmas \ref{lem:voronoi_empirical_dtv_ub} and \ref{lem:voronoi_dtv_ub} may be
combined to yield the following result, which is the $L^2(P_n)$ analog of
Theorem \ref{thm:voronoi_ub}. 

\begin{corollary}
\label{cor:voronoi_empirical_ub}
Under the standard assumptions, for either of the two modified Voronoigram
estimators from Theorem \ref{thm:voronoi_ub}, letting \smash{$\lambda =
  c\sigma\tau_n(\log n)^{1/2+\alpha}$} for any $\alpha>1$ and a constant $c>0$,
there exists another constant $C>0$ such that for all sufficiently large $n$ and 
$f_0\in\BV(L)$,  
\begin{equation}
\label{eq:voronoi_empirical_ub}   
\E\| \hf - f_0 \|_{L^2(P_n)}^2 \leq C \bigg(
\frac{\sigma L(\log n)^{3/2 + \alpha + 1/d}}{n^{1/d}} + 
\frac{(\log n)^\alpha}{n} \bigg).
\end{equation}
\end{corollary}

Note that for a constant $L$ (not growing with $n$), the $L^2(P_n)$ bound in 
\eqref{eq:voronoi_empirical_ub} converges at the rate $n^{-1/d}$, up to log
factors. Interestingly, this $L^2(P_n)$ guarantee does \emph{not} require $f_0$
to be bounded in $L^\infty$, which we saw was required for consistent estimation
in $L^2(P)$ error. Next, we will turn to an $L^2(P)$ upper bound, which does
require $L^\infty$ boundedness on $f_0$. That this is not needed for $L^2(P_n)$
consistency is intuitive (at least in hindsight): recall that we saw from the
proof of Proposition \ref{prop:impossibility} that inconsistency in $L^2(P)$
occurred due to tall spikes with vanishing width but non-vanishing $L^2$ norm,
which could not be witnessed by a finite number of samples. To the $L^2(P_n)$
norm, which only measures error at locations witnessed by the sample points, 
these pathologies are irrelevant. 

\subsection{Analysis of the Voronoigram: $L^2(P)$ risk}
\label{sub:analysis_population}

To close the loop, we derive bounds on the $L^2(P)$ risk of the Voronoigram via 
the $L^2(P_n)$ bounds just established. For this, we need to consider the
behavior of the Voronoigram estimator off of the design points. Recall that an   
equivalent interpretation of the Voronoigram fitted function, \smash{$\hf =
  \sum_{i=1}^n \hf(x_i) \cdot 1_{\Part_i}$}, is that it is given by
1-nearest-neighbor (1NN) extrapolation, applied to \smash{$(x_i, \hf(x_i))$},
$i=1,\dots,n$. Our approach here is to define an analogous 1NN extrapolant
\smash{$\bar{f}_0$} to \smash{$(x_i, f_0(x_i))$}, $i=1,\dots,n$, and then use
the triangle inequality, along with the fact that \smash{$\hf, \bar{f}_0$} are
piecewise constant on the Voronoi diagram, to argue that  
\begin{align}
\nonumber
\|\hf - f_0 \|_{L^2(P)}^2 
&\leq 2 \|\hf - \bar{f}_0 \|_{L^2(P)}^2
+ 2 \|\bar{f}_0 - f_0 \|_{L^2(P)}^2 \\  
\nonumber
&= 2\sum_{i=1}^n \big(\textstyle{\int}_{\Part_i} 1 dP\big)
\big( \hf(x_i) - f_0(x_i) \big)^2 +  
2 \|\bar{f}_0 - f_0 \|_{L^2(P)}^2 \\  
\label{eq:voronoi_triangle_ineq}
&\leq \underbrace{2p_\mx n\cdot\bigg(\max_{i=1,\dots,n} \Leb(\Part_i) \bigg)}_{K_n}
\|\hf - f_0\|_{L^2(P_n)}^2 + 2 \|\bar{f}_0 - f_0 \|_{L^2(P)}^2, 
\end{align} 
where $\Leb(\Part_i)$ denotes the Lebesgue volume of $\Part_i$. The first term 
in \eqref{eq:voronoi_triangle_ineq} is the $L^2(P_n)$ error multiplied by a factor
$K_n$ that is driven by the maximum volume of a Voronoi cell, which we can show
is well controlled (of order $\log n/n$) under Assumption
\ref{assump:density_bounded}. The second term is a kind of $L^2(P)$
approximation error from applying the 1NN extrapolation rule to evaluations of
$f_0$ itself. When $f_0 \in \BV_\infty(L,M)$, this is also well controlled, as
we show next.  

\begin{lemma}
\label{lem:voronoi_extrapolate_ub}
Assume that $x_1,\dots,x_n$ are i.i.d.\ from a distribution satisfying
Assusmption \ref{assump:density_bounded}. Then there is a constant $C > 0$ such
that for all sufficiently large $n$ and \smash{$f_0 \in \BV_{\infty}(L,M)$}, 
\begin{equation}
\label{eq:voronoi_extrapolate_ub}
\E \|\bar{f}_0 - f_0 \|_{L^2(P)}^2 
\leq C \bigg(\frac{L M (\log n)^{1 + 1/d}}{n^{1/d}}\bigg). 
\end{equation}
\end{lemma}

We make two remarks to conclude this subsection. 

\begin{remark}
For nonparametric regression with random design, a standard approach is to use
uniform concentration results that couple the $L^2(P)$ and $L^2(P_n)$ norms in
order to obtain an error guarantee in one norm from a guarantee in the
other; see, e.g., Chapter 14 of \citet{wainwright2019high}. In our setting, such
an approach is not applicable---the simplest explanation being that for any 
$x_1,\dots,x_n$, there will always exist a function $f \in \BV_{\infty}(1,1)$
for which \smash{$\|f\|_{L^2(P_n)} = 0$} but \smash{$\|f\|_{L^2(P)} =
  1/2$}. This is the same issue as that discussed in Remark
\ref{rem:random_design}.  
\end{remark}

\begin{remark}
The contribution of the extrapolation risk in \eqref{eq:voronoi_extrapolate_ub}
to the overall bound in \eqref{eq:voronoi_ub} is not negligible.  This raises
the possibility that, for this problem, extrapolation from random design points
with noiseless function values can be at least as hard as $L^2(P_n)$ estimation
from noisy responses. This is in contrast with conventional wisdom which says
that the noiseless problem is generally much easier. Of course, Lemma
\ref{lem:voronoi_extrapolate_ub} only provides an upper bound on the
extrapolation risk, without a matching lower bound. Resolving the minimax
$L^2(P)$ error in the noiseless setting, and more broadly, studying its precise
dependence on the noise level $\sigma$, is an interesting direction for future
work.
\end{remark}

\subsection{Other minimax optimal estimators}
\label{sub:other_estimators}

Finally, we present $L^2(P)$ guarantees that show that other estimators can also
obtain minimax optimal rates (up to log factors) for the class of functions
bounded in TV and $L^\infty$. First, we consider TV denoising on
$\varepsilon$-neighborhood and kNN graphs, using 1NN extrapolation to turn them
into functions on $\XDom$. The analysis is altogether very similar to that for
the Voronoigram outlined in the preceding subsections, and the details are
deferred to Appendix \ref{app:estimation_theory}. A notable difference, from the 
perspective of methodology, is that these estimators require proper tuning in 
the graph construction itself.  

\begin{theorem}
\label{thm:tvd_eps_knn_ub}
Under the standard assumptions, consider the graph TV denoising estimator
\smash{$\htheta^\Eps$} which solves problem \eqref{eq:generalized_lasso} with
\smash{$D = D(G^\Eps)$}, the edge incidence operator of the
$\varepsilon$-neighborhood graph \smash{$G^\Eps$}, with edge weights as in
\eqref{eq:eps_weights}. Letting \smash{$\varepsilon = c_1((\log
  n)^\alpha/n)^{1/d}$} and \smash{$\lambda = c_2\sigma(\log n)^{1/2-\alpha}$}
for any $\alpha>1$ and constants $c_1,c_2>0$, there is a constant $C>0$ such
that for all sufficiently large $n$ and $f_0 \in \BV_\infty(L,M)$, the 1NN 
extrapolant \smash{$\hf^\Eps = \sum_{i=1}^n \htheta_i^\Eps \cdot 1_{\Part_i}$}
satisfies
\begin{equation}
\label{eq:tvd_eps_ub}   
\E\| \hf^\Eps - f_0 \|_{L^2(P)}^2 \leq C \bigg(
\frac{\sigma L(\log n)^{3/2 + \alpha/d}}{n^{1/d}} + 
\frac{(\log n)^{1+\alpha}}{n} +
\frac{L M(\log n)^{1 + 1/d}}{n^{1/d}} \bigg).
\end{equation}
Consider instead the graph TV denoising estimator \smash{$\htheta^\kNN$} which
solves problem \eqref{eq:generalized_lasso} with \smash{$D = D(G^\kNN)$}, the
edge incidence operator of the kNN graph \smash{$G^\kNN$}, with edge weights 
as in \eqref{eq:knn_weights}. Letting \smash{$k = c'_1(\log n)^3$} and
\smash{$\lambda = c_2'\sigma(\log n)^{1/2-\alpha}$} for any $\alpha>1$ and 
constants $c_1',c_2'>0$, there is a constant $C'>0$ such that for all
sufficiently
large $n$ and $f_0 \in \BV_\infty(L,M)$, the 1NN extrapolant
\smash{$\hf^\kNN = \sum_{i=1}^n \htheta_i^\kNN \cdot 1_{\Part_i}$} satisfies      
\begin{equation}
\label{eq:tvd_knn_ub}   
\E\| \hf^\kNN - f_0 \|_{L^2(P)}^2 \leq C' \bigg(
\frac{\sigma L(\log n)^{9/2 - \alpha + 3/d}}{n^{1/d}} +
\frac{(\log n)^{1+\alpha}}{n} +
\frac{L M(\log n)^{1 + 1/d}}{n^{1/d}} \bigg).
\end{equation}
\end{theorem}

Next, and last, we consider wavelet denoising. For this we assume that the
design density is uniform on $\XDom = (0,1)^d$. The analysis is quite different 
from the preceding ones, but it relies on fairly standard techniques in wavelet 
theory, and we defer the details to Appendix \ref{app:estimation_theory}.  

\begin{theorem}
\label{thm:wavelet_bd}
Under the standard conditions, further assume that $P = \Leb$, the uniform
measure on $\XDom = (0,1)^d$. For an estimator \smash{$\hf^{\mathrm{wav}}$} 
based on hard-thresholding Haar wavelet coefficients, there exist constants 
$c,C>0$ such that for all sufficiently large $n$ and $f_0 \in \BV_\infty(L,M)$,
it holds that     
\begin{equation}
\label{eq:wavelet_bd}
\E\| \hf^{\mathrm{wav}} - f_0 \|_{L^2}^2 \leq \frac{C L M}{n^{1/d}} + C \cdot   
\begin{cases}
L \delta_n^\ast \max\{1, \, 1/M, \, \log_2(M \sqrt n) \} & d = 2 \\  
L^{2/d}(\delta_n^\ast)^{4/(2 + d)}+ LM(\delta_n^\ast/M)^{2/d} & 
d \geq 3,
\end{cases} 
\end{equation}
where \smash{$\delta_n^\ast = (c/\sqrt n)((\log n)^{3/2} + M (\log n)^{1/2})$}. 
\end{theorem}

%% file: discussion.tex
In this paper, we studied total variation as it touches on various aspects of
multivariate nonparametric regression, such as discrete notions of TV based on 
scattered data, the use of discrete TV as a regularizer in nonparametric
estimators, and estimation theory over function classes where regularity is  
given by (continuum) TV.

We argued that a particular formulation of discrete TV, based on the graph
formed by adjacencies with respect to the Voronoi diagram of the design points 
$x_1,\dots,x_n$, has several desirable properties when used as the regularizer
in a penalized least squares context---defining an estimator we call the
Voronoigram. Among these properties:
\begin{itemize}
\item it is user-friendly (requiring no auxiliary tuning parameter unlike other
  geometric graphs, such as $\varepsilon$-neighborhood or $k$-nearest-neighbor  
  graphs); 
\item it tracks ``pure TV'' in large samples, meaning that discrete TV on the
  Voronoi graph converges asymptotically to continuum TV, independent of the
  design density (as opposed to $\varepsilon$-neighborhood or kNN graphs, which
  give rise to certain types of density-weighted TV in the limit);  
\item it achieves the minimax optimal convergence rate in $L^2$ error over a
  class of functions bounded in TV and $L^\infty$;
\item it admits a natural duality between discrete and continuum formulations,
  so the fitted values \smash{$\hf(x_i)$}, $i=1,\dots,n$ have exactly the same
  variation (as measured by discrete TV) over the design points as the fitted
  function \smash{$\hf$} (as measured by continuum TV) over the entire domain.
\end{itemize}
The last property here is completely analogous to the discrete-continuum duality
inherent in trend filtering \citep{tibshirani2014adaptive,
tibshirani2022divided}, which makes the Voronoigram a worthy successor to trend 
filtering for multivariate scattered data, albeit restricted to the polynomial
order $k=0$ (piecewise constant estimation).

Several directions for future work have already been discussed throughout the
paper. We conclude by mentioning one more: extension to the polynomial order
$k=1$, i.e., adaptive piecewise linear estimation, in the multivariate scattered data 
setting. For this problem, we believe the estimator proposed by
\citet{koenker2004penalized}, defined in terms of the Delaunay tessellation
(which is dual to the Voronoi diagram) of the design points, will enjoy many
properties analogous to the Voronoigram, and is deserving of further study.

%% file: app.tex
\renewcommand\theequation{S.\arabic{equation}}
\renewcommand\thefigure{S.\arabic{figure}}
\renewcommand\thetable{S.\arabic{table}}
\renewcommand\thealgorithm{S.\arabic{algorithm}}
\renewcommand\thetheorem{S.\arabic{theorem}}
\renewcommand\thecorollary{S.\arabic{corollary}}
\renewcommand\thelemma{S.\arabic{lemma}}
\renewcommand\theproposition{S.\arabic{proposition}}

\section{Added details and proofs for Sections \ref{sec:introduction} and
  \ref{sec:methods_properties}}       
\label{app:introduction_methods}
\input{app_introduction_methods}

\section{Proofs for Section \ref{sec:asymptotic_limits}} 
\label{app:asymptotic_limits}
\input{app_asymptotic_limits}

\section{Sensitivity analysis for Section \ref{sec:experiments}}
\label{app:sensitivity_analysis}
\input{app_sensitivity_analysis}

\section{Proofs for Section \ref{sec:estimation_theory}}
\label{app:estimation_theory}
\input{app_estimation_theory}

\section{Analysis of graph TV denoising}
\label{sec:discrete_analysis}
\input{discrete-analysis}

\section{Embeddings for random graphs}
\label{sec:graph_embeddings}
\input{graph-embeddings}

\section{Auxiliary lemmas and proofs}
\label{app:auxiliary_lemmas_proofs}
\input{app_technical_lemmas}

%% file: app_introduction_methods.tex
\subsection{Discussion of sampling model for BV functions}
\label{app:precise_representative}

We clarify what is meant by the sampling model in \eqref{eq:model}, since,
strictly speaking, each element $f \in \BV(\XDom)$ is really an equivalence
class of functions, defined only up to sets of Lebesgue measure zero. This issue
is not simply a formality, and becomes a genuine problem for $d \geq 2$, as in
this case the space $\BV(\XDom)$ does not compactly embed into $C^0(\XDom)$,
the space of continuous functions on $\XDom$ (equipped with the $L^\infty$
norm). A key implication of this is that the point evaluation operator is not
continuous over $\BV(\XDom)$.       

In order to make sense of the evaluation map, $x \mapsto f(x)$, we will pick a
representative, denoted \smash{$f^\star \in f$}, and speak of evaluations of
this representative. Our approach here is the same as that taken in
\citet{green2021minimax1, green2021minimax2}, who study minimax estimation of  
Sobolev functions in the subcritical regime (and use an analogous random design 
model). We let \smash{$f^\star$} be the \emph{precise representative}, defined
\citep{evans2015measure} as:
\[
f^\star(x) = 
\begin{cases}
\displaystyle
\lim_{\epsilon \to 0} \frac{1}{\Leb(B(x,\epsilon))} \int_{B(x,\epsilon)} f(z) \,
dz & \text{if the limit exists} \\ 
0 & \text{otherwise}.
\end{cases}
\]
Here $\Leb$ denotes Lebesgue measure and $B(x,\epsilon)$ is the ball of radius
$\epsilon$ centered at $x$. 

Now we explain why the particular choice of representative is not crucial, and
any choice of representative would have resulted in the same interpretation of
function evaluations in \eqref{eq:model}, \emph{almost surely, assuming that
  each $x_i$ is drawn from a continuous distribution on $\XDom$}. Recall that
for a locally integrable function $f$ on $\XDom$, we say that a given point $x
\in \XDom$ is  a \emph{Lebesgue point} of $f$ provided that
\smash{$\lim_{\epsilon \to 0} (\int_{B(x,\epsilon)} f(z) \, dz)
  /\Leb(B(x,\epsilon)) $} exists and equals $f(x)$. By the Lebesgue
differentiation theorem (e.g., Theorem 1.32 of \citealp{evans2015measure}), for
any $f \in L^1(\XDom)$, almost every $x \in \XDom$ is a Lebesgue point of
$f$. This means that each evaluation \smash{$f^\star(x_i)$} of the precise
representative will equal the evaluation of any member of the equivalence class,
almost surely (with respect to draws of $x_i$). This justifies the notation
$f(x_i)$ used in the main text, for $f \in \BV(\XDom)$ and $x_i$ drawn from a
continuous probability distribution.   

\subsection{TV representation for piecewise constant functions} 
\label{app:tv_representation}

Here we will state and prove a more general result from which Proposition 
\ref{prop:tv_representation} will follow. First we give a more general
definition of measure theoretic total variation, wherein the norm used to
constrain the ``test function'' $\phi$ in the supremum is an arbitrary norm
$\|\cdot\|$ on $\R^d$,  
\begin{equation}
\label{eq:total_variation_general}
\TV(f; \XDom, \|\cdot\|) = \sup\left\{
    \int_\XDom f(x) \diver\phi(x) \, dx : \phi\in C_c^1(\XDom;\R^d), \,
    \|\phi(x)\| \leq 1 \; \text{for all $x\in\XDom$} \right\}.
\end{equation}
Note that our earlier definition in \eqref{eq:total_variation} corresponds to
the special case $\TV(f; \XDom, \|\cdot\|_2)$, that is, corresponds to choosing 
$\|\cdot\| = \|\cdot\|_2$ in \eqref{eq:total_variation_general}. In the more
general TV context, this special case is often called \emph{isotropic} TV.   

\begin{proposition}
\label{prop:tv_representation_general}
Let $\Part_1,\dots,\Part_n$ be an open partition of $\XDom$ such that each 
$\Part_i$ is semialgebraic. Let $f$ be of the form 
\[
  f = \sum_{i=1}^n \theta_i \cdot 1_{\Part_i},
\]
for arbitrary $\theta_1,\dots,\theta_n \in \R$. Then, for any norm $\|\cdot\|$
and its dual norm $\|\cdot\|_*$ (induced by the Euclidean inner product), we
have 
\[
\TV(f; \XDom, \|\cdot\|) =
\sum_{i,j=1}^n \bigg(  \int_{\partial\Part_i \cap \partial\Part_j} 
\| n_i(t) \|_*  \, d\cH^{d-1}(t) \bigg) \cdot |\theta_i-\theta_j|, 
\]
where $n_i(t)$ is the measure theoretic unit outer normal for $\Part_i$ at a
boundary point $t \in \partial \Part_i$. In particular, in the isotropic case
$\|\cdot\| = \|\cdot\|_2$,   
\[
\TV(f; \XDom, \| \cdot\|_2) =
\sum_{i,j=1}^n \cH^{d-1}(\partial\Part_i\cap\partial\Part_j) \cdot 
|\theta_i - \theta_j|.
\]
\end{proposition}

\begin{remark}
The condition that each $\Part_i$ is semialgebraic may to weakened to what is
called ``polynomially bounded boundary measure.'' Namely, the proposition still
holds if each map $r \mapsto \cH^{d-1}(\partial\Part_i\cap B(0, r))$ is
polynomially bounded (cf.\ Assumption 2.2 in
\citealp{mikkelsen2018degrees}). This is sufficient to guarantee a locally
Lipschitz boundary (a prerequisite for the application of Gauss-Green) and to 
characterize the outer normals associated with the partition
$\Part_1,\dots,\Part_n$.  
\end{remark}

\begin{proof}
  We begin by deriving an equivalent expression of total variation of piecewise
  constant functions.
  \begin{align}
    \nonumber
    \TV(f; \, &\XDom, \lVert\cdot\rVert)\\
    \nonumber
    &= \sup \left\{
      \int_\XDom
      f(x)\diver\phi(x) dx
      :\phi\in C_c^1(\XDom; \R^{d}),
      \lVert \phi\rVert_*\leq 1\;\forall x
    \right\}
    \\
    \nonumber
    &=  \sup \left\{
      \sum_{i=1}^n \int_{\Part_i}
      \theta_{i} \diver\phi(x) dx
      :\phi\in C_c^1(\XDom; \R^{d}),
      \lVert \phi\rVert_*\leq 1\;\forall x
    \right\}
    \\
    \label{eq:prop-cdt-apply-gauss-green}
    &=  \sup \left\{
      \sum_{i=1}^n
      \theta_{i}
      \int_{\partial \Part_i} \langle \phi(t), n_i(t) \rangle d\cH^{d-1}(t)
      :\phi\in C_c^1(\XDom; \R^{d}),
      \lVert \phi\rVert_*\leq 1\;\forall x
    \right\}
    \\
    \label{eq:prop-cdt-rearrange-to-pairwise-terms}
    \begin{split}
    &=  \sup \Bigg\{
      \sum_{i,j=1}^n \left(
        \theta_{i}\int_{\partial\Part_i\cap\partial\Part_j}
        \langle \phi(t), n_i(t) \rangle d\cH^{d-1}(t)
        +
        \theta_{j}\int_{\partial\Part_i\cap\partial\Part_j}
        \langle \phi(t), n_j(t) \rangle d\cH^{d-1}(t)
      \right)\\
      &\hspace{1.5cm}+
      \sum_{i:\bar\Part_i\cap\partial\XDom\neq\emptyset}
      \theta_{i}\underbrace{
        \int_{\partial\Part_i\cap\partial\XDom}
        \langle \phi(t), n_i(t) \rangle d\cH^{d-1}(t)
      }_{= 0 \text{; ($\phi$ compactly supported)}}
      :\phi\in C_c^1(\XDom; \R^{d}),
      \lVert \phi\rVert_*\leq 1\;\forall t
    \Bigg\}
    \end{split}
    \\
    \label{eq:prop-cdt-opposing-outer-normals}
    &= \sup \Bigg\{
      \sum_{i,j=1}^n
      \int_{\partial\Part_i\cap\partial\Part_j}
      (\theta_{i} - \theta_{j})
      \langle \phi(t), n_i(t) \rangle
      d\cH^{d-1}(t)
      :\phi\in C_c^1(\XDom; \R^{d}),
      \lVert \phi\rVert_*\leq 1\;\forall t
    \Bigg\}
  \end{align}
  we obtain \eqref{eq:prop-cdt-apply-gauss-green} by applying the
  Gauss-Green Theorem \citep[Theorem~5.16]{evans2015measure};
  \eqref{eq:prop-cdt-rearrange-to-pairwise-terms} by observing that
  when the boundaries of three or more
  $\Part_i\neq\Part_j\neq\Part_k\neq\cdots$ intersect, the outer normal vector
  is zero \citep[Lemma~A.2(c)]{mikkelsen2018degrees}; and
  \eqref{eq:prop-cdt-opposing-outer-normals} because when
  the boundaries of exactly two $\Part_i\neq\Part_j$ intersect, they have
  opposing outer normals \citep[Lemma~A.2(b)]{mikkelsen2018degrees}.
  Apply H\"older's inequality to obtain an
  upper bound,
  \begin{align*}
    \label{eq:prop-cdt-upper-bound-holder}
    \TV(f; \, &\XDom, \lVert\cdot\rVert)
    \\
    &\leq \sup \Bigg\{
      \sum_{i,j=1}^n
      |\theta_i-\theta_j|
      \int_{\partial\Part_i\cap\partial\Part_j}
      \lVert \phi(t)\rVert_*
      \lVert n_i(t)\rVert
      d\cH^{d-1}(t)
      :\phi\in C_c^1(\XDom; \R^{d}),
      \lVert \phi\rVert_*\leq 1\;\forall t
    \Bigg\}
    \\
    &= \sum_{i,j=1}^n
      |\theta_i-\theta_j|
      \int_{\partial\Part_i\cap\partial\Part_j}
      \lVert n_i(t)\rVert
      d\cH^{d-1}(t),
  \end{align*}
  where recall $\lVert \cdot\rVert, \lVert \cdot\rVert_*$ are dual
  norms.  Finally, we obtain a matching lower bound via a mollification
  argument.  The target of our approximating sequence will be a pointwise
  duality map with respect to $\lVert\cdot\rVert$, but first we need to do 
  a little bit of work.  Define the function
 \smash{$\phi_0:\cup_{i,j=1}^n \partial\Part_i\cap\partial\Part_j
  \rightarrow\R^{d}$} by
  \begin{equation*}
    \phi_0(t)
    \in\{g/\lVert g\rVert_*:g\in F(n_i(t)),
      t\in\partial\Part_i\cap\partial\Part_j\},
  \end{equation*}
  and its piecewise constant extension to $\XDom$,
  $\tilde\phi:\XDom\rightarrow\R^{d}$ by
  \begin{equation*}
    \tilde\phi(x)
    = \phi_0\left(
      t\in\argmin_t\{\lVert x-t\rVert_2:
        t\in\cup_{i,j=1}^n\partial\Part_i\cap\partial\Part_j\}
    \right),
  \end{equation*}
  where for a Banach space $E$ and its continuous dual $E^*$, we write 
  $F:E\rightarrow P(E^*)$ for the dual map defined by
  \begin{equation*}
    F(x_0) = \left\{
      f_0\in E^*: \lVert f_0\rVert_{E^*}=\lVert x_0\rVert_E
      \text{ and } \langle f_0,x_0\rangle_{(E, E^*)} = \lVert x_0\rVert_E^2
    \right\},
  \end{equation*}
  and moreover when $E^*$ is strictly convex, the duality map is
  singleton-valued \citep{brezis2011functional}.  Observe that
  \smash{$\tilde\phi\in L^p_\text{loc}(\XDom)$}, $1\leq p<\infty$, so 
  there exists an approximating sequence \smash{$\tilde\phi_k\in
    C_c^\infty(\XDom,\R^{d})$}, $k=1,2,3,\ldots$, such that
  \smash{$\lim_{k \to \infty} \tilde\phi_k\rightarrow\tilde\phi$} $\Leb$-a.e.  We
  invoke Fatou's Lemma and properties of the duality map to obtain a matching
  lower bound, 
  \begin{align*}
    \TV(f; \, &\XDom, \lVert\cdot\rVert)
    \\
    &= \sup \Bigg\{
      \sum_{i,j=1}^n
      |\theta_i-\theta_j|
      \int_{\partial\Part_i\cap\partial\Part_j}
      \langle \phi(t), n_i(t)\rangle 
      d\cH^{d-1}(t)
      :\phi\in C_c^1(\XDom; \R^{d}),
      \lVert \phi\rVert_*\leq 1\;\forall t
    \Bigg\}
    \\
    &\geq
      \sum_{i,j=1}^n
      |\theta_i-\theta_j|
      \liminf_{k\rightarrow\infty}
      \int_{\partial\Part_i\cap\partial\Part_j}
      \langle \tilde\phi_k(t), n_i(t)\rangle 
      d\cH^{d-1}(t)
    \\
    &\geq
      \sum_{i,j=1}^n
      |\theta_i-\theta_j|
      \int_{\partial\Part_i\cap\partial\Part_j}
      \left\langle
        \liminf_{k\rightarrow\infty} \tilde\phi_k(t),
        n_i(t)
      \right\rangle
      d\cH^{d-1}(t)
    \\
    &=
      \sum_{i,j=1}^n
      |\theta_i-\theta_j|
      \int_{\partial\Part_i\cap\partial\Part_j}
      \left\langle
        \lim_{k\rightarrow\infty} \tilde\phi_k(t),
        n_i(t)
      \right\rangle
      d\cH^{d-1}(t)
    \\
    &=
      \sum_{i,j=1}^n
      |\theta_i-\theta_j|
      \int_{\partial\Part_i\cap\partial\Part_j}
      \langle
        \tilde\phi(t),
        n_i(t)
      \rangle
      d\cH^{d-1}(t)
    \\
    &= \sum_{i,j=1}^n
      |\theta_i-\theta_j|
      \int_{\partial\Part_i\cap\partial\Part_j}
      \lVert n_i(t)\rVert
      d\cH^{d-1}(t),
  \end{align*}
  establishing equality.
\end{proof}

%% file: app_asymptotic_limits.tex
\subsection{Roadmap for the proof of Theorem \ref{thm:asymptotic_limit_voronoi}}  

The proof of Theorem \ref{thm:asymptotic_limit_voronoi} consists of several
parts, and we summarize them below. Some remarks on notation: throughout this 
section, we use \smash{$\sigma_{\mathrm{Vor}}$} for the constant $c_d$ appearing
in \eqref{eq:asymptotic_limit_voronoi}, and we abbreviate
$\|\cdot\|=\|\cdot\|_2$. Also, we use $C^1(\XDom)$ and $C^2(\XDom)$ to denote
the spaces of continuously differentiable and twice continuously differentiable
functions, respectively, equipped with the $L^\infty$ norm. 

\begin{enumerate}
	\item An edge $\{i,j\}$ in the Voronoi graph depends not only on $x_i$ and
    $x_j$ but also on all other design points $x_k, k \neq i,j$. In Lemma
    \ref{lem:voronoi-tv-u-statistic}, we start by showing that the randomness
    due this dependence on $x_k, k \neq i,j$ is negligible, 
	\begin{equation}
	\label{pf:voronoi-tv-continuum-1}
	\mathbb{E}\Bigl[\bigl(\mathrm{DTV}(f; w^\Vor) -
  U_{n,\mathrm{Vor}}(f)\bigr)^2\Bigr] \leq C \frac{\|f\|_{C^1(\Omega)}^2 (\log
    n)^{(d + 2)/d}}{n^{1/d}},
	\end{equation}
  for a constant $C>0$. The functional $U_{n,\mathrm{Vor}}(f)$ is an order-$2$ 
  U-statistic,   \[
    U_{n,\mathrm{Vor}}(f) = \frac{1}{2}\sum_{i = 1}^{n} \sum_{j = 1}^{n}
    \bigl|f(x_i) - f(x_j)\bigr| H_{\mathrm{Vor}}(x_i,x_j), 
  \]
	with kernel $H_{\mathrm{Vor}}(x,y)$ defined by
  \[
  H_{\mathrm{Vor}}(x,y) = \mathbb{E}\bigl[\mc{H}^{d - 1}(\partial V_i \cap
  \partial V_j)|x_i,x_j\bigr] = \int_{L \cap \Omega} \bigl(1 -
  p_x(z)\bigr)^{(n - 2)} \,dz. 
  \]
	Here $L=L_{xy}$ is the $(d - 1)$-dimensional hyperplane $L = \{z: \|x - z\| 
  = \|y - z\|\}$, and $p_x(z) = P(B(z,\|x - z\|))$. (Note that $p_x(z) = p_y(z)$
  for all $z \in L$).   
	\item We proceed to separately analyze the variance and bias of
    $U_{n,\mathrm{Vor}}(f)$. In Lemma~\ref{lem:u-statistic-variance}, we
    establish that $U_{n,\mathrm{Vor}}(f)$ concentrates around its mean, giving
    the estimate, for a constant $C>0$,
	\begin{equation}
	\label{pf:voronoi-tv-continuum-2}
	\mathrm{Var}\bigl[U_{n,\mathrm{Vor}}(f)\bigr] \leq C\frac{(\log
    n)^3}{n}\|f\|_{C^1(\Omega)}^2. 
	\end{equation}
	\item It remains to analyze the bias, the difference between the
    expectation of $U_{n,\mathrm{Vor}}(f)$ and continuum
    TV. Lemma~\ref{lem:voronoi-tv-bias-holder} leverages the fact that the
    kernel $H_{\mathrm{Vor}}(x,y)$ is close to a spherically symmetric
    kernel---at least at points $x,y$ sufficiently far from the boundary of
    $\Omega$---to show that the expectation of the U-statistic
    $U_{n,\mathrm{Vor}}(f)$ is close to (an appropriately rescaled version of)
    the nonlocal functional  
	\begin{equation}
	\label{eqn:nonlocal-tv}
	\mathrm{TV}_{\varepsilon,K}\bigl(f;\Omega,h\bigr) := \int_{\Omega}
  \int_{\Omega} |f(x) - f(y)| K_{\mathrm{Vor}} \biggl(\frac{\|y -
    x\|}{\varepsilon(x)}\biggr) h(x) \,dy \,dx, 
	\end{equation}
	for bandwidth $\varepsilon(x) = (np(x))^{-1/d}$, weight $h(x) = (p(x))^{(d +
    1)/d}$, and kernel $K_{\mathrm{Vor}}(t)$ defined
  in~\eqref{eqn:voronoi-kernel}. Lemma~\ref{lem:nonlocal-tv-bias-holder} in turn
  shows that this nonlocal functional is close to (a scaling factor) times
  $\int_{\Omega} \|\nabla f\|$. Together, these lemmas imply that   
	\begin{equation}
	\label{pf:voronoi-tv-continuum-3}
	\lim_{n \to \infty} \mathbb{E}[U_{n,\mathrm{Vor}}(f)] = \sigma_{\mathrm{Vor}}
  \int_{\Omega} \|\nabla f(x)\| \,dx. 
	\end{equation}
\end{enumerate}
Combining~\eqref{pf:voronoi-tv-continuum-1},~\eqref{pf:voronoi-tv-continuum-2}, 
and \eqref{pf:voronoi-tv-continuum-3} with Chebyshev's inequality implies the
consistency result stated in~\eqref{eq:asymptotic_limit_voronoi}. In the rest of
this section, across
Sections~\ref{subsec:voronoi-tv-u-statistic}--\ref{subsec:voronoi-ustat-bias}, 
we state and prove the various lemmas referenced above.

\subsection{Step 1: Voronoi TV approximates Voronoi U-statistic}
\label{subsec:voronoi-tv-u-statistic}

Lemma~\ref{lem:voronoi-tv-u-statistic} upper bounds the expected squared
difference between Voronoi TV and the U-statistic $U_{n,\mathrm{Vor}}(f)$. 
\begin{lemma}
	\label{lem:voronoi-tv-u-statistic}
  Suppose $x_{1:n}$ are sampled independently from a distribution $P$
  satisfiying~\ref{assump:density_bounded}. There exists a constant $C > 0$
  such that for all $n \in \mathbb{N}$ sufficiently large, and any $f \in
  C^1(\Omega)$,
	\begin{equation*}
  \mathbb{E}\Bigl[\bigl(\mathrm{DTV}(f;w^\Vor) -
  U_{n,\mathrm{Vor}}(f)\bigr)^2\Bigr] \leq C \frac{\|f\|_{C^1(\Omega)}^2 (\log
n)^{(d + 2)/d}}{n^{1/d}}.
	\end{equation*}
\end{lemma}
\begin{proof}[Proof of Lemma~\ref{lem:voronoi-tv-u-statistic}]
  We begin by introducing some notation and basic inequalities used throughout
  this proof. Take $\varepsilon_{0} = (\log n/n)^{1/d}$. Let $B_{x}(z) :=
  B^{o}(z,\|x - z\|)$ denote the open ball centered at $z$ of radius $\|x -
  z\|$, and note that by our assumptions on $p$, we have $p_x(z) :=
  P(B_{x}(z))$.  We will repeatedly use the estimates
	\begin{equation*}
	p_x(z) \geq \frac{p_{\min}}{2d} \Leb_d \|x - z\|^d,
	\end{equation*}
	and therefore for $c_1 = \frac{p_{\min}}{2d} \Leb_d$,
	\begin{equation*}
	(1 - p_x(z))^{n} \leq \exp(-c_1 n \|x - z\|^d).
	\end{equation*}
  It follows by Lemma~\ref{lem:voronoi-kernel-ub} that for any constants $a,c >
  0$, there exists a constant $C > 0$ depending only on $a,c$ and $d$ such that
	\begin{equation*}
  \int_{L \cap \Omega} (1 - c p_x(z))^{n} \leq C\biggl(\frac{\1\{\|x - y\|\leq
  C\varepsilon_{0}\}}{n^{(d - 1)/d}} + \frac{1}{n^5}\biggr).
	\end{equation*}
  We will assume $n \geq 8$, so that the same estimate holds with respect to $n
  - 4 \geq n/2$. Finally for simplicity write $\Delta(x_i,x_j) := |f(x_i) -
  f(x_j)|\bigl(\mc{H}^{d - 1}(\partial V_i \cap \partial V_j) -
  H_{\mathrm{Vor}}(x_i,x_j)\bigr)$.
	
  We note immediately that, because $x_{1:n}$ are identically distributed, it
  follows from linearity of expectation that
	\begin{align*}
  \mathbb{E}\Bigl[\bigl(\mathrm{DTV}_{n,\mathrm{Vor}}(f;w^\Vor)
  - U_{n,\mathrm{Vor}}(f)\bigr)^2\Bigr] & = {n \choose 2}
  \mathbb{E}[\bigl(\Delta(x_1,x_2)\bigr)^2]\\
	& + {n \choose 3} \mathbb{E}\bigl[\Delta(x_1,x_2)\Delta(x_1,x_3)\bigr] \\
	& + {n \choose 4} \mathbb{E}\bigl[\Delta(x_1,x_2)\Delta(x_3,x_4)\bigr] \\
	& =: {n \choose 2} T_1 + {n \choose 3} T_2 + {n \choose 4} T_3.
	\end{align*}
  We separately upper bound $|T_1|$ (which will make the main contribution to
  the overall upper bound) and $|T_2|$ and $|T_3|$ (which will be comparably
  negligible). In each case, the general idea is to use the fact that the
  fluctuations of the Voronoi edge weights $\mc{H}^{d - 1}(\partial V_1 \cap
  \partial V_2)$ around the conditional expectation $H_{\mathrm{Vor}}(x_1,x_2)$
  are small unless $x_1$ and $x_2$ are close together.
	
\paragraph{Upper bound on $T_1$.}
  We begin by conditioning on $x_1,x_2$, and considering the conditional
  expectation
	\begin{align*}
  \mathbb{E}\bigl[(\Delta(x_1,x_2))^2|x_1,x_2\bigr] = |f(x_1) - f(x_2)|^2
  \mathrm{Var}(\mc{H}^{d - 1}(\partial V_1 \cap \partial V_2)|x_1,x_2).
	\end{align*}
	By Jensen's inequality,
	\begin{align*}
  \mathrm{Var}(\mc{H}^{d - 1}(\partial V_1 \cap \partial V_2)|x_1,x_2) & \leq
  \mc{H}^{d - 1}(L \cap \Omega) \int_{L \cap \Omega}
  \mathrm{Var}\bigl(\1\{P_n(B_{x_1}(z)) = 0\}|x_1\bigr) \,dz \\
  & = \mc{H}^{d - 1}(L \cap \Omega) \int_{L \cap \Omega} \bigl(1 -
  p_{x_1}(z)\bigr)^{(n - 2)}  \,dz \\
  & \leq C\Bigl(\frac{1}{n^{(d - 1)/d}}\1\{\|x_1 - x_2\| \leq
  C\varepsilon_{0}\} + \frac{1}{n^5}\Bigr).
	\end{align*}
	Taking expectation over $x_1$ and $x_2$ gives
	\begin{align*}
  T_1 & \leq C \biggl(\frac{\|f\|_{C^1(\Omega)}^2}{n^{(d - 1)/d}}\int_{\Omega}
  \int_{\Omega} \|x - y\|^2 \1\{\|x - y\|  \leq C\varepsilon_{0} \} \,dy \,dx +
\frac{\|f\|_{L^{\infty}(\Omega)}^2}{n^5}\biggr) \\
  & \leq C \Bigl(\frac{\|f\|_{C^1(\Omega)}^2 \varepsilon_{0}^{(d + 2)}}{n^{(d -
  1)/d}} + \frac{\|f\|_{L^{\infty}(\Omega)}^2}{n^5}\Bigr) \\
  & = C \Bigl(\frac{\|f\|_{C^1(\Omega)}^2 (\log n)^{(d + 2)/d}}{n^{(2 + 1/d)}}
  + \frac{\|f\|_{L^{\infty}(\Omega)}^2}{n^5}\Bigr).
	\end{align*}

\paragraph{Upper bound on $T_2$.}
  Again we begin by conditioning, this time on $x_{1:3}$, meaning we consider
	\begin{equation*}
  \mathbb{E}\bigl[\Delta(x_1,x_2)\Delta(x_1,x_3)|x_{1:3}\bigr] = |f(x_1) -
  f(x_2)| |f(x_1) - f(x_3)| \mathrm{Cov}\bigl[\mc{H}^{d - 1}(\partial V_1 \cap
  \partial V_2), \mc{H}^{d - 1}(\partial V_1 \cap \partial V_3)|x_{1:3}\bigr].
	\end{equation*}
  We begin by focusing on this conditional covariance. Write $L = \{z \in
  \Omega: \|z - x_1\| = \|z - x_2\|\}$ and likewise $L' = \{z \in \Omega: \|z -
x_1\| = \|z - x_3\|\}$. Exchanging covariance with integration gives
	\begin{equation}
	\label{pf:voronoi-tv-u-statistic-1}
	\begin{aligned}
  & \biggl|\mathrm{Cov}\bigl[\mc{H}^{d - 1}(\partial V_1 \cap \partial V_2),
  \mc{H}^{d - 1}(\partial V_1 \cap \partial V_3)|x_{1:3}\bigr]\biggr| \\
  & \leq \int_{L} \int_{L'} \bigl|\mathrm{Cov}\bigl[\1\{P_n(B_{x_1}(z)) = 0\},
  \1\{P_n(B_{x_1}(z')) = 0\}|x_{1:3}\bigr]\bigr| \,dz' \,dz \\
  & \overset{(i)}{\leq} \int_{L} \int_{L'} \Bigl(1 - \frac{p_{x_1}(z) +
  p_{x_1}(z')}{2}\Bigr)^{(n - 3)} \,dz \,dz' \\
  & + \int_{L} \int_{L'}(1 - p_{x_1}(z))^{(n - 3)}(1 - p_{x_1}(z'))^{(n - 3)}
  \,dz' \,dz \\
  & \leq C\Bigl(\frac{1}{n^{(d - 1)/d}}\1\{\|x_1 - x_2\| \leq
  C\varepsilon_{0}\} + \frac{1}{n^5}\Bigr) \Bigl(\frac{1}{n^{(d -
  1)/d}}\1\{\|x_1 - x_3\| \leq C\varepsilon_{0}\} + \frac{1}{n^5}\Bigr) \\
  & \leq C\Bigl(\frac{1}{n^{2(d - 1)/d}}\1\{\|x_1 - x_2\| \leq
  C\varepsilon_{0}\}\1\{\|x_1 - x_3\| \leq C\varepsilon_{0}\} +
\frac{1}{n^5}\Bigr).
	\end{aligned}
	\end{equation}
  The inequality $(i)$ follows first from the standard fact that for positive
  random variables $X$ and $Y$, $\bigl|\mathrm{Cov}[X,Y]\bigr| \leq
  \mathbb{E}[XY] +\mathbb{ E}[Y]\mathbb{E}[X]$, and second from the upper bound
	\begin{align*}
	\mathbb{E}\Bigl[\1\{P_n(B_{x_1}(z)) = 0\}, \1\{P_n(B_{x_1}(z')) = 0\}\Bigr]
    &\leq \Bigl(1 - P\bigl(B_{x_1}(z) \cup B_{x_1}(z')\bigr)\Bigr)^{(n - 3)} \\
    &\leq \biggl(1 - \frac{P\bigl(B_{x_1}(z)\bigr) +
    P\bigl(B_{x_1}(z')\bigr)}{2}\biggr)^{(n - 3)}.
	\end{align*}
	Taking expectation over $x_{1:3}$, we have
	\begin{align*}
	T_2 & \leq C \biggl(\frac{\|f\|_{C^1(\Omega)}^2}{n^{2(d - 1)/d}}\int_{\Omega} \int_{\Omega} \int_{\Omega} \|x - y\| \|x - z\| \1\{\|x - y\|  \leq C\varepsilon_{0} \} \1\{\|x - z\|  \leq C\varepsilon_{0} \} \,dz \,dy \,dx + \frac{\|f\|_{L^{\infty}(\Omega)}^2}{n^5}\biggr) \\
	& \leq C \Bigl(\frac{\|f\|_{C^1(\Omega)}^2 \varepsilon_{0}^{2(d + 1)}}{n^{2(d - 1)/d}} + \frac{\|f\|_{L^{\infty}(\Omega)}^2}{n^5}\Bigr) \\
	& = C \Bigl(\frac{\|f\|_{C^1(\Omega)}^2 (\log n)^{2(d + 1)/d}}{n^4}  + \frac{\|f\|_{L^{\infty}(\Omega)}^2}{n^5}\Bigr).
	\end{align*}
	
  \paragraph{Upper bound on $T_3$.}
	Again we begin by conditioning, this time on $x_{1:4}$, so that
	\begin{equation*}
	\mathbb{E}\bigl[\Delta(x_1,x_2)\Delta(x_3,x_4)|x_{1:4}\bigr] = |f(x_1) - f(x_2)| |f(x_3) - f(x_4)| \mathrm{Cov}\bigl[\mc{H}^{d - 1}(\partial V_1 \cap \partial V_2), \mc{H}^{d - 1}(\partial V_3 \cap \partial V_4)|x_{1:4}\bigr],
	\end{equation*}
	Write $L = \{z \in \Omega: \|z - x_1\| = \|z - x_2\|\}$ and likewise $L' = \{z \in \Omega: \|z - x_3\| = \|z - x_4\|\}$, we focus on the conditional covariance
	\begin{equation*}
	\mathrm{Cov}\bigl[\mc{H}^{d - 1}(\partial V_1 \cap \partial V_2), \mc{H}^{d - 1}(\partial V_3 \cap \partial V_4)|x_{1:4}\bigr] = \int_{L} \int_{L'} \mathrm{Cov}\bigl[\1\{P_n(B_{x_1}(z)) = 0\}, \1\{P_n(B_{x_3}(z')) = 0\}|x_{1:4}\bigr] \,dz' \,dz
	\end{equation*}
	We now show that this covariance is very small unless $x_1$ and $x_3$ are close. Specifically, suppose $\|x_1 - x_3\| > \varepsilon_{0}$. Then either $\|z - x_1\| \geq \varepsilon_{0}/3$, or $\|z' - x_3\| \geq \varepsilon_{0}/3$, or $B_{x_1}(z) \cap B_{x_3}(z') = \emptyset$. In either of the first two cases, we have that
	\begin{align*}
	\Bigl|
    \mathrm{Cov}\bigl[\1\{P_n(B_{x_1}(z)) = 0\},
    \1\{P_n(B_{x_3}(z')) &= 0\}|x_{1:4}\bigr]
  \Bigr|  \\
  & \leq 2\exp(-\frac{p_{\min}}{4d} (n - 4) \|x_1 - z\|^d) \exp(-\frac{p_{\min}}{4d} (n - 4) \|x_3 - z'\|^d)\} \\
	& \leq 2\exp(-\frac{p_{\min}}{4d} (n - 4) \varepsilon_{0}^d) \leq \frac{C}{n^5}.
	\end{align*}
	In the third case, it follows that $P(B_{x_1}(z) \cup B_{x_3}(z')) = p_{x_1}(z) + p_{x_3}(z)$. Assume $x_{3},x_{4} \not\in B_{x_1}(z)$, and likewise $x_1,x_2 \not\in B_{x_3}(z')$, otherwise there is nothing to prove.  We use the definition of covariance $\mathrm{Cov}[X,Y] = \mathbb{E}[XY] - \mathbb{E}[X]\mathbb{E}[Y]$ to obtain the upper bound,
	\begin{align*}
	& \Bigl|\mathrm{Cov}\bigl[\1\{P_n(B_{x_1}(z)) = 0\}, \1\{P_n(B_{x_3}(z')) = 0\}|x_{1:4}\bigr]\Bigr| \\
	& \quad = \bigl|(1 - (p_{x_1}(z) + p_{x_3}(z)))^{(n - 4)} - (1 - p_{x_1}(z))^{(n - 4)}(1 - p_{x_3}(z))^{(n - 4)}\bigr| \\
	& \quad = (1 - p_{x_1}(z))^{(n - 4)}(1 - p_{x_3}(z))^{(n - 4)}\biggl|\Bigl(1 - \frac{p_{x_1}(z)p_{x_3}(z)}{(1 - p_{x_1}(z))(1 - p_{x_3}(z))}\Bigr)^{(n - 4)} - 1\Bigr| \\
	& \quad \leq (1 - p_{x_1}(z))^{(n - 4)}(1 - p_{x_3}(z))^{(n - 4)}p_{x_1}(z)p_{x_3}(z)n \\
	& \quad \leq p_{\max}^2 \Leb_d^2 \exp(-\frac{p_{\min}}{4d} (n - 4) \|x_1 - z\|^d) \exp(-\frac{p_{\min}}{4d} (n - 4) \|x_2 - z\|^d) \|x_1 - z\|^d \|x_3 - z'\|^d n \\
	& \quad \leq C \exp(-\frac{p_{\min}}{4d} (n - 4) \|x_1 - z\|^d) \exp(-\frac{p_{\min}}{4d} (n - 4) \|x_2 - z\|^d) \varepsilon_{0}^{2d}n 
	\end{align*}
	Integrating over $z,z'$, it follows that if $\|x_1 - x_3\| > \varepsilon_{0}$, then
	\begin{equation*}
	\Bigl|\mathrm{Cov}\bigl[\mc{H}^{d - 1}(\partial V_1 \cap \partial V_2), \mc{H}^{d - 1}(\partial V_3 \cap \partial V_4)|x_{1:4}\bigr]\Bigr| \leq C\Bigl(\frac{\varepsilon_{0}^{2d}}{n^{(d - 2)/d}}\1\{\|x_1 - x_2\| \leq C\varepsilon_{0}\}\1\{\|x_3 - x_4\| \leq C\varepsilon_{0}\} + \frac{1}{n^5}\Bigr).
	\end{equation*}
	Otherwise $\|x_1 - x_3\| \leq \varepsilon_{0}$, and using the same inequalities as in~\eqref{pf:voronoi-tv-u-statistic-1}, we find that
	\begin{align*}
	& \Bigl|\mathrm{Cov}\bigl[\mc{H}^{d - 1}(\partial V_1 \cap \partial V_2), \mc{H}^{d - 1}(\partial V_3 \cap \partial V_4)|x_{1:4}\bigr]\Bigr| \\
	& \leq C\Bigl(\frac{1}{n^{2(d - 1)/d}}\1\{\|x_1 - x_2\| \leq C\varepsilon_{0}\}\1\{\|x_3 - x_4\| \leq C\varepsilon_{0}\}\{\|x_1 - x_3\| \leq \varepsilon_{0}\} + \frac{1}{n^5}\Bigr).
	\end{align*}
	Taking expectation over $x_{1:4}$, we conclude that
	\begin{align*}
	T_3 & \leq C \biggl(\frac{\varepsilon_{0}^{2d}\|f\|_{C^1(\Omega)}^2}{n^{(d - 2)/d}}\int_{\Omega} \int_{\Omega} \int_{\Omega} \int_{\Omega} \|x - y\| \|h - z\| \1\{\|x - y\|  \leq C\varepsilon_{0} \} \1\{\|h - z\|  \leq C\varepsilon_{0} \} \,dh \,dz \,dy \,dx \\
	& + \frac{\|f\|_{C^1(\Omega)}^2}{n^{2(d - 1)/d}} \int_{\Omega} \int_{\Omega} \int_{\Omega} \int_{\Omega} \|x - y\| \|h - z\| \1\{\|x - y\|  C\leq \varepsilon_{0} \} \1\{\|h - z\|  \leq C\varepsilon_{0}, \|x - h\|  \leq \varepsilon_{0} \} \,dh \,dz \,dy \,dx \\
	& + \frac{\|f\|_{L^{\infty}(\Omega)}^2}{n^5}\biggr) \\
	& \leq C \Bigl(\frac{\|f\|_{C^1(\Omega)}^2 \varepsilon_{0}^{4d + 2}}{n^{(d - 2)/d}} + \frac{\|f\|_{C^1(\Omega)}^2 \varepsilon_{0}^{3d + 2}}{n^{2(d - 1)/d}} + \frac{\|f\|_{L^{\infty}(\Omega)}^2}{n^5}\Bigr) \\
	& = C \Bigl(\frac{\|f\|_{C^1(\Omega)}^2 (\log n)^{(4d + 2)/d}}{n^5}  + \frac{\|f\|_{L^{\infty}(\Omega)}^2}{n^5}\Bigr).
	\end{align*}
	Combining our upper bounds on $T_1$-$T_3$ gives the claim of the lemma.
\end{proof}

\subsection{Step 2: Variance of Voronoi U-statistic}
\label{sub:voronoi-ustat-variance}

Lemma~\ref{lem:u-statistic-variance} leverages classical theory regarding order-2 U-statistics to show that the Voronoi U-statistic $U_{n,\mathrm{Vor}}(f)$ concentrates around its expectation. This is closely related to an estimate provided in~\citet{garciatrillos2017estimating}, but not strictly implied by that result: it handles a specific kernel $H_{\mathrm{Vor}}$ that is not compactly supported, and functions $f$ besides $f(x) = \1\{x \in A\}$ for some $A \subseteq \Omega$.
\begin{lemma}
	\label{lem:u-statistic-variance}
	Suppose $x_{1:n}$ are sampled independently from a distribution $P$ satisfiying~\ref{assump:density_bounded}. There exists a constant $C > 0$ such that for any $f \in C^1(\Omega)$,
	\begin{equation}
	\label{eqn:u-statistic-variance}
	\mathrm{Var}\bigl[U_{n,\mathrm{Vor}}(f)\bigr] \leq C\frac{(\log n)^3}{n}\|f\|_{C^1(\Omega)}^2.
	\end{equation}
\end{lemma}
Lemma~\ref{lem:u-statistic-variance} can be strengthened in several respects. Under the assumptions of the lemma, better bounds are available than~\eqref{eqn:u-statistic-variance} which do not depend on factors of $\log n$. Additionally, under weaker assumptions than $f \in C^1(\Omega)$, it is possible to obtain bounds which are looser than~\eqref{eqn:u-statistic-variance} but which still imply that $\mathrm{Var}\bigl[U_{n,\mathrm{Vor}}(f)\bigr] \to 0$ as $n \to \infty$. Neither of these are necessary to prove Theorem~\ref{thm:asymptotic_limit_voronoi}, and so we do not pursue them further.
\begin{proof}[Proof of Lemma~\ref{lem:u-statistic-variance}]
	We will repeatedly use the following fact, which is a consequence of Lemma~\ref{lem:voronoi-kernel-ub}: there exists a constant $C > 0$ not depending on $n$ such that for any $x,y \in \Omega$,
	\begin{equation}
	\label{pf:u-statistic-variance-1}
	H_{\mathrm{Vor}}(x,y) \leq \int_{L \cap \Omega} \exp\bigl(-(p_{\min}/2d)\|x - z\|^d\bigr) \,dz \leq C\Bigl(\frac{1}{n^{(d - 1)/d}}\1\{\|x - y\| \leq C\varepsilon_{0}\} + \frac{1}{n^2}\Bigr).
	\end{equation}
	Now, we recall from Hoeffding's decomposition of U-statistics~\citep{hoeffding1948class} that the variance of $U_{n,\mathrm{Vor}}(f)$ can be written as
	\begin{equation}
	\label{pf:u-statistic-variance-2}
	\mathrm{Var}[U_{n,\mathrm{Vor}}(f)] = \frac{1}{4}\Bigl(n(n - 1)\mathrm{Var}[h(x_1,x_2)] + n(n - 1)(n - 2)\mathrm{Var}[h_1(x_1)]\Bigr)
	\end{equation}
	where $h(x,y) = |f(x) - f(y)| H_{\mathrm{Vor}}(x,y)$ and $h_1(x) = \mathbb{E}[h(x_1,x_2)|x_1]$. 
	
	We now use~\eqref{pf:u-statistic-variance-1} to upper bound the variance of $h$ and $h_1$. For $h$, we have that
	\begin{align*}
	\mathrm{Var}[h(x_1,x_2)] & \leq \mathbb{E}[h^2(x_1,x_2)] \\
	& \leq p_{\max}^2 \|f\|_{C^1(\Omega)}^2 \int_{\Omega} \int_{\Omega} \|y - x\|^2 \bigl(H_{\mathrm{Vor}}(x,y)\bigr)^2 \,dy \,dx \\
	& \leq C\|f\|_{C^1(\Omega)}^2\biggl(\frac{1}{n^{2(d - 1)/d}}\int_{\Omega} \int_{\Omega} \|y - x\|^2 \1\{\|x - y\| \leq C\varepsilon_{0}\} \,dy \,dx + \frac{1}{n^4}\biggr) \\
	& \leq C\biggl(\varepsilon_{0}^{3d} \|f\|_{C^1(\Omega)}^2 + \frac{\|f\|_{C^1(\Omega)}^2}{n^4}\biggr).
	\end{align*}
	For $h_1$, we have that for every $x \in \Omega$,
	\begin{align*}
	|h_1(x)| & \leq \|f\|_{C^1(\Omega)} p_{\max} \int_{\Omega} |y - x\| H_{\mathrm{Vor}}(y,x) \,dy \\
	& \leq C\|f\|_{C^1(\Omega)}\biggl(\frac{1}{n^{(d - 1)/d}} \int_{\Omega} |y - x\| \1\{\|y - x\| \leq C\varepsilon_{0}\} \,dy + \frac{1}{n^2}\biggr) \\
	& \leq C\|f\|_{C^1(\Omega)}\biggl(\varepsilon_{0}^{2d} + \frac{1}{n^2}\biggr).
	\end{align*}
	Integrating over $x \in \Omega$, we conclude that
	\begin{equation*}
	\mathrm{Var}[h_1(x_1)] \leq \mathbb{E}[(h_1(x_1))^2] \leq C\|f\|_{C^1(\Omega)}^2\biggl(\varepsilon_{0}^{4d} + \frac{1}{n^4}\biggr).
	\end{equation*}
	Plugging these estimates back into~\eqref{pf:u-statistic-variance-2} gives the upper bound in~\eqref{eqn:u-statistic-variance}.
\end{proof}

\subsection{Step 3: Bias of Voronoi U-statistic}
\label{subsec:voronoi-ustat-bias}

Under appropriate conditions, the expectation of $U_{n,\mathrm{Vor}}(f)$ is approximately equal to (an appropriately rescaled version of) the nonlocal functional~\eqref{eqn:nonlocal-tv} for bandwidth $\varepsilon_{(1)}(x) = (np(x))^{-1/d}$, weight $(p(x))^{(d + 1)/d}$, and kernel
\begin{equation}
\label{eqn:voronoi-kernel}
K_{\mathrm{Vor}}(t) = \int_{0}^{\infty} \exp\biggl(-\Leb_d\Bigl\{\frac{t^2}{4} + s^2\Bigr\}^{d/2}\biggr) s^{d - 2} \,ds.
\end{equation} 
\begin{lemma}
	\label{lem:voronoi-tv-bias-holder}
	Suppose $x_{1:n}$ are sampled independently from a distribution $P$ satisfying~\ref{assump:density_bounded}. For any $f \in C^1(\Omega)$,
	\begin{equation*}
	\mathbb{E}\bigl[U_{n,\mathrm{Vor}}(f)\bigr] = n^{(d + 1)/d} \frac{\eta_{d - 2}}{2}\cdot\mathrm{TV}_{\varepsilon_{(1)},K_{\mathrm{Vor}}}\bigl(f;\Omega,p^{(d + 1)/d}\bigr) + O\biggl(\frac{(\log n)^{3 + 1/d}}{n^{1/d}}\|f\|_{C^1(\Omega)}\biggr).
	\end{equation*}
\end{lemma}
\begin{proof}
	We will use Lemma~\ref{lem:voronoi-boundary}, which shows that at points $x,y \in \Omega$ sufficiently far from the boundary of $\Omega$, the kernel $H_{\mathrm{Vor}}(x,y)$ is approximately equal to a spherical kernel. To invoke this lemma, we need to restrict our attention to points sufficiently far from the boundary. In particular, letting $h = h_n$ be defined as in Lemma~\ref{lem:voronoi-boundary}, we conclude from~\eqref{eqn:voronoi-boundary-2} that 
	\begin{multline}
	\label{pf:voronoi-tv-bias-holder-1}
	\int_{\Omega} \int_{\Omega} |f(y) - f(x)| H_{\mathrm{Vor}}(x,y) p(y) p(x) \,dy \,dx ={} \\ \int_{\Omega_h} \int_{\Omega} |f(y) - f(x)| H_{\mathrm{Vor}}(x,y) p(y) p(x) \,dy \,dx + O\biggl(\frac{h}{n^2}\|f\|_{C^1(\Omega)}\biggr),
	\end{multline}
	where we have used the assumption $f \in C^1(\Omega)$ and~\eqref{eqn:voronoi-boundary-2} to control the boundary term, since
	\begin{equation}
	\label{pf:voronoi-tv-bias-holder-2}
	\begin{aligned}
	\int_{\Omega \setminus \Omega_h} &\int_{\Omega} |f(y) - f(x)| H_{\mathrm{Vor}}(x,y) p(y) p(x) \,dy \,dx \\ 
& \leq \frac{C_3p_{\max}^2 \eta_{d - 2} \|f\|_{C^1(\Omega)}}{n^{(d - 1)/d}}  \int_{\Omega \setminus \Omega_h} \int_{\Omega} \|y - x\| K_{\mathrm{Vor}}\biggl(\frac{\|y - x\|}{C_4n^{1/d}}\biggr) \,dy \,dx \\
	& \overset{(i)}{\leq} \frac{C_3C_4^{(d + 1)/d}p_{\max}^2 \eta_{d - 2} \|f\|_{C^1(\Omega)}}{n^{2}}  \int_{\Omega \setminus \Omega_h} \int_{\Rd} \|h\| K_{\mathrm{Vor}}(\|h\|) \,dh \,dx \\
	& \overset{(ii)}{\leq} \frac{C_3C_4^{(d + 1)/d}p_{\max}^2 \eta_{d - 2} \eta_{d - 1} \|f\|_{C^1(\Omega)}}{n^{2}}  \int_{\Omega \setminus \Omega_h} \int_{0}^{\infty} t^{d} K_{\mathrm{Vor}}(t) \,dt \,dx \\
	& \overset{(iii)}{\leq} \frac{C\|f\|_{C^1(\Omega)}}{n^{2}}  \Leb(\Omega \setminus \Omega_h) \\
	& \leq \frac{Ch\|f\|_{C^1(\Omega)}}{n^{2}},
	\end{aligned}
	\end{equation}
	where $(i)$ follows by changing variables $h = (y - x)/C_3n^{1/d}$, $(ii)$ by converting to polar coordinates, and $(iii)$ upon noticing that $\int_{0}^{\infty} t^{d}K_{\mathrm{Vor}}(t) < \infty$. 
	
	Returning to the first-order term in~\eqref{pf:voronoi-tv-bias-holder-1}, we can use~\eqref{eqn:voronoi-boundary-1} to replace the integral with $H_{\mathrm{Vor}}$ by an integral with the Voronoi kernel $K_{\mathrm{Vor}}$. Precisely,
	\begin{align}
  \nonumber
	\int_{\Omega_h} \int_{\Omega}
  |f(y) - f(x)| &H_{\mathrm{Vor}}(x,y) p(y) p(x) \,dy \,dx \\
  \nonumber
  & = \frac{\eta_{d - 2}}{n^{(d - 1)/d}} \int_{\Omega_{h}}\int_{\Omega} |f(y) - f(x)| K_{\mathrm{Vor}}\biggl(\frac{\|x - y\|}{\varepsilon_{(1)}}\biggr) p(y) \bigl(p(x)\bigr)^{1/d} \,dy \,dx \\
  \nonumber
	& \quad + O\biggl(\frac{1}{n^3}\int_{\Omega}\int_{\Omega} |f(y) - f(x)| \,dy \,dx\biggr) \\
  \nonumber
	& \quad + O\biggl(\frac{(\log n)^2}{n} \int_{\Omega} \int_{\Omega} |f(y) - f(x)|\1\Bigl\{\|x - y\| \leq C (\log n/n)^{1/d}\Bigr\} \,dy \,dx\biggr) \\
  \nonumber
	& = \frac{\eta_{d - 2}}{n^{(d - 1)/d}} \int_{\Omega_h}\int_{\Omega} |f(y) - f(x)| K_{\mathrm{Vor}}\biggl(\frac{\|x - y\|}{\varepsilon_{(1)}}\biggr) p(y) \bigl(p(x)\bigr)^{1/d} \,dy \,dx \\
  \nonumber
	& \quad + O\biggl(\frac{\|f\|_{C^1(\Omega)}}{n^3} + \frac{(\log n)^{3 + 1/d}}{n^{2 + 1/d}} \|f\|_{C^1(\Omega)}\biggr) \\
  \nonumber
	& = \frac{\eta_{d - 2}}{n^{(d - 1)/d}} \int_{\Omega}\int_{\Omega} |f(y) - f(x)| K_{\mathrm{Vor}}\biggl(\frac{\|x - y\|}{\varepsilon_{(1)}}\biggr) p(y) \bigl(p(x)\bigr)^{1/d} \,dy \,dx \\
	\label{pf:voronoi-tv-bias-holder-2.5}
	& \quad + O\biggl(\frac{\|f\|_{C^1(\Omega)}}{n^3} + \frac{(\log n)^{3 + 1/d}}{n^{2 + 1/d}} \|f\|_{C^1(\Omega)} + \frac{h\|f\|_{C^1(\Omega)}}{n^{2}}\biggr),
	\end{align}
	with the second equality following from the upper bound~\eqref{eqn:dtv-ub-neighborhood}, and the third equality from exactly the same argument as in~\eqref{pf:voronoi-tv-bias-holder-2}. Finally, we use the Lipschitz property of $p$ to conclude that
	\begin{align}
    \nonumber
    \int_{\Omega}\int_{\Omega}
    |f(y) - f(x)|
    &K_{\mathrm{Vor}}\biggl(
      \frac{\|x - y\|}{\varepsilon_{(1)}}
    \biggr) p(y) \bigl(p(x)\bigr)^{1/d} \,dy \,dx 
    \\
    \label{pf:voronoi-tv-bias-holder-3}
    & = \int_{\Omega}\int_{\Omega}
    |f(y) - f(x)| K_{\mathrm{Vor}}\biggl(
      \frac{\|x - y\|}{\varepsilon_{(1)}}
    \biggr) \bigl(p(x)\bigr)^{(d + 1)/d} \,dy \,dx
    + O\biggl(\frac{\|f\|_{C^1(\Omega)}}{n^{(d + 2)/2}}\biggr),
	\end{align}
	since
	\begin{align*}
	& \int_{\Omega} \int_{\Omega} |f(y) - f(x)| K_{\mathrm{Vor}}\biggl(\frac{\|x - y\|}{\varepsilon_{(1)}}\biggr) |p(y) - p(x)| \bigl(p(x)\bigr)^{1/d} \,dy \,dx \\
	&\quad \leq C\|f\|_{C^1(\Omega)} p_{\max}^{1/d} \int_{\Omega} \int_{\Omega} \|y - x\|^2 K_{\mathrm{Vor}}\biggl(\frac{\|x - y\|}{\varepsilon_{(1)}}\biggr) \,dy \,dx \\
	&\quad \leq C\frac{\|f\|_{C^1(\Omega)}p_{\max}^{1/d}}{p_{\min}^{1/d} n^{(2 + d)/d}} \int_{\Omega} \int_{\Reals^d} \|h\|^2 K_{\mathrm{Vor}}(\|h\|) \,dh \,dx \\
	&\quad = C\frac{\|f\|_{C^1(\Omega)} p_{\max}^{1/d} \eta_{d - 1}}{p_{\min}^{1/d} n^{(2 + d)/d}} \int_{\Omega} \int_{0}^{\infty} t^{d + 1} K_{\mathrm{Vor}}(t) \,dt \,dx \\
	&\quad \leq C\frac{\|f\|_{C^1(\Omega)}}{n^{(2 + d)/d}},
	\end{align*}
	with the last inequality following since $\int_{0}^{\infty} t^{d + 1} K_{\mathrm{Vor}}(t) \,dt = C < \infty$. Combining~\eqref{pf:voronoi-tv-bias-holder-1},~\eqref{pf:voronoi-tv-bias-holder-2.5} and~\eqref{pf:voronoi-tv-bias-holder-3} yields the final claim.
\end{proof}
Finally, Lemma~\ref{lem:nonlocal-tv-bias-holder} shows that the kernelized TV
$\mathrm{TV}_{\varepsilon,K}(f;\Omega,h)$ converges to a continuum TV under
appropriate conditions.
\begin{assumption}{A2}{}
	\label{asmp:kernel}
  The bandwidth $\varepsilon(x)  = \bar{\varepsilon}_n g(x)$ for a sequence
  $\bar{\varepsilon}_n \to 0$ and a bounded function $g \in
  L^{\infty}(\Omega)$. The kernel function $K$ satisfies $\int_{0}^{\infty}
  K(t) t^{d + 1} \,dt < \infty$. The weight function $h \in
  L^{\infty}(\Omega)$. 
\end{assumption}
Note that Assumption~\ref{assump:density_bounded} implies that
Assumption~\ref{asmp:kernel} is satisfied by bandwidth $\varepsilon_{(1)}$,
kernel $K_{\mathrm{Vor}}$ and weight function $h = p^{(d + 1)/d}$. 
\begin{lemma}
	\label{lem:nonlocal-tv-bias-holder}
	Assuming~\ref{asmp:kernel}, for any $f \in C^2(\Omega)$,
	\begin{equation}
	\label{eqn:nonlocal-tv-bias-holder}
	\lim_{n \to \infty} (\bar{\varepsilon}_n)^{-(d + 1)} \mathrm{TV}_{\varepsilon,K}(f;\Omega, h) = \sigma_K \int_{\Omega} \|\nabla f(x)\| h(x) (g(x))^{d + 1} \,dx
	\end{equation}
	where
	\begin{equation}
	\label{eqn:surface-tension}
	\sigma_K := \frac{2\eta_{d - 2}}{(d - 1)}\int_{0}^{\infty} K(t) t^d \,dt.
	\end{equation}
\end{lemma}
\begin{proof}
	The proof of Lemma~\ref{lem:nonlocal-tv-bias-holder} follows closely the proof
  of some related results, e.g., Lemma~4.2
  of~\cite{garciatrillos2016continuum}. We begin by summarizing the major steps.  
	\begin{enumerate}
		\item We use a $2$nd-order Taylor expansion to replace differencing by derivative inside the nonlocal TV.
		\item Naturally, the nonlocal TV behaves rather differently than a local functional near the boundary of $\Omega$. We show that the contribution of the integral near the boundary is negligible. 
		\item Finally, we reduce from a double integral to a single integral involving the norm $\|\nabla f\|$.
	\end{enumerate}
	
	\paragraph{Step 1: Taylor expansion.}
	Since $f \in C^2(\Omega)$ we have that
	\begin{equation*}
	f(y) - f(x) = \nabla f(x)^{\top} (y - x) + O(\|f\|_{C^2(\Omega)}\|y - x\|^2).
	\end{equation*}
	Consequently, 
	\begin{align*}
	\mathrm{TV}_{\varepsilon,K}(f;\Omega,h) = \int_{\Omega} \int_{\Omega} \Bigl(|\nabla f(x)^{\top} (y - x)| + O(\|f\|_{C^2(\Omega)})\Bigr) K\biggl(\frac{\|y - x\|}{\varepsilon(x)}\biggr) h(x) \,dy \,dx.
	\end{align*}
	We now upper bound the contribution of the $O(\|y - x\|^2)$-term. For each $x \in \Omega$,
	\begin{align*}
	\int_{\Omega} \|y - x\| K\biggl(\frac{\|y - x\|^2}{\varepsilon(x)}\biggr) \,dy \leq C |\varepsilon_n(x)|^{d + 2} \int_{\Rd} \|z\|^{2} K(\|z\|) \,dz \leq C |\varepsilon_n(x)|^{d + 2} \leq C |\varepsilon_n(x)|^{d + 2},
	\end{align*}
	with the final inequality following from the assumption $\int_{0}^{\infty} t^{d + 1}K(t)\,dt < \infty$. Integrating over $\Omega$ gives the upper bound
	\begin{equation*}
	\int_{\Omega} \int_{\Omega} O(\|f\|_{C^2(\Omega)} \|y - x\|^2) K\biggl(\frac{\|y - x\|}{\varepsilon(x)}\biggr)  h(x) \,dy \,dx = O(\|f\|_{C^2(\Omega)} \bar{\varepsilon}_n^{d + 2}), 
	\end{equation*}
	recalling that $h(x),g(x) \in L^{\infty}(\Omega)$. 
	
	\paragraph{Step 2: Contribution of boundary to nonlocal TV.}
	Take $r = r_n$ to be any sequence such that $r_n/\bar{\varepsilon}_n \to \infty, r_n \to 0$. Breaking up the integrals in the definition of nonlocal TV gives
	\begin{align*}
	\int_{\Omega} \int_{\Omega} |\nabla f(x)^{\top} (y - x)| K\biggl(\frac{\|y - x\|}{\varepsilon(x)}\biggr) h(x) \,dy \,dx & = \int_{\Omega_r} \int_{\Rd} |\nabla f(x)^{\top} (y - x)| K\biggl(\frac{\|y - x\|}{\varepsilon(x)}\biggr) h(x) \,dy \,dx \\
	& \quad - \int_{\Omega_r} \int_{\Rd \setminus \Omega} |\nabla f(x)^{\top} (y - x)| K\biggl(\frac{\|y - x\|}{\varepsilon(x)}\biggr) h(x) \,dy \,dx \\
	& \quad + \int_{\Omega \setminus \Omega_r} \int_{\Omega} |\nabla f(x)^{\top} (y - x)| K\biggl(\frac{\|y - x\|}{\varepsilon(x)}\biggr) h(x) \,dy \,dx \\
	& =: I_1 + I_2 + I_3.
	\end{align*}
	Now we are going to show that $I_2$ and $I_3$ are negligible. For $I_2$, noting that $r/\varepsilon(x) \to \infty$ for all $x$, we have that for any $x \in \Omega_r$,
	\begin{align*}
	\int_{\Rd \setminus \Omega} |\nabla f(x)^{\top} (y - x)| K\biggl(\frac{\|y - x\|}{\varepsilon(x)}\biggr) h(x) \,dy & \leq \|f\|_{C^1(\Omega)} \int_{\Rd \setminus \Omega} K\biggl(\frac{\|y - x\|}{\varepsilon(x)}\biggr) \|y - x\| \,dy \\
	& \leq \|f\|_{C^1(\Omega)} (\varepsilon(x))^{1} \int_{\Rd \setminus B(0,r/\varepsilon(x))}  \|z\| K(\|z\|) \,dz \\
	& \overset{(i)}{\leq} C \|f\|_{C^1(\Omega)} (\varepsilon(x))^{d + 1} \int_{r/\varepsilon(x))}^{\infty} t^{d + 1} K(t) \,dt \\
	& \overset{(ii)}{=} o(\|f\|_{C^1(\Omega)} (\varepsilon(x))^{d + 1}),
	\end{align*}
	where $(i)$ follows from converting to polar coordinates and $(ii)$ follows by the assumption $\int_{0}^{\infty} t^{d + 1} K(t) \,dt < \infty$. Integrating over $x$ yields $I_2 = o(\|f\|_{C^1(\Omega)} \wb{\varepsilon}_n^{d + 1})$, since $h,g \in L^{\infty}(\Omega)$. \\
	
	On the other hand for $I_3$, similar manipulations show that for every $x \in \Omega$,
	\begin{align*}
	\int_{\Omega} |\nabla f(x)^{\top} (y - x)| K\biggl(\frac{\|y - x\|}{\varepsilon(x)}\biggr) \,dy & \leq C \|f\|_{C^1(\Omega)} (\varepsilon(x))^{d + 1}.
	\end{align*}
	Noting that the tube $\Omega \setminus \Omega_r$ has volume at most $Cr$, we conclude that
	\begin{equation*}
	I_3 \leq C \|f\|_{C^1(\Omega)} (\varepsilon(x))^{d + 1} \Leb(\Omega\setminus\Omega_r) \leq C r \|f\|_{C^1(\Omega)} (\varepsilon(x))^{d + 1} = o(\|f\|_{C^1(\Omega)} (\varepsilon(x))^{d + 1}),
	\end{equation*}
	with the last inequality following since $r = o(1)$.
	
	\paragraph{Step 3: Double integral to single integral.}
	Now we proceed to reduce the double integral in $I_1$ to a single integral. Changing variables to $z = (y - x)/\varepsilon(x)$, converting to polar coordinates, and letting $w(x) = \nabla f(x)/\|\nabla f(x)\|$, we have that
	\begin{align*}
	\int_{\Rd} \|\nabla f(x)^{\top} (y - x)| K\biggl(\frac{\|y - x\|}{\varepsilon(x)}\biggr) \,dy & = (\varepsilon(x))^{d + 1} \int_{\Rd} |\nabla f(x)^{\top} z| K(\|z\|) \,dz \\
	& = (\varepsilon(x))^{d + 1} \biggl(\int_{\mathbb{S}^{d - 1}} |\nabla f(x)^{\top} \phi| \,d\mc{H}^{d - 1}\biggr) \biggl(\int_0^{\infty} t^d K(t) \,dt\biggr) \\
	& = (\varepsilon(x))^{d + 1} \|\nabla f(x)\| \biggl(\int_{\mathbb{S}^{d - 1}} |w(x)^{\top} \phi| \,d\mc{H}^{d - 1}\biggr) \biggl(\int_0^{\infty} t^d K(t) \,dt\biggr) \\
	& = (\varepsilon(x))^{d + 1} \|\nabla f(x)\| \biggl(\int_{\mathbb{S}^{d - 1}} |\phi_1| \,d\mc{H}^{d - 1}\biggr) \biggl(\int_0^{\infty} t^d K(t) \,dt\biggr) \\
	& = \sigma_K (\varepsilon(x))^{d + 1} \|\nabla f(x)\|,
	\end{align*}
	with the second to last equality following from the spherical symmetry of the integral, and the last equality by definition of $\sigma_K$. Integrating over $x \in \Omega_r$ gives
	\begin{align*}
	I_1 & = \sigma_K \bar{\varepsilon}_n^{d + 1} \int_{\Omega_r} \|\nabla f(x)\| h(x) (g(x))^{d + 1} \,dx \\
	& = \sigma_K \bar{\varepsilon}_n^{d + 1} \int_{\Omega} \|\nabla f(x)\| h(x) (g(x))^{d + 1} \,dx + o(\bar{\varepsilon}_n^{d + 1}\|f\|_{C^1(\Omega)}),
	\end{align*}
	with the second equality following from the same reasoning as was used in analyzing the integral $I_3$.
	
	\paragraph{Putting the pieces together.}
	We conclude that
	\begin{align*}
    (\bar{\varepsilon}_n)^{-(d + 1)} &\mathrm{TV}_{\varepsilon,K}(f;\Omega, h)
  \\
  & = (\bar{\varepsilon}_n)^{-(d + 1)} \int_{\Omega} \int_{\Omega} \Bigl(|\nabla f(x)^{\top} (y - x)|)\Bigr) K\biggl(\frac{\|y - x\|}{\varepsilon(x)}\biggr) h(x) \,dy \,dx + O(\bar{\varepsilon}_n \|f\|_{C^2(\Omega)}) \\
	& = (\bar{\varepsilon}_n)^{-(d + 1)} \int_{\Omega_r} \int_{\Rd} \Bigl(|\nabla f(x)^{\top} (y - x)|)\Bigr) K\biggl(\frac{\|y - x\|}{\varepsilon(x)}\biggr) h(x) \,dy \,dx + O(\bar{\varepsilon}_n \|f\|_{C^2(\Omega)}) + o(\|f\|_{C^1(\Omega)}) \\
	& = \sigma_K \int_{\Omega} \int_{\Omega} \|\nabla f(x)\| h(x) (g(x))^{d + 1} \,dx + O(\bar{\varepsilon}_n \|f\|_{C^2(\Omega)}) + o(\|f\|_{C^1(\Omega)}),
	\end{align*}
	completing the proof of Lemma~\ref{lem:nonlocal-tv-bias-holder}.
\end{proof}

%% file: app_sensitivity_analysis.tex
\def\graphscale{50}

In Section~\ref{sec:experiments}, we chose the scale $k$, $\varepsilon$ in the
$k$-nearest neighbor and $\varepsilon$-neighborhood graphs to be such that their
average degree would roughly match that of the Voronoi adjacency graph, and we
remarked that mildly better results are attainable if one increases the
connectivity of the graphs.  Here, we present an analogous set of results to
those found in Section~\ref{sec:experiments}, where the average degree of the
$k$-nearest neighbor and $\varepsilon$-neighborhood graphs are roughly twice
that of the graphs in Section~\ref{sec:experiments}.  All other details of the
experimental setup remain the same.
\begin{itemize}
  \item In Figure~\ref{fig:04_TVEstimationResults_HighScale}, the estimates of
    TV by the $k$-nearest neighbor and $\varepsilon$-neighborhood graphs
    approach their density-weighted limits more quickly than in
    Section~\ref{sec:experiments}, with slightly narrower variability bands.
  \item In Figure~\ref{fig:02_FunctionEstimationResults_HighScale}, we see that 
    $\varepsilon$-neighborhood TV denoising is now competitive with $k$-nearest
    neighbor TV denoising and the unweighted Voronoigram for the ``low inside
    tube'' setting.  In the ``high inside tube'' and uniform sampling settings,
    the performance of $k$-nearest neighbor TV denoising improves slightly.
\end{itemize}
As previously remarked, the Voronoigram has no such auxiliary tuning parameter,
so the weighted and unweighted Voronoigram results here are the same as in
Section~\ref{sec:experiments}.  We also note that with greater connectivity in
the $k$-nearest neighbor and $\varepsilon$-neighborhood graphs comes greater
computational burden in storing the graphs, as well as performing calculations
with them. Therefore, it is advantageous to the practitioner to use the sparsest
graph capable of achieving favorable performance.

\begin{figure}[H]
  \centering
  \includegraphics[width=0.95\linewidth]{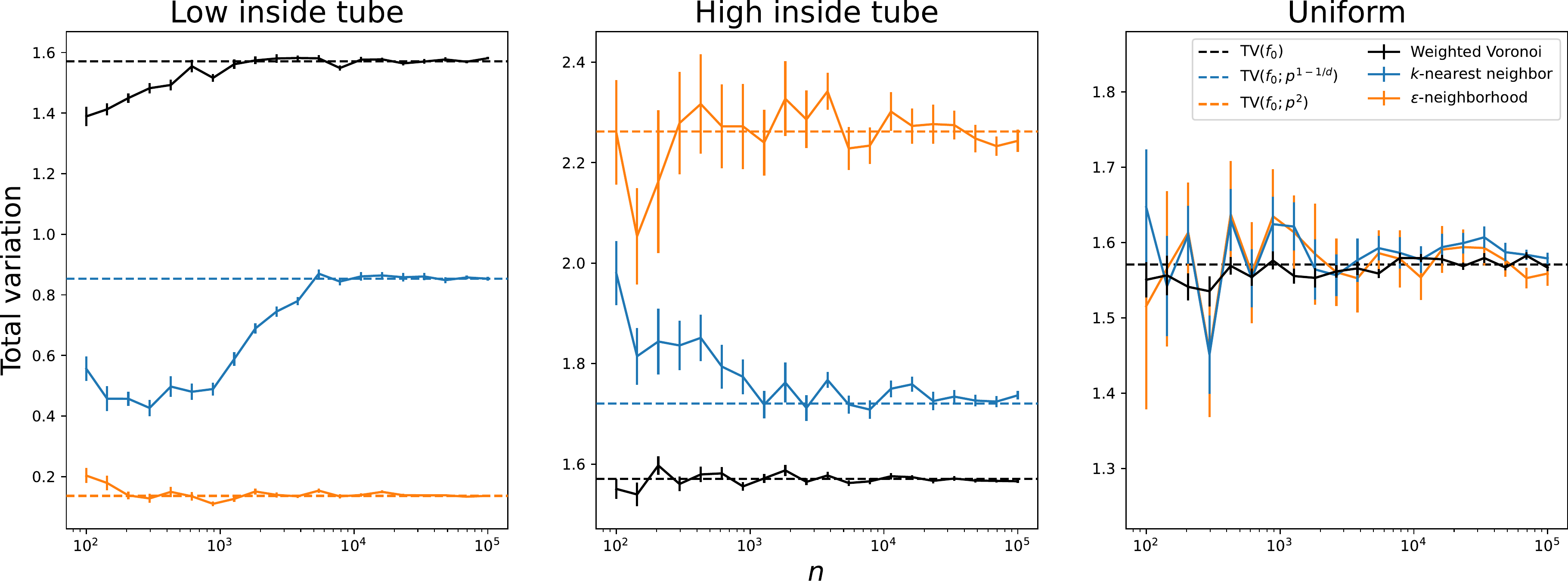}
  \caption{\it
    Results from the TV estimation experiment, with greater connectivity
    in the kNN and $\varepsilon$-neighborhood graphs.  Compare
    these results to those in Figure~\ref{fig:04_TVEstimationResults}.
  }%
  \label{fig:04_TVEstimationResults_HighScale}

\bigskip
  \includegraphics[width=0.95\linewidth]{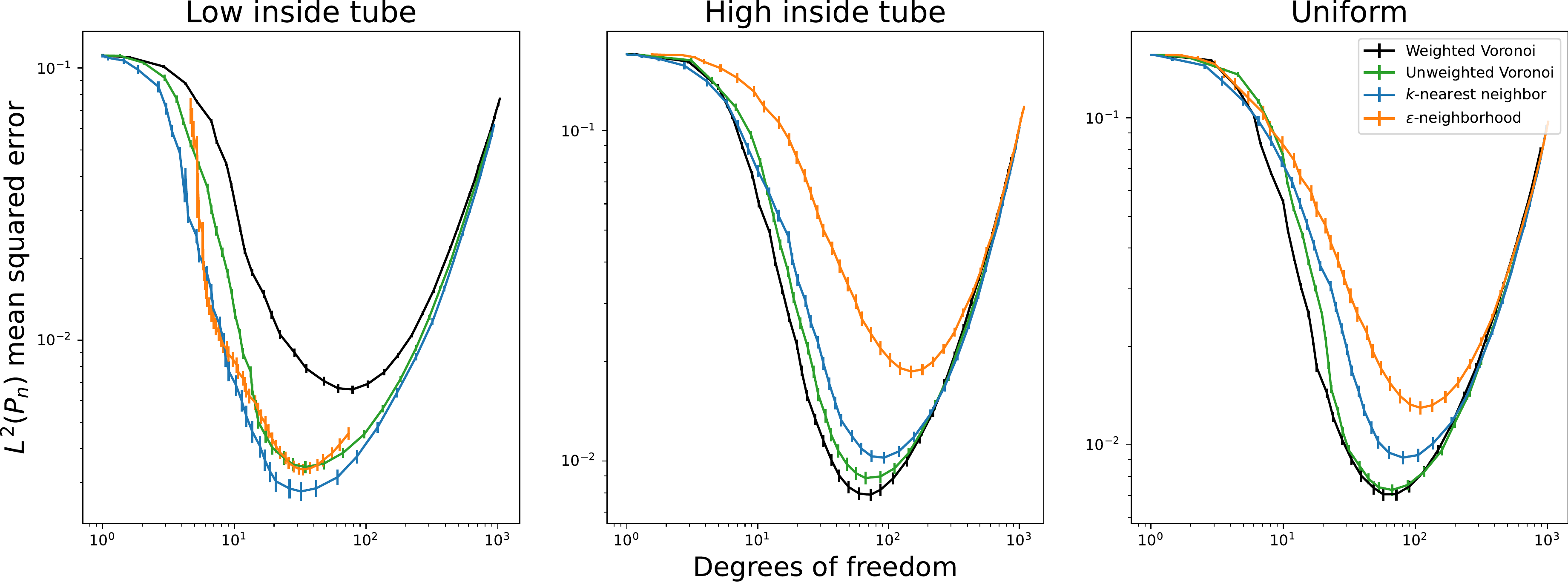}
  \caption{\it
    Results from the function estimation experiment, with greater connectivity
    in the kNN and $\varepsilon$-neighborhood graphs.  Compare
    these results to those in Figure~\ref{fig:02_FunctionEstimationResults}.
  }%
  \label{fig:02_FunctionEstimationResults_HighScale}
\end{figure}

\begin{figure}[H]
  \centering
  \includegraphics[width=0.95\linewidth]{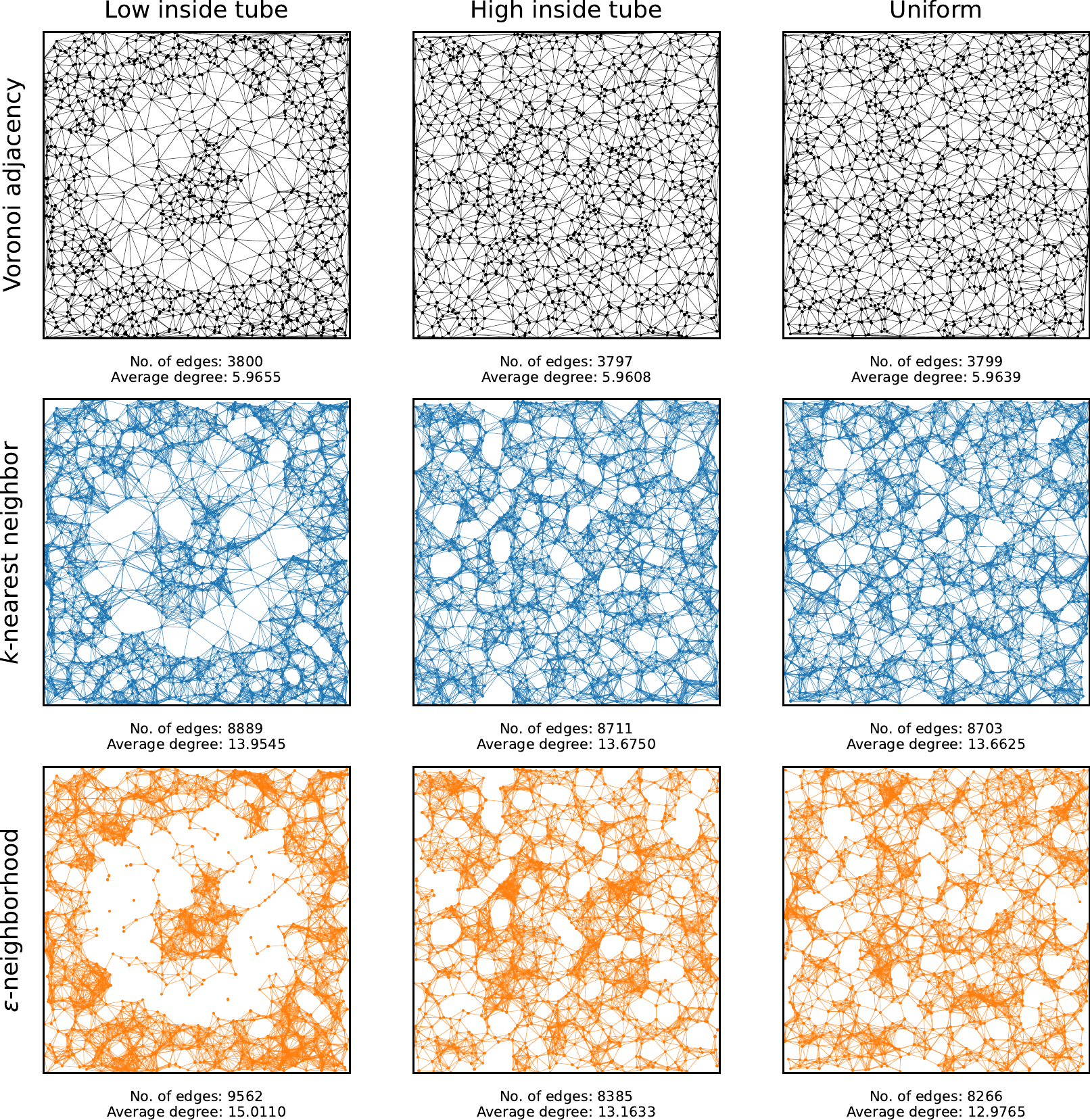}
  \caption{\it
    Visualization of the Voronoi, kNN, and $\varepsilon$-neighborhood graphs,
    with greater connectivity in the latter two graphs. (The Voronoi graph  does
    not have such an auxiliary tuning parameter.)  Compare these graphs to those
    in Figure~\ref{fig:05_TVEstimationGraphs}. 
  }%
  \label{fig:05_TVEstimationGraphs_HighScale}
\end{figure}

\begin{figure}[H]
  \centering
  \includegraphics[width=0.9\linewidth]{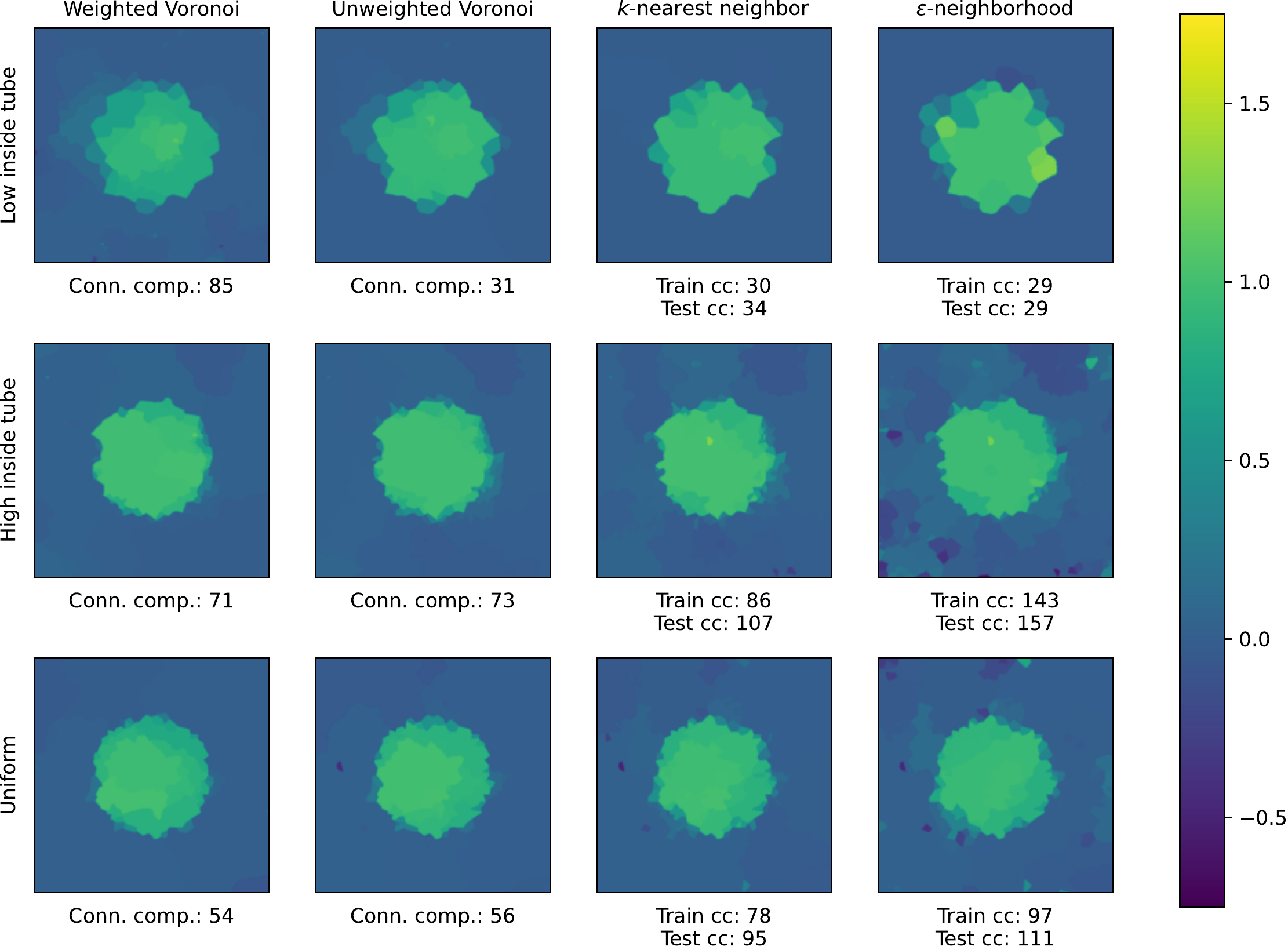}
  \caption{\it
    Extrapolants from graph TV denoising, with greater connectivity in the
    kNN and $\varepsilon$-neighborhood graphs.  Compare these
    results to those in Figure~\ref{fig:03_FunctionEstimationPredictions}.
    }%
  \label{fig:03_FunctionEstimationPredictions_HighScale}
\end{figure}

%% file: app_estimation_theory.tex
\subsection{Proof of Theorem \ref{thm:voronoi_ub}}
\label{app:pf-thm-voronoi-ub}

From \eqref{eq:voronoi_triangle_ineq}, in the discussion preceding Lemma
\ref{lem:voronoi_extrapolate_ub}, we have  
\begin{equation}
  \label{pf:voronoi-triangle-ineq}
  \E \|\hf - f_0 \|_{L^2(P)}^2
  \leq \E \Big[ K_n \|\hf - f_0\|_{L^2(P_n)}^2 \Big]
  + 2 \E \|\bar{f}_0 - f_0 \|_{L^2(P)}^2, 
\end{equation}
where
\[
K_n = 2p_\mx n \cdot \bigg(\max_{i=1,\dots,n} \Leb(\Part_i) \bigg).
\]
The second term is bounded by Lemma~\ref{lem:voronoi_extrapolate_ub}.  We now
outline the analysis of the first term.  As in the $L^2(P_n)$ case
we will decompose the error into the case where the design points
are well spaced and the case where they are not.  This is formalized by
the set $\XSet = \XSet_1\cap\XSet_2$, where $\XSet_1,\XSet_2$ are defined
in Appendix~\ref{sec:graph_embeddings}.  $x_{1:n}$ falls within this set
with probability at least $1-3/n^4$, and notably on this set,
\begin{equation*}
  \max_i\Leb(\Part_i) \leq C_1\log n/n,
\end{equation*}
for some $C_1>0$, since $\XSet_2$ is the set upon which the conclusion of
Lemma~\ref{lem:voronoi-cell-upper-bounds} holds.  We proceed by conditioning,
\begin{equation}
  \label{pf:weighted-l2p-error-decompose-cases}
  \begin{aligned}
  \E\left[
    K_n \lVert \hf - f_0\rVert_{L^2(P_n)}^2
  \right]
  &= 2p_\mx \Bigg(
    \E_x\left[
      \E_{z|x}\left[
        \max_i\Leb(\Part_i) \lVert \hat\theta-\theta_0\rVert_2^2
      \right] 
      \one\{x_{1:n}\in\XSet\}
    \right]
    \\
  &\hspace{4cm}+
    \E_x\left[
      \E_{z|x}\left[
        \max_i\Leb(\Part_i) \lVert \hat\theta-\theta_0\rVert_2^2
      \right] 
      \one\{x_{1:n}\not\in\XSet\}
    \right]
  \Bigg)
  \end{aligned}
\end{equation}
Using the fact that $x_{1:n}\in\XSet$, the first term on the RHS
of~\eqref{pf:weighted-l2p-error-decompose-cases} may be bound,
\begin{align}
  \nonumber
  \E_{z|x}\left[
    \max_i\Leb(\Part_i) \lVert \hat\theta-\theta_0\rVert_2^2
  \right] 
  \one\{x_{1:n}\in\XSet\}
  &\leq C_1(\log n) \;
  \E_{z|x} \left[
    \frac{1}{n} \lVert \hat\theta-\theta_0\rVert_2^2
  \right] \cdot \one\{x_{1:n}\in\XSet\}
  \\
  \label{pf:weighted-l2p-error-good-case}
  &\leq C_2(\log n)\left(
    \frac{\lambda \lVert D\theta_0\rVert}{n} 
    + \frac{\log^\alpha n}{n} 
  \right),
\end{align}
where the latter inequality is obtained by following the analysis of
Lemma~\ref{lem:voronoi_empirical_dtv_ub}.  For the second term on the RHS
of~\eqref{pf:weighted-l2p-error-decompose-cases}, we apply the crude
upper bound that $\Leb(\Part_i)\leq\Leb(\XDom)=1$ for all $i=1,\dotsc,n$.
Then apply~\eqref{pf:ub-mse-tv-1} to obtain,
\begin{align}
  \nonumber
  \E_{z|x}\left[
    \max_i\Leb(\Part_i) \lVert \hat\theta-\theta_0\rVert_2^2
  \right] 
  \one\{x_{1:n}\not\in\XSet\}
  &\leq 
  \E_{z|x}\left[
    16 \lVert z_{1:n}\rVert_2^2 + 2\lambda \lVert D\theta_0\rVert_1
  \right] 
  \one\{x_{1:n}\not\in\XSet\}
  \\
  \nonumber
  &=
  \left(
    16n + 2\lambda \lVert D\theta_0\rVert
  \right) \one\{x_{1:n}\not\in\XSet\}.
  \\
  \nonumber
  &\leq 
  \left(
    16n + 4n^2\lambda \lVert \theta_0\rVert_\infty \lVert w\rVert_\infty
  \right) \one\{x_{1:n}\not\in\XSet\}.
  \\
  &\leq 
  \left(
    16n + 4n^2\lambda \lVert \theta_0\rVert_\infty
  \right) \one\{x_{1:n}\not\in\XSet\},
  \label{pf:weighted-l2p-error-bad-case}
\end{align}
where we also use crude upper bounds on the discrete TV.  Substitute
\eqref{pf:weighted-l2p-error-good-case}~and~\eqref{pf:weighted-l2p-error-bad-case}
into~\eqref{pf:weighted-l2p-error-decompose-cases} to obtain,
\begin{align}
  \nonumber
  \E\left[
    K_n \lVert \hf-f_0\rVert_{L^2(P_n)}^2
  \right] 
  &\leq C_3\left(
    \frac{(\log n)\lambda \E\lVert D\theta_0\rVert}{n} 
    + \frac{(\log n)^{1+\alpha}}{n} 
    + \lambda n^2\P\{x_{1:n}\not\in\XSet\}
  \right)
  \\
  \nonumber
  &\leq C_4\left(
    \frac{(\log n)\lambda \E\lVert D\theta_0\rVert}{n} 
    + \frac{(\log n)^{1+\alpha}}{n} 
    + \frac{\lambda}{n^2} 
  \right)
  \\
  \label{pf:weighted-l2p-error-bound-in-edtv}
  &\leq C_5\left(
    \frac{\sigma\tau_n(\log n)^{3/2+\alpha}\E \lVert D\theta_0\rVert}{n} 
    + \frac{(\log n)^{1+\alpha}}{n} 
  \right),
\end{align}
where in the final line we have substituted in the value
of $\lambda = c\sigma\tau_n(\log n)^{1/2+\alpha}$.  Apply
Lemma~\ref{lem:voronoi_dtv_ub} to~\eqref{pf:weighted-l2p-error-bound-in-edtv}
and substitute back into~\eqref{pf:voronoi-triangle-ineq} to obtain the
claim.
\qed

\subsection{Proof of Theorem~\ref{thm:minimax_lb}}
\label{sec:pf-minimax-lower-bound}

To establish the lower bound in~\eqref{eq:minimax_lb}, we follow a classical
approach, similar to that outlined in~\citep{delalamo2021frameconstrained}:
first we reduce the problem to estimating binary sequences, then we apply
Assouad's lemma (Lemma~\ref{lem:assouad}). This results in a constrained
maximization problem, which we analyze to establish the ultimate lower bound.  

\paragraph{Step 1: Reduction to estimating binary sequences.}
We begin by associating functions $f_{\theta}$ with vertices of the hypercube
$\Theta_S = \{0,1\}^S$, where $S \subseteq [m]^d$ for some $m \in \mathbb{N}$.
To construct these functions $f_{\theta}$, we partition $\XDom$ into cubes,
\begin{equation*}
Q_i = \frac{1}{m}(i_1 - 1,i_1) \times \cdots \times \frac{1}{m}(i_d - 1,i_d),
\quad \textrm{for $i \in [m]^d$,}
\end{equation*}
and for each $\theta \in \Theta_S$ take $f_{\theta}$ to be the piecewise
constant function
\begin{equation}
\label{eqn:hard-test-functions}
f_{\theta}(x) := a \cdot \sum_{i \in S} \theta_i \1_{Q_i}(x),
\end{equation}
where $\1_{Q_i}(x) = \1(x \in Q_i)$ is the characteristic function of $Q_i$.
Observe that for all $\theta \in \Theta_S$, letting $\epsilon := 1/m$,
\begin{equation}
\label{eqn:hard-test-functions_tv}
\mathrm{TV}(f_{\theta}) \leq 2d a |S| \epsilon^{d - 1},
\quad\textrm{and}\quad \|f_{\theta}\|_{L^\infty(\XDom)} \leq a.
\end{equation}
So long as the constraints in~\eqref{eqn:hard-test-functions_tv} are satisfied
$\{f_\theta: \theta \in \Theta_S\} \subseteq \mathrm{BV}_{\infty}(L,M)$, and
consequently
\begin{equation}
\label{pf:bounded-variation_minimax-lb_2}
\inf_{\wh{f}} \sup_{f_0 \in \mathrm{BV}_{\infty}(L,M)}
\Ebb_{f_0} \|\wh{f} - f_0\|_{L^2(\XDom)}^2
\geq \inf_{\wh{f}} \max_{\theta \in \Theta_S}
\Ebb_{\theta} \|\wh{f} - f_{\theta}\|_{L^2(\XDom)}^2
\geq \frac{a^2\epsilon^d}{4}\inf_{\wh{\theta}} \max_{\theta \in \Theta}
\Ebb_{\theta} \rho(\wh{\theta},\theta),
\end{equation}
where $\rho(\theta,\theta') = \sum_{i \in S} |\theta_i - \theta_i'|$ is the
Hamming distance between vertices $\theta,\theta' \in \Theta_S$. The second
inequality in~\eqref{pf:bounded-variation_minimax-lb_2} is verified as follows:
for a given $\wh{f}$, letting 
\begin{equation*}
\wh{\theta}_i = 
\begin{dcases*}
1, \quad\textrm{if $\oint_{Q_i} \wh{f}(x) \,dx \geq a/2$}, \\
0, \quad\textrm{otherwise,}
\end{dcases*}
\end{equation*}
it follows that 
\begin{align*}
\|\wh{f} - f_{\theta}\|_{L^2(P)}^2
& = \sum_{i \in [m]^d} \|\wh{f} - f_{\theta}\|_{L^2(Q_i)}^2 \\
& \geq \sum_{i \in S} \|\wh{f} - f_{\theta}\|_{L^2(Q_i)}^2 \\
& \geq \frac{a^2\epsilon^d}{4} \sum_{i \in S} \1\{\wh{\theta}_i \neq \theta_i\}.
\end{align*}

\paragraph{Step 2: application of Assouad's lemma.}
Given a measurable space $(\mc{Z}, \mc{A})$, and a set of probability measures
$\mc{M} = \{\mu_{\theta}: \theta \in \Theta\}$ on $(\mc{Z}, \mc{A})$, Assouad's
lemma lower bounds the minimax risk over $\Theta_S$, when loss is measured by
the Hamming distance
$\rho(\wh{\theta},\theta) := \sum_{i \in S} |\wh{\theta}_i - \theta_i|$. We use
a form of Assouad's lemma given in~\citet{tsybakov2009introduction}.
\begin{lemma}[Lemma 2.12 of~\cite{tsybakov2009introduction}]
	\label{lem:assouad}
  Suppose that for each $\theta,\theta' \in \Theta_S: \rho(\theta,\theta') =
  1$, we have that
  $\mathrm{KL}(\mu_{\theta},\mu_{\theta'}) \leq \alpha < \infty$. It follows
  that
	\begin{equation*}
	\inf_{\wh{\theta}} \sup_{\theta \in \Theta_S} \Ebb_{\theta}
  \rho(\wh{\theta},\theta)
  \geq \frac{|S|}{2}
  \max\biggl(\frac{1}{2}\exp(-\alpha),(1 - \sqrt{\alpha/2})\biggr).
	\end{equation*} 
\end{lemma}
To apply Assouad's lemma in our context, we take $\mc{Z} = (\XDom \times
\Reals)^{\otimes n}$, and associate each $\theta \in \Theta_S$ with the measure
$\mu_{\theta}^{(n)}$, the $n$-times product of measure $\mu_{\theta} =
\mathrm{Unif}(\XDom) \times N(f_{\theta}(x),1)$. We now lower bound the KL
divergence $\mathrm{KL}(\mu_{\theta}^{(n)},\mu_{\theta'}^{(n)})$ when
$\rho(\theta,\theta') = 1$; letting $i \in S$ be the single index at which
$\theta_i \neq \theta_i'$,
\begin{align*}
\mathrm{KL}(\mu_{\theta},\mu_{\theta'})
& = \int_{\XDom \times \Reals}
\log\biggl(\frac{\phi(y - f_{\theta}(x))}{\phi(y - f_{\theta'}(x))}\biggr)
\phi(y - f_{\theta}(x)) \,dy \,dx \\
& = \int_{Q_i \times \Reals}
\log\biggl(\frac{\phi(y - a\theta_i)}{\phi(y -a\theta_i')}\biggr)
\phi(y - a\theta_i) \,dy \,dx \\
& = \epsilon^d \int_{\Reals}
\log\biggl(\frac{\phi(y - a\theta_i}{\phi(y - a\theta_i')}\biggr)
\phi(y - a\theta_i) \,dy \\
& = \frac{\epsilon^d a^2}{2},
\end{align*}
and it follows that $\mathrm{KL}(\mu_{\theta}^{(n)},\mu_{\theta'}^{(n)}) \leq n
\epsilon^d a^2/2$. Consequently, so long as~\eqref{eqn:hard-test-functions_tv}
is satisfied and 
\begin{equation*}
\frac{n \epsilon^d a^2}{2} \leq 1,
\end{equation*}
we may apply Lemma~\ref{lem:assouad}, and deduce
from~\eqref{pf:bounded-variation_minimax-lb_2} that 
\begin{equation}
\label{pf:bounded-variation_minimax-lb_4}
\inf_{\wh{f}} \sup_{f_0 \in \mathrm{BV}_{\infty}(L,M)}
\Ebb_{f_0} \|\wh{f} - f_0\|_{L^2(\XDom)}^2
\geq \frac{a^2\epsilon^d}{4}\inf_{\wh{\theta}}
\max_{\theta \in \Theta} \Ebb_{\theta} \rho(\wh{\theta},\theta)
\geq \frac{a^2\epsilon^d|S|}{16 \exp(1)}.
\end{equation}

\paragraph{Step 3: Lower bound.}
The upshot of Steps 1 and 2 is that the solution to the following constrained
maximization problem yields a lower bound on the minimax risk: letting
$s = |S|$,
\begin{equation*}
\begin{aligned}
& \mathrm{maximize} \quad && \frac{a^2 \epsilon^d s}{16 \exp(1)},\\
& \mathrm{subject~to} \quad && 1 \leq s \leq \epsilon^{-d}, \\
& && a s \epsilon^{d - 1} \leq \frac{L}{2d}, \\
& && a \leq M, \\
& && \frac{n a^2 \epsilon^d}{2} \leq 1.
\end{aligned}
\end{equation*}
Setting $a = M, \epsilon = (\frac{2}{a^2n})^{1/d}$, and $s = \frac{L}{2da}
\epsilon^{-(d - 1)}$ is feasible for this problem if $
2dM (\frac{M^2n}{2})^{-\frac{(d - 1)}{d}}\leq L \leq
2dM(\frac{M^2n}{2})^{1/d}$, and implies that the optimal value is at least
$\frac{2^{1/d}}{32\exp(1)d}LM(M^2n)^{-1/d}$. This implies the
claim~\eqref{eq:minimax_lb} upon suitable choices of constants. 
\qed

\subsection{Proof of Lemma~\ref{lem:voronoi_empirical_dtv_ub}}
In this proof, write $\theta_0 := (f_0(x_1),\ldots,f_0(x_n))$ and
$\mathbb{E}_{z|x}[\cdot] = \mathbb{E}[\cdot|x_{1:n}]$. 
We will use $\PenMat$ to represent the modified edge incidence operator
with either clipped edge weights or unit weights; the following analysis,
which uses the scaling factor $\tau_n$, applies to both.  Let
\begin{equation}
\label{pf:voronoi-estimation--1}
  \XSet = \XSet_1\cap\XSet_2,
\end{equation} 
with $\XSet_1, \XSet_2$ as in Section~\ref{sec:graph_embeddings}.
By the law of iterated expectation, 
\begin{equation}
\label{pf:voronoi-estimation-0}
\mathbb{E}\biggl[\frac{1}{n} \|\wh{\theta} - \theta_0\|_2^2\biggr]
  = \mathbb{E}_x\biggl[
    \mathbb{E}_{z|x}\Bigl[\frac{1}{n} \|\wh{\theta} - \theta_0\|_2^2\Bigr]
    \cdot \1\{x_{1:n} \in \XSet\}
  \biggr] + \mathbb{E}_x\biggl[
    \mathbb{E}_{z|x}\Bigl[\frac{1}{n} \|\wh{\theta} - \theta_0\|_2^2\Bigr]
    \cdot \1\{x_{1:n} \not \in \XSet\}
  \biggr].
\end{equation}
We now upper bound each term on the right hand side separately. 

For the first term, we will proceed by comparing the penalty operator
$\PenMat$ to the averaging operator~\eqref{eq:mesh-averaging-operator}
and surrogate operator $\SurPenMat$ corresponding to the
graph~\eqref{eq:lattice-graph}.  By construction $x_{1:n} \in \XSet$ implies,
for $(\xi_k,u_k)$ the $k$th singular value/left singular vector of
$\SurPenMat$, that
\begin{align*}
  \lambda
  &\geq C_1 \sigma \tau_n
    (\log n)^{1/2+\alpha}
  \\
  &\geq \max\left\{
    8\max_\ell|\Cell_\ell|^{1/2}
    \Phi_1(\PenMat,\SurPenMat,\SqAvMat)
    \cdot
    \sigma
    \sqrt{\log 2n^4 \cdot \sum_{k = 2}^{n} \frac{\|u_k\|_{\infty}^2}{\xi_k^2}},
    \Phi_2(\PenMat,\SurPenMat,\SqAvMat)
    \cdot
    \sigma\sqrt{2\log n}
  \right\},
\end{align*}
where the latter inequality follows
from combining 
\eqref{eqn:grid-svd-scaling},~\eqref{eq:uniform-upper-bound-cell-count}
with \eqref{eq:surr-diff-av-embedding-cvor},~\eqref{eq:av-error-bound-cvor}
in the clipped weights case, or
\eqref{eq:surr-diff-av-embedding-uvor},~\eqref{eq:av-error-bound-uvor}
in the unit weights case, for an appropriately chosen $C_1$.
We may therefore apply Theorem~\ref{thm:graph-tv-denoising-surrogate-graph}
with $\PenMat$, $\SurPenMat$, and $\SqAvMat$, which gives
\begin{equation}
  \label{pf:voronoi-estimation-1}
  \mathbb{E}_{z|x}\biggl[
    \frac{1}{n}\|\wh{\theta} - \theta_0\|_2^2
  \biggr] \cdot \1\{x_{1:n} \in \XSet\}
  \leq C\biggl(
    \frac{\lambda \|D\theta_0\|_1}{n} + \frac{\log^\alpha n}{n}
  \biggr),
\end{equation}
On the other hand, to upper bound the second term
in~\eqref{pf:voronoi-estimation-0} we use~\eqref{pf:ub-mse-tv-1}, 
\begin{equation}
  \label{pf:voronoi-estimation-2}
  \begin{aligned}
    \mathbb{E}_{z|x}\Bigl[
      \frac{1}{n} \|\wh{\theta} - \theta_0\|_2^2
    \Bigr]\cdot \1\{x_{z|x} \not \in \XSet\}
    & \leq \mathbb{E}_{z|x}\Bigl[
      \frac{16 \|z_{1:n}\|_2^2}{n} + \frac{2 \lambda \|D\theta_0\|_1}{n}
    \Bigr] \1\{x_{1:n} \not \in \XSet\} \\
    & \leq \Bigl(
      16 + \frac{2 \lambda \|D\theta_0\|_1}{n}
    \Bigr)\1\{x_{1:n} \not \in \XSet\}.
  \end{aligned}
\end{equation}
Substituting~\eqref{pf:voronoi-estimation-1}
and~\eqref{pf:voronoi-estimation-2}
into~\eqref{pf:voronoi-estimation-0}, we conclude that
\begin{align}
  \nonumber
  \mathbb{E}\biggl[
    \frac{1}{n} \|\wh{\theta} - \theta_0\|_2^2
  \biggr]
  & \leq C\Bigl(
    \frac{\lambda \mathbb{E}\|D\theta_0\|_1}{n}
    + \frac{\log^\alpha n}{n}
    + \mathbb{P}(x_{1:n} \not\in \XSet)\Bigr) \\
  \nonumber
  & \leq C\Bigl(
    \frac{\lambda \mathbb{E}\|D\theta_0\|_1}{n}
    + \frac{\log^\alpha n}{n}
  \Bigr)
  \\
  \label{eq:gtv-ub-voronoigram-Dtheta}
  &= C\Bigl(
    \frac{
      \sigma\tau_n(\log n)^{1/2+\alpha}
      \E \lVert \PenMat\theta_0\rVert_1
    }{n} + \frac{
    \log^\alpha n}{n} 
  \Bigr),
\end{align}
with the second inequality following from
Lemma~\ref{lem:random-graph-embedding}, and the equality from the choice of
$\lambda=C_1\sigma\tau_n(\log n)^{1/2+\alpha}$.
\qed

\subsection{Proof of Lemma~\ref{lem:voronoi_dtv_ub}}
We prove the claim~\eqref{eq:voronoi_dtv_ub} separately for the unit weights
and clipped weights case (recall that they differ by a scaling factor
$\bar\tau_n$.
We will subsequently abbreviate $f:=f_0$ and use the notation
$\DTV(\;\cdot\;; w^{\Eps\leftarrow r})$ to denote the
$\varepsilon$-neighborhood graph TV, having set $\varepsilon=r$.

\subsubsection{Unit weights}
Our goal is to upper bound
\begin{equation*}
  \E\Big[ \DTV \Big( f(x_{1:n}); \, \check{w}^\Vor\Big)\Big] 
  = n(n - 1) \mathbb{E}\Bigl[
    |f(x_1) - f(x_2)| \1\{\mc{H}^{d - 1}(\bar \Part_1 \cap \bar \Part_2) > 0\}
  \Bigr].
\end{equation*}
By conditioning, we can rewrite the expectation above as
\begin{equation}
  \label{pf:dtv-ub-1}
  p_\mx^2\int_{\XDom}\int_{\XDom}
    |f(y) - f(x)| \mathbb{P}_{x_{3:n}}\{
      \mc{H}^{d - 1}(\bar \Part_{x} \cap \bar \Part_{y}) > 0
    \}
    \,dy\,dx,
\end{equation}
where $V_{x} = \{z:  \|z - x\|_2 < \|z - x_i\|~\forall i = 2,3,\ldots,n\}$,
and likewise for $V_{y}$. Note that $V_x$ and $V_y$ are random subsets of
$\Reals^d$.

We now give an upper bound on the probability that the random cells $\bar \Part_x$
and $\bar \Part_y$ intersect on a set of positive Hausdorff measure, by relating
the problem
to uniform concentration of the empirical mass of balls in $\Rd$. The upper
bound will be crude, in that it may depend on suboptimal multiplicative
constants, but sufficient for our purposes. Define $r(V_x) := \sup\{\|z - x\|:
z \in \Part_x\}$. Observe that if $\|y - x\| > r(V_x) + r(V_y)$, then
$\bar \Part_x \cap \bar \Part_y
= \emptyset$, since for any $z \in \Part_{x}$, by the triangle inequality
\begin{equation*}
  \{
    \|z - y\| \geq \|y - x\| - \|z - x\| > r(V_y)
  \}
  \Longrightarrow
  \{
    z \not\in \Part_y
  \};
\end{equation*}
therefore
\begin{equation*}
  \{
    \cH^{d - 1}(\bar \Part_{x} \cap \bar \Part_{y}) > 0
  \}
  \Longrightarrow
  \{
    \|y - x\| \leq r(V_x) + r(V_y)
  \}.
\end{equation*}
Now, choose $z \in \Part_x$ for which $\|z - x\| = r(V_x)$. Observe that the ball
$B(z,r(V_x)/2)$ must have empirical mass $0$, i.e., $B(z,r(V_x)/2) \cap
\{x_3,\ldots,x_n\} = \emptyset$ (indeed, this same fact must hold for any
$r < r(V_x)$). Therefore,
\begin{equation*}
  \mathbb{P}_{x_{3:n}}\{r(V_x) \geq t\}
  \leq \mathbb{P}_{x_{3:n}}\Bigl\{
    \exists z:B(z,t/2) \cap \{x_3,\ldots,x_n\} = \emptyset
  \Bigr\}.
\end{equation*}
It follows from Lemma~\ref{lem:vc-balls} that if
$t_{n,\delta} = c\bigl(\frac{1}{n}(d \log n + \log(1/\delta)\bigl)^{1/d} <
  t_0$,
where $t_0$ is a constant not depending on $n,\delta$, then
\begin{equation*}
  \mathbb{P}_{x_{3:n}}\{
    \exists z:B(z,t_{n,\delta}/2) \cap \{x_3,\ldots,x_n\} = \emptyset
  \} \leq \delta.
\end{equation*}
Summarizing this reasoning, we have
\begin{align*}
  \mathbb{P}_{x_{3:n}}\{\mc{H}^{d - 1}(\bar \Part_{x} \cap \bar \Part_{y}) > 0\}
  & \leq \mathbb{P}_{x_{3:n}}\Bigl\{\|y - x\| \leq {r}(V_x) + {r}(V_y)\Bigr\} \\
  & \leq \mathbb{P}_{x_{3:n}}\Bigl\{
    \|y - x\| \leq 2{r}(V_x)
  \Bigr\} + \mathbb{P}_{x_{3:n}}\Bigl\{
    \|y - x\| \leq 2 {r}(V_y)
  \Bigr\} \\
  & \leq \mathbb{P}_{x_{3:n}}\Bigl\{
    \exists z:|B(z,\|x - y\|/4) \cap \{x_3,\ldots,x_n\}| = \emptyset
  \Bigr\} \\
  & \qquad\quad+ \mathbb{P}_{x_{3:n}}\Bigl\{
    \exists z:|B(z,\|x - y\|/4) \cap \{x_3,\ldots,x_n\}| = \emptyset
  \Bigr\} \\
  & \leq \begin{dcases}
    2, & \quad \textrm{if $\|x - y\|_2 \leq 2t_{n,\delta}$}, \\
    2\delta, & \quad\textrm{otherwise.}
  \end{dcases}
\end{align*}
Setting $\delta_n = n^{-(d + 1)/d}$ and plugging this back
into~\eqref{pf:dtv-ub-1}, we conclude that if $t_{n,\delta_n} < t_0$, then
\begin{align}
  \nonumber
  \E\Big[ \DTV \Big( f(x_{1:n}); \, \check{w}^\Vor\Big)\Big] 
  & \leq 2n(n - 1) \int_{\XDom} \int_{\XDom}
      |f(y) - f(x)|\Bigl(\1\{\|x - y\| \leq 2t_{n,\delta_n}\}
    + 2\delta_n\Bigr) \,dy \,dx \\
  \label{pf:dtv-ub-2}
  & \leq 2\mathbb{E}[\mathrm{DTV}(f;w^{\Eps\leftarrow t_{n,\delta}})]
    + 2n^{1 - 1/d} \int_{\XDom} \int_{\XDom} |f(y) - f(x)| \,dy \,dx.
\end{align}
Note that since $\lim_{n \to \infty} t_{n,\delta_n} = 0$, the condition
$t_{n,\delta_n} < t_0$ will automatically be satisfied for all $n$ sufficiently
large.

We now conclude the proof by upper bounding each term in~\eqref{pf:dtv-ub-2}.
The first term refers to the expected $\varepsilon$-neighborhood graph total
variation of $f$ when $\varepsilon = t_{n,\delta_n}$, and
by~\eqref{eqn:dtv-ub-neighborhood} satisfies
\begin{equation*}
  \mathbb{E}[\mathrm{DTV}_{n,t_{n,\delta}}(f)]
  \leq C n^2 (t_{n,\delta_n})^{d + 1} \mathrm{TV}(f;\XDom)
  \leq C n^{1 - 1/d}(\log n)^{(d + 1)/d}\mathrm{TV}(f;\XDom).
\end{equation*}
The second term above can be upper bounded using a Poincar\'{e} inequality for
$\mathrm{BV}(\XDom)$ functions, i.e.,
\begin{equation*}
  \int_{\XDom} \int_{\XDom} |f(y) - f(x)| \,dy \,dx
  \leq 2 \int_{\XDom} |f(x) - \wb{f}(x)| \,dx \leq C \mathrm{TV}(f;\XDom).
\end{equation*}
Plugging these upper bounds back into~\eqref{pf:dtv-ub-2} yields the claimed
result~\eqref{eq:voronoi_dtv_ub} in the unit weights case.
\qed

\subsubsection{Clipped weights}%
We now show~\eqref{eq:voronoi_dtv_ub} using clipped weights.  Our goal is to
upper bound
\begin{equation*}
  \E\Big[ \DTV \Big( f(x_{1:n}); \, \tilde{w}^\Vor\Big)\Big] 
  = n(n - 1) \mathbb{E}\Bigl[
    |f(x_1) - f(x_2)|
    \max\{
      c_0n^{-(d-1)/d}\one\{\cH^{d-1}(\bar \Part_1\cap\bar \Part_2) > 0\},
      \cH^{d-1}(\bar \Part_1\cap\bar \Part_2)
    \}
  \Bigr].
\end{equation*}
By conditioning, we may rewrite the expectation above as
\begin{equation}
  \label{pf:dtv-clipped-ub-1}
  p_\mx^2\int_{\XDom}\int_{\XDom}
    |f(y) - f(x)|
    \E_{x_{3:n}}\left[\max\{
      c_0n^{-(d-1)/d}\one\{\cH^{d-1}(\bar \Part_x\cap\bar \Part_y) > 0\},
      \cH^{d-1}(\bar \Part_x\cap\bar \Part_y)
    \}\right] 
    \,dy\,dx,
\end{equation}
where $V_{x} = \{z:  \|z - x\|_2 < \|z - x_i\|~\forall i = 2,3,\ldots,n\}$,
and likewise for $V_{y}$. Note that $V_x$ and $V_y$ are random subsets of
$\R^d$.  We now focus on controlling the inner expectation
of~\eqref{pf:dtv-clipped-ub-1}.  Upper bound the maximum of two positive
functions with their sum to obtain,
\begin{equation}
  \label{pf:dtv-clipped-ub-decompose-max}
  \begin{aligned}
    \E_{3:n}\Big[\max\{
      c_0n^{-(d-1)/d}
      &\one\{\cH^{d-1}(\bar \Part_x\cap\bar \Part_y) > 0\},
      \cH^{d-1}(\bar \Part_x\cap\bar \Part_y)
    \}\Big] 
    \\
    &\leq c_0n^{-(d-1)/d}\P\{\cH^{d-1}(\bar \Part_x\cap\bar \Part_y) > 0\}
    + \E\left[
      \cH^{d-1}(\bar \Part_x\cap\bar \Part_y)
    \right].
  \end{aligned}
\end{equation}
We recognize the first term on the RHS
of~\eqref{pf:dtv-clipped-ub-decompose-max} as having already been analyzed in
the unit weights case; we now focus on the second term.  The latter
``Voronoi kernel'' term may be rewritten,
\begin{equation*}
  \E_{x_{3:n}}\left[
    \cH^{d-1}(\bar \Part_x\cap\bar \Part_y)
  \right] 
  = \int_{L\cap\Omega} (1-p_x(z))^{n-2} dz,
\end{equation*}
where $L = \{z: \lVert x-z\rVert = \lVert y-z\rVert\}$ and
$p_x(z) = P(B(z, \lVert x-z\rVert))$.  Observe by
Assumption~\ref{assump:density_bounded} that
$p_x(z) \geq p_\mn\Leb_d \lVert x-z\rVert^d$, and therefore
\begin{equation*}
  \int_{L\cap\Omega} (1-p_x(z))^{n-2}
  \leq \exp(-cn \lVert x-z\rVert^d),
\end{equation*}
for some $c>0$.  Apply Lemma~\ref{eqn:surface-tension} with $a=2$ to therefore
bound,
\begin{equation}
  \label{pf:dtv-clipped-ub-voronoi-kernel}
  \E_{x_{3:n}}\left[
    \cH^{d-1}(\bar \Part_x\cap\bar \Part_y)
  \right] 
  \leq C_1\left(
    \frac{
      \one\{\lVert x-y\rVert \leq C_2(\log n/n)^{1/d}\}
    }{
      n^{(d-1)/d}
    } + \frac{1}{n^2} 
  \right),
\end{equation}
for constants $C_1, C_2 > 0$.  
Substitute~\eqref{pf:dtv-clipped-ub-voronoi-kernel}
into~\eqref{pf:dtv-clipped-ub-decompose-max}
and~\eqref{pf:dtv-clipped-ub-1} to obtain,
\begin{align*}
  \E\Big[ &\DTV \left( f(x_{1:n}); \, \tilde{w}\right)\Big] 
  \\
  &\leq 
  p_\mx^2 n^2 \int_\XDom\int_\XDom
  |f(y) - f(x)|\Bigg(
    c_0n^{-(d-1)/d}\P_{3:n}\{\cH(\bar \Part_x\cap\bar \Part_y)>0\}
    \\
    &\hspace{8cm}
    + C_1 \frac{
      \one\{\lVert x-y\rVert \leq C_2(\log n/n)^{1/d}\}
    }{
      n^{(d-1)/d}
    }
    + \frac{C_1}{n^2} 
  \Bigg) \;dy\; dx
  \\
  &\leq
  p_\mx^2c_0n^{-(d-1)/d} \E\left[
    \DTV(f(x_{1:n}); \check{w}^\Vor)
  \right]
  + p_\mx^2C_1 n^{-(d-1)/d} \E\left[
    \DTV(f(x_{1:n}); w^{\Eps\leftarrow C_2(\log n/n)^{1/d}})
  \right]
  \\
  &\hspace{9cm} + p_\mx^2C_1 \int_\XDom\int_\XDom |f(y)-f(x)|\;dy\;dx
  \\
  &= T_1 + T_2 + T_3.
\end{align*}
We bound each of the terms above in turn.  The first term appeals
to~\eqref{eq:voronoi_dtv_ub} in the unit weights case, which we have already
proved.
\begin{align*}
  T_1
  &=
  p_\mx^2c_0n^{-(d-1)/d} \E\left[
    \DTV(f(x_{1:n}); \check{w}^\Vor)
  \right]
  \\
  &\leq C_3 n^{-(d-1)/d} n^{(d-1)/d}(\log n)^{1+1/d} \TV(f)
  \\
  &= C_3(\log n)^{1+1/d}\TV(f).
\end{align*}
The second term refers to the expected $\varepsilon$-neighborhood
graph total variation of $f$ when $\varepsilon=C_2(\log n/n)^{1/d}$,
which by~\eqref{eqn:dtv-ub-neighborhood} satisfies,
\begin{align*}
  T_2
  &= p_\mx^2C_1 n^{-(d-1)/d} \E\left[
    \DTV\Big(f(x_{1:n});w^{\Eps\leftarrow C_2(\log n/n)^{1/d}}\Big)
  \right]
  \\
  &\leq C_4 n^{-(d-1)/d} n^2 (\log n/n)^{(d+1)/d} \TV(f)
  \\
  &\leq C_4(\log n)^{1+1/d}\TV(f).
\end{align*}
The third term can be controlled via the Poincar\'{e} inequality,
\begin{align*}
  T_3
  &= p_\mx^2 C_1 \int_\XDom\int_\XDom |f(y) - \bar f + \bar f -f(x)|\;dy\;dx
  \\
  &\leq C_5 \int_\XDom | f(x) - \bar f|\;dx
  \\
  &\leq C_5\TV(f),
\end{align*}
where $\bar f := \AvInt_\XDom f$.
\qed

\subsubsection{\texorpdfstring{$\varepsilon$}{eps}-neighborhood and kNN expected discrete TV}%
\input{app_edtv_nbd_knn}

\subsection{Proof of Lemma~\ref{lem:voronoi_extrapolate_ub}} 
Recall that $\lVert g\rVert_{L^2(P)} \leq p_\mx \lVert g\rVert_{L^2(\Leb)}$
for any $g\in L^2(\Leb)$ and 
note that $\|\bar{f_0}\|_{L^{\infty}(\Leb)} \leq M$ with probability one. By
H\"{o}lder's inequality, 
\begin{equation}
  \label{pf:1nn-risk-0}
  \begin{aligned}
    \mathbb{E} \|\bar{f_0} - f_0\|_{L^2(\Leb)}^2
      & \leq \mathbb{E}\Bigl[\|\bar{f_0} - f_0\|_{L^1(\Leb)}
        \cdot \|\bar{f_0} - f_0\|_{L^{\infty}(\Leb)}\Bigr] \\
      & \leq 2M~\mathbb{E}\|\bar{f_0} - f_0\|_{L^1(\Leb)},
  \end{aligned}
\end{equation}
and the problem is reduced to upper bounding the expected $L^1(\Leb)$ loss of
$\bar{f_0}$. By Fubini's Theorem we may exchange expectation with integral, giving
\begin{align}
  \label{pf:1nn-risk-1}
  \mathbb{E}\|\bar{f_0} - f_0\|_{L^1(\Leb)}
    & = \int_{\XDom} \mathbb{E}|\bar{f_0}(x) - f_0(x)| \,dx \nonumber \\
    & = \int_{\XDom} \int_{\XDom} |f_0(y) - f_0(x)| p_x^{(1)}(y) \,dy \,dx,
\end{align}
where $p_x^{(1)}(\cdot)$ is the density of $x_{(1)}(x)$. We now give a closed
form expression for this density, before proceeding to lower
bound~\eqref{pf:1nn-risk-1}.

\paragraph{Closed-form expression for $p_x^{(1)}$.}
Suppose $P$ satisfies Assumption~\ref{assump:density_bounded}.
For any $y \in \XDom$ and 
$0 < r < \mathrm{dist}(y,\partial\XDom)$, we have
\begin{align*}
  \mathbb{P}\bigl\{x_{(1)}(x) \in B(y,r)\bigr\}
    &\leq n~\mathbb{P}\bigl\{x_1 \in B(y,r)\bigr\}
    \bigl(\mathbb{P}\{x_2\not\in B(x, \lVert y-x\rVert\}\bigr)^{(n - 1)} \\
    &\leq n p_\mx\Leb\bigl(B(y,r)\bigr)
      \Bigl(1 - P\bigl(B(x,\|y - x\|)\bigr)\Bigr)^{(n - 1)}.
\end{align*}
Taking limits as $r \to 0$ gives
\begin{equation*}
  p_x^{(1)}(y)
    = \lim_{r \to 0} \frac{
      \mathbb{P}\bigl\{x_{(1)}(x) \in B(y,r)\bigr\}
    }{
      \Leb(B(y,r))
    }
    = np_\mx\Bigl(
      1 - P\bigl(B(x,\|y - x\|)\bigr)
    \Bigr)^{(n - 1)}.
\end{equation*}

\paragraph{Upper bound on~\eqref{pf:1nn-risk-1}.}
There exists a constant $C_d$ such that for all $x,y \in \XDom$, 
\begin{equation*}
  P\bigl(B(x,\|y - x\|)\bigr)
    \geq \frac{p_\mn}{C_d} \Leb(B(x,\|y - x\|))
    = \frac{p_\mn\Leb_d}{C_d} \|y - x\|^{d}.
\end{equation*}
This implies an upper bound on the density of $x_{(1)}(x)$,
\begin{align*}
  p_x^{(1)}(y)
    & \leq n \biggl(
      1 - \frac{p_\mn\Leb_d}{C_d} \|y - x\|^{d}
    \biggr)^{(n - 1)} \\
    & \leq n \exp\biggl(
      -\frac{p_\mn\Leb_d}{C_d}\Bigl(\frac{\|y - x\|}{n^{-1/d}}\Bigr)^d
    \biggr),
\end{align*}
where we have used the inequality $(1 - x)^{n} \leq \exp(-nx)$ for $|x| \leq
1$.  Using the inequality, valid for all monotone non-increasing functions
$g:[0,\infty) \to [0,\infty)$, that $g(t) \leq \1\{t \leq t_0\}g(0) + g(t_0)$,
we further conclude that
\begin{equation*}
  p_x^{(1)}(y)
  \leq n \1\{\|y - x\| \leq \varepsilon_n^{(1)}\} + \frac{1}{n},
\end{equation*}
for $\varepsilon_n^{(1)} := (\frac{2C_d}{p_\mn\Leb_d} (\log n/n))^{1/d}$.
Plugging back into~\eqref{pf:1nn-risk-1}, we see that the expected
$L^1(\Leb)$ error is upper bounded by the expected discrete TV of a
neighborhood graph with particular kernel and radius, plus a remainder term.
Specifically,
\begin{align}
\nonumber
\mathbb{E}\|\bar{f_0} - f_0\|_{L^1(\Leb)}
  & \leq n \int_{\XDom} \int_{\XDom}
      |f_0(y) - f_0(x)| \1\{\|y - x\| \leq \varepsilon_n^{(1)}\}
    \,dy \,dx
    + \frac{1}{n} \int_{\XDom} \int_{\XDom}
      |f_0(y) - f_0(x)|
    \,dy \,dx \\
  \label{pf:1nn-risk-2}
  & \leq n \int_{\XDom} \int_{\XDom}
      |f_0(y) - f_0(x)| \1\{\|y - x\| \leq \varepsilon_n^{(1)}\}
    \,dy \,dx
    + \frac{C~\mathrm{TV}(f_0;\XDom)}{n} \\
  \nonumber
  & = \frac{1}{n} \mathrm{E}[
    \DTV(f_0; w^{\Eps\leftarrow \varepsilon_n^{(1)}}))
    ] + \frac{C~\TV(f_0;\XDom)}{n},
\end{align}
where \eqref{pf:1nn-risk-2}~above follows from the Poincar\'{e}
inequality~\eqref{eqn:poincare}. We can therefore
apply~\eqref{eqn:dtv-ub-neighborhood}, which upper bounds the expected
$\varepsilon$-neighborhood graph TV, and conclude that 
\begin{equation*}
  \mathbb{E}\|\bar{f_0} - f_0\|_{L^1(\Leb)}
    \leq C\biggl(\frac{(\log n)^{1 + 1/d}}{n^{1/d}}
      + \frac{1}{n}\biggr)\mathrm{TV}(f_0;\XDom)
    \leq C\biggl(\frac{L(\log n)^{1 + 1/d}}{n^{1/d}}\biggr).
\end{equation*}
Inserting this upper bound into~\eqref{pf:1nn-risk-0} completes the proof of
Lemma~\ref{lem:voronoi_extrapolate_ub}.
\qed

\subsection{Proof of Theorem~\ref{thm:tvd_eps_knn_ub}}%
\label{sub:proof_of_theorem_thm_tvd_eps_knn_ub}
The analysis of the $\varepsilon$-neighborhood and kNN TV denoising estimators
proceeds identically, so we consider them together.  Henceforth let $D$
denote the penalty operator for either estimator and $\hf$ denote their
1NN extrapolants.  Follow the proof of Theorem~\ref{thm:voronoi_ub} (given in
Appendix~\ref{app:pf-thm-voronoi-ub}) to decompose the $L^2(P)$ error for some
$C>0$,
\begin{equation}
  \label{eq:eps-knn-tvd-l2p-decomp}
  \E\left[
    \lVert \hf-f_0\rVert_{L^2(P)}^2
  \right] 
  \leq C\left(
    \frac{
      \lambda\log n \;\E \lVert D\theta_0\rVert
    }{n} 
    + \frac{
      (\log n)^{1+\alpha}
    }{n} 
    + \frac{
      LM(\log n)^{1+1/d}
    }{n^{1/d}} 
  \right),
\end{equation}
where we have applied Lemma~\ref{lem:voronoi_extrapolate_ub} which controls
the 1NN extrapolation error.  Lemma~\ref{lem:dtv-ub-eps-k} provides that
under the standard assumptions, there exist constants $C_1, C_1'>0$
such that for all sufficiently large $n$ and $\theta_0 = f_0(x_{1:n})$,
$f_0\in\BV(\XDom)$,
\begin{itemize}
  \item setting $\varepsilon = c_1(\log^\alpha n/n)^{1/d}$,
    \begin{equation}
      \label{pf:eps-dtv-ub}
      \E \lVert D^\Eps\theta_0\rVert_1
      \leq C_1n^{(d-1)/d}(\log n)^{\alpha+\alpha/d} \TV(f_0);
    \end{equation}
  \item setting $k = c_1'(\log n)^3$,
    \begin{equation}
      \label{pf:knn-dtv-ub}
      \E \lVert D^\kNN\theta_0\rVert_1
      \leq C_1' n^{(d-1)/d}(\log n)^{3+3/d}\TV(f_0).
    \end{equation}
\end{itemize}
Take these values of $\varepsilon, k$ and
$\lambda=c\sigma(\log n)^{1/2-\alpha}$, $c=c_2, c_2'$, and substitute
\eqref{pf:eps-dtv-ub},~\eqref{pf:knn-dtv-ub}
into~\eqref{eq:eps-knn-tvd-l2p-decomp} to obtain the claim.  
\qed

Note that the $L^2(P_n)$ in-sample error may be obtained similarly, beginning
with an analysis identical to that of Lemma~\ref{lem:voronoi_empirical_dtv_ub}
to obtain the preliminary upper bound,
\begin{equation*}
  \E\left[
    \lVert \hat f - f_0\rVert_{L^2(P_n)}^2
  \right]
  \leq C\left(
    \frac{\lambda\;\E \lVert D\theta_0\rVert}{n} 
    + \frac{(\log n)^{1+\alpha}}{n} 
  \right).
\end{equation*}

\subsection{Proof of Theorem~\ref{thm:wavelet_bd}}
\label{subsec:pf-minimax-rates-bv-ub-wavelet}

In this section we prove the upper bound~\eqref{eq:wavelet_bd}. The proof
is comprised of several steps and we start by giving a high-level summary.
\begin{itemize}
  \item We begin in Section~\ref{subsubsec:wavelet} by formalizing the
    estimator $\wh{f}_{\mathrm{wave}}$ alluded to in
    Theorem~\ref{thm:wavelet_bd}, based on hard thresholding of Haar
    wavelet empirical coefficients. 
  \item Section~\ref{subsubsec:wavelet-decay} reviews wavelet coefficient decay
    of $\mathrm{BV}(\XDom)$ and $L^{\infty}(\XDom)$ functions. These rates of
    decay imply that the wavelet coefficients of $f_0 \in
    \mathrm{BV}_{\infty}(L,M)$ must belong to the normed balls in a pair of
    Besov bodies, defined formally in~\eqref{eqn:besov-body_inf}. Besov bodies
    are sequence-based spaces that reflect the wavelet coefficient decay of
    functions in Besov spaces. 
  \item Section~\ref{subsubsec:wavelet-deterministic-ub} gives a deterministic
    upper bound on the squared-$\ell^2$ error of thresholding wavelet
    coefficients when the population-level coefficients belong to intersections
    of Besov bodies. This deterministic upper bound is based on analyzing two
    functionals---a modulus of continuity~\eqref{eqn:modulus-continuity} and
    the tail width~\eqref{eqn:tail-width}--- in the spirit
    of~\citep{donoho1995wavelet}; the difference is that we are
    considering~\emph{intersections} of Besov bodies. 
  \item The aforementioned modulus of continuity measures the size of the
    $\ell^2$-norm $\|\theta - \theta'\|_2$ relative to $\ell^{\infty}$-norm
    $\|\theta - \theta'\|_{\infty}$.  In
    Section~\ref{subsubsec:wavelet-convergence}, we give an upper bound on the
    $\ell^{\infty}$ norm of the difference between sample and population-level
    wavelet coefficients. 
  \item Finally, in Section~\ref{subsubsec:pf-minimax-rates-bv-ub} we combine
    the results of Sections~\ref{subsubsec:wavelet-deterministic-ub}
    and~\ref{subsubsec:wavelet-convergence} to establish upper bounds on the
    expected squared-$\ell^2$ error of hard thresholding sample wavelet
    coefficients. The same upper bound will apply to the expected
    squared-$L^2(\XDom)$ error of $\wh{f}_{\mathrm{wave}}$, by Parseval's
    theorem.
\end{itemize} 

\subsubsection{Step 1: Hard-thresholding of wavelet coefficients}
\label{subsubsec:wavelet}
To define the estimator $\wh{f}_{\mathrm{wave}}$ that achieves the upper bound
in~\eqref{eq:wavelet_bd}, we first review the definition of tensor product
Haar wavelets.
\begin{definition}[Haar wavelet]
	The \emph{Haar wavelet} $\psi:(0,1) \to \Reals$ is defined by
	\begin{equation}
	\label{eqn:haar}
	\psi(x) := \1\{x \in (0,1/2]\} - \1\{x \in (1/2,1)\}.
	\end{equation}
  For each ${\bf i} \in \{0,1\}^d \setminus \{(0,\ldots,0)\}$, the \emph{tensor
  product Haar wavelet} $\Psi^{\bf i}:(0,1)^d \to \Reals$ is defined by
	\begin{equation}
	\label{eqn:haar-tensor}
	\Psi^{i}(x) := \psi^{i_1}(x_1)\ldots\psi^{i_d}(x_d),
	\end{equation}
  where $\psi^{1}(x) = \psi(x)$ and $\psi^{0}(x) = 1$. 
  To ease notation, let $\mc{I} = \{0,1\}^d \setminus \{(0,\ldots,0)\}$ and
  $\mc{K}(\ell) = [2^{\ell} - 1]^d$. For each $\ell \in \mathbb{N} \cup \{0\},
  k \in \mc{K}(\ell)$ and ${\bf i} \in \mc{I}$, put $\Psi_{\ell k}^{{\bf i}}(x)
  := 2^{\ell d/2} \Psi^{{\bf i}}(2^{\ell}x - k)$. Finally, let $\Phi(x) = \1\{x
  \in (0,1)^d\}$. The Haar wavelet basis is the collection $\{\Psi_{\ell
  k}^{{\bf i}}: \ell \in \mathbb{N}, k \in \mc{K}(\ell), {\bf i} \in \mc{I}\}
  \cup \{\Phi\}$, and it forms an orthonormal basis of $L^2((0,1)^d)$.
\end{definition}
We now describe the estimator $\wh{f}_{\mathrm{wave}}$, which applies hard
thresholding to sample wavelet coefficients. For each $\ell \in \mathbb{N} \cup
\{0\}, k \in \mc{K}(\ell)$ and ${\bf i} \in \mc{I}$, write 
\begin{equation*}
\theta_{\ell k {\bf i}}(f) := \int_{\XDom} \Psi_{\ell k}^{{\bf i}}(x) f(x)
\,dx, \quad  \wt{\theta}_{\ell k {\bf i}}(f) := \frac{1}{n}\sum_{j = 1}^{n}
f(x_j) \Psi_{\ell k}^{{\bf i}}(x_j),
\end{equation*}
for the population-level and empirical wavelet coefficients of a given $f \in
L^2(\XDom)$. The sample wavelet coefficient is $\wt{\theta}_{\ell k {\bf
i}}(y_{1:n})$. The hard thresholding estimator we use is defined with respect
to a threshold $\lambda > 0$ and a truncation level $\ell^{\ast} \in \mathbb{N}
\cup \{0\}$ as
\begin{equation}
\label{eqn:hard-thresholding}
\wh{\theta}_{\ell k {\bf i}}^{(\lambda, \ell^{\ast})} :=
\begin{dcases}
  \wt{\theta}_{\ell k {\bf i}}(y_{1:n}) \cdot \1\{\wt{\theta}_{\ell k {\bf
  i}}(y_{1:n}) \geq \lambda\},
  & \textrm{$\ell = 0,\ldots,\ell^{\ast}$} \\
  0,
  & \textrm{$\ell \geq \log_2(n)/d + 1,$}
\end{dcases}
\end{equation}
and we map the sequence estimate $\wh{\theta}^{(\lambda,\ell^{\ast})} =
\bigl(\wh{\theta}_{\ell k {\bf i}}^{(\lambda,\ell^{\ast})}: \ell \in
\mathbb{N}, k \in \mc{K}(\ell), {\bf i} \in \mc{I} \bigr)$ to the function
\begin{equation}
\label{eqn:hard-thresholding-estimate}
\wh{f}^{(\lambda,\ell^{\ast})}(x) = \wb{y} + \sum_{\ell \in \mathbb{N}}
\sum_{k \in \mc{K}(\ell), {\bf i} \in \mc{I}}
\wh{\theta}_{\ell k {\bf i}}^{(\lambda,\ell^{\ast})} \Psi_{\ell k}^{{\bf i}}(x),
\end{equation}
where $\wb{y} = \frac{1}{n}\sum_{i = 1}^{n}y_i$ is the sample average of the
responses.~\eqref{eqn:hard-thresholding-estimate} defines a family of
estimators depending on the threshold $\lambda$, and the estimator
$\wh{f}_{\mathrm{wave}}$ is the hard thresholding estimate
$\wh{f}^{(\lambda,\ell^{\ast})}$ with the specific choices $\lambda = 8
n^{-1/2} \log^{3/2}(2n/\delta)$ and $\ell^{\ast} = \log_2(n)/d$.

\subsubsection{Step 2: Wavelet decay}
\label{subsubsec:wavelet-decay}

In this section we recall the wavelet coefficient decay of functions in
$\mathrm{BV}(\XDom)$ and in $L^{\infty}(\XDom)$. For each $\ell \in \mathbb{N}
\cup \{0\}$, define
\begin{equation*}
\theta_{\ell \cdot}(f) = \bigl(\theta_{\ell k}^{{{\bf i}}}(f):
k \in \mc{K}(\ell), {\bf i} \in \mc{I}\bigr).
\end{equation*}
and write $\theta(f)$ for the vector with entries $\theta(f)_{\ell} :=
\theta_{\ell \cdot}(f)$.
\begin{lemma}
	\label{lem:linf-wavelet}
	Let $f \in L^\infty(\XDom)$. Then for all $\ell \in \mathbb{N} \cup \{0\}$,
	\begin{equation}
	\label{eqn:linf-wavelet}
	\|\theta_{\ell \cdot}(f)\|_{\infty} \leq 2^{-\ell d/2} \|f\|_{L^\infty(\XDom)}. 
	\end{equation}
\end{lemma}
\begin{lemma}
	\label{lem:bv-wavelet}
  There exists a constant $C_1$ such that for all $f \in \mathrm{BV}(\XDom)$
  and $\ell \in \mathbb{N} \cup \{0\}$,
	\begin{equation}
	\label{eqn:bv-wavelet}
	\|\theta_{\ell \cdot}(f)\|_{1} \leq C_1 2^{-\ell(1 - d/2)} \mathrm{TV}(f;\XDom). 
	\end{equation}
\end{lemma}
The decay rates established by Lemmas~\ref{lem:linf-wavelet}
and~\ref{lem:bv-wavelet} imply that if $f_0 \in \mathrm{BV}_{\infty}(L,M)$,
then $\theta(f_0)$ belongs to $\Theta_\infty^{0,\infty}(M)$ and
$\Theta_{\infty}^{1,1}(L)$, where $\Theta_\infty^{s,p}(C)$ consists of
sequences $\theta$ for which
\begin{equation}
\label{eqn:besov-body_inf}
\|\theta\|_{\Theta_{\infty}^{s,p}} := \sup_{\ell \in \mathbb{N} \cup \{0\}}
2^{\ell(s + d/2 - d/p)} \|\theta_{\ell \cdot}\|_{p} < C.
\end{equation}
The sets $\Theta_\infty^{s,p}(C)$ can be interpreted as normed balls in Besov
bodies, since a function $f$ belongs to the Besov space $B_{\infty}^{s,p}$ if
and only if its coefficients in a suitable wavelet basis satisfy
$\|\theta(f)\|_{\Theta_{\infty}^{s,p}} < \infty$.

The conclusions of Lemmas~\ref{lem:linf-wavelet} and~\ref{lem:bv-wavelet} are
generally well-understood (see for instance
\citet{gine2021mathematical}
for the upper bound on wavelet decay of $L^{\infty}(\XDom)$ functions when $d =
1$, and \citet{cohen2003harmonic} for the wavelet decay of
$\mathrm{BV}(\XDom)$ functions). For purposes of completeness only, we include
proofs of these results in
Appendix~\ref{ssub:proofs_of_lemmas_linf_wavelet_bv_wavelet}.

\subsubsection{Step 3: Deterministic upper bound on \texorpdfstring{$\ell^2$}{ell2}-error}
\label{subsubsec:wavelet-deterministic-ub}
In this section, we analyze the $\ell^2$-error of the hard-thresholding
estimator $\wh{\theta}^{(\lambda,\ell^{\ast})}$. Specifically, we upper bound
the magnitude of $\|\wh{\theta}^{(\lambda,\ell^{\ast})} - \theta(f_0)\|_2$ as a
function of  the $\ell^{\infty}$ distance between the (truncated) sample and
population-level wavelet coefficients, i.e the quantity 
\begin{equation*}
\epsilon_n := \|(\wt{\theta}(y_{1:n}) - \theta(f_0))_{\leq \ell^{\ast}}\|_{\infty},
\end{equation*}
where
\begin{equation*}
(\theta_{\leq \ell^{\ast}})_{\ell k}^{{\bf i}} := 
\begin{dcases*}
\theta_{\ell k}^{{\bf i}},& \quad \textrm{if $\ell \leq \ell^{\ast}$,} \\
0,& \quad \textrm{otherwise.}
\end{dcases*}
\end{equation*}
Note that this upper bound is purely deterministic.
\begin{proposition}
	\label{prop:wavelet-estimation-error-ub}
  Suppose $\theta(f_0) \in \Theta_{\infty}^{0,\infty}(M) \cap
  \Theta_{\infty}^{1,1}(L)$.
  Then there exists a constant $C_3$ that does not depend on $n,M$ or $L$ for
  which the following statement holds: if $\lambda \geq 2 \epsilon_n$, then the
  estimator $\wh{\theta}^{(\lambda,\ell^{\ast})}$
  of~\eqref{eqn:hard-thresholding-estimate} satisfies the upper bound
	\begin{equation}
	\label{eqn:wavelet-estimation-error-ub}
	\|\wh{\theta}^{(\lambda,\ell^{\ast})} - \theta(f_0)\|_2^2
  \leq 4 C_1 L M 2^{-\ell^{\ast}} + C_3 \cdot 
	\begin{dcases*}
	L \lambda \max\{1,1/M,\log_2(M/\lambda)\},& \quad \textrm{if $d = 2$} \\
	L^{2/d}\lambda^{4/(2 + d)} + LM\Bigl(\frac{\lambda}{M}\Bigr)^{2/d},
                                            & \quad \textrm{if $d \geq 3$.}
	\end{dcases*}
	\end{equation}
\end{proposition}
The proof is deferred to
Appendix~\ref{ssub:proof_prop_wavelet_estimation_error_ub}.

\subsubsection{Step 4: Uniform convergence of wavelet coefficients}
\label{subsubsec:wavelet-convergence}
Lemma~\ref{lem:wavelet-convergence} gives an upper bound on the maximum
difference between sample and population-level wavelet coefficients that holds
uniformly over all $\ell = 0,\ldots,\log_2(n)/d$, $k \in \mc{K}(\ell)$ and
${\bf i} \in \mc{I}$.  Its proof is deferred to
Appendix~\ref{ssub:proof_lemma_wavelet_convergence}.
\begin{lemma}
	\label{lem:wavelet-convergence}
  Suppose we observe data $(x_1,y_1),\ldots,(x_n,y_n)$ according
  to~\eqref{eq:model}, where $f_0 \in L^{\infty}(\XDom;M)$. There exists a
  constant $C_4$ not depending on $n$ such that the following statement holds
  for all $\delta > 0$: with probability at least $1 - C_4\delta$, 
	\begin{equation}
	\label{eqn:wavelet-convergence}
	\|(\wt{\theta}(y_{1:n}) - \theta(f_0))_{\leq \log_2(n)/d}\|_{\infty}
  \leq \underbrace{\frac{4 \log^{3/2}(2n/\delta)}{\sqrt{n}}
  + \frac{\sqrt{12} M \sqrt{\log(2n/\delta)}}{\sqrt{n}}}_{:=\delta_n}.
	\end{equation}
\end{lemma} 

\subsubsection{Step 5: Upper bound on risk}
\label{subsubsec:pf-minimax-rates-bv-ub}
We are now ready to prove the stated upper bound~\eqref{eq:wavelet_bd}. In
this section we take $\lambda = 2 \delta_n$ and $\ell^{\ast} = \log_2(n)/d$.
Combining Proposition~\ref{prop:wavelet-estimation-error-ub} and
Lemma~\ref{lem:wavelet-convergence}, we have that with probability at least $1
- C_4\delta$,
\begin{equation}
\label{pf:minimax-rates-bv-ub-1}
\|\wh{\theta}^{(\lambda,\ell^{\ast})} - \theta_0(f)\|_2^2
\leq \frac{4 C_1 L M}{n^{1/d}}  + C_3 \cdot 
\begin{dcases*}
2 L \delta_n \max\{1,1/M,\log_2(M/2 \delta_n)\},
& \quad \textrm{if $d = 2$,} \\
L^{2/d}(2 \delta_n)^{4/(2 + d)} + LM\Bigl(\frac{2 \delta_n}{M}\Bigr)^{2/d},
& \quad \textrm{if $d \geq 3$.}
\end{dcases*}
\end{equation}
The following lemma allows us to convert this upper bound, which holds with
probability $1 - C_4\delta$, to an upper bound which holds in expectation.
Its proof is deferred to
Appendix~\ref{ssub:proof_lem_prob_to_expectation}.
\begin{lemma}
	\label{lem:probability-to-expectation}
  Let $X > 0$ be a positive random variable. Suppose there exist positive
  numbers $A_1,\ldots,A_K$, $a_1,\ldots,a_K$, $b_1,\ldots,b_K > 1$ and $B$ such
  that for all $\delta \in (0,1)$,
	\begin{equation*}
	\mathbb{P}\Bigl(X > \sum_{k = 1}^{K}A_k \log^{a_k}(b_k/\delta)\Bigr)
  \leq B\delta.
	\end{equation*}
  Then there exists a constant $C_5$ depending only on $a_1,\ldots,a_k$ and $B$
  such that
	\begin{equation*}
	\mathbb{E}[X] \leq C_5 \sum_{k = 1}^{K} A_k (\log b_k)^{a_k}.
	\end{equation*}
\end{lemma}
Now we use Lemma~\ref{lem:probability-to-expectation} to complete the proof of
Theorem~\ref{thm:wavelet_bd}. Note that for any $a > 0$,
\begin{equation*}
\delta_n^{a} \leq \frac{2^a}{\sqrt{n}}\Bigl((\log(2n/\delta))^{3a/2}
+ M \sqrt{\log(2n/\delta)}\Bigr).
\end{equation*}
Thus we may can apply Lemma~\ref{lem:probability-to-expectation}
to~\eqref{pf:minimax-rates-bv-ub-1}, which, setting $\delta_n^{\ast} = \frac{4}{\sqrt n}\Bigl((\log 2n)^{3/2} + M \log(n)^{1/2}\Bigr)$, gives
\begin{equation}
\label{pf:minimax-rates-bv-ub-2}
\mathbb{E} \|\wh{\theta}^{(\lambda,\ell^{\ast})} - \theta_0(f)\|_2^2
\leq \frac{4 C_1 L M}{n^{1/d}}  + C_6 \cdot 
\begin{dcases*}
2 L \delta_n^{\ast} \max\{1,1/M,\log_2(M \sqrt{n})\},&
\quad \textrm{if $d = 2$,} \\
L^{2/d}(2 \delta_n^{\ast})^{4/(2 + d)}
+ LM\Bigl(\frac{2 \delta_n^{\ast}}{M}\Bigr)^{2/d}, & \quad \textrm{if $d \geq 3$.}
\end{dcases*} 
\end{equation}
where $C_6 = 2 C_3 C_5$. 

Finally, we translate this to an upper bound on the expected risk of
$\wh{f}_{\mathrm{wave}} = \wh{f}^{(\lambda,\ell^{\ast})}$. Since $\{\Psi_{\ell
k}^{{\bf i}}, \ell \in \mathbb{N} \cup \{0\}, k \in \mc{K}(\ell), {\bf i} \in
\mc{I}\} \cup \{\Phi\}$ forms an orthonormal basis of $L^2(\XDom)$, by
Parseval's theorem we have
\begin{equation*}
\|\wh{f}^{(\lambda,\ell^{\ast})} - f_0\|_{L^2(\XDom)}^2
= (\wb{y} - \mathbb{E}[f_0])^2
+ \|\wh{\theta}^{(\lambda,\ell^{\ast})} - \theta(f_0)\|_2^2.
\end{equation*}
Taking expectation on both sides, the claimed upper
bound~\eqref{eq:wavelet_bd} follows by~\eqref{pf:minimax-rates-bv-ub-2},
upon proper choice of constant $C = \max\{4C_1,16C_6\}$. 
\qed

%% file: app_edtv_nbd_knn.tex
\begin{lemma}
  \label{lem:dtv-ub-eps-k}
  Under Assumption~\ref{assump:density_bounded}, there exist constants
  $c,C_1,C_2>0$ such that for all sufficiently large $n$ and
  $f_0\in\BV(\XDom)$,
	\begin{itemize}
    \item The $\varepsilon$-neighborhood graph total variation, for
      any $\varepsilon > 0$, satisfies
      \begin{equation}
        \label{eqn:dtv-ub-neighborhood}
        \E\Big[
          \DTV \Big(
            f_0(x_{1:n}; \; w^\Eps)
          \Big)
        \Big] 
        \leq C_1 n^2 \varepsilon^{d + 1} \TV(f_0).
      \end{equation}
		\item The $k$-nearest neighbors graph total variation,
      for any $k \in \mathbb{N}$, satisfies
      \begin{equation}
        \label{eqn:dtv-ub-knn}
        \E\Big[
          \DTV \Big(
            f_0(x_{1:n}; \; w^\kNN)
          \Big)
        \Big] 
        \leq C_2\left(
          n^{1-1/d} k^{(d+1)/d}
          + n^2\exp(-ck)
        \right) \TV(f_0).
      \end{equation}
	\end{itemize}
\end{lemma}

\begin{proof}~\vspace{-0.3cm}
  \paragraph{$\varepsilon$-neighborhood expected discrete TV.}
	This follows the proof of Lemma~1 in \cite{green2021minimax1}, with two
	adaptations to move from Sobolev $H^2(\XDom)$ to the $\BV(\XDom)$: we deal in
	absolute differences rather than squared differences, and an approximation
	argument is invoked at the end to account for the existence of non-weakly
	differentiable functions in $\BV(\XDom)$.

	Begin by rewriting,
	\begin{equation}
		\E\left[
		\sum_{i,j=1}^n |f(x_i) - f(x_j)|
		\cdot \one\{\lVert x_i-x_j\rVert\leq \varepsilon\}
		\right]
		= \frac{n(n-1)}{2}\E\left[
				|f(X') - f(X)| K\left(\frac{\lVert X'-X\rVert}{\varepsilon}\right)
			\right],
	\end{equation}
	where $X$ and $X'$ are random variables independently drawn from $P$
	following Assumption~\ref{assump:density_bounded} and
	$K(t) = \one\{t\leq 1\}$.  Now,
	take $\XDom'$ to be an arbitrary bounded open set such that
	$B(x, c_0) \subseteq \XDom'$ for all $x \in \XDom$.

	For the remainder of this proof, we assume that (i) $f \in BV(\XDom')$ and
	(ii) $\lVert f\rVert_{BV(\XDom')} \leq C'\lVert f\rVert_{BV(\XDom)}$ for some
	constant $C'$ independent of $f$.  These conditions are guaranteed by the
	Extension Theorem
	(\citealp{evans2010partial}; Section 5.4 Theorem 1), which promises an
	extension operator $E:W^{1, p}(\XDom) \rightarrow W^{1, p}(\XDom)$ (take
	$p=1$ and the BV case is established through an approximation argument).  We
	also assume that $f\in C^\infty(\XDom)$, which is addressed through via an
	approximation argument at the end.  Since $f\in C^\infty(\XDom)$, we may
	rewrite a difference in terms of an integrated derivative:
	\begin{equation}
		\label{eq:first-order-taylor-appx}
		f(x') - f(x) = \int_0^1 \nabla f(x + t(x'-x))^\top(x'-x)dx.
	\end{equation}
	It follows that
	\begin{equation}
		\E\left[
			|f(X') - f(X)| K\left(
				\frac{\lVert X'-X\rVert}{\varepsilon}
			\right) \leq
			p_\text{max}^2\int_\XDom\int_\XDom
			|f(x') - f(x)| K\left(
				\frac{\lVert x'-x\rVert}{\varepsilon}
			\right)
			dx' dx
		\right] ,
	\end{equation}
	and the final step is to bound the double integral.  We have
	\begin{align*}
		\int_\XDom\int_\XDom
			&|f(x') - f(x)| K\left(
				\frac{\lVert x'-x\rVert}{\varepsilon}
			\right)
			dx' dx
		\\
		&=  \int_\XDom\int_\XDom
		\Big|\int_0^1 \nabla f(x+t(x'-x))^\top(x'-x)dt\Big|K\left(
			\frac{\lVert x'-x\rVert}{\varepsilon} 
		\right) dx' dx
		&& \text{(FTC)}
		\\
		&\leq \int_\XDom\int_\XDom
			\int_0^1 |\nabla f(x+t(x'-x))^\top(x'-x)|K\left(
			\frac{\lVert x'-x\rVert}{\varepsilon} 
		\right)dt dx' dx
		&& \text{(Jensen)}
		\\
		&=  \int_\XDom\int_{B(0, 1)}\int_0^1
			|\nabla f(x+t\varepsilon z)^\top (\varepsilon z)|
			K(\lVert z\rVert) \varepsilon^ddtdzdx
		&& \text{($z = (x'-x)/\varepsilon$)}
		\\
		&=  \varepsilon^{d+1}\int_\XDom\int_{B(0, 1)}\int_0^1
			|\nabla f(x+t\varepsilon z)^\top z| K(\lVert z\rVert) dtdzdx
		\\
		&\leq \varepsilon^{d+1}\int_{\XDom'}\int_{B(0, 1)}\int_0^1
			|\nabla f(\tilde x)^\top z|K(\lVert z\rVert) dtdzd\tilde x
		&&\text{($\tilde x=x+t\varepsilon z$)}.
	\end{align*}
	Next, we apply the Cauchy-Schwarz to $|\nabla f(\tilde x)^\top z|$ to obtain,
	\begin{align*}
		\int_{B(0, 1)}|\nabla f(\tilde x)^\top z|K(\lVert z\rVert)dz
		&\leq \int_{B(0, 1)}
			\lVert \nabla f(\tilde x)\rVert
			\lVert z\rVert K(\lVert z\rVert)dz
		\\
		&= \lVert \nabla f(\tilde x)\rVert
			\int_{(B(0,1)} \lVert z\rVert K(\lVert z\rVert) dz
		\\
		&= C_d \lVert \nabla f(\tilde x)\rVert
	\end{align*}
	Substituting back in to the previous derivation, we obtain
	\begin{align*}
		\int_\XDom\int_\XDom |f(x')-f(x)|K \left(
			\frac{\lVert x'-x\rVert}{\varepsilon}
			\right)
			dx'dx
		&\leq C_d\varepsilon^{d+1}\int_{\XDom'}
			\int_0^1 \lVert \nabla f(\tilde x)\rVert_1 dtd\tilde x
		\\
		&= C_d \varepsilon^{d+1}\lVert Df\rVert(\XDom')
		\\
		&\leq C_dC' \varepsilon^{d+1} \lVert Df\rVert(\XDom)
	\end{align*}
	Hence,
	\begin{align*}
		\E\left[
			\frac{1}{2}\sum_{i,j=1}^n
			|f(x_i) - f(x_j)|
			\cdot \one\{\lVert x_i-x_j\rVert\leq \varepsilon\}
		\right]
		&\leq
		\frac{n(n-1)}{2} p_\text{max}^2
		C_dC' \varepsilon^{d+1} \lVert Df\rVert(\XDom)
		\\
		&\leq C_2n^2 \varepsilon^{d+1} \lVert Df\rVert(\XDom)
	\end{align*}
	Finally, we provide an approximation argument to justify the assumption that
	$f\in C^1(\XDom')$.  For a function $f\in \BV(\XDom')$,
	we may construct a sequence of functions $f_k \in C^\infty(\XDom')$ via
	mollification such that $f_k\rightarrow f$ $\Leb$-a.e. (specifically, at all
	Lebesgue points) and $\lVert Df_k\rVert(\XDom')\rightarrow \lVert
	Df\rVert(\XDom')$ as $k\rightarrow\infty$
	(\citealp{evans2015measure}; Theorems 4.1 \& 5.3).  Via an application of
	Fatou's lemma, we find that
	\begin{align*}
		\E\Bigg[
			\frac{1}{2} \sum_{i,j=1}^n |f(x_i)-f(x_j)|
			&
			\cdot \one\{\lVert x_i-x_j\rVert\leq \varepsilon\}
		\Bigg] 
		\\
		&=  
		\E\left[
			\frac{1}{2} \sum_{i,j=1}^n
			|\lim_{k\rightarrow\infty}f_k(x_i)-f_k(x_j)|
			\cdot \one\{\lVert x_i-x_j\rVert\leq \varepsilon\}
		\right] 
		\\&=  
		\E\left[
			\liminf_{k\rightarrow\infty}
			\frac{1}{2} \sum_{i,j=1}^n
			|f_k(x_i)-f_k(x_j)|
			\cdot \one\{\lVert x_i-x_j\rVert\leq \varepsilon\}
		\right] 
			&&\text{(Continuity)}
		\\&\leq
		\liminf_{k\rightarrow\infty}
		\E\left[
			\frac{1}{2} \sum_{i,j=1}^n
			|f_k(x_i)-f_k(x_j)|
			\cdot \one\{\lVert x_i-x_j\rVert\leq \varepsilon\}
		\right] 
			&&\text{(Fatou's lemma)}
		\\&\leq
		\liminf_{k\rightarrow\infty}
		Cn^2\varepsilon^{d+1} \lVert Df_k\rVert(\XDom)
		\\&= 
		Cn^2\varepsilon^{d+1} \lVert Df\rVert(\XDom)
	\end{align*}

	\paragraph{$k$-nearest neighbors expected discrete TV.}
	Let $\varepsilon_k(x) := \|x - x_{(k)}(x)\|_2$ and $\varepsilon_k(x,y) =
	\max\{\varepsilon_k(x), \varepsilon_k(y)\}$ be data-dependent radii. Notice
	that
	\begin{equation*}
		\mathrm{DTV}_{n,k}(f)
		= \frac{1}{2}\sum_{i,j = 1}^{n}|f(x_i) - f(x_j)|\cdot \1\bigl\{
			\|x_i - x_j\| \leq \varepsilon_k(x_i,x_j)
		\bigr\}.
	\end{equation*}
	By linearity of expectation and conditioning, the expected $k$-nearest
	neighbor TV can be written as a double integral,
	\begin{align*}
		\mathbb{E}[\DTV(f; w^{\kNN})]
		& = n(n - 1) \mathbb{E}\Bigl[
			|f(x_i) - f(x_j)|~\1\bigl\{
				\|x_i - x_j\| \leq \varepsilon_k(x_i,x_j)
			\bigr\}
		\Bigr] \\
		& = n(n - 1) \mathbb{E}\Bigl[\mathbb{E}\Bigl[
			|f(x_i) - f(x_j)|~\1\bigl\{
				\|x_i - x_j\| \leq \varepsilon_k(x_i,x_j)
			\bigr\}|x_i,x_j
		\Bigr]\Bigr] \\
		& \leq n(n - 1) \int_{\XDom} \int_{\XDom} |f(y) - f(x)|~\mathbb{P}\bigl\{
			\|x - y\| \leq \varepsilon_k(x,y)
		\bigr\} \,dx \,dy \\
		& \leq n(n - 1) \int_{\XDom} \int_{\XDom} |f(y) - f(x)|~\Bigl(
			\mathbb{P}\bigl\{
				\|x - y\| \leq \varepsilon_k(x)
			\bigr\} + \mathbb{P}\bigl\{
				\|x - y\| \leq \varepsilon_k(x)
			\bigr\}
		\Bigr) \,dx \,dy
	\end{align*}
	(The first inequality above is nearly an equality for large $n$, and the second
	inequality follows by a union bound.) 

	We now derive an upper bound
	$\mathbb{P}\bigl\{\|x - y\| \leq \varepsilon_k(x)\bigr\}$. First, observe that
	the event $\|x - y\| \leq \varepsilon_k(x)$ is equivalent to
	$|B(x,\|y - x\|) \cap x_{1:n}| < k$. Suppose
	$\|y - x\| \geq C(k/n)^{1/d}$ for $C = (\frac{2d}{p_\mn\Leb_d})^{1/d}$. Then
	\begin{equation*}
		p_k(x,y) := P(B(x,\|y - x\|))
		\geq \frac{p_\mn}{2d}\Leb_d \|y - x\|^d \geq \frac{2k}{n},
	\end{equation*}
	and applying standard concentration bounds (Bernstein's inequality) to the
	tails of a binomial distribution, it follows that 
	\begin{align*}
		\mathbb{P}\biggl\{
			|B(x,\|y - x\|) \cap x_{1:n}| < k
		\biggr\}
		& =  \mathbb{P}\biggl\{
			|B(x,\|y - x\|) \cap x_{1:n}| - np_k(x,y) < k - np_k(x,y)\biggr
		\} \\
		& \leq \exp\biggl(
			-\frac{c(np_k(x,y) - k)^2}{np_k(x,y) + |np_k(x,y) - k|}
		\biggr) \\
		& \leq \exp(-ck).
	\end{align*} 
	Otherwise if $\|y - x\| < C(k/n)^{1/d}$, we use the trivial upper bound $1$ on
	the probability of an event. To summarize, we have shown 
	\begin{equation*}
		\mathbb{P}\bigl(\|x - y\| \leq \varepsilon_k(x)\bigr) \leq 
		\begin{dcases}
			1,& \quad\textrm{if $\|x - y\| < C(k/n)^{1/d}$,} \\
			\exp(-ck),& \quad\textrm{otherwise}.
		\end{dcases}
	\end{equation*}
	It follows from~\eqref{eqn:dtv-ub-knn} that
	\begin{align}
		\mathbb{E}[\DTV(f; w^{\kNN})]
		& \leq 2n^2 \int_{\XDom} \int_{\XDom} |f(y) - f(x)|~\Bigl(
			\bigl(\1\{\|x - y\| < C(k/n)^{1/d}\}\bigr) + \exp(-ck)
		\Bigr)\,dx \,dy \nonumber \\
		\label{pf:dtv-ub-knn-2}
		& \leq C\bigl(
			\mathbb{E}[
				\DTV(f; w^{\Eps\leftarrow C(k/n)^{1/d}})
			] + n^2\exp(-ck)\mathrm{TV}(f,\XDom)
		\bigr);
	\end{align}
	the first term on the right hand side of the second inequality is the expected
	$\varepsilon$-neighborhood graph TV of $f$, with radius $C(k/n)^{1/d}$, while
	the second term is obtained from the Poincar\'{e} inequality
	\begin{equation}
	\label{eqn:poincare}
		\int_{\XDom} \int_{\XDom} |f(y) - f(x)| \,dy \,dx
			= \int_{\XDom} \int_{\XDom}
				\biggl|
					f(y) - \bar{f} + \bar{f} - f(x)
				\biggr| \,dy \,dx
			\leq C\bigl(
				\mathrm{TV}(f;\XDom)
			\bigr),
	\end{equation}
	where $\bar{f} = \AvInt_{\XDom} f(x) \,dx$ is the average of $f$ over
	$\XDom$.
	The claimed upper bound~\eqref{eqn:dtv-ub-knn} follows from applying
	inequality~\eqref{eqn:dtv-ub-neighborhood}, with $\varepsilon = C(k/n)^{1/d}$,
	to~\eqref{pf:dtv-ub-knn-2}. 
\end{proof}

%% file: discrete-analysis.tex
In this section, we review tools for analyzing graph total variation denoising.
Suppose an unknown $\theta_0\in\R^n$ and observations
$y_1, \dotsc, y_n$,
\begin{equation}
  \label{eqn:model-discrete}
  y_i = \theta_{0i} + z_i, \qquad i=1,\dotsc,n,
\end{equation}
where $z_i\sim\cN(0, \sigma^2)$.  The graph total variation denoising estimator
$\hat\theta$ associated with a graph $G=(V, E)$, $|V|=n$, is given by
\begin{equation}
  \label{eq:tv-denoising-estimator}
  \hat\theta
  = \argmin_{\theta\in\R^n} \half \lVert y_{1:n}-\theta\rVert_2^2
    + \lambda \lVert \PenMat\theta\rVert_1,
\end{equation}
where $\PenMat\in\R^{m\times n}$ is the edge incidence matrix of $G$.  

The initial analysis of graph total variation denoising was performed by
\cite{hutter2016optimal} for the two-dimensional grid.
\cite{sadhanala2016total} subsequently generalized the analysis to
$d$-dimensional lattices, and \cite{wang2016trend} provided tools for the
analysis of general graphs.  These techniques rely on direct analysis 
of properties of graph $G$ and the penalty $\PenMat$ in induces, which is
tractable when $G$ has a known and regular properties (e.g., it is a lattice
graph).

Unfortunately, direct analysis on $\PenMat$ may not always be feasible.  It may
be possible, however, to compare the operator $\PenMat$ to a \emph{surrogate
operator} whose properties we analyze instead.  For our purposes, we compare
$\PenMat$ to a linear operator which first takes averages on a partition, and
then computes differences across cells of the partition.  Comparison to this
type of surrogate operator was used by \cite{padilla2020adaptive} to bound the
risk of graph total variation denoising in probability; the following theorem
provides an analogous risk bound in expectation.  We note that elements of the
``surrogate operator analysis'' are also found in \cite{padilla2018dfs}.

\begin{theorem}
  \label{thm:graph-tv-denoising-surrogate-graph}
  Suppose we observe data according to model~\eqref{eqn:model-discrete} and
	compute the graph TV denoising estimator $\hat\theta$
  of~\eqref{eq:tv-denoising-estimator}. 
  Let $A\in\R^{n\times n}$ denote an averaging operator over $\bar N$ groups of
  the form,
  \begin{equation*}
    \SqAvMat = \begin{bmatrix}
      n_1^{-1}\mathbf{1}_{n_1} \mathbf{1}_{n_1}^\top
      & 0 & \hdots & 0  \\
      0 & 
      n_2^{-1}\mathbf{1}_{n_2} \mathbf{1}_{n_2}^\top
      & \hdots & 0  \\
      \vdots & \vdots & \ddots & \vdots  \\
      0 & 0 & \hdots & 
      n_{\bar N}^{-1}\mathbf{1}_{n_{\bar N}} \mathbf{1}_{n_{\bar N}}^\top
    \end{bmatrix},
  \end{equation*}
  with $\CellMax := \max_j n_j$,
  and let $\AvMat\in\R^{\bar N\times n}$ be the same matrix with redundant
  rows removed.  Further let $\SurPenMat\in\R^{\bar m\times \bar N}$ be a 
  surrogate penalty operator, with singular value decomposition
  $\SurPenMat = U\Sigma V^\top$, such that
  \begin{align}
    \label{eq:surr-diff-av-embedding}
    \lVert \SurPenMat\AvMat\theta\rVert_1
      &\leq \Phi_1(\PenMat, \SurPenMat, \SqAvMat)\lVert \PenMat\theta\rVert_1,
    \\
    \label{eq:av-error-bound}
    \lVert (I-\SqAvMat)\theta\rVert_1 
      &\leq \Phi_2(\PenMat, \SurPenMat, \SqAvMat)\lVert \PenMat\theta\rVert_1,
  \end{align}
  for quantities
  $\Phi_1(\PenMat, \SurPenMat, \SqAvMat), \Phi_2(\PenMat, \SurPenMat, \SqAvMat)$
  that may depend on $n$, for all $\theta\in\R^n$.  If the penalty parameter
  \begin{equation}
    \lambda > \max\left\{
      8\CellMax^{1/2}\Phi_1(\PenMat, \SurPenMat, \SqAvMat)\cdot\sigma
      \sqrt{
        \log(2n^4)\sum_{k=2}^{\bar N}
        \frac{\lVert u_k\rVert_\infty^2}{\xi_k^2}
      },
      \Phi_2(\PenMat,\SurPenMat,\SqAvMat)\cdot\sigma
      \sqrt{2\log(n)}
    \right\}
  \end{equation}
  where $u_k$ is the $k$th column of $U$ and $\xi_k$ the $k$th diagonal entry
  of $\Sigma$, then there exists a constant $C>0$ such that
  \begin{equation}
		\label{eqn:ub-mse-tv}
    \E\left[
      \frac{1}{n} \lVert \hat\theta-\theta_0\rVert_2^2
    \right]
    \leq C\left(
      \frac{\lambda \lVert D\theta_0\rVert_1}{n}
      + \frac{\CellMax\nuli(\SurPenMat)}{n} 
    \right).
  \end{equation}
\end{theorem}

\begin{proof}
  We follow the approach of \cite{padilla2020adaptive}, with adaptations to
  provide a bound in expectation rather than in probability.  From the basic
  inequality,
  \begin{equation*}
    \lVert \hat\theta-\theta_0\rVert_2^2
    \leq 2 \langle \ErrVec, \hat\theta-\theta_0\rangle 
    + \lambda(\lVert D\theta_0\rVert_1-\lVert D\hat\theta\rVert_1),
  \end{equation*}
  where $\ErrVec\in\R^n$ is the vector of error terms $z_i$, $i=1,\dotsc, n$.
  We provide two deterministic bounds under the ``good case'' that the error
  term falls into the set,
  \begin{equation}
    \cZ_\lambda = \left\{\ErrVec:
    \max\left\{
      \CellMax^{1/2}\Phi_1(\PenMat,\SurPenMat,\SqAvMat)
      \sup_{\AvMat\theta\in \row(\SurPenMat):
      \lVert \SurPenMat\AvMat\theta\rVert_1\leq 1}
      |\langle \widebar{\ErrVec}, \AvMat\theta\rangle|,
      \Phi_2(\PenMat,\SurPenMat,\SqAvMat)
      \lVert \ErrVec\rVert_\infty
    \right\} \leq \frac{\lambda}{8} 
  \right\},
  \end{equation}
  where $\widebar{\ErrVec}\in\R^{\bar N}$ has entries
  $\widebar{\ErrVec}_j = n_j^{1/2}(\AvMat\ErrVec)_j$,
  and under the ``bad case'' that $\ErrVec\not\in\cZ_\lambda$.

  \paragraph{Upper bound in the ``good case''.}%
  \label{par:upper_bound_in_the_good_case_}
  Decompose the first term on the RHS,
  \begin{align*}
    \langle \ErrVec, \hat\theta-\theta_0\rangle 
    &=  \langle \ErrVec, \hat\theta-\SqAvMat\hat\theta\rangle
      + \langle \ErrVec, \SqAvMat\theta_0-\theta_0\rangle
      + \langle \ErrVec, \SqAvMat(\theta_0-\hat\theta)\rangle 
    \\
    &\leq \langle \ErrVec, \SqAvMat(\theta_0-\hat\theta)\rangle 
    + \lVert \ErrVec\rVert_\infty(
      \lVert (I-\SqAvMat)\hat\theta\rVert_1
      + \lVert (I-\SqAvMat)\theta_0\rVert_1)
      \\
    &\leq \langle \ErrVec, \SqAvMat(\theta_0-\hat\theta)\rangle 
    + \lVert \ErrVec\rVert_\infty \Phi(\PenMat, \SurPenMat, \SqAvMat)(
      \lVert D\theta_0\rVert_1 +\lVert D\hat\theta\rVert_1),
  \end{align*}
  where the final inequality follows from \eqref{eq:av-error-bound}.
  Observe that we may rewrite, for any $\theta\in\R^n$,
  \begin{align*}
    \langle \ErrVec, \SqAvMat\theta\rangle 
    &=  \sum_{j=1}^{\bar N} \sum_{i=1}^{n_j}
      \ErrVec_{(\sum_{k=1}^{j-1}n_k)+i} (\AvMat\theta)_j
      \\
    &\stackrel{d}{=}\sum_{j=1}^{\bar N} n_j^{1/2} \ErrVec_j(\AvMat\theta)_j
    \\
    \Rightarrow
    \langle \ErrVec, A\theta\rangle
    &\leq \CellMax^{1/2}
      |\langle \widebar{\ErrVec}, \AvMat\theta\rangle|
    \\
    &\leq \CellMax^{1/2}\left|
      \langle \proj_V(\widebar{\ErrVec}), \AvMat\theta\rangle 
      + \langle \proj_{V^\perp}(\widebar{\ErrVec}), \AvMat\theta\rangle 
    \right|
    \\
    &\leq \CellMax^{1/2}\left|
      \lVert\proj_V(\widebar{\ErrVec})\rVert_2
      \lVert\AvMat\theta\rVert_2
      + \langle \proj_{V^\perp}(\widebar{\ErrVec}), \AvMat\theta\rangle 
    \right|
    \\
    &\leq \CellMax^{1/2}\left(
      \lVert\proj_V(\widebar{\ErrVec})\rVert_2
      \lVert\theta\rVert_2
      + |\langle \proj_{V^\perp}(\widebar{\ErrVec}), \AvMat\theta\rangle|
    \right|,
  \end{align*}
  where $\widebar{\ErrVec}\in\R^{\bar N}$ has independent $\cN(0, \sigma^2)$
  entries and $V = \nul(\SurPenMat)$.
  Substitute back in to obtain,
  \begin{align*}
    \lVert \hat\theta-\theta_0\rVert_2^2
    &\leq 2\CellMax^{1/2} (
    \lVert \proj_V(\widebar{\ErrVec})\rVert_2 \lVert \hat\theta-\theta_0\rVert_2
      + |\langle \proj_{V^\perp}(\widebar{\ErrVec}),
        \AvMat(\hat\theta-\theta_0)\rangle|
    )
    \\
    &\qquad\qquad+ 2 \lVert \ErrVec\rVert_\infty
    \Phi(\PenMat, \SurPenMat, \SqAvMat)(
      \lVert D\theta_0\rVert_1 + \lVert D\hat\theta\rVert_1
    ) + \lambda (
      \lVert D\theta_0\rVert_1 - \lVert D\hat\theta\rVert_1
    ),
  \end{align*}
  and consequently,
  \begin{align*}
  \lVert \hat\theta-\theta_0\rVert_2(
    &\lVert \hat\theta-\theta_0\rVert_2
    - 2\CellMax^{1/2} \lVert \proj_V(\widebar{\ErrVec})\rVert_2
    )
    \\
  &\leq 
  2\CellMax^{1/2} |\langle \proj_{V^\perp}(\widebar{\ErrVec}),
        \AvMat(\hat\theta-\theta_0)\rangle|
    + 2 \lVert \ErrVec\rVert_\infty
    \Phi(\PenMat, \SurPenMat, \SqAvMat)(
      \lVert D\theta_0\rVert_1 + \lVert D\hat\theta\rVert_1
    ) + \lambda (
      \lVert D\theta_0\rVert_1 - \lVert D\hat\theta\rVert_1
    )
  \end{align*}
  \emph{Case 1.} $\lVert \hat\theta-\theta_0\rVert_2
    \leq 4\CellMax^{1/2} \lVert \proj_V(\widebar{\ErrVec})\rVert_2$.
  \\
  \emph{Case 2.} $\lVert \hat\theta-\theta_0\rVert_2
    > 4\CellMax^{1/2} \lVert \proj_V(\widebar{\ErrVec})\rVert_2$.  Then,
  \begin{align*}
    \lVert \hat\theta-\theta_0\rVert_2^2
    &\leq 4\CellMax^{1/2}|\langle\proj_{V^\perp}(\widebar{\ErrVec}),
    \AvMat(\hat\theta-\theta_0)\rangle|
    + 4 \lVert \ErrVec\rVert_\infty
    \Phi(\PenMat, \SurPenMat, \SqAvMat)(
      \lVert D\theta_0\rVert_1+\lVert D\hat\theta\rVert_1
    ) + \lambda (
      \lVert D\theta_0\rVert_1-\lVert D\hat\theta\rVert_1
    ).
  \end{align*}
  We then bound,
  \begin{align*}
    |\langle \proj_{V^\perp}(\widebar{\ErrVec}),
    \AvMat(\hat\theta-\theta_0)\rangle|
    &= \left|
      \left\langle 
      \proj_{V^\perp}(\widebar{\ErrVec}),
      \frac{
        \AvMat(\hat\theta-\theta_0)
      }{
        \lVert \SurPenMat\AvMat(\hat\theta-\theta_0)\rVert_1
      } \right\rangle 
      \lVert \SurPenMat\AvMat(\hat\theta-\theta_0)\rVert_1
    \right|
    \\
    &\leq \sup_{\AvMat\theta\in V^\perp:
      \lVert \SurPenMat\AvMat\theta\rVert_1\leq 1}
    |\langle \widebar{\ErrVec}, \AvMat\theta\rangle|
      \lVert \SurPenMat\AvMat(\hat\theta-\theta_0)\rVert_1
    \\
    &\leq \sup_{\AvMat\theta\in V^\perp:
      \lVert \SurPenMat\AvMat\theta\rVert_1\leq 1}
    |\langle \widebar{\ErrVec}, \AvMat\theta\rangle|
    \Phi(\PenMat, \SurPenMat, \SqAvMat)
    (\lVert \PenMat\hat\theta\rVert_1+\lVert \PenMat\theta_0\rVert_1),
  \end{align*}
  where the last inequality follows by \eqref{eq:av-error-bound}.
  Conditioning on $\ErrVec\in \cZ_\lambda$, we find that under Case 2,
  \begin{align*}
    \lVert \hat\theta-\theta_0\rVert_2^2
    &\leq \frac{\lambda}{2} (
      \lVert \PenMat\hat\theta\rVert_1+\lVert \PenMat\theta_0\rVert_1
    ) + \frac{\lambda}{2} (
      \lVert \PenMat\hat\theta\rVert_1+\lVert \PenMat\theta_0\rVert_1
    ) + \lambda(
      \lVert \PenMat\theta_0\rVert_1-\lVert \PenMat\hat\theta\rVert_1
    )
    \\
    &\leq 2\lambda \lVert \PenMat\theta_0\rVert_1.
  \end{align*}
  Therefore, conditioning on $\ErrVec\in \cZ_\lambda$ and combining Case 1
  and Case 2, we obtain that
  \begin{align*}
    \lVert \hat\theta-\theta_0\rVert_2^2
    &\leq 16\CellMax \lVert \proj_V(\widebar{\ErrVec})\rVert_2^2
    + 2\lambda \lVert \PenMat\theta_0\rVert_1.
  \end{align*}

  \paragraph{Upper bound in the ``bad case''.}%
  \label{par:upper_bound_in_the_bad_case_}
  On the ``bad event'' that $\ErrVec\not\in \cZ_\lambda$, we apply
  H\"older directly to the basic inequality to bound,
  \begin{align*}
    \lVert \hat\theta-\theta_0\rVert_2^2
    \leq 2 \lVert \ErrVec\rVert_2 \lVert \hat\theta-\theta_0\rVert_2
      + \lambda \lVert \PenMat\theta_0\rVert_1,
  \end{align*}
  and rearrange to obtain,
  \begin{equation}
    \label{pf:ub-mse-tv-1}
    \lVert \hat\theta-\theta_0\rVert_2^2
    \leq 16 \lVert \ErrVec\rVert_2^2 + 2\lambda \lVert \PenMat\theta_0\rVert_1.
  \end{equation}

  \paragraph{Combining the ``good case'' and ``bad case'' upper bounds.}%
  \label{par:combining_the_good_case_and_bad_case_upper_bounds_}
  \begin{align*}
    \frac{1}{n} \E \lVert \hat\theta-\theta_0\rVert_2^2
    &= \frac{1}{n} \E\left[
      \lVert \hat\theta-\theta_0\rVert_2^2
        \mathbf{1}\{\ErrVec\in \cZ_\lambda\}
      + 
      \lVert \hat\theta-\theta_0\rVert_2^2
        \mathbf{1}\{\ErrVec\not\in \cZ_\lambda\}
    \right] 
    \\
    &\leq \frac{1}{n} \left[
      \E\left[
        16\CellMax \lVert \proj_V(\widebar{\ErrVec})\rVert_2^2
        + 2\lambda \lVert \PenMat\theta_0\rVert_1
      \right] 
      + \E\left[
        (16 \lVert \ErrVec\rVert_2^2 + 2\lambda \lVert D\theta_0\rVert_1)
        \mathbf{1}\{\ErrVec\not\in \cZ_\lambda\}
      \right]
    \right] 
    \\
    &\leq \frac{1}{n} \left[
      16\CellMax\dim(V) + 4\lambda \lVert \PenMat\theta_0\rVert_1
      + \sqrt{\E[\lVert \ErrVec\rVert_2^4]}\cdot
      \P[\ErrVec\not\in \cZ_\lambda]
    \right] 
    \\
    &\leq \frac{1}{n} \left[
      16\CellMax\dim(V) + 4\lambda \lVert \PenMat\theta_0\rVert_1
      + \sqrt{3}n\cdot \P[\ErrVec\not\in \cZ_\lambda]
    \right] 
  \end{align*}
  It remains to bound the probability of the bad case,
  \begin{align*}
    \P\{\ErrVec\not\in \cZ_\lambda\}
    &\leq \P\left\{
      \CellMax^{1/2}\sup_{\AvMat\theta\in V^\perp:
      \lVert \SurPenMat\AvMat\theta\rVert_1\leq 1}
      |\langle \widebar{\ErrVec}, \AvMat\theta\rangle|
      \geq \lambda/8\Phi(\PenMat, \SurPenMat, \SqAvMat)
    \right\} + \P\left\{
      \lVert \ErrVec\rVert_\infty \geq \lambda/8\Phi(\PenMat, \SurPenMat, \SqAvMat)
    \right\}
    \\
    &\leq
    \P\{\CellMax^{1/2}\Phi(\PenMat, \SurPenMat, \SqAvMat)
      \lVert (\SurPenMat^+)^\top\widebar{\ErrVec}\rVert_\infty\geq\lambda/8\}
    + \P\{\Phi(\PenMat, \SurPenMat, \SqAvMat)\lVert \ErrVec\rVert_\infty
        \geq \lambda/8\}.
  \end{align*}
  Standard results on the maxima of Gaussians provide that,
  \begin{align*}
    \P\left\{
      \CellMax^{1/2} \Phi_1(\PenMat, \SurPenMat, \SqAvMat)
      \lVert (T^+)^\top\widebar{\ErrVec}\rVert_\infty
      \geq \CellMax^{1/2} \Phi_1(\PenMat, \SurPenMat, \SqAvMat)
      \cdot\sigma
      \sqrt{
        \log(2n^2/\delta)
        \cdot\sum_{k=2}^{\bar N}\frac{\lVert u_k\rVert_\infty^2}{\xi_k^2} 
      }
    \right\}&\leq \delta,
    \\
    \P\left\{\Phi_2(\PenMat, \SurPenMat, \SqAvMat)
      \lVert \ErrVec\rVert_\infty \geq
      \Phi_2(\PenMat, \SurPenMat, \SqAvMat) \cdot \sigma
      \sqrt{
        \log(2n^2/\delta)
      }
    \right\}&\leq\delta.
  \end{align*}
  Recalling the choice of penalty parameter,
  \begin{equation*}
    \lambda > \max\left\{
      8\CellMax^{1/2}\Phi_1(\PenMat, \SurPenMat, \SqAvMat)\cdot\sigma
      \sqrt{
        \log(2n^4)\sum_{k=2}^{\bar N}
        \frac{\lVert u_k\rVert_\infty^2}{\xi_k^2}
      },
      \Phi_2(\PenMat,\SurPenMat,\SqAvMat)\cdot\sigma
      \sqrt{2\log(n)}
    \right\},
  \end{equation*}
  we conclude that
  \begin{equation*}
    \P\{\ErrVec\not\in \cZ_\lambda\} \leq \frac{2}{n^2},
  \end{equation*}
  completing the proof.
\end{proof}

We now state a well-known result controlling certain functionals of the lattice
difference operator.  These quantities have been analyzed by others studying
graph total variation denoising on lattices, e.g.,
\cite{hutter2016optimal} and \cite{sadhanala2017higher}.
\begin{lemma}
  Let $\SurPenMat$ be the edge incidence operator of the $d$-dimensional
  lattice graph $N$ elements per direction.  Denote $n=N^d$.  
  The left singular vectors of $\SurPenMat$ satisfy an incoherence condition,
  \begin{equation*}
    \lVert u_j\rVert_\infty \leq \frac{C_d}{\sqrt{n}}, \qquad j=1,\dotsc,n, 
  \end{equation*}
  for some $C_d>0$, and its singular values satisfy an
  asymptotic scaling,
  \begin{equation*}
    c_d (j/n)^{1/d} \leq \xi_j \leq C_d(j/n)^{1/d}, \qquad j=2,\dotsc,n, 
  \end{equation*}
  for some $0<c_d<C_d$.  Consequently,
  \begin{equation}
  \label{eqn:grid-svd-scaling}
    \sum_{j=2}^n \frac{\lVert u_j\rVert_\infty^2}{\xi_j^2} 
    = C_d\begin{cases}
      \log n  & d=2,\\
      1       & d>2.
    \end{cases}
  \end{equation}
\end{lemma}

%% file: graph-embeddings.tex
We begin by providing a result that controls the number of sample points
that fall into each cell of a lattice mesh.
\begin{lemma}
  \label{lem:control-mesh-counts}
  Suppose $x_1, \dotsc, x_n$ are sampled from a distribution $P$
  supported on $(0,1)^d$ with density $p$ such that $0<p_\mn<p(x)<p_\mx<1$ for
  all $x\in(0,1)^d$.  Form a partition of $(0,1)^d$ using an equally spaced
  mesh with $N = C_1(p_\mn n/\log^\alpha n)^{1/d}$, $\alpha>1$, along each
  dimension.  Let $\cC_\ell$ denote the $\ell$th cell of the mesh, and
  let $|\cC_\ell|$ denote its empirical content.
  Then for all $x_{1:n}\in\XSet_1$, with
  $\P\{x_{1:n}\in\XSet_1\} \geq 1-2/n^4$,
  \begin{align}
    \label{eq:uniform-upper-bound-cell-count}
    \max_\ell|\Cell_\ell|
    &\leq C_3\log^\alpha n,
    \\
    \label{eq:uniform-lower-bound-cell-count}
    \min_\ell|\Cell_\ell|
    &\geq c_4 \log^\alpha n,
  \end{align}
  for $n$ sufficiently large, where $C_3, c_4>0$ depend only on
  $p_\mn, p_\mx, d$.
\end{lemma}

\begin{proof}
  From standard concentration bounds
  (e.g., \citealp{vonluxberg2014hitting}; Proposition 27)
  on a random variable $m\sim\mathrm{Bin}(n, p)$, for all $\delta\in(0,1]$,
  \begin{align*}
    \P\{m\geq(1+\delta)np\}
      &\leq \exp\{-\frac{1}{3} \delta^2np\},
    \\
    \P\{m\leq(1-\delta)np\}
      &\leq \exp\{-\frac{1}{3} \delta^2np\}.
  \end{align*}
  Apply these bounds with $p=\P\{x\in \Cell_\ell\}$ to obtain that,
  \begin{align*}
    \P\left\{
      \max_{\ell} |\Cell_\ell|
      \geq (1+\delta)C_1^d\frac{p_\mx}{p_\mn} \log^\alpha n 
    \right\} &\leq N^d\exp\left\{
    -\frac{1}{3} \delta^2 C_1^d\log^\alpha n
    \right\},
    \\
    \P\left\{
      \min_{\ell} |\Cell_\ell| \leq
      (1-\delta)C_1^d\log^\alpha n 
    \right\} &\leq N^d\exp\left\{
    -\frac{1}{3} \delta^2C_1^d\log^\alpha n
    \right\},
  \end{align*}
  for all $\delta\in(0, 1)$.  Setting the RHS to $1/n^4$,
  \begin{align*}
    \frac{C_1^d p_\mn n}{\log^\alpha n}
    \exp\{-\frac{1}{3} \delta^2 C_1^{-d}\log^\alpha n\}
    &\leq \frac{1}{n^4} 
    \\
    \log(C_1^dp_\mn) - \log(\log^\alpha n)
    - \frac{1}{3} \delta^2 C_1^{-d}\log^\alpha n
    &\leq -5\log n
    \\
    \Rightarrow
    \frac{1}{3} \delta^2 C_1^{-d}\log^\alpha n
    &\geq 5\log n + \log(C_1^dp_\mn) - \log(\log^\alpha n)
    \\
    \delta^2
    &\geq 3C_1^d\left(
      5\log^{1-\alpha}n
      + \frac{\log(C_1^dp_\mn)}{\log^\alpha n} 
      - \frac{\log(\log^\alpha n)}{\log^\alpha n}
    \right)
    \\
    \delta
    &\geq C_2 \log^{(1-\alpha)/2} n,
  \end{align*}
  for some $C_2>0$ for all $n$ sufficiently large.  Therefore deduce that,
  \begin{align*}
    \P\left\{
      \max_\ell|\Cell_\ell|
      \geq C_1^d \frac{p_\mx}{p_\mn} \log^\alpha n
        + C_1^dC_2 \frac{p_\mx}{p_\mn} \log^{(1+\alpha)/2}n
    \right\}
    &\leq \frac{1}{n^4},
    \\
    \P\left\{
      \min_\ell|\Cell_\ell|
      \leq C_1^d  \log^\alpha n
        - C_1^dC_2 \log^{(1+\alpha)/2}n
    \right\}
    &\leq \frac{1}{n^4}.
  \end{align*}
  Recall that $\alpha>1$ by assumption, and choose $C_3, c_4>0$ with $n$
  sufficiently large to obtain the claim.
\end{proof}

The following lemma establishes embeddings from certain random graphs into
a coarser lattice graph.
\begin{lemma}
  \label{lem:random-graph-embedding}
  Partition the domain $(0, 1)^d$ using an equally spaced mesh with
  $N = C_1(p_\mn n/\log^\alpha n)^{1/d}$ elements per direction.  Suppose
  that $x_{1:n}\in\XSet_1$, with $x_{1:n}$ re-indexed such that
  \begin{align*}
    x_1, \dotsc, x_{|\Cell_1|}
      &\in \Cell_1,  \\
    x_{|\Cell_1|+1}, \dotsc, x_{|\Cell_1|+|\Cell_2|}
      &\in \Cell_2,  \\
      &\vdots       \\
    x_{\sum_{\ell=1}^{N^d-1}|\Cell_\ell|+1}, \dotsc, x_{N^d}
      &\in \Cell_{N^d}.
  \end{align*}
  Consider the averaging operators,
  \begin{equation}
    \label{eq:mesh-averaging-operator}
    \SqAvMat = \begin{bmatrix}
      |\Cell_1|^{-1}\mathbf{1}_{|\Cell_1|} \mathbf{1}_{|\Cell_1|}^\top
      & 0 & \hdots & 0  \\
      0 & 
      |\Cell_2|^{-1}\mathbf{1}_{|\Cell_2|} \mathbf{1}_{|\Cell_2|}^\top
      & \hdots & 0  \\
      \vdots & \vdots & \ddots & \vdots  \\
      0 & 0 & \hdots & 
      |\Cell_{N^d}|^{-1}
      \mathbf{1}_{|\Cell_{N^d}|} \mathbf{1}_{|\Cell_{N^d}|}^\top
    \end{bmatrix},
  \end{equation}
  and the lattice difference operator $\SurPenMat$ based on the graph
  \begin{equation}
    \label{eq:lattice-graph}
    G_\SurPenMat = (\{1, \dotsc, N^d\}, E_\SurPenMat),
  \end{equation}
  where $(i, j)\in E_\SurPenMat$ if the midpoints of $\Cell_i, \Cell_j$ are
  $1/N$ apart.  Also, let $\AvMat\in\R^{N^d\times n}$ be the matrix obtained
  by dropping the redundant rows of $\SqAvMat$.

  \begin{itemize}
    \item Build the Voronoi graph from $x_{1:n}$, and let $\CVorPenMat$ denote
      the edge incidence operator with edge set \smash{$E^{\Vor}$}
			and edge weights 
			$\tilde{w}^\Vor_{ij} = \max\{c_0 n^{-(d-1)/d}, w^\Vor_{ij}\}$ for each
      $i, j$.  Further
      condition on the set $\XSet_2$ such that the result of
      Lemma~\ref{lem:vc-balls} holds with probability $1-1/n^4$
      (equivalently, the set that the result
      of Lemma~\ref{lem:voronoi-cell-upper-bounds} holds with
      probability $1-1/n^4$).  Then there
      exists a constant $C_6>0$ such that for all $\theta\in\R^n$,
      \begin{align}
        \label{eq:surr-diff-av-embedding-cvor}
        \lVert \SurPenMat\AvMat\theta\rVert_1
        &\leq C_6 n^{(d-1)/d}
        \lVert \CVorPenMat\theta\rVert_1.
        \\
        \label{eq:av-error-bound-cvor}
        \lVert (I-\SqAvMat)\theta\rVert_1
        &\leq C_6 (\log n)^\alpha n^{(d-1)/d}
        \lVert\CVorPenMat\theta\rVert_1,
      \end{align}
    \item Build the Voronoi graph from $x_{1:n}$, and let $\UVorPenMat$ denote
      the edge incidence operator with edge set \smash{$E^{\Vor}$}
			and edge weights 
      $\check{w}^\Vor_{ij} = 1$ for each $i, j$ such that $w^\Vor_{i,j}>0$.
      Further condition on the set $\XSet_2$ such that the result of
      Lemma~\ref{lem:vc-balls} holds with probability $1-1/n^4$.  Then there
      exists a constant $C_7>0$ such that for all $\theta\in\R^n$,
      \begin{align}
        \label{eq:surr-diff-av-embedding-uvor}
        \lVert \SurPenMat\AvMat\theta\rVert_1
        &\leq C_7
        \lVert \UVorPenMat\theta\rVert_1.
        \\
        \label{eq:av-error-bound-uvor}
        \lVert (I-\SqAvMat)\theta\rVert_1
        &\leq C_7 (\log n)^\alpha
        \lVert\UVorPenMat\theta\rVert_1,
      \end{align}
    \item Build the $\varepsilon$-neighborhood graph from $x_{1:n}$, with
      $\varepsilon \geq 2\sqrt{d}/N$.  Then with the constant $c_4$ from
      Lemma~\ref{lem:control-mesh-counts}, it holds that for all
      $\theta\in\R^n$,
      \begin{align}
        \label{eq:surr-diff-av-embedding-eps}
        \lVert \SurPenMat\AvMat\theta\rVert_1
        &\leq \frac{1}{c_4^2\log^{2\alpha} n}
        \lVert \EpsPenMat\theta\rVert_1.
        \\
        \label{eq:av-error-bound-eps}
        \lVert (I-\SqAvMat)\theta\rVert_1
        &\leq \frac{2}{c_4\log^\alpha n}
        \lVert\EpsPenMat\theta\rVert_1,
      \end{align}
    \item Build the $k$-nearest neighbors graph from $x_{1:n}$, with
      $k\geq C_5\log^3n$.  Further condition on the
      set $\XSet_2$ such that the result of 
      Lemma~\ref{lem:vc-balls} holds with probability $1-1/n^4$.
      Then with the constant $c_4$ from Lemma~\ref{lem:control-mesh-counts}, it
      holds for all $\theta\in\R^n$,
      \begin{align}
        \label{eq:surr-diff-av-embedding-knn}
        \lVert \SurPenMat\AvMat\theta\rVert_1
        &\leq \frac{1}{c_4^2\log^{2\alpha} n}
        \lVert \KnnPenMat\theta\rVert_1.
        \\
        \label{eq:av-error-bound-knn}
        \lVert (I-\SqAvMat)\theta\rVert_1
        &\leq \frac{2}{c_4\log^\alpha n}
        \lVert\KnnPenMat\theta\rVert_1,
      \end{align}
  \end{itemize}
\end{lemma}

\begin{proof}
  \textbf{$\varepsilon$-neighborhood graph.}
  First, we prove \eqref{eq:surr-diff-av-embedding-eps} and
  \eqref{eq:av-error-bound-eps}.  For the former, observe that
  \begin{align*}
    \lVert \SurPenMat\AvMat\theta\rVert_1
    &=  \sum_{(k,\ell)\in E_\SurPenMat}
      \left|
      |\Cell_k|^{-1}\sum_{i\in\Cell_k}\theta_i
      - |\Cell_\ell|^{-1}\sum_{j\in\Cell_\ell}\theta_j
      \right|
    \\
    &\leq \sum_{(k, \ell)\in E_\SurPenMat}
      \frac{1}{|\Cell_k||\Cell_\ell|} 
      \sum_{i\in|\Cell_k|, j\in|\Cell_\ell|}
      |\theta_i-\theta_j|
    \\
    &\leq \frac{1}{c_4^2\log^{2\alpha}n} 
    \sum_{(k,\ell)\in E_\SurPenMat}
    \sum_{i\in\Cell_k,j\in\Cell_\ell}
    |\theta_i-\theta_j|
    \\
    &\leq \frac{1}{c_4^2\log^{2\alpha}n} 
    \lVert \EpsPenMat\theta\rVert_1,
  \end{align*}
  as $\varepsilon=2\sqrt{d}/N$.  For the latter, similarly deduce that
  \begin{align*}
    \lVert (I-\SqAvMat)\theta\rVert_1
    &= \sum_{i=1}^{n} \left|
      \theta_i
      - |\Cell(i)|^{-1}
      \sum_{j\in\Cell(i)} \theta_j
    \right|
    \\
    &\leq \sum_{i=1}^{n} 
    |\Cell(i)|^{-1}
    \left|
      \sum_{j\in\Cell(i)} \theta_j-\theta_i
    \right|
    \\
    &\leq \sum_{i=1}^{n} 
    |\Cell(i)|^{-1}
    \sum_{j\in\Cell(i)} |\theta_i-\theta_j|
    \\
    &= \sum_{\ell=1}^{N^d}
    |\Cell_\ell|^{-1}
    \sum_{i\in\Cell_\ell}
    \sum_{j\in\Cell_\ell} |\theta_i-\theta_j|
    \\
    &\leq \frac{2}{c_4\log^{\alpha}n}
    \sum_{\ell=1}^{N^d} \sum_{i<j\in\Cell_\ell} |\theta_i-\theta_j|
    \\
    &\leq \frac{2}{c_4\log^{\alpha}n} \lVert \EpsPenMat\theta\rVert_1.
  \end{align*}
  \\
  \paragraph{$k$-nearest neighbors graph.}
  Recall that we have conditioned on the set $\XSet_2$ such that the result
  of Lemma~\ref{lem:vc-balls} holds.  In particular,
  \eqref{eqn:vc-balls-2}~ gives that
  \begin{equation*}
    \min_{i=1, \dotsc, n}
    \varepsilon_k(x_i) \geq C\left(\frac{k}{n}\right)^{1/d},
  \end{equation*}
  where $\varepsilon_k(x_i):=\lVert x_i-x_{(k)}(x_i)\rVert_2$.
  The results \eqref{eq:surr-diff-av-embedding-knn} and
  \eqref{eq:av-error-bound-knn} then follow by observing that on the
  event $\XSet_2$, the $k$-nearest neighbors graph with $k\geq C_5\log^3n$
  dominates the $\varepsilon$-neighborhood graph with $\varepsilon =
  2\sqrt{d}/N$.

  \paragraph{Voronoi adjacency graph.}  
  We will prove the results
  \eqref{eq:surr-diff-av-embedding-uvor} and
  \eqref{eq:av-error-bound-uvor} by providing a graph comparison inequality
  between the $\varepsilon$-neighborhood graph with
  $\varepsilon=2\sqrt{d}/N$ and the Voronoi adjacency graph.
  The results
  \eqref{eq:surr-diff-av-embedding-cvor},~\eqref{eq:av-error-bound-cvor}
  follow from the inequality
  $\lVert \UVorPenMat\theta\rVert_1\leq c_0^{-1}n^{(d-1)/d}
  \lVert \CVorPenMat\theta\rVert_1$ for all $\theta\in\R^n$.

  \emph{Intuition and outline.}  The central goal of this proof is to show that
  \begin{equation*}
    \lVert \EpsPenMat\theta\rVert_1
    \leq C(n) \lVert \UVorPenMat\theta\rVert_1,
  \end{equation*}
  for all $\theta\in\R^n$, where $C(n)$ is at most polylogarithmic in $n$.
  This will be accomplished by
  \begin{enumerate}[label=(\roman*)]
    \item verifying that for any $\{x_i, x_j\}\in E_\nEps$, there exists a
      path $\{x_i, x_{k_1}\}, \{x_{k_1}, x_{k_2}\}, \dotsc,
      \{x_{k_{ij}}, x_j\}\in E_\nVor$, and
    \item showing that if one uses the shortest path in the Voronoi
      adjacency graph $G_\nVor$ to connect each $\{x_i,x_j\}\in E_\nEps$,
      then no one edge is used more than $C_9 \log^{2\alpha} n$ times,
      where $C_9$ is a positive constant and $\alpha>1$ may be chosen.
  \end{enumerate}

  \emph{Step (i).}  Consider $x_i, x_j$ such that $\{x_i, x_j\}\in E_\nEps$.
  We will show the existence of a path between $x_i$ and $x_j$ in $G_\nVor$
  and also characterize some properties of the path for step (ii).

  By definition, $\lVert x_i-x_j\rVert \leq \varepsilon$.  
  Denote
  \begin{align*}
    x_{ij}  &:= \frac{x_i+x_j}{2},      \\
    r_{ij}  &:= \lVert x_i-x_{ij}\rVert.
  \end{align*}
  Consider the subgraph $G^{ij} = (V^{ij}, E^{ij})$, where
  \begin{align*}
    V^{ij}  &:= \{\Part_k: \Part_k\cap B(x_{ij}, r_{ij})\neq\emptyset\},  \\
    E^{ij}  &:= \{\{\Part_k, \Part_\ell\}:
      \Part_k, \Part_\ell\in V^{ij},
      \cH^{d-1}(\partial\Part_k\cap\partial\Part_\ell) > 0
    \},
  \end{align*}
  where $B(x_{ij},r_{ij})$ is the closed ball centered at $x_{ij}$ with
  radius $r_{ij}$.  By construction, $x_i, x_j\in V^{ij}$, and
  by Lemma~\ref{lem:topological-connectedness-equiv-graph-connectedness},
  $G^{ij}$ is connected.  Therefore a path between $x_i$ and $x_j$
  exists in the graph $G^{ij}$ (one can use, e.g., breadth-first search or
  Dijkstra's algorithm to find such a path).  

  \emph{Step (ii).}
  For any $\{x_i,x_j\}\in E_\nEps\setminus E_\nVor$, we create a path in
  $G_\nVor$ as prescribed in step (i).  With these paths created, we upper
  bound the number of times any edge in $E_\nVor$ is used.  We do so by
  uniformly bounding above the number of times a vertex $x_k$ appears in
  these paths (and since each edge involves two vertices, this immediately
  yields an upper bound on the number of times an edge appears in these
  paths).  We split this into two substeps:
  \begin{enumerate}[label=(\alph*)]
    \item first, we derive a necessary condition for $x_k$ to appear in the path
      between $x_i$ and $x_j$;  
    \item then, we will upper bound the number of possible pairs $x_i, x_j$
      such that this necessary condition is satisfied.
  \end{enumerate}

  \emph{Step (ii~a).}
  For $x_k$ to appear in the path between $x_i$ and $x_j$ as designed in
  step (i), it is necessary for $\Part_k\in V^{ij}$.  Consider
  $x\in \Part_k\cap B(x_{ij}, r_{ij})$.  Since $x$ belongs to the Voronoi
  cell $\Part_k$,
  \begin{equation*}
    \lVert x-x_k\rVert < \min\{
      \lVert x-x_i\rVert, \lVert x-x_j\rVert
    \},
  \end{equation*}
  but since $x$ also lies in $B(x_{ij}, r_{ij})$,
  \begin{equation*}
    \lVert x-x_{ij}\rVert < r_{ij}.
  \end{equation*}
  It follows that,
  \begin{align*}
    \lVert x_k - x_{ij}\rVert
    &\leq \lVert x - x_k\rVert + \lVert x - x_{ij}\rVert  \\
    &\leq \lVert x - x_i\rVert + \lVert x - x_{ij}\rVert  \\
    &\leq \lVert x - x_{ij}\rVert + \lVert x_i - x_{ij}\rVert
      + \lVert x - x_{ij}\rVert \\
    &\leq 3r_{ij},
  \end{align*}
  thus if $\Part_k\in V^{ij}$, then it is necessary for
  $x_k\in B(x_{ij}, 3r_{ij})$.

  \emph{Step (ii~b).}
  Recalling $\varepsilon = 2\sqrt{d}/N$, where
  $N=C_1(p_\mn n/\log^\alpha n)^{1/d}$, we have a uniform upper bound of
  \begin{equation*}
    \max_{\{x_i,x_j\}\in E_\nEps} r_{ij}
    \leq C_8 \left(
      \frac{\log^\alpha n}{n} 
    \right)^{1/d},
  \end{equation*}
  for some $C_8>0$.  Thus, we conclude that for an edge of
  $x_k$ to be involved in a path between $x_i$ and $x_j$, it is necessary
  for
  \begin{equation*}
    x_{ij}\in B(x_k, 3C_8(\log n/n)^{1/d}),
  \end{equation*}
  or more loosely,
  \begin{equation*}
    x_i, x_j \in B(x_k, 4C_8(\log n/n)^{1/d})
  \end{equation*}
  recalling that $r_{ij} = \lVert x_{ij}-x_i\rVert = \lVert x_{ij}-x_j\rVert$
  and the uniform upper bound on $r_{ij}$.  Therefore, the number of paths
  in which any $x_k$ may appear is bounded above,
  \begin{align*}
    (nP_n(\cdot, 4C_8(\log^\alpha n/n)^{1/d}))^2
    &\leq C_9 \log^{2\alpha} n,
  \end{align*}
  where the final inequality is obtained
  by~\eqref{eq:vc-ball-ub-k}.
\end{proof}

%% file: app_technical_lemmas.tex
\subsection{Useful concentration results}

The following is an immediate consequence of the well-known fact that the set
of balls $B$ in $\Rd$ has VC dimension $d + 1$, e.g., Lemma~16
of~\citep{chaudhuri2010rates}. 
\begin{lemma}
	\label{lem:vc-balls}
  Suppose $x_1,\ldots,x_n$ are drawn from $P$ satisfying
  Assumption~\ref{assump:density_bounded}. There exist
  constants $C_1$-$C_5$ depending only on $d$, $p_\mn$, and $p_\mx$ such that
  the following statements hold: with
  probability at least $1 - \delta$, for any $z \in \XDom$,
	\begin{equation}
    \label{eq:vc-ball-ub-0}
    \left\{
      |B(z,r) \cap \{x_1,\ldots,x_n\}| = 0
    \right\}
    \Longrightarrow
    \left\{
      r < C_1\biggl(\frac{\log n + \log(1/\delta)}{n}\biggr)^{1/d}
    \right\},
	\end{equation}
  and 
  \begin{equation}
    \label{eq:vc-ball-ub-k}
    \begin{aligned}
      \left\{
        r < C_2\left(
          \frac{
            k - C_3(
              d \log n + \log(1/\delta) + \sqrt{k(d\log n + \log(1/\delta))}
            )
          }{n}
        \right)^{1/d}
      \right\}
        \hspace{-3cm}&
      \\
      &
      \Longrightarrow
      \left\{
        |B(z, r)\cap \{x_1,\dotsc,x_n\}| < k
      \right\}.
    \end{aligned}
  \end{equation}
  In particular, if $k \geq C_4 (\log(1/\delta))^2 \log n$, then
	\begin{equation}
	\label{eqn:vc-balls-2}
  \left\{
    |B(z,r) \cap \{x_1,\ldots,x_n\}|
      \geq k
  \right\}
  \Longrightarrow
  \left\{
    r \geq C_5\biggl(\frac{k}{n}\biggr)^{1/d}
  \right\}.
	\end{equation}
\end{lemma}

\subsection{Properties of the Voronoi diagram}%
\label{sub:properties_of_the_voronoi_diagram}
\input{app_technical_lemmas_voronoi}

\subsection{Proofs of technical lemmas for Theorem~\ref{thm:wavelet_bd}}%
\label{sub:proofs_of_technical_lemmas_for_theorem_thm_wavelet_bd}
\input{app_technical_lemmas_wavelet_bd}

%% file: app_technical_lemmas_voronoi.tex
\subsubsection{High probability control of cell geometry}%
\label{ssub:high_probability_control_of_cell_geometry}
The following lemma shows that with high probability, no Voronoi cell is very
large. Let $r(V_i) := \max\{\|x - x_i\|: x \in V_i\}$ be the radius of the
Voronoi cell $V_i$.
\begin{lemma}
  \label{lem:voronoi-cell-upper-bounds}
  Suppose $x_1,\ldots,x_n$ are drawn from $P$ satisfying
  Assumption~\ref{assump:density_bounded}.
  There exist constants $C_1$ and
  $C_2$ such that the following statement holds: for any $\delta \in (0,1)$,
  with probability at least $1 - \delta$
	\begin{equation}
    \label{eq:voronoi-radius}
    \max_{i = 1,\ldots,n} r(V_i)
    \leq C_1\biggl(\frac{\log n + \log(1/\delta)}{n}\biggr)^{1/d},
	\end{equation}
	and 
	\begin{equation}
    \label{eq:voronoi-mass}
    \max_{i = 1,\ldots,n} \Leb(V_i)
    \leq C_2\biggl(\frac{\log n + \log(1/\delta)}{n}\biggr).
	\end{equation}
\end{lemma}
\begin{proof}
  If $x \in V_i$, then $|B(x,\frac{1}{2}\|x - x_i\|) \cap \{x_1,\ldots,x_n\}| =
  0$. (Note that the same holds true if $\frac{1}{2}$ is replaced with any $a
  \in [0,1)$). Taking $x$ to be such that $\|x - x_i\| = r(V_i)$, it follows by
  Lemma~\ref{lem:vc-balls} that
	\begin{equation*}
    \frac{1}{2}r(V_i)
      = \frac{1}{2}\|x - x_i\|
      \leq C\biggl(\frac{\log n + \log(1/\delta)}{n}\biggr)^{1/d},
	\end{equation*}
  with probability at least $1 - \delta$. Multiplying both sides by $2$ and
  taking a maximum over $i = 1,\ldots,n$ gives~\eqref{eq:voronoi-radius}. The
  upper bound~\eqref{eq:voronoi-mass} on the maximum Lebesgue measure of $V_i$
  follows immediately, since $V_i \subseteq B(x,r(V_i))$. 
\end{proof}

\subsubsection{Connectedness of the Voronoi adjacency graph}%
\label{ssub:connectedness_of_the_voronoi_adjacency_graph}
The following lemma relates graph theoretic connectedness to a kind of
topological connectedness that excludes connectedness using sets of
$\cH^{d-1}$-measure zero.

\begin{lemma}
  \label{lem:topological-connectedness-equiv-graph-connectedness}
  Let $\XDom\subset\R^d$ be open such that there does \emph{not} exist
  any set $S\subsetneq\XDom$ with $\cH^{d-1}(S)=0$ such that
  $\XDom\setminus S$ is disconnected.  Let $\{\Part_1,\dotsc,\Part_m\}$
  denote an open polyhedral partition of $\XDom$.  Then the graph
  $G=(\{\Part_1,\dotsc,\Part_m\}, E)$, where
  \begin{equation*}
    E   =  \{\{\Part_i,\Part_j\}:
        \cH^{d-1}(\partial\Part_i\cap\partial\Part_j) > 0\},
  \end{equation*}
  is connected.
\end{lemma}
\begin{proof}
  Assume by way of contradiction that $G$ is disconnected.  Therefore there
  exists sets of vertices $\cC_1$, $\cC_2$ such that
  \begin{equation}
    \label{eq:app-proof-disconnected-graph}
    \cH^{d-1}(\bar\Part_i\cap\bar\Part_j) = 0,
  \end{equation}
  for all $\Part_i\in\cC_1$, $\Part_j\in\cC_2$.  Next, define
  \begin{align*}
    \XDom_1 &:= (\cup_{\Part_i\in\cC_1}\bar\Part_i)^\circ,  \\
    \XDom_2 &:= (\cup_{\Part_j\in\cC_2}\bar\Part_j)^\circ,
  \end{align*}
  such that $\{\XDom_1, \XDom_2\}$ constitutes an open partition of $\XDom$.
  Let 
  \begin{align}
    \label{eq:app-proof-disconnected-set}
    S &:= \XDom\setminus(\XDom_1\cup\XDom_2)  \\
      \nonumber
      &=  \XDom \cap(
        (\partial\XDom_1\cap\partial\XDom_2)
        \cup ((\XDom_1^c)^\circ\cap(\XDom_2^c)^\circ)
      ) \\
      \nonumber
      &=  \XDom \cap (
        (\partial\XDom_1\cap\partial\XDom_2)
        \cup (\XDom_2\cap\XDom_1)
      ) \\
      \label{eq:app-proof-disconnected-set-2}
      &=  \XDom\cap\partial\XDom_1\cap\partial\XDom_2.
  \end{align}
  From~\eqref{eq:app-proof-disconnected-graph}
  and~\eqref{eq:app-proof-disconnected-set-2} we see that
  $\cH^{d-1}(S) = 0$.  On the other hand,
  \eqref{eq:app-proof-disconnected-set} yields that
  $\XDom\setminus S = \XDom_1\cup\XDom_2$ is disconnected
  ($\XDom_1,\XDom_2$ are open and disjoint).
\end{proof}

\subsubsection{Analysis of the Voronoi kernel}%
\label{ssub:analysis_of_the_voronoi_kernel}
Recall that in the proof of Theorem~\ref{thm:asymptotic_limit_voronoi}, we
compare Voronoi TV to a U-statistic involving the kernel function
\begin{equation*}
H_{\mathrm{Vor}}(x,y) = \mathbb{E}[\mc{H}^{d - 1}(\partial V_{x_1} \cap
\partial V_{x_2})|x_1 = x,x_2 = y] = \int_{L \cap \Omega} (1 - p_x(z))^{(n -
2)} \,dz.
\end{equation*}
The following lemma shows that this kernel function is close to a spherically
symmetric kernel.
\begin{lemma}
	\label{lem:voronoi-boundary}
  Suppose $x_1,\ldots,x_n$ are sampled from distribution $P$
  satisfying~\ref{assump:density_bounded}. There exist constants $C_1$-$C_4 >
  0$ such that for $h = h_n = C_1(3\log n/n)^{1/d}$, the following statements
  hold. 
	\begin{itemize}
		\item For any $x,y \in \Omega_{h}$,
		\begin{equation}
		\label{eqn:voronoi-boundary-1}
    H_{\mathrm{Vor}}(x,y) = \frac{\eta_{d - 2}}{(np(x))^{\frac{d - 1}{d}}}
    K_{\mathrm{Vor}}\biggl(\frac{\|y - x\|}{\varepsilon_{(1)}}\biggr) +
    O\biggl(\frac{1}{n^3} + \frac{(\log n)^2}{n}\1\bigl\{\|x - y\| \leq
    C_2(\log n/n)^{1/d}\bigr\}\biggr)
		\end{equation} 
		\item For any $x,y \in \Omega$,
		\begin{equation}
		\label{eqn:voronoi-boundary-2}
    H_{\mathrm{Vor}}(x,y) \leq \frac{C_3}{n^{(d - 1)/d}}
    K_{\mathrm{Vor}}\biggl(\frac{\|y - x\|}{C_4n^{1/d}}\biggr).
		\end{equation}
	\end{itemize}
\end{lemma} 
\begin{proof}[Proof of~\eqref{eqn:voronoi-boundary-1}]
	
  We now replace the integral above with one involving an exponential function
  that can be more easily evaluated. Then we evaluate this latter integral.
	
	\paragraph{Step 1: Reduction to easier integral.}
  Let $\Omega_x = \{z \in \Omega: \mathrm{dist}(z,\partial\Omega) > \|z -
  x\|\}$. (Note that $L \cap \Omega_x = L \cap \Omega_y$.) Separate the
  integral into two parts,
	\begin{align*}
  \int_{L \cap \Omega} \Bigl(1 - p_x(z)\Bigr)^{(n - 2)} \,dz & = \int_{L \cap
  \Omega_x} \Bigl(1 - p_x(z)\Bigr)^{(n - 2)} \,dz + \int_{L \cap
(\Omega\setminus\Omega_x)} \Bigl(1 - p_x(z)\Bigr)^{(n - 2)} \,dz.
	\end{align*} 
  We start by showing that the second term above is negligible for $x,y \in
  \Omega_h$.  For any $z \in \Omega \setminus \Omega_x$, it follows by the
  triangle inequality that
	\begin{equation*}
  \mathrm{dist}(x,\partial \Omega) \leq \|x - z\| + \mathrm{dist}(z,\Omega)
  \leq 2 \|x - z\|.
	\end{equation*}  
  Since $x \in \Omega_h$, it follows that $p_x(z) \geq (p_{\min}/2d) \|z -
  x\|^d \geq (p_{\min}/2^{d + 1}d) (\mathrm{dist}(x,\partial\Omega))^d \geq
  (p_{\min}/2^{d + 1}d) h^d$. Integrating over $z \in \Omega \setminus
  \Omega_x$ implies an upper bound on the second term, 
	\begin{align*}
  \int_{L \cap (\Omega\setminus\Omega_x)} \Bigl(1 - p_x(z)\Bigr)^{(n - 2)} \,dz
  & \leq \int_{L \cap (\Omega\setminus\Omega_x)} \exp(-(n - 2)p_x(z)) \,dz \\
	& = O(\exp(-(p_{\min}/2^{d + 2}d)nh^d)) \\
	& = O(\frac{1}{n^3}),
	\end{align*} 
  with the last line following upon choosing $C_1 \geq (p_{\min}/2^{d +
  2}d)^{-1/d}$ in the definition of $h$.
	
  On the other hand, if $z \in \Omega_x$ then $B(z,\|z - x\|) \subset \Omega$.
  Consequently, letting $\wt{p}_x(z) := p(x) \Leb_d \|z - x\|^d$, it follows by
  the Lipschitz property of $p$ that
	\begin{equation*}
  |p_x(z) - \wt{p}_x(z)| \leq \int_{B(z,\|z - x\|)} |p(z) - p(x)| \,dz \leq C
  \Leb_d \|z - x\|^{d + 1},
	\end{equation*}
	and
	\begin{equation*}
  |\exp\bigl(-np_x(z)\bigr) - \exp\bigl(-n\wt{p}_x(z)\bigr)| \leq C \Leb_d \|z
  - x\|^{d + 1} n. 
	\end{equation*}
  Additionally recall that $\exp(-np) \geq (1 - p)^n \geq \exp(-np)(1 - np^2)$
  for any $|p| < 1$. Combining these facts, we conclude that
	\begin{equation}
	\label{pf:voronoi-boundary-1}
	\begin{aligned}
  \int_{L \cap \Omega_x} \Bigl(1 - p_x(z)\Bigr)^{n} \,dz & = \int_{L \cap
  \Omega_x} \exp(-np_x(z)) \bigl(1 + O(np_x(z)^2)\bigr) \,dz \\
  & = \int_{L \cap \Omega_x} \exp(-n\wt{p}_x(z)) \Bigl(1 + O(n\|z - x\|^{2d}) +
  O(n\|z - x\|^{d + 1})\Bigr) \,dz \\
  & \overset{(i)}{=} \int_{L \cap \Omega_x} \exp(-np(x)\Leb_d\|x - z\|^d) \,dz
  + O\Bigl(\frac{1}{n^3} + \frac{1}{n}\1\bigl\{\|x - y\| \leq C_2(\log
  n/n)^{1/d}\bigr\}\Bigr) \\
  & \overset{(ii)}{=} \int_{L} \exp\bigl(-np(x)\Leb_d\|x - z\|^d\bigr) \,dz +
  O\Bigl(\frac{1}{n^3} + \frac{(\log n)^2}{n}\1\bigl\{\|x - y\| \leq C_2(\log
  n/n)^{1/d}\bigr\}\Bigr).
	\end{aligned}
	\end{equation}
  We prove the last two equalities, which control the remainder terms, after
  completing our analysis of the leading order term.  
	
	\paragraph{Step 2: Leading order term.}
  Let $r = \|x - y\|/2$. Due to rotational symmetry, we may as well take $x =
  re_1, y = -re_1$, in which case the integral becomes
	\begin{align*}
  \int_{L} \exp\bigl(-np(x)\Leb_d\|x - z\|^d\bigr) \,dz & = \int_{\{0\} \times
  \Reals^{d - 1}} \exp\bigl(-np(x)\Leb_d\|re_1 - z\|^d\bigr) \,dz \\
	& = \int_{\Reals^{d - 1}} \exp\bigl(-np(x)\Leb_d(r^2 + \|z\|^2)^{d/2}\bigr) \,dz,
	\end{align*}
  with the latter equality following from the Pythagorean theorem. Converting
  to polar coordinates, we see that
	\begin{align*}
  \int_{\Reals^{d - 1}} \exp\bigl(-np(x)\Leb_d(r^2 + \|z\|^2)^{d/2}\bigr) \,dz
  & = \int_{0}^{\infty} \int_{\mathbb{S}^{d - 2}} \exp\bigl(-np(x)\Leb_d(r^2 +
  t^2)^{d/2}\bigr)t^{d - 2} \,d\theta \,dt  \\
  & = \eta_{d - 2}\int_{0}^{\infty} \exp\bigl(-np(x)\Leb_d(r^2 +
  t^2)^{d/2}\bigr)t^{d - 2} \,dt \\
  & = \frac{\eta_{d - 2}}{(np(x))^{\frac{d - 1}{d}}} \int_{0}^{\infty}
  \exp\Bigl(-\Leb_d\bigl\{\bigl(r (n p(x))^{1/d}\bigr)^{2} +
  s^2\bigr\}^{d/2}\Bigr)s^{d - 2}\,ds, \\
  & = \frac{\eta_{d - 2}}{(np(x))^{\frac{d - 1}{d}}}
  K_{\mathrm{Vor}}\biggl(\frac{\|y - x\|}{\varepsilon_{(1)}}\biggr),
	\end{align*}
  with the second to last equality following by substituting $s =
  t/(np(x))^{-1/d}$. 
	
	\paragraph{Controlling remainder terms.}
  We complete the proof of~\eqref{eqn:voronoi-boundary-1} by establishing $(i)$
  and $(ii)$ in \eqref{pf:voronoi-boundary-1}. 
	
	\underline{Proof of $(i)$.}
  Take $\varepsilon_{0} = (4\log n/\Leb_d p_{\min} n)^{1/d}$, and note that if
  $\|z - x\| \geq \varepsilon_{0}$ then $\exp\bigl(-\Leb_d n p(x) \|z -
  x\|^{d}\bigr) \leq \frac{1}{n^4}$. Recalling the definition of $\wt{p}_x(z)$,
  we have 
	\begin{equation}
	\label{pf:voronoi-boundary-2}
	\begin{aligned}
  n\int_{L \cap \Omega_x} \exp\bigl(-n\wt{p}_x(z)\bigr)\|z - x\|^{d + 1} \,dz &
  = n\int_{L \cap \Omega_x} \exp\bigl(-\Leb_d n p(x) \|z - x\|^{d}\bigr)\|z -
  x\|^{d + 1} \,dz \\
  & \leq n\int_{L \cap B(x,\varepsilon_{0})} \exp\bigl(-\Leb_d n p_{\min} \|z -
  x\|^{d}\bigr)\|z - x\|^{d + 1} \,dz + \frac{\mc{H}^{d - 1}(L \cap
\Omega)}{n^3} \\
  & \leq n \varepsilon_0^{d + 1}\int_{L \cap B(x,\varepsilon_{0})}
  \exp\bigl(-\Leb_d n p_{\min} \|z - x\|^{d}\bigr) \,dz + \frac{\mc{H}^{d -
  1}(L \cap \Omega)}{n^3} \\
  & \leq  n \varepsilon_0^{d + 1} \mc{H}^{d - 1}(L \cap B(x,\varepsilon_{0})) +
  \frac{\mc{H}^{d - 1}(L \cap \Omega)}{n^3}.
	\end{aligned}
	\end{equation}
  For any $x,y$ we have $\mc{H}^{d - 1}(L \cap B(x,\varepsilon_0)) \leq \Leb_{d
  - 1}\varepsilon_{0}^{d - 1}$. If additionally $\|x - y\|/2 > \varepsilon_0$
  then $L \cap B(x,\varepsilon_0) = \emptyset$, and so $\mc{H}^{d - 1}(L \cap
  B(x,\varepsilon_0)) = 0$. Compactly, these estimates can be written as
	\begin{equation*}
  \mc{H}^{d - 1}(L \cap B(x,\varepsilon_{0})) \leq \Leb_{d - 1}\1\{\|x - y\|
  \leq 2\varepsilon_0\} \varepsilon_{0}^{d - 1}.
	\end{equation*} 
  Plugging this back into~\eqref{pf:voronoi-boundary-2}, we conclude that
	\begin{align*}
  n\int_{L \cap \Omega_x} \exp\bigl(-n\wt{p}_x(z)\bigr)\|z - x\|^{d + 1} \,dz &
  \leq n\varepsilon_{0}^{2d}\1\{\|x - y\| \leq 2 \varepsilon_{0}\} +
  \frac{\mc{H}^{d - 1}(L \cap \Omega)}{n^3} \\
  & \leq C\Bigl(\frac{(\log n)^2}{n}\1\{\|x - y\| \leq C_2 (\log n/n)^{1/d}\} +
  \frac{1}{n^3}\Bigr),
	\end{align*}
  for $C_2 = 2(4/(p_{\min}\Leb_d))^{1/d}$. This is precisely the claim.
	
	\underline{Proof of $(ii)$.}
  Recall the fact established previously, that if $z \in L \setminus \Omega_x$
  then $\|z - x\| \geq h/2$. Therefore,
  \begin{align*}
	\int_{L \setminus \Omega_x} \exp\bigl(-n\wt{p}_x(z)\bigr) \,dz &
  \leq \int_{L \setminus \Omega_x} \exp\bigl(-\Leb_d n p_{\min} \|z - x\|^{d}\bigr) \,dz \\
  & \leq \int_{L \setminus B(x,2)} \exp\bigl(-\Leb_d n p_{\min} \|z -
  x\|^{d}\bigr) \,dz + \int_{(L \cap B(x,2)\setminus\Omega_x)}
  \exp\bigl(-\Leb_d p_{\min} n(h/2)^{d}\bigr)\,dz \\
  & \leq \int_{L \setminus B(x,2)} \exp\bigl(-\Leb_d n p_{\min} \|z -
  x\|^{d}\bigr) \,dz + \frac{\mc{H}^{d - 1}(L \cap B(x,2))}{n^3},
	\end{align*}
  with the last inequality following upon choosing $C_1 \geq 2/(\Leb_d
  p_{\min})^{1/d}$ in the definition of $h$. The remaining integral is
  exponentially small in $n$, proving the upper bound $(ii)$.
\end{proof}
\begin{proof}[Proof of~\eqref{eqn:voronoi-boundary-2}]
	Note immediately that 
	\begin{equation*}
  H_{\mathrm{Vor}}(x,y) \leq \int_{L \cap \Omega} \exp(-n p_x(z)) \,dz \leq
  \int_{L} \exp(-n \Leb_d p_{\min} \|x - z\|^d/2d) \,dz. 
	\end{equation*}
  We have already analyzed this integral in the proof
  of~\eqref{eqn:voronoi-boundary-1}, with the analysis implying that 
	\begin{equation*}
  \int_{L} \exp(-n \Leb_d p_{\min} \|x - z\|^d/2d) \,dz = \frac{\eta_{d -
  2}(2d)^{\frac{d - 1}{d}}}{(np_{\min})^{(d - 1)/d}}
  K_{\mathrm{Vor}}\biggl(\frac{\|y - x\|}{(2dn/p_{\min})^{1/d}}\biggr).
	\end{equation*}
  This is exactly~\eqref{eqn:voronoi-boundary-2} with $C_3 = \eta_{d -
  2}(2d/p_{\min})^{(d - 1)/d}$ and $C_4 = (2d/p_{\min})^{1/d}$.
\end{proof}

\subsubsection{Compact kernel approximation}

The kernel function $H_{\mathrm{Vor}}(x,y)$ is not compactly supported, and in
our analysis it will frequently be convenient to approximate it by a compactly
supported kernel. The following lemma does the trick. Let $\varepsilon_{0} :=
(\log n/n)^{1/d}$.
\begin{lemma}
	\label{lem:voronoi-kernel-ub}
  Let $x,y \in \Omega$, and $L = \{z: \|x - z\| = \|y - x\|\}$. For any $a,c >
  0$, there exists  a constant $C > 0$ depending only on $a,c$ and $d$ such
  that
	\begin{equation}
	\label{eqn:voronoi-kernel-ub}
	\int_{L \cap \Omega} \exp(-cn\|x - z\|^d)\,dz
  \leq C\biggl(\frac{\1\{\|x - y\|
  \leq C\varepsilon_{0}\}}{n^{(d - 1)/d}} + \frac{1}{n^a}\biggr)
	\end{equation}
  where $\varepsilon_0 := (\log n/n)^{1/d}$.
\end{lemma}
\begin{proof}
  Let $\wt{\varepsilon}_{0} = C_1 \varepsilon_{0}$ for $C_1 = (a/c)^{1/d}$. The
  key is that if $\|x - z\| \geq \wt{\varepsilon}_{0}$, then
	\begin{equation*}
	\exp(-cn\|x - z\|^d) \leq \frac{1}{n^{a}}.
	\end{equation*}
  Now suppose $\|y - x\| > 2\wt{\varepsilon}_{0}$. Then $\|x - z\| \geq
  \wt{\varepsilon}_{0}$ for all $z \in L$, and
	\begin{equation*}
	\int_{L \cap \Omega} \exp(-cn\|x - z\|^d)\,dz
  \leq \frac{\mc{H}^{d - 1}(L \cap \Omega)}{n^a}.
	\end{equation*}
	It follows that
	\begin{equation*}
	\begin{aligned}
	\label{pf:voronoi-kernel-ub-1}
	\int_{L \cap \Omega} \exp(-cn\|x - z\|^d)\,dz &
  \leq \1\{\|y - x\| \leq 2\wt{\varepsilon}_{0}\}
  \int_{L \cap \Omega} \exp(-cn\|x - z\|^d)\,dz
  + \frac{\mc{H}^{d - 1}(L \cap \Omega)}{n^a} \\
	& \leq \1\{\|y - x\| \leq 2\wt{\varepsilon}_{0}\}
  \int_{B_{d - 1}((x + y)/2,\wt{\varepsilon}_{0})}
  \exp(-cn\|x - z\|^d)\,dz + 2\frac{\mc{H}^{d - 1}(L \cap \Omega)}{n^a} \\
	& \leq \frac{\1\{\|y - x\|
  \leq 2\wt{\varepsilon}_{0}\}}{n^{(d - 1)/d}}
  \int_{\R^{d - 1}} \exp(-\|z\|^d)\,dz
  + 2\frac{\mc{H}^{d - 1}(L \cap \Omega)}{n^a} \\
	& \leq C_2\biggl(
    \frac{\1\{\|y - x\| \leq 2\wt{\varepsilon}_{0}\}}{n^{(d - 1)/d}}
    + \frac{1}{n^a}\biggr).
	\end{aligned}
	\end{equation*}
	for $C_2 = \max\{\int_{\R^{d - 1}}
  \exp(-\|z\|^d)\,dz,2\mc{H}^{d - 1}(L \cap \Omega)\}$.
  Equation~\eqref{eqn:voronoi-kernel-ub} follows upon taking
  $C = \max\{2C_1,C_2\}$.
\end{proof}

%% file: app_technical_lemmas_wavelet_bd.tex
\subsubsection{Proofs of Lemmas~\ref{lem:linf-wavelet} and~\ref{lem:bv-wavelet}}%
\label{ssub:proofs_of_lemmas_linf_wavelet_bv_wavelet}

\begin{proof}[Proof of Lemma~\ref{lem:linf-wavelet}.]
  For each $k \in \mc{K}(\ell)$ and $\bf{i} \in \mc{I}$, it follows from
  H\"{o}lder's inequality that
	\begin{equation*}
	|\theta_{\ell k}^{{\bf i}}(f)| \leq \|f\|_{L^{\infty}(\XDom)}
  \|\Psi_{\ell k}^{{\bf i}}\|_{L^1(\XDom)} = \|f\|_{L^{\infty}(\XDom)} 2^{-\ell d/2},
	\end{equation*}
  and taking supremum over $k \in \mc{K}(\ell)$ and $\bf{i} \in \mc{I}$ gives
  the result.
\end{proof}

\begin{proof}[Proof of Lemma~\ref{lem:bv-wavelet}.]
  The proof hinges on an application of an integration by parts
  identity~\eqref{pf:bv-wavelet1}, valid for all $f \in C^1(\XDom)$. We thus
  first derive~\eqref{eqn:bv-wavelet} for all $f \in C^1(\XDom)$, before
  returning to complete the full proof. 
	
  Now, taking $f \in C^1(\XDom)$, a simple calculation verifies that for each
  $i = 1,\ldots,d$, and all $\Psi_{\ell k}^{{\bf i}}$ such that ${\bf i} \in
  \mc{I}$, we have
	\begin{equation}
	\label{pf:bv-wavelet1}
	\int_{\XDom} f(x) \Psi_{\ell k}^{{\bf i}}(x) \,dx
  = - \int_{\XDom} D_if I_i\Psi_{\ell k}^{{\bf i}}(x) \,dx.
	\end{equation}
  Here $I_1f(x) = \int_{-\infty}^{x_1} f((t,x_2,\ldots,x_d)) \,dt$ is the
  partial integral operator in the $1$st coordinate, and $I_i$ are defined
  likewise. Now, we introduce some notation: for all $x, y \in \XDom$, and for
  each $i = 1,\ldots,d$, take
	\begin{equation*}
	K_{\ell}^{i}(x,y)
  = \sum_{k \in \mc{K}(\ell), {\bf i} \in \mc{I}_i}
  \Psi_{\ell,k}^{{\bf i}}(x) \Psi_{\ell,k}^{{\bf i}}(y),
  \quad \Lambda_{\ell}^{i}(x,y)
  = \sum_{k \in \mc{K}(\ell), {\bf i} \in \mc{I}_i}
  \Psi_{\ell,k}^{{\bf i}}(x) I_i\Psi_{\ell,k}^{{\bf i}}(y).
	\end{equation*}
	By definition, $K_{\ell}^i$ is the integral operator such that
	\begin{equation*}
	P_{\ell,i}f(x)
  := \sum_{k \in \mc{K}(\ell), {\bf i} \in \mc{I}_i}
  \theta_{k\ell}^{{\bf i}}(f) \Psi_{\ell k}^{{\bf i}}(x)
  = \int f(y) K_{\ell}^{i}(x,y) \,dy,
	\end{equation*}
  and we may use the integration by parts identity~\eqref{pf:bv-wavelet1} to
  obtain
	\begin{equation}
	P_{\ell,i}f = - \int D_if(y) \Lambda_{\ell}^{i}(\cdot,y) \,dy.
	\end{equation}
  Taking absolute value, integrating over $\XDom$ and applying Fubini's
  theorem, we determine that
	\begin{equation*}
	\begin{aligned}
	\|P_{\ell,i}f\|_{L^1(\XDom)}
  & \leq \int_{\XDom} \int_{\XDom} |D_if(y)| |\Lambda_{\ell}^{i}(x,y)| \,dx \,dy \\
	& \leq \|D_if\|_{L^1(\XDom)}
  \cdot \sup_{y \in \XDom} \|\Lambda_{\ell}^{i}(\cdot,y)\|_{L^1(\XDom)} \\
	& \leq \|D_if\|_{L^1(\XDom)}
  \cdot \sup_{y \in \XDom} \sum_{k \in \mc{K}(\ell), {\bf i}
  \in \mc{I}_i} \|\Psi_{\ell k}^{{\bf i}}\|_{L^1(\XDom)}
  \cdot |I_i\Psi_{\ell,k}^{{\bf i}}(y)| \\
	& \leq 2^d 2^{-\ell}\|D_if\|_{L^1(\XDom)}.
	\end{aligned}
	\end{equation*}
	Now, for each $i = 1,\ldots,d$, take
	\begin{equation*}
	\theta_{\ell \cdot}^{i}(f)
  = \bigl(\theta_{\ell k}^{{\bf i}}(f): k \in \mc{K}(\ell),
  {\bf i} \in \mc{I}_{i}\bigr),
	\end{equation*}
  where $\mc{I}_i \subset \mc{I}$ contains all indices ${\bf i} \in \mc{I}$ for
  which ${\bf i}_j = 0$ for all $j < i$, and ${\bf i}_i = 1$. 
	
  Using the $L^2(\XDom)$ orthogonality property of the Haar basis and applying
  H\"{o}lder's inequality gives
	\begin{equation*}
	\|\theta_{\ell\cdot}^i(f)\|_1
  = \|\theta_{\ell\cdot}^i(P_{\ell,i}f)\|_1
  \leq \|P_{\ell,i}f\|_{L^1(\XDom)}
  \cdot \|\sum_{k \in \mc{K}, {\bf i} \in \mc{I}_i}
  \Psi_{\ell k}^{{\bf i}}\|_{L^{\infty}(\XDom)}
  \leq 2^{{\ell d/2}} \|P_{\ell,i}f\|_{L^1(\XDom)}.
	\end{equation*}
  and summing up over $i = 1,\ldots,d$ gives the desired upper bound on
  $\|\theta_{\ell \cdot}(f)\|_1$.
	
  Finally, a density argument will imply the same result holds for any $f \in
  \mathrm{BV}(\XDom)$. In particular, there exists a sequence $\{f_n\} \subset
  C^1(\XDom)$ for which 
	\begin{equation}
	\lim_{n \to \infty}\|f_n - f\|_{L^1(\XDom)} \to 0,
  \quad \lim_{n \to \infty} \mathrm{TV}(f_n;\XDom) = \mathrm{TV}(f;\XDom),
	\end{equation}
  see, e.g., \citet[Theorem 5.3]{evans2015measure}; consequently
	\begin{equation*}
	\|\theta_{k\cdot}(f)\|_1 = \lim_{n \to \infty}
  \|\theta_{k\cdot}(f_n)\|_1 \leq d2^d 2^{\ell(1 - d/2)}
  \cdot \lim_{n \to \infty} \mathrm{TV}(f_n;\XDom)
  = d2^d 2^{\ell(1 - d/2)}  \mathrm{TV}(f;\XDom),
	\end{equation*}
	and the proof is complete upon taking $C_1 = d2^d$.
\end{proof}

\subsubsection{Proof of Proposition~\ref{prop:wavelet-estimation-error-ub}}%
\label{ssub:proof_prop_wavelet_estimation_error_ub}
The proof of Proposition~\ref{prop:wavelet-estimation-error-ub} follows in
spirit the analysis of~\citep{donoho1995wavelet}. First we upper bound the
$\ell^2$-loss by the sum of two terms, a modulus of continuity and tail
width, then we proceed to separately upper bound each term. The primary
difference between our situation and that of~\citep{donoho1995wavelet} is
that we are working with respect to intersections of Besov bodies rather than
Besov bodies themselves.\\

For the rest of this proof, take $\wh{\theta} =
\wh{\theta}^{(\lambda,\ell^{\ast})}$ and $\theta_0 = \theta(f_0)$.\\

\noindent \underline{Step 1: Upper bound involving modulus of continuity and
tail width.} We are going to establish that
\begin{equation}
\label{eqn:optimal-recovery}
\|\wh{\theta} - \theta_0\|_2 \leq \XDom\Bigl(\Theta_{\infty}^{0,\infty}(2M)
\cap \Theta_{\infty}^{1,1}(2L), \epsilon_n + \lambda\Bigr) +
\Delta\Bigl(\Theta_{\infty}^{0,\infty}(M) \cap \Theta_{\infty}^{1,1}(L),
\ell^{\ast}\Bigr),
\end{equation}
where $\XDom(\Theta,\epsilon)$ is the modulus of continuity
\begin{equation}
\label{eqn:modulus-continuity}
\XDom(\Theta,\epsilon) := \sup_{\theta,\theta \in \Theta} \bigl\{\|\theta -
\theta'\|_{2}: \|\theta - \theta'\|_{\infty} \leq \epsilon\bigr\},
\end{equation}
and $\Delta(\Theta,\ell)$ is the tail width
\begin{equation}
\label{eqn:tail-width}
\Delta(\Theta,\ell) := \sup_{\theta \in \Theta} \|\theta - \theta_{\leq \ell}\|_2.
\end{equation}
Observe that by the triangle inequality,
\begin{equation*}
\|\wh{\theta} - \theta_0\|_2 \leq \|\wh{\theta} - \theta_{0,\leq
\ell^{\ast}}\|_2 + \|\theta_0 - \theta_{0,\leq \ell^{\ast}}\|_2.
\end{equation*}
The second term on the right hand side of the previous expression is upper
bounded by $\Delta\bigl(\Theta_{\infty}^{0,\infty}(M) \cap
\Theta_{\infty}^{1,1}(L), \ell^{\ast}\bigr)$. We turn now to upper bounding
the first term by the modulus of continuity. Observe that for each index, we
are in one of two situations: either
\begin{equation*}
|\wt{\theta}_{\ell,k}^{{\bf i}}| < \lambda
\Longrightarrow |\wh{\theta}_{\ell k}^{{\bf i}}|
= 0~~\quad\textrm{and}~~|\wh{\theta}_{\ell k}^{{\bf i}}
- \theta_{0,\ell k}^{{\bf i}}|
= |{\theta}_{0,\ell k}^{{\bf i}}| \leq \lambda + \epsilon_n, 
\end{equation*}
or 
\begin{equation*}
|\wt{\theta}_{\ell,k}^{{\bf i}}|
\geq \lambda \Longrightarrow |\theta_{0,\ell k}^{{\bf i}}|
\geq |\wt{\theta}_{\ell,k}^{{\bf i}}| 
- \epsilon_n \geq \frac{1}{2}|\wt{\theta}_{\ell k}|
= \frac{1}{2}|\wh{\theta}_{\ell k}| ~~\textrm{and}~~|\wh{\theta}_{\ell k}
- \theta_{0,\ell k}^{{\bf i}}| \leq \epsilon_n.
\end{equation*}
It follows that for every $\ell \in \mathbb{N} \cup \{0\}, k \in
\mc{K}(\ell)$, we have $|\wh{\theta}_{\ell k}| \leq 2|\theta_{0,\ell k}^{{\bf
i}}|$, and so $\wh{\theta} \in \Theta_{\infty}^{0,\infty}(2M) \cap
\Theta_{\infty}^{1,1}(2L)$. Moreover, the above calculations also confirm
that $\|\wh{\theta} - \theta_{0,\leq \ell^{\ast}}\|_{\infty} \leq \lambda +
\epsilon_n$. Thus 
\begin{equation*}
\|\wh{\theta} - \theta_{0,\leq \ell^{\ast}}\|_2
\leq \XDom\bigl(\Theta_{\infty}^{0,\infty}(2M)
  \cap \Theta_{\infty}^{1,1}(2L), \epsilon_n + \lambda),
\end{equation*}
establishing~\eqref{eqn:optimal-recovery}.\\

\noindent \underline{Step 2: Tail width.}
Fix $\theta \in \Theta_{\infty}^{0,\infty}(M) \cap \Theta_{\infty}^{1,1}(L)$.
For each $\ell \in \mathbb{N} \cup \{0\}$ we have
\begin{equation*}
\|\theta_{\ell \cdot}\|_2^2
\leq \|\theta_{\ell \cdot}\|_1 \|\theta_{\ell \cdot}\|_{\infty}
\leq C_1 2^{-\ell} L M,
\end{equation*}
with the first inequality following from~H\"{o}lder, and the second
inequality from Lemmas~\ref{lem:linf-wavelet} and~\ref{lem:bv-wavelet}.
Summing over $\ell = \ell^{\ast} + 1,\ell^{\ast} + 2,\ldots$ gives
\begin{equation*}
\|\theta - \theta_{\leq \ell^{\ast}}\|_2^2
= \sum_{\ell = \ell^{\ast} + 1}^{\infty} \|\theta_{\ell \cdot}\|_2^2
\leq 2 C_1 L M 2^{-\ell^{\ast}},
\end{equation*}
and it follows that
\begin{equation}
\label{pf:wavelet-estimation-error-ub-1}
\Delta(\Theta_{\infty}^{0,\infty}(M)
\cap \Theta_{\infty}^{1,1}(L), \ell^{\ast})
\leq \sqrt{2 C_1 L M} 2^{-\ell^{\ast}/2}.
\end{equation}
\noindent \underline{Step 3: Modulus of continuity.}
Fix $\theta,\theta' \in \Theta_{\infty}^{0,\infty}(2M) \cap
\Theta_{\infty}^{1,1}(2L)$ such that $\|\theta - \theta'\|_{\infty} \leq
\epsilon$. For each $\ell \in \mathbb{N} \cup \{0\}$ we have 
\begin{equation}
\label{pf:modulus-continuity-1}
\begin{aligned}
\|\theta_{\ell \cdot} - \theta_{\ell \cdot}'\|_2^2
& \leq \min\{2^{\ell d} \|\theta_{\ell \cdot}
- \theta_{\ell \cdot}'\|_{\infty}^2, \|\theta_{\ell \cdot}
- \theta_{\ell \cdot}'\|_{\infty} \|\theta_{\ell \cdot}
- \theta_{\ell \cdot}'\|_{1}\} \\
& \leq \min\{2^{\ell d} \epsilon^2,
2 C_1 \epsilon L 2^{\ell(d/2 - 1)}  ,4 C_1 M L 2^{-\ell} \},
\end{aligned}
\end{equation}
with the final inequality following from the triangle inequality, i.e.
$\|\theta - \theta'\| \leq \|\theta\| + \|\theta'\|$, and
Lemmas~\ref{lem:linf-wavelet} and~\ref{lem:bv-wavelet}. 

The three upper bounds in~\eqref{pf:modulus-continuity-1} divide $\mathbb{N}
\cup \{0\}$ into three zones, based on the indices $\ell$ for which each
upper bound is tightest. 
\begin{itemize}
  \item The first~\emph{dense} zone contains $\ell = 0,\ldots,\lceil
    \log_2(2C_1L/\epsilon) \cdot \frac{2}{2 + d}\rceil =: N_1$. In the dense
    zone, the extremal vectors $\theta$ are dense, i.e. everywhere non-zero. 
  \item The second~\emph{intermediate} zone contains $\ell = N_1 + 1, \ldots,
    \lceil\log_2(2M/\epsilon) \cdot \frac{2}{d}\rceil =: N_2$. In the
    intermediate zone, the extremal vectors $\theta$ are neither dense nor
    sparse. 
  \item The third~\emph{sparse} zone contains $\ell = N_2 + 1, \ldots$. In
    the sparse zone, the extremal vectors are sparse, i.e. they have exactly
    one non-zero entry.
\end{itemize}
Summing over $\ell \in \mathbb{N} \cup \{0\}$ gives
\begin{align}
\label{pf:modulus-continuity-2}
\|\theta - \theta'\|_2^2
& = \sum_{\ell = 0}^{\infty} \|\theta_{\ell \cdot}
- \theta_{\ell \cdot}'\|_2^2 \nonumber \\
& \leq \epsilon^2 \sum_{\ell = 0}^{N_1} 2^{\ell d}
+ 2 C_1 L \epsilon \sum_{\ell = N_1 + 1}^{N_2} 2^{-\ell(1 - d/2)}
+ 4MC_1L\sum_{\ell = N_2 + 1}^{\infty} 2^{-\ell},
\end{align}	
where we use the convention that $\sum_{\ell = N_1 + 1}^{N_2} \cdot = 0$ if
$N_1 \geq N_2$. There are three terms on the right hand side
of~\eqref{pf:modulus-continuity-2}, and we derive upper bounds on each.
\begin{itemize}
  \item For the first term, recalling that $\sum_{\ell = 0}^{N} a^{\ell} =
    \frac{a^{N + 1} - 1}{a - 1}$ for any $a > 1$, and noting that $\lceil
    b\rceil \leq b + 1$, we have
  \begin{equation*}
  \epsilon^2 \sum_{\ell = 0}^{N_1} 2^{\ell d} = \epsilon^2
  \Bigl(\frac{(2^d)^{N_1 + 1} - 1}{2^d  - 1}\Bigr) \leq 4^d \epsilon^2
  \Bigl(\frac{2C_1L}{\epsilon}\Bigr)^{2/d} = 4^d(2C_1)^{2/d} L^{2/d}
  \epsilon^{4/(2 + d)}.
  \end{equation*}
  \item For the second term, we obtain separate upper bounds depending on
    whether $d = 2$ or $d \geq 3$: when $d = 2$,
  \begin{equation*}
  2 C_1 L \epsilon \sum_{\ell = N_1 + 1}^{N_2} 2^{-\ell(1 - d/2)} = 2 C_1 L
  \epsilon (N_2 - (N_1 + 1))_{+} \leq 2 C_1 L \epsilon N_2 \leq 2 C_1
  L\epsilon\bigl(\log_2(2M/\epsilon) + 1\bigr),
  \end{equation*}
  and for $d \geq 3$,
  \begin{equation*}
  2 C_1 L \epsilon \sum_{\ell = N_1 + 1}^{N_2} 2^{-\ell(1 - d/2)} \leq 2 C_1
  L \epsilon 2^{(N_2 + 1)(d/2 - 1)} \leq 2^{d} C_1 L \epsilon
  \Bigl(\frac{2M}{\epsilon}\Bigr)^{1 - 2/d} = 2^{d + 1 - 2/d} C_1 L M
  (M^{-1}\epsilon)^{2/d}.
  \end{equation*}
  \item For the third term,
  \begin{equation*}
  4MC_1L\sum_{\ell = N_2 + 1}^{\infty} 2^{-\ell} = 8MC_1 L 2^{-(N_2 + 1)}
  \leq 4MC_1 L\Bigl(\frac{\epsilon}{2M}\Bigr)^{2/d} = \frac{4C_1}{2^{2/d}} L
  (M^{-1}\epsilon)^{2/d}.
  \end{equation*}
\end{itemize}
Combining these upper bounds, we conclude that for an appropriate choice of
constant \\$C_2 := 3\max\{4^d(2C_1)^{2/d}, 2C_1, 2^{d + 1 - 2/d C_1},
4C_12^{2/d}\}$,
\begin{equation*}
\|\theta - \theta'\|_2^2 \leq C_2 \cdot 
\begin{dcases*}
L \epsilon \max\{1,1/M,\log_2(M/\epsilon)\},& \quad \textrm{if $d = 2$} \\
L^{2/d}\epsilon^{4/(2 + d)} + LM\Bigl(\frac{\epsilon}{M}\Bigr)^{2/d}, & \quad \textrm{if $d \geq 3$.}
\end{dcases*}
\end{equation*}
Since this holds for all $\theta,\theta'$, it follows that the modulus of
continuity is likewise upper bounded, i.e.
\begin{equation}
\label{pf:wavelet-estimation-error-ub-2}
\Bigl\{\XDom\Bigl(\Theta_{\infty}^{0,\infty}(2M) \cap
\Theta_{\infty}^{1,1}(2L), \epsilon_n + \lambda\Bigr)\Bigr\}^2 \leq C_3 \cdot 
\begin{dcases*}
L \lambda \max\{1,1/M,\log_2(M/\lambda)\},& \quad \textrm{if $d = 2$} \\
L^{2/d}\lambda^{4/(2 + d)} + LM\Bigl(\frac{\lambda}{M}\Bigr)^{2/d},
                                          & \quad \textrm{if $d \geq 3$,}
\end{dcases*}
\end{equation}
where $C_3 := 2^{1 + 2/d}C_2$,  and we recall the assumption $\lambda \geq
2\epsilon_n$, which implies $\epsilon_n + \lambda \leq 2\lambda$.
Combining~\eqref{eqn:optimal-recovery},
\eqref{pf:wavelet-estimation-error-ub-1} and
\eqref{pf:wavelet-estimation-error-ub-2} gives the desired upper
bound~\eqref{eqn:wavelet-estimation-error-ub}.
\qed

\subsubsection{Proof of Lemma~\ref{lem:wavelet-convergence}}%
\label{ssub:proof_lemma_wavelet_convergence}
We are going to show that
\begin{equation}
\label{pf:wavelet-convergence-1}
\mathbb{P}(\bigl|\wt{\theta}_{\ell k}^{{\bf i}}(y_{1:n})
- \theta_{\ell k}^{{\bf i}}(f_0)\bigr| \geq \delta_n) \leq \frac{3\delta}{n}.
\end{equation}
From~\eqref{pf:wavelet-convergence-1}, taking a union bound over all $\ell =
0,\ldots,\log_2(n)/d$, $k \in \mc{K}(\ell)$ and ${\bf i} \in \mc{I}$ implies
the claim with $C_4 := 2^{(d + 2)}$, noting that $|\mc{I}| = 2^d - 1$ and so
\begin{equation*}
\sum_{\ell = 0}^{\frac{1}{d}\log_2(n)} |\mc{I}||\mc{K}(\ell)|
\leq 2^d\sum_{\ell = 0}^{\frac{1}{d}\log_2(n)} 2^{\ell} \leq 2^{(d + 2)}n.
\end{equation*}	
It remains to show~\eqref{pf:wavelet-convergence-1}. For ease of notation, in
the remainder of this proof we write $\wt{\theta}(\cdot) = \wt{\theta}_{\ell
k}^{{\bf i}}(\cdot)$ and $\theta(\cdot) = \theta_{\ell k}^{{\bf i}}(\cdot)$.
Decomposing $y_{i} = f_0(x_{i}) + z_{i}$, we have
\begin{equation}
\label{pf:wavelet-convergence-2}
\bigl|\wt{\theta}(y_{1:n}) - \theta(f_0)\bigr|
\leq \bigl|\wt{\theta}(y_{1:n}) - \wt{\theta}(f_0)\bigr|
+ \bigl|\wt{\theta}(f_0) - \theta(f_0)\bigr|,
\end{equation}
and we now proceed to give high-probability upper bounds on each term
in~\eqref{pf:wavelet-convergence-2}. To do so, recall Bernstein's inequality:
if $x_1,\ldots,x_n$ are independent, zero-mean random variables such that
$|x_i| \leq b$ and $\mathbb{E}[x_i^2] \leq \sigma^2$, for all $i =
1,\ldots,n$, then
\begin{equation}
\label{eqn:bernstein}
\mathbb{P}\Biggl(\biggl|\frac{1}{n} \sum_{i = 1}^{n} x_i\biggr|
\geq t\Biggr)
\leq 2 \exp\biggl(-\frac{\frac{1}{2}nt^2}{\sigma^2 + \frac{1}{3}bt}\biggr).
\end{equation}
\underline{Term 1 in~\eqref{pf:wavelet-convergence-2}: response noise.}
To upper bound $|\wt{\theta}(y_{1:n}) - \wt{\theta}(f_0)| =
|\wt{\theta}(z_{1:n})|$, we condition on the event
\begin{equation*}
\mc{Z} = \bigl\{\max_{i = 1,\ldots,n} z_i \leq \sqrt{4 \log(2n/\delta)}\bigr\},
\end{equation*}
which occurs with probability at least $1 - \delta/n$. Note that the
following statements hold conditional on $\mc{Z}$:
\begin{enumerate}
  \item The noise variables $z_i$ are conditionally independent, $z_i \perp
    z_j | \mc{Z}$, and have conditional mean $\mathbb{E}[z_i|\mc{Z}] = 0$.
  \item The conditional variance of $z_i \Psi_{\ell k}^{{\bf i}}(x_i)$ is upper bounded,
  \begin{equation*}
  \mathrm{Var}\bigl(z_i \Psi_{\ell k}^{{\bf i}}(x_i)|\mc{Z}\bigr)
  \leq \mathbb{E}\bigl[z_i^2
    \bigl(\Psi_{\ell k}^{{\bf i}}(x_i)\bigr)^2|\mc{Z}\bigr]
  \leq 2 \log(2n/\delta) \mathbb{E}\bigl[
    \bigl(\Psi_{\ell k}^{{\bf i}}(x_i)\bigr)^2|\mc{Z}
  \bigr] = 2 \log(2n/\delta),
  \end{equation*}
  with the last equality following from the $L^2(\XDom)$ normalization of
  $\Psi_{\ell k}^{{\bf i}}$, along with the independence of $x_i$ and $z_i$.
  \item For each $i = 1,\ldots,n$,
  \begin{equation*}
  |z_i \Psi_{\ell k}^{{\bf i}}(x_i)|
  \leq \sqrt{2 \log(2n/\delta)} 2^{\ell d/2} \leq \sqrt{2 n \log(2n/\delta)}.
  \end{equation*}
\end{enumerate}
We may therefore apply Bernstein's inequality~\eqref{eqn:bernstein}
conditional on $\mc{Z}$, and conclude that for $\delta_{1,n} = 4
\log^{3/2}(2n/\delta)/\sqrt{n}$, 
\begin{equation}
\label{pf:wavelet-convergence-3}
\mathbb{P}\Bigl(|\wt{\theta}(z_{1:n})|
\geq \delta_{1,n}\Bigr) \leq \mathbb{P}(\mc{Z}^c)
+ \mathbb{P}(|\wt{\theta}(z_{1:n})|
\geq \delta_{1,n} |\mc{Z}) \leq \frac{2\delta}{n}.
\end{equation}
\underline{Term 2 in~\eqref{pf:wavelet-convergence-2}: empirical
coefficient.}
To upper bound $|\wt{\theta}(f_0)- \theta(f_0)|$, observe that the random
variables $\Psi_{\ell k}^{{\bf}}(x_i) f_0(x_i) - \theta(f_0)$ for $i =
1,\ldots,n$ are mean-zero and independent. Additionally,
\begin{equation*}
\mathrm{Var}\bigl(f_0(x_i) \Psi_{\ell k}^{{\bf i}}(x_i)\bigr) \leq M^2
\end{equation*}
and 
\begin{equation*}
|f_0(x) \Psi_{\ell k}^{{\bf i}}(x)| \leq 2^{\ell d/2} M.
\end{equation*}
Applying Bernstein's inequality~\eqref{eqn:bernstein} again, unconditionally
this time, we conclude that for $\delta_{2,n} = \sqrt{12} M
\sqrt{\log(2n/\delta)}/\sqrt{n}$, 
\begin{equation}
\label{pf:wavelet-convergence-4}
\mathbb{P}\Bigl(|\wt{\theta}(f_0)- \theta(f_0)|
\geq \delta_{2,n}\Bigr) \leq \frac{\delta}{n}.
\end{equation}
Together~\eqref{pf:wavelet-convergence-3}
and~\eqref{pf:wavelet-convergence-4} imply~\eqref{pf:wavelet-convergence-1},
noting that $\delta_n = \delta_{1,n} + \delta_{2,n}$.
\qed

\subsubsection{Proof of Lemma~\ref{lem:probability-to-expectation}}%
\label{ssub:proof_lem_prob_to_expectation}
Set $f(\delta) = \sum_{k = 1}^{K}A_k \log^{a_k}(b_k/\delta)$. We use the
identity
\begin{equation}
\label{pf:probability-to-expectation-1}
\mathbb{E}[X] = \int_{0}^{\infty}\mathbb{P}(X > t) \,dt \leq f(1) + \int_{f(1)}^{\infty} \mathbb{P}(X > t) \,dt.
\end{equation}
Note $f^{-1}(f(1)) = 1$, $f^{-1}(\infty) = 0$ and
$f'(\delta) = \sum_{k = 1}^{K} A_k a_k (\log b/\delta)^{a_k - 1}/ \delta$.
U-substitution with $t = f(\delta)$ gives
\begin{align*}
\mathbb{E}[X] & = - \int_{0}^{1}
\mathbb{P}(X > f(\delta)) f'(\delta) \,d\delta \\
& = \sum_{k = 1}^{K} a_kA_k \int_{0}^{1} \mathbb{P}(X > f(\delta))
\frac{(\log b_k/\delta)^{a_k - 1}}{\delta} \,d\delta \\
& \leq B \sum_{k = 1}^{K} a_kA_k \Bigl(\int_{0}^1 (\log b_k/\delta)^{a_k - 1}
\,d\delta \Bigr) \\
& \leq B \sum_{k = 1}^{K} a_kA_k \max\{2^{a_k},1\} \Bigl((\log b_k)^{a_k - 1}
+ \int_{0}^1 (\log 1/\delta)^{a_k - 1} \,d\delta\Bigr) \\
& =  B \sum_{k = 1}^{K} a_kA_k \max\{2^{a_k},1\} \Bigl((\log b_k)^{a_k - 1}
+ \Gamma(a_k)\Bigr),
\end{align*}
where the last inequality follows by the algebraic fact $(x + y)^{a} \leq
2^{a}(x^a + y^a)$ for all $a > 0$, and the last equality comes from
substituting $h = \log(1/\delta)$. Combining this
with~\eqref{pf:probability-to-expectation-1}, we conclude that
\begin{align*}
\mathbb{E}[X] & \leq \sum_{k = 1}^{K} A_k
(\log b_k)^{a_k} + B \sum_{k = 1}^{K} a_kA_k \max\{2^{a_k},1\}
\Bigl((\log b_k)^{a_k - 1} + \Gamma(a_k)\Bigr) \\
& \leq C_5\sum_{k = 1}^{K} A_k (\log b_k)^{a_k}
\end{align*}
for $C_5 := 2B \max_{k = 1,\ldots,K} \{a_k 2^{a_k}\}$.
\qed